\documentclass[a4paper,plainchapterheads,yschapters,twoside,truedoublelespace,openright]{iitbthesis}
\usepackage{amsmath,amsfonts,bm}

\def\eqref#1{equation~\ref{#1}}

\def\1{\bm{1}}

\newcommand{\Rcal}{\mathcal{R}}
\newcommand{\iprod}[2]{\langle #1, #2 \rangle}
\newcommand{\order}[1]{O\left(#1\right)}
\newcommand{\abs}[1]{\left|#1\right|}
\newcommand{\ipbar}{\overline{ip}}
\newcommand{\expbar}{\overline{exp}}
\newcommand{\defeq}{:=}

\def\va{{\bm{a}}}

\def\ve{{\bm{e}}}
\def\vf{{\bm{f}}}
\def\vg{{\bm{g}}}

\def\vx{{\bm{x}}}

\DeclareMathAlphabet{\mathsfit}{\encodingdefault}{\sfdefault}{m}{sl}
\SetMathAlphabet{\mathsfit}{bold}{\encodingdefault}{\sfdefault}{bx}{n}

\newcommand{\E}{\mathbb{E}}

\newcommand{\R}{\mathbb{R}}

\DeclareMathOperator*{\argmax}{arg\,max}
\DeclareMathOperator*{\argmin}{arg\,min}

\usepackage{epsfig}
\usepackage[T1]{fontenc}
\usepackage[utf8]{inputenc}
\usepackage{newtxtext} 
\usepackage{amssymb}
\usepackage{amsmath,epsfig}
\usepackage{graphicx,graphics}
\usepackage{url}
\usepackage[colorlinks=true,allcolors=blue!50!black,pagebackref=true]{hyperref}
\usepackage[capitalise]{cleveref}
\usepackage{textcomp}
\usepackage{enumitem}
\usepackage{natbib}
\usepackage{pdflscape}

\usepackage{titlesec}
\usepackage{bm,bbm}
\usepackage{float,lscape}

\usepackage{epstopdf}
\usepackage{tabu}
\usepackage{array}
\usepackage{booktabs}
\usepackage{graphicx}
\usepackage[hypcap=false,font=small]{caption}
\usepackage{lettrine}
\usepackage{enumitem}
\usepackage{lscape}
\usepackage{rotating}
\usepackage{booktabs}
\usepackage{longtable}
\usepackage{anyfontsize}
\usepackage{t1enc}
\usepackage{relsize}
\usepackage{placeins}
\usepackage{nomencl}
\usepackage{rotating}
\usepackage{glossaries}
\usepackage{scrlayer}
\usepackage{enumitem}
\usepackage{microtype}
\usepackage{graphicx}
\usepackage{amsmath,amsfonts,amssymb,amsthm,bm}
\usepackage{amsfonts}
\usepackage{booktabs} 
\usepackage{gensymb}
\usepackage{colortbl}
\usepackage{soul} 
\usepackage{changepage}
\usepackage{makecell}
\usepackage{multirow}
\usepackage{adjustbox}

\usepackage{algorithm,algpseudocode}
\usepackage{bm}
\usepackage{subcaption}
\usepackage{dsfont}

\usepackage{tikz}
\usetikzlibrary{bayesnet}
\usetikzlibrary{positioning}
\usetikzlibrary{arrows}
    
\usepackage{wrapfig}
\usepackage[titletoc]{appendix}
\usepackage[nottoc,notlot]{tocbibind}

\usepackage{natbib}

\usepackage{url}
\usepackage{xspace}
\usepackage{latexsym}

\usepackage{nicefrac}       
\usepackage{mathtools}
\usepackage{commath}

\RequirePackage{enumitem}
\setlist[itemize]{nosep,leftmargin=*,labelwidth=0pt}
\setlist[enumerate]{nosep, leftmargin=*}
\setlist[description]{nosep,leftmargin=.8em}

\RequirePackage[normalem]{ulem}
\RequirePackage{enumitem} \setlist{nosep}
\RequirePackage{tabularx}
\RequirePackage{pgfplots}
\RequirePackage{pgfplotstable,subcaption}
\pgfplotsset{compat=1.14}
\pgfplotsset{nice/.style={grid=major, major grid style={line width=.2pt,draw=gray!30},}}
\RequirePackage[most]{tcolorbox}
\newtcolorbox{insight}[1][default]{boxsep=.5mm, left=0mm, right=0mm, top=0mm, bottom=0mm, colback=gray!4, colframe=gray!25, boxrule=.2mm, }

\RequirePackage{array}
\newcolumntype{L}[1]{>{\raggedright\let\newline\\\arraybackslash\hspace{0pt}}m{#1}}
\newcolumntype{C}[1]{>{\centering\let\newline\\\arraybackslash\hspace{0pt}}m{#1}}
\newcolumntype{R}[1]{>{\raggedleft\let\newline\\\arraybackslash\hspace{0pt}}m{#1}}

\def\ignore#1{}
\DeclareMathOperator*{\stability}{stability}
\DeclareMathOperator*{\wscore}{\mathit{R}}  
\DeclareMathOperator*{\sscore}{\mathit{Q}}  

\renewcommand*\backref[1]{\ifx#1\relax \else (Cited on #1) \fi}
\newcommand{\vek}[1]{{\bf {#1}}}

\newcommand{\DT}{{{\mathcal D_T}}}
\newcommand{\DS}{{{\mathcal D_S}}}
\newcommand{\ES}{{{\mathcal E}}}
\def\R{\mathbb{R}}

\newcommand{\tgt}{{{Tgt}}}
\newcommand{\src}{{{Src}}}
\newcommand{\srcinit}{{{SrcTune}}}
\newcommand{\concat}{{{Concat}}}
\newcommand{\avg}{{{Avg}}}

\newcommand{\yangC}{{{RegFreq}}}

\newcommand{\srcselReg}{{{RegSense}}}

\newcommand{\srcselR}{{{\shortname:R}}}
\newcommand{\srcselC}{{{\shortname:c}}}
\newcommand{\srcsel}{{{\shortname}}}

\definecolor{LightGreen}{rgb}{0.8,1,0.8}

\newcommand{\dro}{Group-DRO}
\newcommand{\erm}{ERM}
\newcommand{\cg}{CGD}

\newcommand{\ermuw}{ERM-UW}

\newcommand{\pgi}{PGI}

\def\jac{\mathbb{J}}
\def\R{\mathbb{R}}
\newcommand{\vgmap}{\hat{\vek{g}}}
\newcommand{\vgmapcoord}{\hat{g}}

\newcommand{\cD}{{\mathcal{D}}}
\newcommand{\cY}{{\mathcal{Y}}}

\newcommand{\dan}{\textsc{DAN}\xspace}
\newcommand{\goodfellow}{\textsc{LabelGrad}\xspace}
\newcommand{\crossgrad}{\textsc{CrossGrad}\xspace}
\newcommand{\mypara}[1]{\medskip\noindent\textbf{#1}~}

\newtheorem{theorem}{Theorem}[chapter]

\newtheorem{defn}{Definition}

\newtheorem{lemma}[theorem]{Lemma}

\newcommand{\mos}{CSD}
\newcommand{\mosdescr}{Common-specific Decomposition.}

\newcommand{\numC}{C}
\newcommand{\rank}{k}
\newcommand{\Span}[1]{\textrm{Span}\left({#1}\right)}
\newcommand{\Rank}[1]{\textrm{rank}\left({#1}\right)}
\newcommand{\What}{\widehat{W}}
\newcommand{\vtilde}{\widetilde{v}}
\newcommand{\ones}{\mathds{1}}
\newcommand{\frob}[1]{\left\|{#1} \right\|_F}
\newcommand{\trans}[1]{{#1}^{\top}}
\newcommand{\wtilde}{\widetilde{w}}
\newcommand{\Utilde}{\widetilde{U}}
\newcommand{\Stilde}{\widetilde{\Sigma}}
\newcommand{\Vtilde}{\widetilde{V}}
\newcommand{\Wtilde}{\widetilde{W}}
\begin{document}

\abovedisplayskip=8mm
\abovedisplayshortskip=8mm
\belowdisplayskip=8mm
\belowdisplayshortskip=8mm


\pagenumbering{gobble} 

\title{Robustness, Evaluation and Adaptation of Machine Learning Models in the Wild}
\author{Vihari Piratla}
\date{2022}

\rollnum{173050039} 

\iitbdegree{Doctor of Philosophy}

\thesis

\department{Department of Computer Science and Engineering}

\setguide{Prof. Sunita Sarawagi}
\setcoguide{Prof. Soumen Chakrabarti}

\maketitle



\makeapproval

\clearpage
\thispagestyle{empty}

\begin{center}
\Large  {\bf Declaration }
\end{center}
\vspace{-6in}
I declare that this written submission represents my ideas in my own words and where others ideas or words have been included, I have adequately cited and referenced the original sources. I also declare that I have adhered to all principles of academic honesty and integrity and have not misrepresented or fabricated or falsified any idea/data/fact/source in my submission. I understand that any violation of the above will be cause for disciplinary action by the Institute and can also evoke penal action from the sources which have thus not been properly cited or from whom proper permission has not been taken when needed.
\vspace{0.5in}


%
 \begin{table}[h]
 \begin{flushleft}

\vspace{-3.2in} 
 \begin{tabular}{ccccc}
 \rule[5ex]{0pt}{-10ex}&& Date: 2023-02-13 && \\ 
 \end{tabular}
\end{flushleft}

\vspace{-0.5in} 
\begin{flushright}
 \begin{tabular}{ccccc}
 
 \hline 	\rule[5ex]{0pt}{-10ex}&& Vihari Piratla&& \\ 
 \rule[5ex]{0pt}{-10ex}&& Roll No. 173050039&& \\ \\
 \end{tabular}
\end{flushright}
\end{table}

\pagebreak
\joiningdate{July 2017}
\begin{coursecertificate}
\addcourse{CS 752}{System Dynamics: Modeling \& Simulation for Development}{6}
\addcourse{CS 754}{Advanced Image Processing}{6}
\addcourse{DE 403}{Studio Project I}{6}
\addcourse{CS 691}{R \& D Project}{6}
\addcourse{CS 694}{Seminar}{4}
\addcourse{CS 726}{Advanced Machine Learning}{6}
\addcourse{CS 728}{Organization of Web Information}{6}	
\addcourse{CS 792}{Communication Skills -II}{4}
\addcourse{HS 791}{Communication Skills -I}{2}

\addcourse{CS 601}{Algorithms and Complexity}{6}
\addcourse{CS 635}{Information Retrieval \& Mining for Hypertext \& the Web}{6}
\addcourse{CS 699}{Software Lab}{8}
\addcourse{CS 725}{Foundations of Machine Learning}{6}	
\addcourse{CS 747}{Foundations of Intelligent and Learning Agents}{6}
\end{coursecertificate}


\begin{abstract}

Our goal is to improve reliability of Machine Learning (ML) systems deployed in the wild. ML models perform exceedingly well when test examples are similar to train examples. However, real-world applications are required to perform on any distribution of test examples. For example, ML in healthcare applications can encounter any distribution of patients or hospital systems. Current ML systems can fail silently on test examples with distribution shifts. In order to improve reliability of ML models due to covariate or domain shift, we propose algorithms that enable models to: (a) generalize to a larger family of test distributions, (b) evaluate accuracy under distribution shifts, (c) adapt to a target/test distribution.


We study causes of impaired robustness to domain shifts and present algorithms for training domain robust models. A key source of model brittleness is due to domain overfitting, which our new training algorithms suppress and instead encourage domain-general hypotheses. While we improve robustness over standard training methods for certain problem settings, performance of ML systems can still vary drastically with domain shifts. It is crucial for developers and stakeholders to understand model vulnerabilities and operational ranges of input, which could be assessed on the fly during the deployment, albeit at a great cost. Instead, we advocate for proactively estimating accuracy {\it surfaces} over any combination of prespecified and interpretable domain shifts for performance forecasting. We present a label-efficient Bayesian estimation technique to address estimation over a combinatorial space of domain shifts. Further, when a model's performance on a target domain is found to be poor, traditional approaches adapt the model using the target domain's resources. Standard adaptation methods assume access to sufficient computational and labeled target-domain resources, which may be impractical for deployed models. We initiate a study of lightweight adaptation techniques with only unlabeled data resources with a focus on language applications. Broadly, our methods infuse context from selected portions of unlabeled data for better representations without extensive parameter tuning.  
\end{abstract}

\tableofcontents



\include{abbreviations}
\include{symbols}

\setlength{\parskip}{2.5mm}
\titlespacing{\chapter}{0cm}{55mm}{10mm}
\titleformat{\chapter}[display]
  {\normalfont\huge\bfseries\centering}
  {\chaptertitlename\ \thechapter}{20pt}{\Huge}
  
  \titlespacing*{\section}
  {0pt}{8mm}{8mm}
  \titlespacing*{\subsection}
  {0pt}{8mm}{8mm}
\pagebreak
\pagenumbering{arabic}

\makeatletter
\def\cleardoublepage{\clearpage\if@twoside \ifodd\c@page\else
	\hbox{}
	\vspace*{\fill}
	\begin{center}
		This page was intentionally left blank.
	\end{center}
	\vspace{\fill}
	\thispagestyle{empty}
	\newpage
	\if@twocolumn\hbox{}\newpage\fi\fi\fi}
	
\newcommand\placeholder[2]{%
  \@ifundefined{r@#1}{%
    #2%
  }{%
    \ref{#1}%
  }%
}
\makeatother

\newpage
\pagebreak
\cleardoublepage
\chapter{Introduction}
\label{chap:intro}

In the past few decades, machine-learned models have provided significant improvements on various cognitive tasks, surprisingly even beating humans at times. Machine learning (ML) is now being hailed as ``the new electricity''\footnote{\url{https://www.gsb.stanford.edu/insights/andrew-ng-why-ai-new-electricity}} because of the vast number of real-world situations where it is being deployed. 
ML innovations fuel inspiring applications such as early diagnosis of blindness due to diabetic retinopathy, predicting protein folding, discovering new planets and detecting gravitational waves. Further, ML has huge potential in improving health care, personalised education, and automation. 

The success of ML is primarily due to statistical learning methods that train end-to-end on large datasets.
Despite progress on standard benchmarks and leaderboards, ML systems often perform poorly when deployed. 
They are known to be brittle and can fail in inexplicable ways on unfamiliar inputs~\citep{szegedy14,perspective_shift,Zech2018MetalTokenxray,Gururangan18,Beery18,david20,bandi18}. 
Standard training and testing algorithms assume that train and test examples are sampled from the same underlying distribution. This assumption, however, is often violated in the wild. Real-world performance of an ML model, therefore, could be far worse than what is expected through an identically distributed test split~\citep{plank2016non,Geirhos20}. Poor performance due to train-test distribution mismatch is referred to as the data shift problem. 

Several studies illustrate data shift challenges for ML in practice.
\citet{perspective_shift} found that the average object recognition accuracy of Faster-RCNN~\citep{faster-rcnn} varies dramatically with simple changes in viewpoint as shown in Figure~\ref{fig:chap1:perspective_shift}. 
\citet{Awasthi21} found that the word error rate (WER) of state-of-art Automatic Speech Recognition (ASR) systems varies with accent type (4\% WER for American-English and 11\% to 55\% WER for Indian-English accent). 
These failures can be attributed to training data bias towards certain perspective angles or accent types respectively. 



\begin{figure}[htb]
    \centering
    \includegraphics[height=2.8cm]{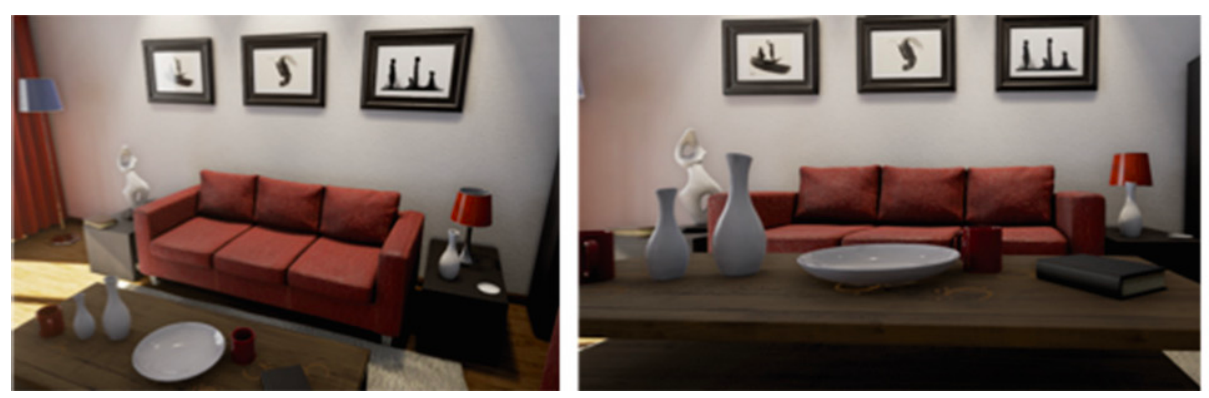}
    \includegraphics[height=2.8cm]{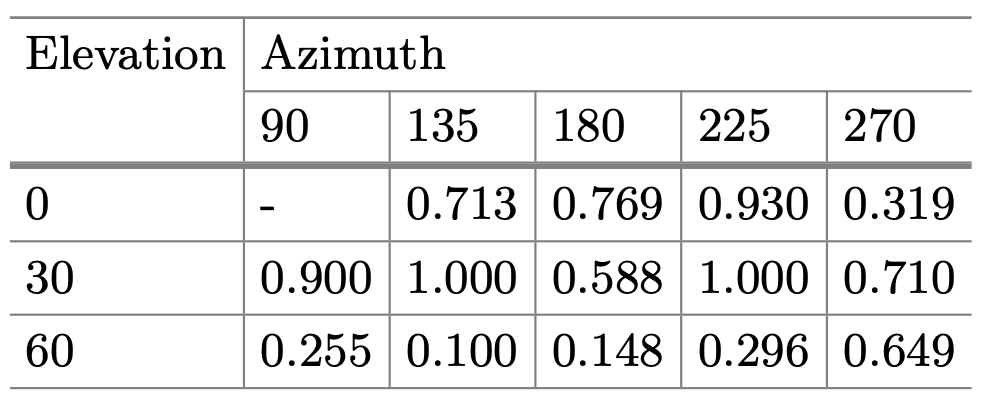}
    \caption{Figure borrowed from \citet{perspective_shift}. Average Precision (AP) of Faster-RCNN~\citep{faster-rcnn} detection of the sofa was found to vary from 0.1 to 1 when the perspective angle is changed.}
    \label{fig:chap1:perspective_shift}
\end{figure}

Here is another interesting anecdote\footnote{\url{https://www.jefftk.com/p/detecting-tanks}} on the lack of generalization of ML models in practice: a team trained a model to discriminate pictures of war tanks in the forest from normal forest pictures. Despite the high performance of the model on held-out test split, they found the model to be no better than the random baseline when deployed on the ground. This lack of generalization was later found to be because the model exploited a co-occurring feature: the brightness of an image; war-tank pictures in the training data were all photographed during a cloudy day while the pictures without war-tanks were from clear days. The training process latched on the easy brightness feature for classification, and thus failed to generalize in the wild where such bias was absent. 


Such lack of generalization in the wild is emerging as a major concern in ML applications, particularly since failures go unnoticed in embedded ML deployments. The central motivation behind this thesis is: 

\begin{quote}
    {\it How can we improve the reliability of ML models deployed in the wild?}
\end{quote}

\section{Background}
\label{sec:intro:background}
Unreliability of a deployed model arises out of distribution shifts between train and test examples. To understand various kinds of shifts, let us establish some notation. Denote by $\vx$, the input to an ML model for predicting the label $y$. 
Denote by $\Pr_{tr}(\vx, y)=\Pr_{tr}(\vx)\Pr_{tr}(y\mid \vx)$ the training distribution and $\Pr_{te}(\vx, y)$ the test distribution. The training and test distribution are sometimes referred to as the source and the target distribution respectively. The test distribution can shift due to three types of shifts: (1)~input distribution  $\Pr(\vx)$ called covariate shift, (2) label distribution $\Pr(y)$ called label shifts, (3)~label conditional distribution $\Pr(y\mid \vx)$ called concept shift. These shifts are summarized in Table~\ref{tab:chap1:shifts}, where highlighted region shows the shifting factor. We refer to~\citet{Castro2020Causality,Jin2021Causality,pml2Book} for a more thorough treatment of data shifts. While any combination of shifts is possible in the real world, this thesis deals only with covariate shift, also called domain shift in this thesis. 

\begin{table}[htb]
    \centering
    \begin{tabular}{l|c|c}
         \thead{Name} & \thead{Source} & \thead{Target} \\\hline 
         Covariate or domain shift & $\Pr_{tr}(\vx)\Pr_{tr}(y\mid \vx)$ & \colorbox{yellow}{$\Pr_{te}(\vx)$}$\Pr_{tr}(y\mid \vx)$\\
         Concept shift & $\Pr_{tr}(\vx)\Pr_{tr}(y\mid \vx)$ & $\Pr_{tr}(\vx)$\colorbox{yellow}{$\Pr_{te}(y\mid \vx)$} \\
         Label (prior) shift & $\Pr_{tr}(y)\Pr_{tr}(\vx\mid y)$ & \colorbox{yellow}{$\Pr_{te}(y)$}$\Pr_{tr}(\vx\mid y)$
    \end{tabular}
    \caption{Dataset shifts. Highlighted regions shows the shifting factor. }
    \label{tab:chap1:shifts}
\end{table}

Domain shifts arise out of shifts in the input distribution ($\vx$), but {\emph what characterizes shifts in a high-dimensional input distribution?} 
The inputs can be seen as being generated by a latent set of factors. Some examples of such latent factors are fonts in an optical character dataset, speakers in a speech utterances dataset, and writers in a review dataset.
Domain refers to a collection of coherent inputs with shared latent factors~\citep{Ramponi20}.
Some examples of domain are Newswire articles, Twitter tweets, Amazon product reviews, or Yelp reviews, because they all share the style and topic of the product platform. 
Domain captures similarities in real-world examples and is therefore understandably overloaded.
For this reason, domain is best understood in the context of an application. If the application is to predict pulmonary diseases from chest x-ray scans that is to be used in other hospitals anywhere in the world, then the domain can be defined by patient groups with the same age or sex, or hospital groups with the same location or scanning device. Consequently, domain shift, thereby, refers to the shift in latent factors that influence the input distribution.  

\noindent
{\bf Why is domain shift a problem?} As shown in Table~\ref{tab:chap1:shifts}, we assume that the label conditional distribution: $\Pr(y\mid \vx)$ does not shift when the domain shifts; why then does an ML model underperform on unseen domains? This is because the stable conditional distribution assumption may only hold over a latent subset of input features ($\vx_c\in \vx$).  
The input may also contain domain-specific features ($\vx_s$, $\vx=\vx_c\cup \vx_s$) with incidental correlations with the label, which do not generalize to other domains. When training using finite data that is likely biased, the estimated model owing to its dependence on incidental domain-specific features may underperform on unseen domains. We will discuss multiple problem scenarios in this thesis, which share the objective of discovering shift stable features ($\vx_c$).

\section{Related Work}
\noindent
Addressing domain shifts is a long-standing and well-studied challenge. We present an overview in this section. 

\subsection*{Domain Adaptation}
Domain adaptation is a popular approach for mitigating domain shift, which assumes access to labeled or unlabeled data from the target/test domain. The objective is to leverage labeled instances from the source domain to learn a predictor suited for the target domain. Approaches vary based on whether we are adapting from single or multiple source domains, or have access to labeled or unlabeled target data~\citep{Ramponi20,Daume2007,Saenko2010,Zhuang20,Csurka17,Sankaranarayanan18}. With ever-increasing model sizes, conventional adaptation techniques incur significant (re-)training costs. An emerging approach to reduce adaptation cost in text applications is to simply prompt a pretrained model by describing the target task in text and providing some labeled data~\citep{FLAN,gpt3}. Although promising, the scenario where prompt based methods work is not well-understood~\citep{Reynolds21, Webson21prompt, Razeghi22}. Moreover, in many applications, target domain labeled or unlabeled resources is not readily available. For example, in healthcare applications where domain shifts between patients, it is impractical to procure data from each patient in advance before we deploy the model~\citep{MuandetBS13}.




\subsection*{Domain Robustness}
Instead of post-facto adaptation, one may wish to train models that are robust by design to domain shifts.
This research question is extensively studied from several fronts. We will now present popular methods for training domain-robust models. 
\begin{itemize}[leftmargin=*]
    \item {\bf Regularized training} techniques that are targeted to reduce overfitting on training data can discourage learning irrelevant features with incidental correlations with the label, thereby improving robustness. These techniques include augmentation~\citep{augmix,VihariSSS18} of training data with a transformation of original examples such that a transformed example has the same label but different domain, and regularization losses~\citep{Zhao20,Kim21selfreg}.
    \item {\bf Inductive biases.} Imposing appropriate inductive biases that align human and machine cognition can generalize to unfamiliar inputs similar to humans. For example, \citet{Geirhos18,Shi20} argued with psychophysical experiments that humans rely on shape more than texture to recognize objects unlike convolutional networks (when trained on ImageNet) that rely on texture. Further, they introduced a method for overcoming the texture bias of CNN, and found it to have improved robustness to unseen image manipulations. However, discovering inductive biases of humans and enforcing them in to the model is non-trivial, which had only limited success in the past. 
    \item {\bf Diverse training data.} The performance of ML model is known to deteriorate with increasing shift between the train and test domains~\citep{HendrycksD19}. When the training data covers examples from many domains, the shift between the test domain and its closest training domain is reduced, with simultaneous gains in performance on the test domain. Some of our results in Chapter~\placeholder{chap:dg}{2} and~\citet{bridgedata} corroborate this claim through empirical results. \citet{bridgedata} further recommend improving diversity of training data for improving domain robustness.  
    \item {\bf Adversarial training} methods focus training on the most difficult parts of the training distribution as per the current state of the model to prepare it for the worst possible test distribution. Robust training~\citep{goodfellow2014,Bai21} methods minimize training risk over adversarial examples in the $\epsilon$-neighbourhood of the original examples; Distributionally Robust Optimization~\citep{duchi18} minimize training risk over adversarial distribution drawn from a predefined family of input distributions. 
    \item  {\bf Domain Generalization}, a practically motivated problem formulation, assumes that the training and the test domains are drawn from a latent distribution over domains. Under this assumption, domain-shifts between the training domains captures anticipated domain-shift between the train and the test domains. Methods for domain generalization, therefore, exploit domain-induced variation during train time~\citep{Zhou2021domain}, which we will discuss in more detail later.
    \item {\bf Causality.} In causal representation learning, causality based invariance principles instead of statistical correlations drive feature learning. This technique holds the promise of human-like generalization. Nevertheless, isolating causal from correlated features with limited data and supervision is challenging~\citep{Scholkopf2021toward}.
    In the same vein, \citet{Lake2015} argue that the key to generalization is to train a generative model that explains the label through a composition of its parts. Although the causal promise for domain robustness is well-founded, to date they have not been scaled to learn end-to-end from large data. 
    \item {\bf Data and Model Scaling.} Some researchers rely on simply scaling data and models to their limits (if there is one)~\citep{FLAN, gpt3, Radford21clip} for superior generalization and robustness results. Although their results are impressive, the underpinnings of their improved performance is yet to be understood. 
\end{itemize}

\subsection*{Evaluation}
Evaluation under domain shifts, on the other hand, is under-studied. Given that the performance of a model varies when domain shifts, several studies suggested that we rethink evaluation~\citep{modelcards19,Yuille21}. A popular approach is to estimate accuracy when provided unlabeled data from the target domain~\citep{sawade10,katariya12,Karimi2020,Garg22,Deng21}. Such evaluation measures that rely on resources from the target domain, however, cannot inform the developers and stakeholders of the model vulnerabilities in advance. 



\section{Our Contribution}

ML systems deployed in the wild may encounter any input. They are trained on a small subset of the input space with unknown behaviour on unseen inputs. Training and evaluating ML models that work in any context form the reliability challenge~\citep{Varshney22}. 
This thesis envisions tackling the reliability challenges due to domain shifts through addressing the following sub-problems. 
\begin{itemize}
    \item[$\bullet$] Training for domain robustness: algorithms for training domain-robust models.
    \item[$\bullet$] Evaluation methods for ML in the wild: techniques suitable for performance queries on any domain. 
    \item[$\bullet$] Unlabeled adaptation: algorithms that can adapt the model with user-provided unlabeled resources. 
\end{itemize}

\noindent
The thesis is organized to discuss contributions in the order presented above. 

\subsection{Training for Domain Robustness}

The goal of training domain-robust models is to recover the latent set of input features that are stable under domain shifts; but, can we have models that can generalize to any domain?
The set of useful features may shrink with increasing number of domains.
A universally robust model may therefore compromise accuracy when compared with a model designed for robustness in a limited setting. 
Our focus is to train models that are robust over a distribution of domains. For most practical applications, it is possible to define such distributions in terms of characteristics of the anticipated usage scenarios of the model.
For example, a global tool for classifying gender given a subject's portrait should generalize to subjects of any race under any lighting condition.

We study algorithms for domain-robust models on multi-domain training data under two challenging scenarios: (1)~the number of training domains is small (2)~the training population size of some of the domains is small. When training data covers a sufficiently large number of domains, features with incidental correlations with label that hold only for a subset of domains are naturally suppressed, thereby generalizing better to unseen domains with increasing diversity of training data. 
Moreover, since large datasets are often collated from multiple sources, training domains can have arbitrary domain population size.
Standard learning algorithms are influenced by majority domains and overfit to minority domains.
As a result, fewer training domains and disproportionate domain population sizes both cause overfitting to the training domains. 
We discuss our solutions to both these challenges below.

\subsubsection*{Domain Generalization}

Domain generalization poses the following question. {\it If the training data includes examples from multiple domains sampled from a latent distribution of domains, can we generalize better to any seen or unseen domain also sampled from the domain distribution?}


In the domain generalization task, we are given domain-demarcated data from multiple domains during training, and our goal is to create a model that will generalize to instances from new domains during testing. 
Domain generalization requires zero-shot generalization to instances from unseen domains. It is of particular significance in large-scale deep learning networks because large training sets often entail aggregation from multiple distinct domains, on which today's high-capacity networks easily overfit. Standard methods of regularization only address generalization
to unseen instances sampled from training domains, and have been shown to perform poorly on instances from unseen domains.




We made one of the early contributions to the domain generalization problem~\citep{VihariSSS18}. We argue with empirical findings that domain generalization is improved by simply increasing the number of train domains. Acting on this observation, we propose a simple data augmentation strategy. Because augmenting examples from new domains is challenging, we first design an (imperfect) domain classifier network to be trained with a suitable loss function. Given an example, we search in its $\epsilon$-neighbourhood for an input that perturbs domain but preserves label using the domain and label classifiers. 

Generalization through domain augmentation can be limiting because the space of inputs is large and synthesizing new examples is a non-trivial exercise. 
With the objective of explicitly decomposing out non-generalizing classifier components of the classifier, we design an algorithm called CSD (Common Specific Decomposition) that decomposes model parameters into a common component and a low-rank domain-specific component. CSD, which only modifies the final linear classification layer, improved domain generalization with only a slight computation overhead.
Through analysis with a simple generative setting, we provide principled understanding of lack of domain generalization due to domain-specific component and merits of CSD for domain generalization.


\noindent
We discuss these solutions more thoroughly in Chapter~\placeholder{chap:dg}{2}. 


\subsubsection*{Subpopulation Shift}
The objective of subpopulation shift problem is to learn models that minimize generalization disparity between training domains. Standard training methods generalize poorly to domains with low training population size~\citep{Sagawa20}. 
Poor generalization accuracy on minority domains, despite being represented in the training data, raises practical concerns; as an example, \citet{gendershades} reported that state-of-art gender classification models performed far worse on portraits of colored subjects perhaps because they are under-represented in training data. Subpopulation shift robustness, therefore, can address data costs and fairness concerns. 



Subpopulation shift is closely related to domain generalization where different population mixture of training domains define the distribution of domains. 
We extend our best algorithm for domain generalization (CSD) to propose CGD (Common Gradient Descent) algorithm, which recovers domain-generalizing common component of the parameter gradients without any decomposition assumption of CSD. CGD avoids minority overfitting by updating parameters with only gradients from common domains, which enables near uniform generalization to all domains. Theoretically, we show that the proposed algorithm is a descent method and finds first-order stationary points of smooth nonconvex functions. This work is described in Chapter~\placeholder{chap:cgd}{3}.


\subsection{Evaluation of Deployed Models}

The performance of an ML model can vary under domain shifts. Different users may need to deploy the model on very different data distributions, with possibly widely different accuracy. Yet, traditional evaluation methods report a single aggregate number on benchmarks or through leaderboards. Alternative practical evaluation measures are desired, especially with the increasing prevalence of ML models~\citep{modelcards19}. 

We propose to evaluate black-box classification model, not as one or few aggregate numbers, but as a {\it surface} defined on a space
of input instance attributes that capture the variability of user expectations. Indoor/outdoor,
day/night, urban/rural may be attributes of input images for visual object recognition tasks. Speaker
age, gender, ethnicity/accent may be attributes of input audio for speech recognition tasks. We call
a combination of attributes in their Cartesian space an arm\footnote{borrowing from bandit terminology}.
We present a Gaussian Process (GP)-based probabilistic estimator of the accuracy surface. Each arm is associated with a Beta density from which the model's accuracy is sampled.  We expect the GP to smooth the parameters of the Beta density over related arms to mitigate sparsity.
We show that the obvious application of GPs cannot address the challenge of heteroscedastic uncertainty over a huge attribute space that is sparsely and unevenly populated.
In response, we present two enhancements: pooling sparse observations, and regularizing the scale parameter of the Beta densities. After introducing these innovations, we establish the effectiveness through extensive experiments. This work is  described in Chapter~\placeholder{chap:aaa}{4}.

\subsection{Unlabeled Domain Adaptation}

When the space of domains is extremely large, such as in language applications, it is impossible to train accurate models that are robust to any domain shift. A better strategy is to instead adapt the model to the target domain if the model's performance is found inferior. Fine-tuning using labeled data from the target domain is known to improve performance. Users, though, may have no labeled data or access to computational resources to fine-tune giant state-of-the-art models. We study lightweight adaptation approaches, with a focus on text applications, using only unsupervised data.

In text applications, words are represented using low-dimensional dense vectors called word embeddings that encode their meaning for end-to-end learning using statistical models~\citep{MikolovSCCD2013word2vec,PenningtonSM2014GloVe}. When topic distribution shifts due to domain shifts, word meaning could drift. Even if pretrained word embeddings are trained on a large `universal' corpus, considerable sense shift may exist in the meaning of polysemous words and their cooccurrences and similarities with other words.  In a corpus about Unix, `cat' and `print' are more related than in a generic English corpus such as Wikipedia; in a corpus about Physics, `Charge' and `potential' are more related than in a generic English corpus. 
Topic shifts can as a result lead to misinterpretation and poor performance in the new domain.
On the other hand, available unlabeled data from the target domain may be too small to train word embeddings to sufficient quality from scratch.
Adapting word embeddings can therefore positively influence many downstream language tasks.
A thorough comparison across ten topics, spanning three tasks, reveals that even the best embedding adaptation strategies provide small gains beyond well-tuned baselines. 
We then reach the surprising conclusion that even limited corpus augmentation of target unlabeled data is more useful than adapting embeddings, which suggests that non-dominant sense information may be irrevocably obliterated from pretrained embeddings and cannot be salvaged by adaptation.

Adaptation of any kind requires parameter tuning that can be computationally expensive given the ever-increasing model complexity. 
A common source of domain shift in text applications is mismatch of word salience~\citep{paik2013novel} between source and target domains~\citep{Ruder2019Neural}.
In this respect, our multi-domain training setting presents a new opportunity.
Users are numerous and form natural clusters, e.g., healthcare, sports, politics. We want the model to exploit commonalities in train domains, and provide some level of generalization to new domains without re-training or fine-tuning.
We investigate practical protocols for lightweight domain adaptation of natural language processing (NLP) models. We propose a system in which each domain registers with the model using a simple sketch derived from its (unlabeled) corpus. The model then processes target domain inputs along with its sketch. Our model design provides accuracy benefits to new domains immediately.
We show that a simple late-stage intervention in the model network gives visible accuracy benefits, and provide diagnostic analyses and insights.

\noindent
Chapter~\placeholder{chap:adaptation}{5} describes our adaptation methods in more detail.


\chapter{Domain Generalization}
\label{chap:dg}

In this chapter, we study algorithms for training models that are robust to domain shifts. 
We assume that the train and test domains are sampled from the same underlying distribution of domains.
Domain-shifts between the training domains, therefore, also capture domain-shifts between the train and test domains. For example, multi-domain training data for a gender classification task could represent shifts in the lighting or subject, of the portrait, which are also the latent factors that are expected to shift between the train and test distributions. For this reason, many algorithms exploit robustness to domain-shifts during the train time, which in effect also renders robustness to domain-shifts during the test time.


\section{Problem Statement}

Let $\cD$ be a space of domains. During training we get labeled data from a proper subset $D \subset \cD$ of these domains.  Each labeled example during training is a triple $(\vx, y, d)$ where $\vx\in \mathbb{R}^m$ is the m-dimensional input, $y \in \cY$ is the true class label from a finite set of labels $\cY$ and $d \in D$ is the domain from which this example is sampled. Each example is assigned an explicit, discrete domain id, which we refer to as $d$. We must train a classifier to predict the label $y$ for examples sampled from all domains, including the subset $\cD \setminus D$ not seen in the training set.  Our goal is high accuracy for both in-domain (i.e., in $D$) and out-of-domain  (i.e., in $\mathcal{D}\setminus D$) test instances.

We make the following assumption. {\it There exist common features whose correlation with the label is consistent across domains and domain-specific  features whose correlation with the label varies 
across domains}. 
Therefore, label classifiers that rely on common features are likely to generalize to unseen domains better than those that also rely on domain-specific  features. 


Let us consider the following simple example:
\begin{align}\label{eqn:dg:simple}
    \vx_d = \vx_c + \vx_s + \mathbf{n}_d\\
    \vx_c = y\ve_c \nonumber\\
    \vx_s = \beta_d y\ve_s \nonumber \\
    \mathbf{n}_d \sim \mathcal{N}(0,\Sigma_d) \nonumber
\end{align}

The input $\vx_d$ of a domain $d$ is a superposition of common features $\vx_c$, domain-specific  feature $\vx_s$, and normally distributed noise $\mathbf{n}_d$. For a binary task $y\in \{-1, +1\}$, the feature $\vx_c$ is formed out of the sampled label $y$ and a constant vector $\ve_c$. Thus, $\vx_c$ is positively correlated with the label in any domain. The part $\vx_s$ is the product of a domain-specific  scalar $\beta_d$, sampled label $y$, and the constant vector $\ve_s$. Thus, $\vx_s$ correlation with the label varies with domain. With $\ve_c=[1, 0]^T, \ve_s=[0, 1]^T$, Figure~\ref{fig:dg:a} shows two train and one test domain for $\beta=-1, 2, -4; \sigma=0.2, 0.5, 0.4$ respectively, where $\Sigma=\sigma I_{2\times 2}$. Observe how, in all domains, positive classes are to the right and negative to the left, but positive classes are sometimes above or below negative examples depending on the domain. Therefore classification based on $\vx_c$, i.e. Y axis is the optimal classifier that can generalize to domain shifts. 

We fit a linear classifier using standard training (ERM), which is shown in Figure~\ref{fig:dg:b}. It learns a solution that perfectly classifies the two training domains, but generalizes poorly to the test domain shown in highlighted oval regions. Standard training methods predict based on any label correlated feature and therefore learn the predictor that is a function of both the coordinates as shown. 

\begin{figure}
     \centering
     \begin{subfigure}[b]{0.48\textwidth}
         \centering
         \includegraphics[width=\textwidth]{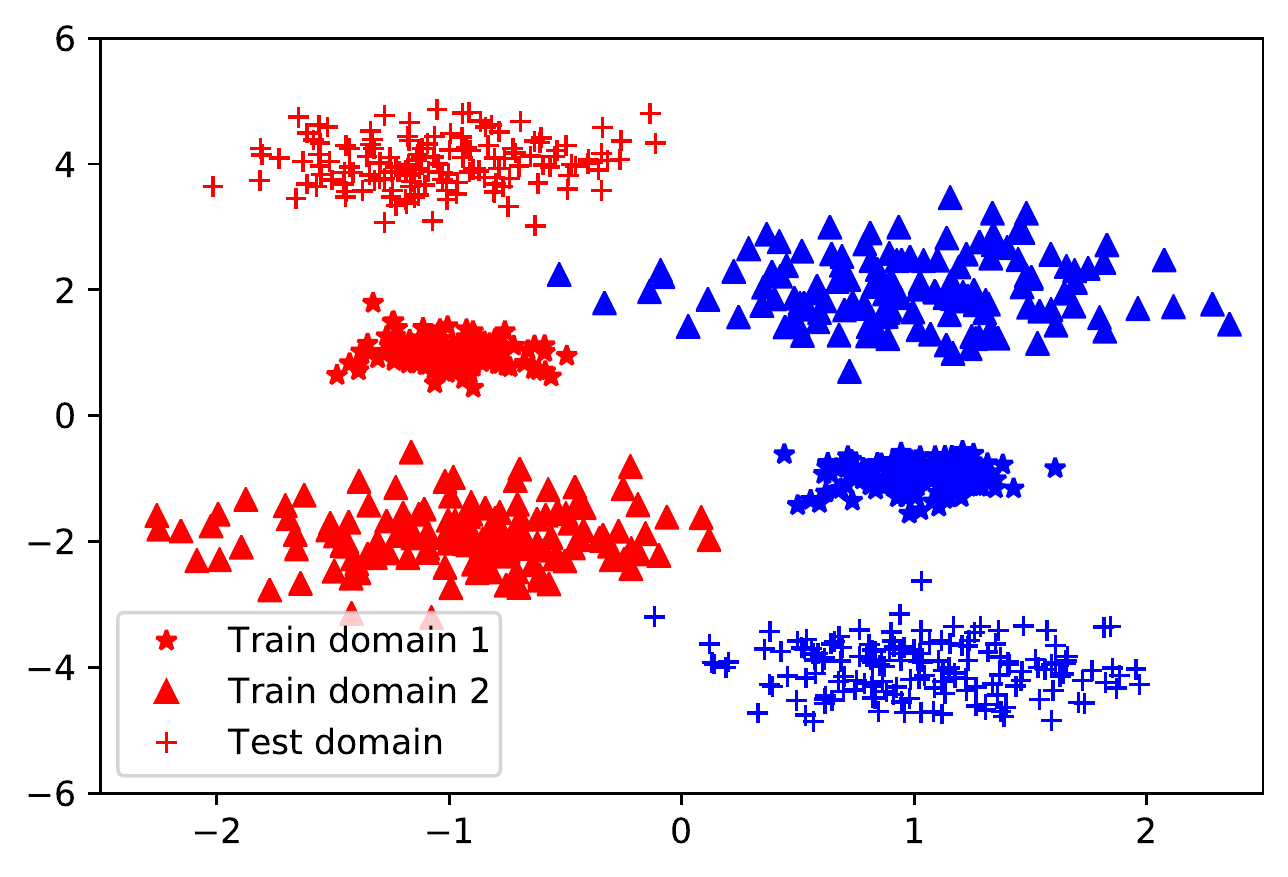}
         \caption{Train and test data}
         \label{fig:dg:a}
     \end{subfigure}
     \hfill
     \begin{subfigure}[b]{0.48\textwidth}
         \centering
         \includegraphics[width=\textwidth]{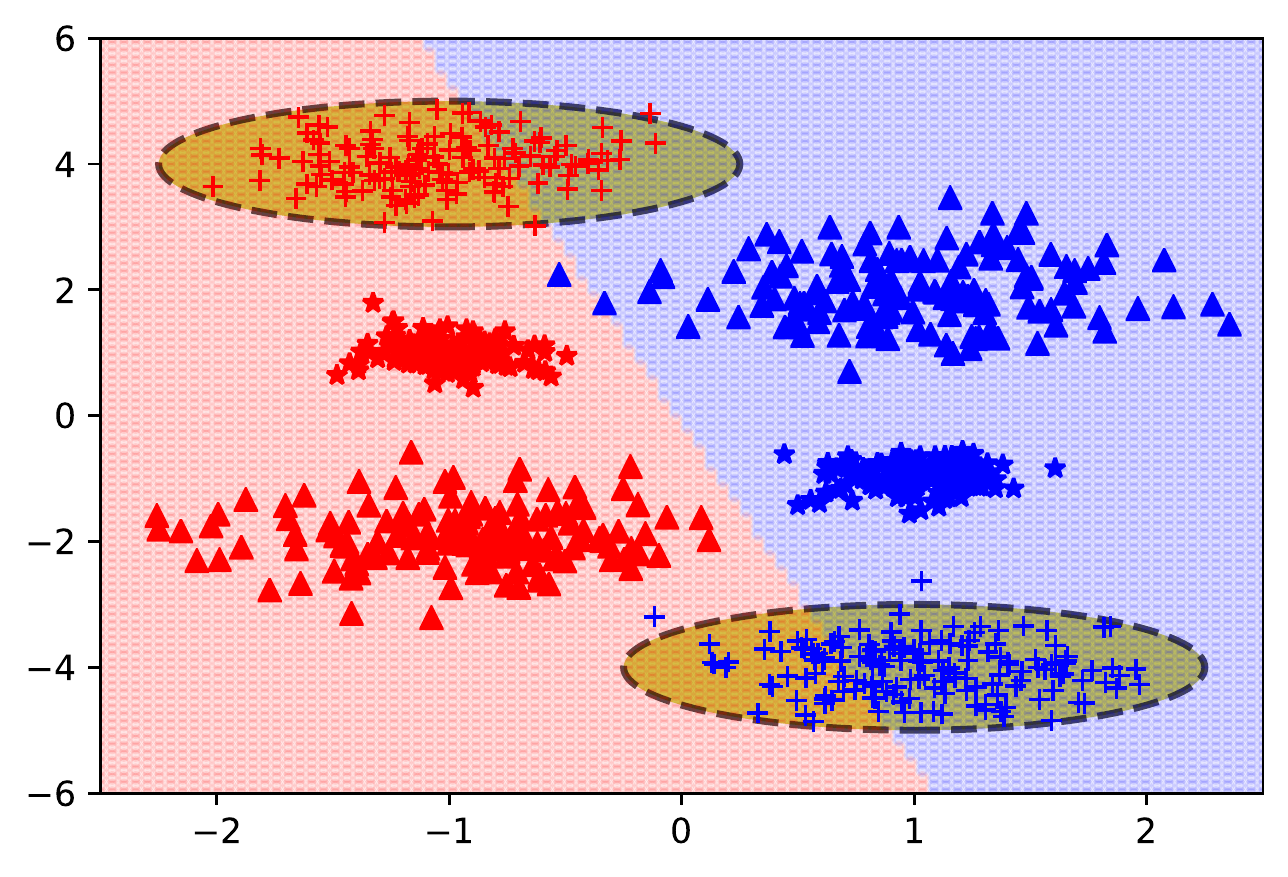}
         \caption{ERM fitted classifier}
         \label{fig:dg:b}
     \end{subfigure}\\
     \begin{subfigure}[b]{0.48\textwidth}
         \centering
         \includegraphics[width=\textwidth]{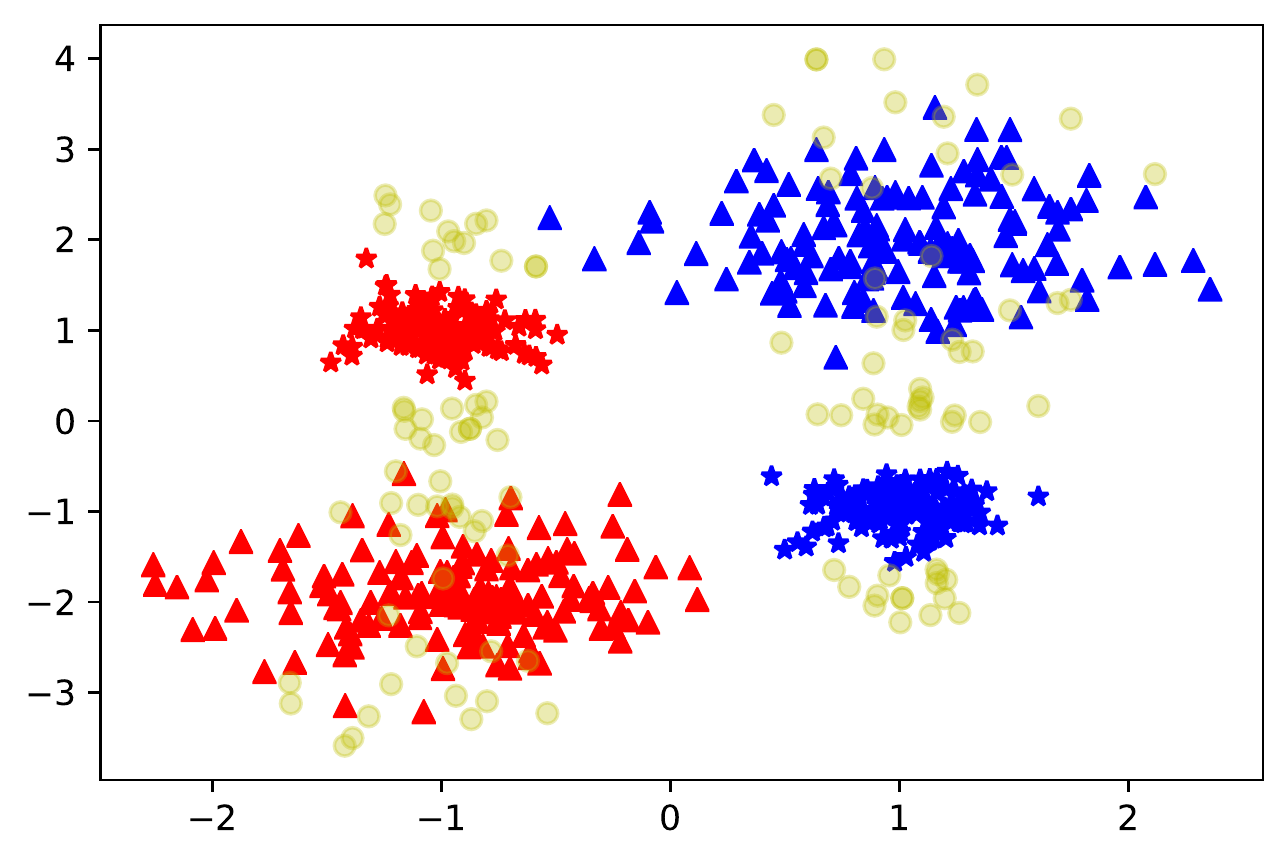}
         \caption{Augmentation of \crossgrad{} in circles}
         \label{fig:dg:c}
     \end{subfigure}
     \hfill
     \begin{subfigure}[b]{0.48\textwidth}
         \centering
         \includegraphics[width=\textwidth]{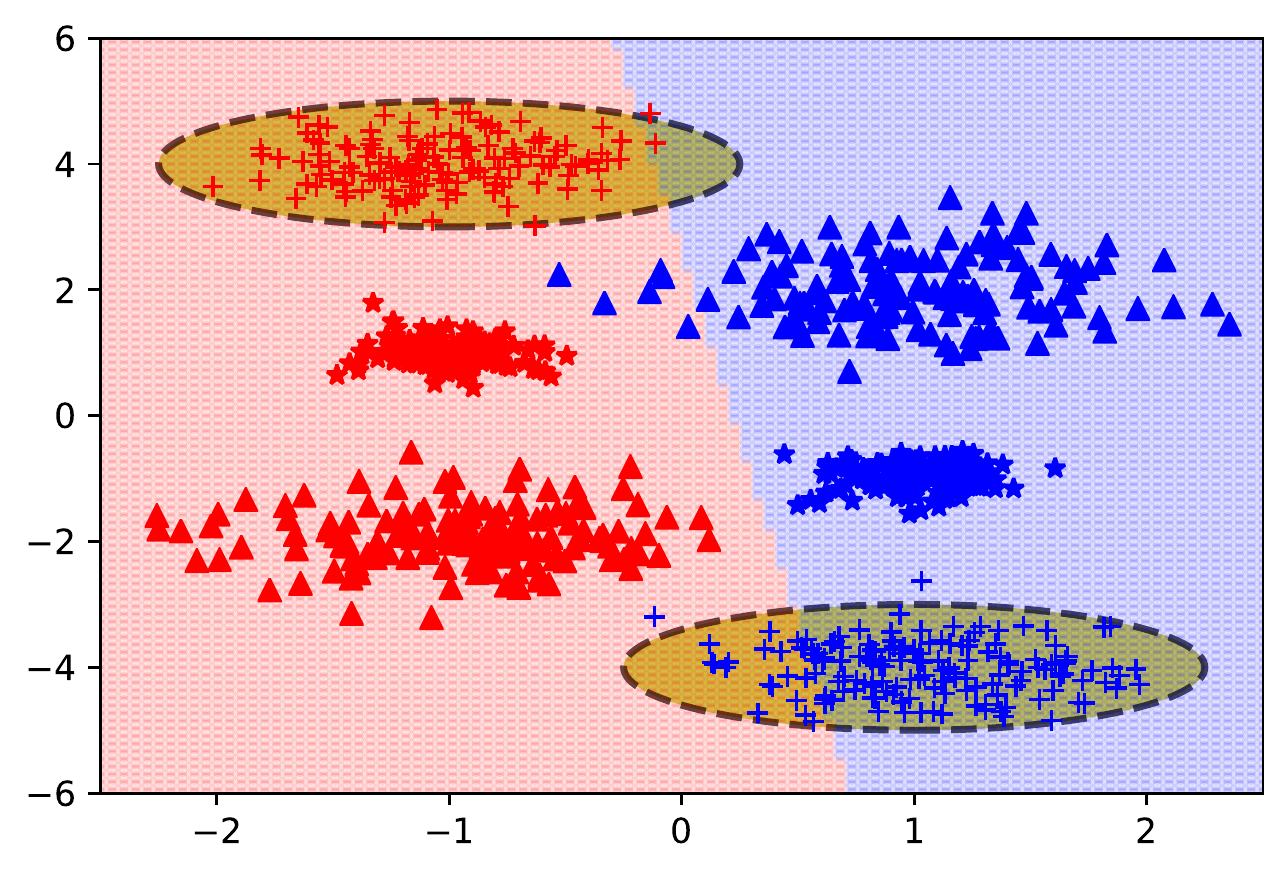}
         \caption{\crossgrad{} fitted classifier}
         \label{fig:dg:d}
    \end{subfigure}\\
    \begin{subfigure}[b]{0.48\textwidth}
         \centering
         \includegraphics[width=\textwidth]{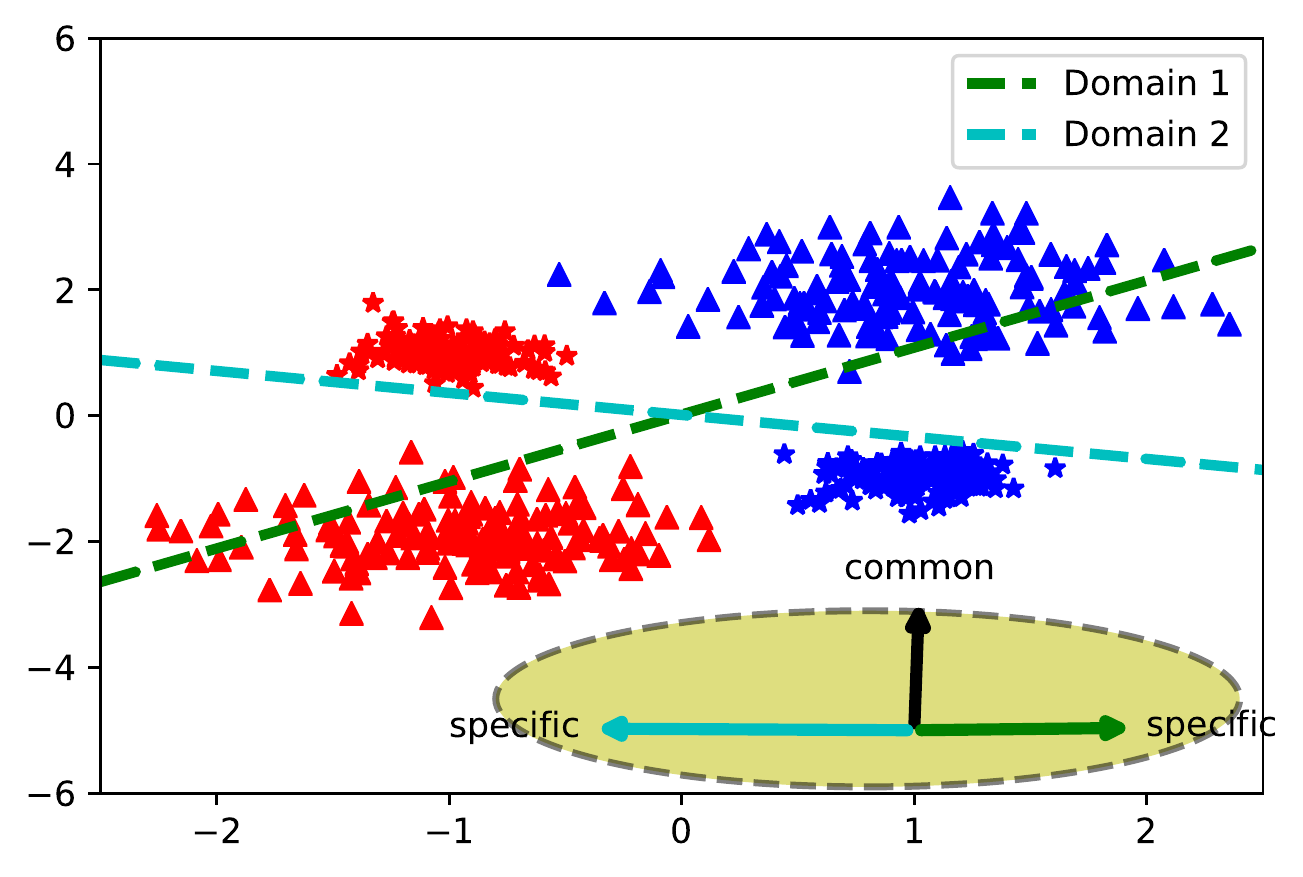}
         \caption{Common-specific decomposition of CSD}
         \label{fig:dg:e}
     \end{subfigure}
     \hfill
     \begin{subfigure}[b]{0.48\textwidth}
         \centering
         \includegraphics[width=\textwidth]{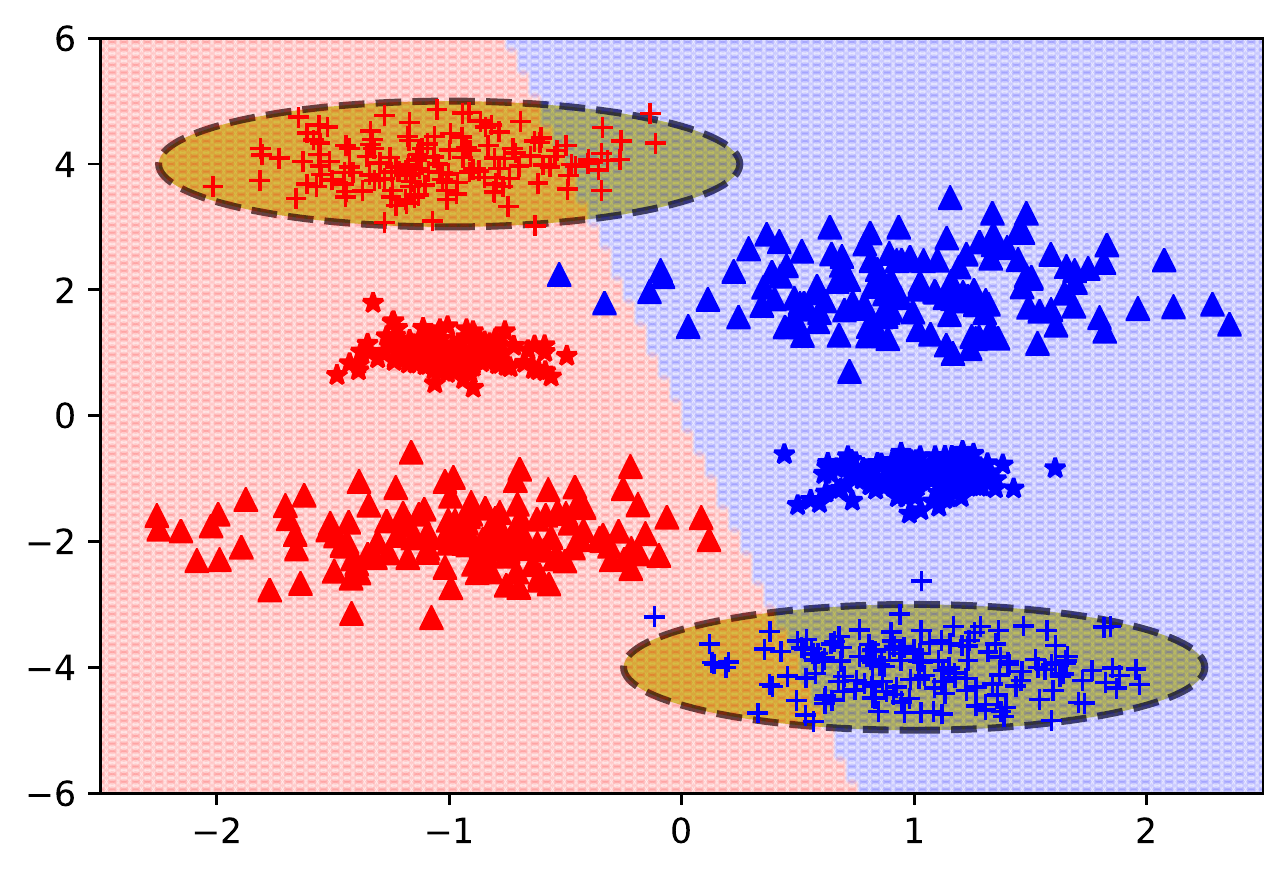}
         \caption{CSD fitted classifier}
         \label{fig:dg:f}
     \end{subfigure}
    \caption{When $\ve_c=[1, 0]^T, \ve_s=[0, 1]^T$, figure~\ref{fig:dg:a} shows two train and one test domain for $\beta=-1, 2, -4; \sigma=0.2, 0.5, 0.4$ respectively, where $\Sigma=\sigma I_{2\times 2}$. Figure~\ref{fig:dg:b},~\ref{fig:dg:d},~\ref{fig:dg:f} shows the linear classifier fitted by the ERM, \crossgrad{}, CSD algorithms and highlights the test domain in dotted regions. \crossgrad{} augments examples from new domains (along $e_s$) as shown by yellow circles in figure~\ref{fig:dg:c}. CSD decomposes common component from per-domain classifiers, figure~\ref{fig:dg:e}. \crossgrad{}, CSD fitted classifiers are closer to the common component when compared with ERM's fit. Plots generated using a \href{https://colab.research.google.com/drive/1KwzIdfjqFWoycTGnMJydxZmg2k3M7Sd2?usp=sharing}{colab notebook}.}
    \label{fig:csd:dg_example}
\end{figure}

We present two algorithms for recovering the generalizing common component. Our first algorithm, called \crossgrad{}~\citep{VihariSSS18} described in Section~\ref{sec:crossgrad:probmodel}, recovers the common component by suppressing the specific component implicitly through data augmentation from new domains, as illustrated in Figure~\ref{fig:dg:c}; augmented examples are shown as yellow circles. Intuitively, augmentations along the Y axis can learn a classifier close to the generalizing classifier as presented in Figure~\ref{fig:dg:d}. Our second algorithm, called CSD~\citep{VihariNS2020} described in Section~\ref{sec:csd}, explicitly recovers the common component through decomposition of per-domain classifiers as illustrated in Figure~\ref{fig:dg:e}. The fitted classifier with CSD is shown in Figure~\ref{fig:dg:f}. We describe each of the algorithms in detail next, present empirical results on real-world datasets and conclude the chapter with a discussion.

\section{{\crossgrad}: Cross-Gradient Training}
\label{sec:crossgrad:probmodel}

\noindent
We describe our first algorithm {\crossgrad} in this section. 


We use a Bayesian network to model the dependence among the label $y$, domain $d$, and input $\vx$ as shown in Figure~\ref{fig:crossgrad:bn}.  Variables $y \in \cY$ and $d \in \cD$ are discrete and $\vg \in \R^q,\vx \in \R^r$ lie in continuous multi-dimensional spaces.
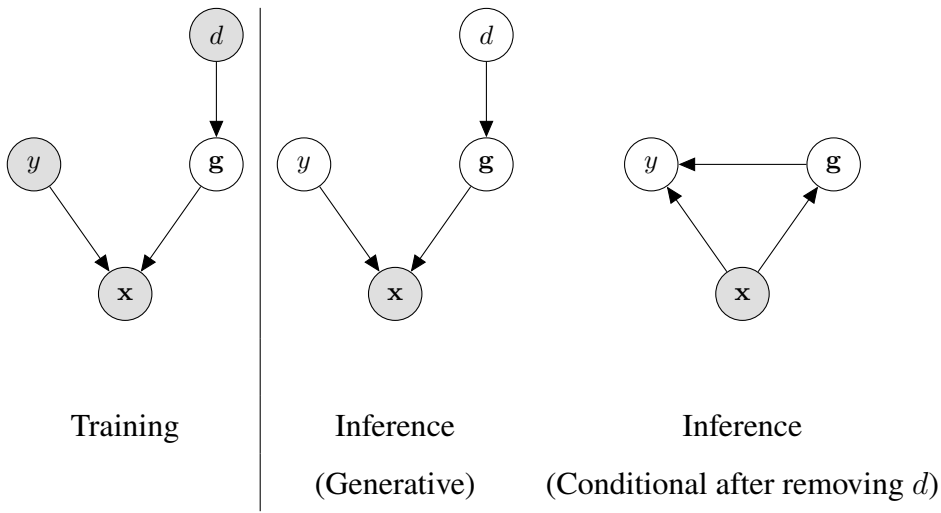
\begin{figure}[ht]
  \begin{center}
  \begin{tabular}{c|cc}
\begin{tikzpicture}
  \node[obs]                               (x) {$\mathbf{x}$};
  \node[latent, above=of x, xshift=1.2cm]  (g) {$\mathbf{g}$};
  \node[obs, above=of x, xshift=-1.2cm]  (y) {$y$};
  \node[obs, above=of g]            (d) {$d$};
  \edge {g, y} {x} ; 
  \edge {d} {g} ;
\end{tikzpicture}
&
\begin{tikzpicture}
  \node[obs]                               (x) {$\mathbf{x}$};
  \node[latent, above=of x, xshift=1.2cm]  (g) {$\mathbf{g}$};
  \node[latent, above=of x, xshift=-1.2cm]  (y) {$y$};
  \node[latent, above=of g]            (d) {$d$};
  \edge {g, y} {x} ; %
  \edge {d} {g} ;
\end{tikzpicture}
&
\begin{tikzpicture}
  \node[obs]                               (x) {$\mathbf{x}$};
  \node[latent, above=of x, xshift=1.2cm]  (g) {$\mathbf{g}$};
  \node[latent, above=of x, xshift=-1.2cm]  (y) {$y$};
  \edge {x} {g, y} ; %
  \edge {g} {y} ;
\end{tikzpicture}\\
& \\
Training & Inference & Inference \\
 & (Generative) & (Conditional after removing $d$)
\end{tabular}
\end{center}
  \caption{\label{fig:crossgrad:bn} Bayesian network to model dependence between label $y$, domain $d$, and input $\vx$.}
\end{figure}
The domain $d$ induces a set of latent domain features $\vg$.  The input $\vx$ is obtained by a complicated, un-observed mixing\footnote{The dependence of $y$ on $\vx$ could also be via continuous hidden variables but our model for domain generalization is agnostic of such structure.} of $y$ and $\vg$. In the training sample $L$, nodes $y,d,\vx$ are observed but $L$ spans only a proper subset $D$ of the set of all domains $\cD$.  During inference, only $\vx$ is observed and we need to compute the posterior $\Pr(y|\vx)$.
As reflected in the network, $y$ is not independent of $d$ given $\vx$.  However, since $d$ is discrete and we observe only a subset of $d$'s during training, we need to make additional assumptions to ensure that we can generalize to a new $d$ during testing.
The assumption we make is that integrated over the training domains the distribution $\Pr(\vg)$ of the domain features is well-supported in $L$.  More precisely, generalizability of a training set of domains $D$ to the universe of domains $\cD$ requires that the training data cover the $\vg$ distribution well. 
\begin{align}
\label{eq:crossgrad:dg}
\Pr(\vg) &= \sum_{d \in \cD} \Pr(\vg | d) \Pr(d)
\approx \sum_{d \in D} \Pr(\vg | d) \frac{\Pr(d)}{\Pr(D)} \tag{A1}
\end{align}
Under this assumption $\Pr(\vg)$ can be modeled during training, so that during inference we can infer $y$ for a given $\vx$ by estimating
\begin{equation}
\Pr(y|\vx) = \sum_{d\in \cD}\Pr(y|\vx,d)\Pr(d|\vx) = \int_\vg \Pr(y|\vx,\vg)\Pr(\vg|\vx) \approx \Pr(y|\vx,\hat{\vg})
\end{equation}
where $\hat{\vg} = \argmax_\vg \Pr(\vg|\vx)$ is the inferred continuous representation of the domain of $\vx$. 

Existence of low-dimensional continuous representation of the domain is key to our being able to claim generalization to new domains even though most real-life domains are discrete.  For example, domains like fonts and speakers are discrete, but their variation can be captured via latent continuous features (e.g. slant, ligature size etc. of fonts; speaking rate, pitch, intensity, etc. for speech).  The assumption states that as long as the training domains span the latent continuous features we can generalize to new fonts and speakers. 

We next elaborate on how we estimate $\Pr(y|\vx,\hat{\vg})$ and $\hat{\vg}$ using the domain-labeled data \mbox{$L=\{(\vx, y, d)\}$}.
The main challenge in this task is to ensure that the model for $\Pr(y|\vx,\hat{\vg})$ is not over-fitted on the inferred $\vg$ values of the training domains. In many applications, the per-domain $\Pr(y|\vx,d)$ is significantly easier to train.  So, an easy local minima is to choose a different $\vg$ for each training $d$ and generate separate classifiers for each distinct training domain.  We must encourage the network to stay away from such easy solutions.
We strive for generalization by moving along the continuous space $\vg$ of domains to sample new training examples from hallucinated domains.
Ideally, for each training instance $(\vx_i, y_i)$ from a given domain $d_i$, we wish to generate a new $\vx'$ by transforming its (inferred) domain $\vg_i$ to a random domain sampled from
$\Pr(\vg)$, keeping its label $y_i$ unchanged.   Under the domain continuity assumption \eqref{eq:crossgrad:dg}, a model trained with such an ideally augmented dataset is expected to generalize to domains in $\cD\setminus D$.

However, there are many challenges to achieving such ideal augmentation.  To avoid changing $y_i$, it is convenient to draw a sample $\vg$ by perturbing $\vg_i$.  But $\vg_i$ may not be reliably inferred, leading to a distorted sample of $\vg$.   For example, if the $\vg_i$ obtained from an imperfect extraction contains label information, then big jumps in the approximate $\vg$ space could change the label too.  We propose a more cautious data augmentation strategy that perturbs the input to make only small moves along the estimated domain features, while changing the label as little as possible.  We arrive at our method as follows.

\mypara{Domain inference.}
We create a model $G(\vx)$ to extract domain features $\vg$ from an input $\vx$.  
We supervise the training of $G$ to predict the domain label $d_i$ as $S(G(\vx_i))$ where $S$ is a softmax transformation.
We use $J_{d}$ to denote the cross-entropy loss function of this classifier.  Specifically, 
$J_d(\vx_i,d_i)$ is the domain loss at the current instance.

\noindent
{\bf Domain perturbation.}
Given an example $(\vx_i,y_i,d_i)$, we seek to sample a new example $(\vx'_i,y_i)$ (i.e., with the same label $y_i$), whose domain is as ``far'' from $d_i$ as possible. 
\begin{wrapfigure}{r}{0.45\textwidth}
\centering
\includegraphics[scale=.4]{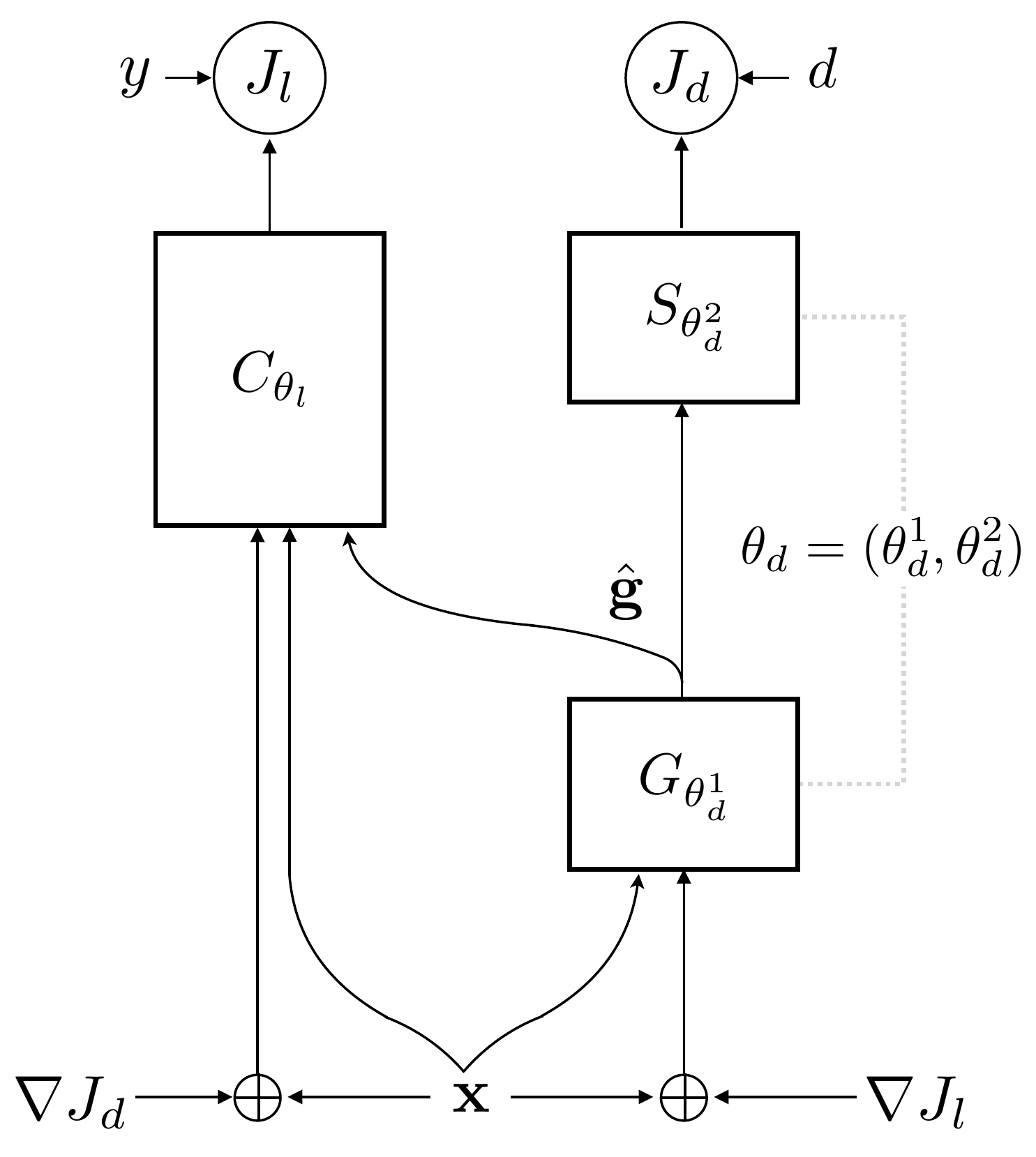}
\caption{{\crossgrad} network design.}
\label{fig:crossgrad:network}
\end{wrapfigure}
To this end, consider setting $\vx'_i = \vx_i + \epsilon \nabla_{\vx_i} J_d(\vx_i, d_i)$. Intuitively, this perturbs the input along the direction of greatest domain change\footnote{We use $\nabla_{\vx_i} J_d$ as shorthand for the gradient $\nabla_{\vx} J_d$ evaluated at $\vx_i$.}, for a given budget of $||\vx'_i-\vx_i||$. However, this approach presupposes that the direction of domain change in our domain classifier is 
not highly correlated with the direction of label change. To enforce this in our model, we shall train the domain feature extractor $G$ to avoid domain shifts when the data is perturbed to cause label shifts.

\begin{algorithm}[htb]
  \caption{\crossgrad\ training pseudocode.}\label{alg}
  \begin{algorithmic}[1]
    \State {\bf Input:} Labeled data $\{(\vx_i, y_i, d_i)\}_{i=1}^M$, step sizes $\epsilon_l, \epsilon_d$, learning rate $\eta$, data augmentation weights $\alpha_l, \alpha_d$, number of training steps $n$.
    \State {\bf Output:} Label and domain classifier parameters $\theta_l,\theta_d$
    \State Initialize $\theta_l,\theta_d$
    \Comment {$J_l, J_d$ are loss functions for the label and domain classifiers, respectively}
    \For{$n$ training steps}
      \State Sample a labeled batch $(X,Y,D)$
      \State $X_d := X + \epsilon_l \cdot \nabla_{X} J_d(X, D; \theta_d)$
      \State $ X_l  := X + \epsilon_d \cdot \nabla_{X} J_l(X,Y;\theta_l)$
      \State $\theta_l \leftarrow \theta_l - \eta\nabla_{\theta_l} ((1 - \alpha_l) J_l(X, Y;\theta_l) + \alpha_l J_l(X_d, Y;\theta_l))$
      \State $\theta_d \leftarrow \theta_d - \eta\nabla_{\theta_d} ((1 - \alpha_d) J_d(X, D; \theta_d) + \alpha_d J_d(X_l, D; \theta_d))$
    \EndFor
  \end{algorithmic}
\end{algorithm}

\subsection{Algorithm}
The above development leads to the network sketched in Figure~\ref{fig:crossgrad:network}, and
an accompanying training algorithm, \crossgrad, shown in Algorithm~\ref{alg}. Here $X,Y,D$ correspond to a minibatch of instances. Our proposed method integrates data augmentation and batch training as an alternating sequence of steps. In line 6, 7, we obtain perturbations of the input that change domain and label in the direction of greatest change using domain and label classifier respectively. The label classifier is trained with original and domain perturbed examples. The domain classifier is simultaneously trained with the perturbations from the label classifier network so as to be robust to label changes. 
Thus, we construct cross-objectives $J_{l}$ and $J_{d}$, and update their respective parameter spaces. We found this scheme of simultaneously training both networks to be empirically superior to independent training even though the two classifiers do not share parameters.


In practice, we may see some variations to the Bayesian network shown in Figure~\ref{fig:crossgrad:bn}. For example, we may have that $y$ is independent of $d$ given $\vx$ as shown to the left of Figure~\ref{fig:crossgrad:bn_alternate}. Therefore, conditioning on $\hat{\vg}$ in $\Pr(y\mid \mathbf{x}, \hat{\vg})$ may not help. However, practical applications could be more complicated as shown to the right of Figure~\ref{fig:crossgrad:bn_alternate}. That is, input is an unknown mixture of its latent features: ($\mathbf{x}_c, \mathbf{x}_s$), where we may have mix of dependencies. As a result, $d$ is correlated to varying extent with the label $y$. \crossgrad{} is designed to handle the whole spectrum of correlations that emerge in applications. 

\begin{figure}[htb]
\centering
    \begin{tabular}{c|c}
        \begin{tikzpicture}
          \node[obs]                               (x) {$\mathbf{x}$};
          \node[obs, above=of x, xshift=-1.8cm]  (y) {$y$};
          \node[obs, above=of x, xshift=1.8cm]  (d) {$d$};
          \edge {x} {y} ; 
          \edge {d} {x} ;
        \end{tikzpicture} &
        \begin{tikzpicture}
          \node[latent, above=of x, xshift=-1.2cm] (xc) {$\mathbf{x}_c$};
          \node[latent, above=of x, xshift=1.2cm]  (xs) {$\mathbf{x}_s$};
          \node[obs]                               (x) {$\mathbf{x}$};
          \node[obs, above=of xc, xshift=-1.2cm](y) {$y$};
          \node[obs, above=of xs, xshift=1.2cm] (d) {$d$};
          \edge {xc} {y} ;
          \edge {d} {xc} ;
          \edge {xc} {x} ;
          \edge {xs} {x} ;
          \edge {y} {xs} ;
          \edge {d} {xs} ;
        \end{tikzpicture} \\
        Label C.I. domain & General model
    \end{tabular}
    \caption{}
    \label{fig:crossgrad:bn_alternate}
\end{figure}

\subsection{Why does \crossgrad\ work?}
\label{sec:crossgrad:why}

We present insights on the working of \crossgrad\ via experiments on the Rotated-MNIST dataset. In this dataset, the training data contains MNIST digits rotated by varying angles from 0$^\degree$ to 75$^\degree$ in the interval of 15$^\degree$. The domains corresponding to image rotations are easy to interpret. 

In Figure~\ref{fig:crossgrad:304560}, we show PCA projections
of the $\vg$ embeddings for images from three different domains, corresponding to rotations by \textit{30, 45, 60} degrees in green, blue, and yellow respectively. The $\vg$ embeddings of domain \textit{45} (blue) lies in between the $\vg$ of domains \textit{30} (green) and \textit{60} (yellow) showing that the domain classifier has successfully extracted continuous representation of the domain even when the input domain labels are categorical. Figure~\ref{fig:crossgrad:01530} shows the same pattern for domains \textit{0, 15, 30}.  Here again we see that the embedding of domain \textit{15} (blue) lies in-between that of domain \textit{0}~(yellow) and \textit{30}~(green).

Next, we show that the $\vg$ perturbed along gradients of domain loss does manage to generate images that substitute for the missing domains (rotations) during training.
For example, the embeddings of the domain \textit{45}, when perturbed, scatters towards the domain \textit{30} and \textit{60} as can be seen in Figure~\ref{fig:crossgrad:304560p}: note the scatter of \textit{ perturbed  45} (red) points inside the \textit{30} (green) zone, without any \textit{45} (blue) points. Figure~\ref{fig:crossgrad:01530p} depicts a similar pattern with perturbed domain embedding points (red) scattering towards domains \textit{30} and \textit{0} more than unperturbed domain \textit{15} (blue).  For example, between x-axis -1 and 1 dominated by the green domain (domain \textit{30}) we see many more red points (perturbed domain \textit{15}) than blue points  (domain \textit{15}).  Similarly in the lower right corner of domain \textit{0} shown in yellow.
%
This highlights the mechanism of \crossgrad\ working; that it is able to augment training with samples closer to unobserved domains.

Finally, we observe in Figure~\ref{fig:crossgrad:embed_comparison_label} that the embeddings are not correlated with labels. For both domains \textit{30} and \textit{45} the colors corresponding to different labels are not clustered. This is a consequence of \crossgrad's symmetric training of the domain classifier via label-loss perturbed images.

\begin{figure}[thb]
\begin{subfigure}[b]{0.5\linewidth}
\begin{center}
  \includegraphics[width=200pt]{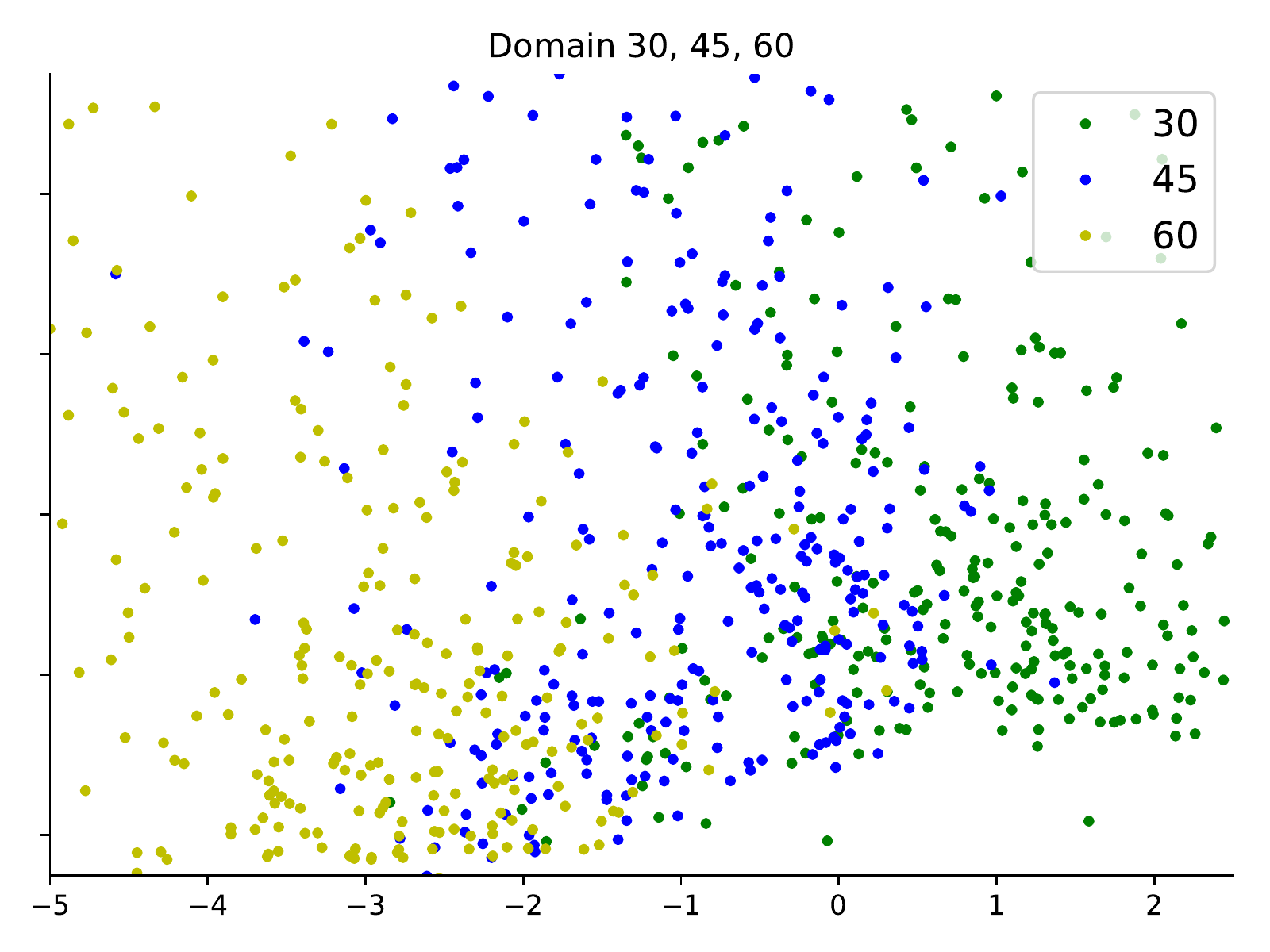}
  \caption{$\vg$ of domain \textit{45} lies between $\vg$ of \textit{60} and \textit{30}.}
  \label{fig:crossgrad:304560}
  \end{center}
 \end{subfigure}
 \begin{subfigure}[b]{0.5\textwidth}
\begin{center}
  \includegraphics[width=200pt]{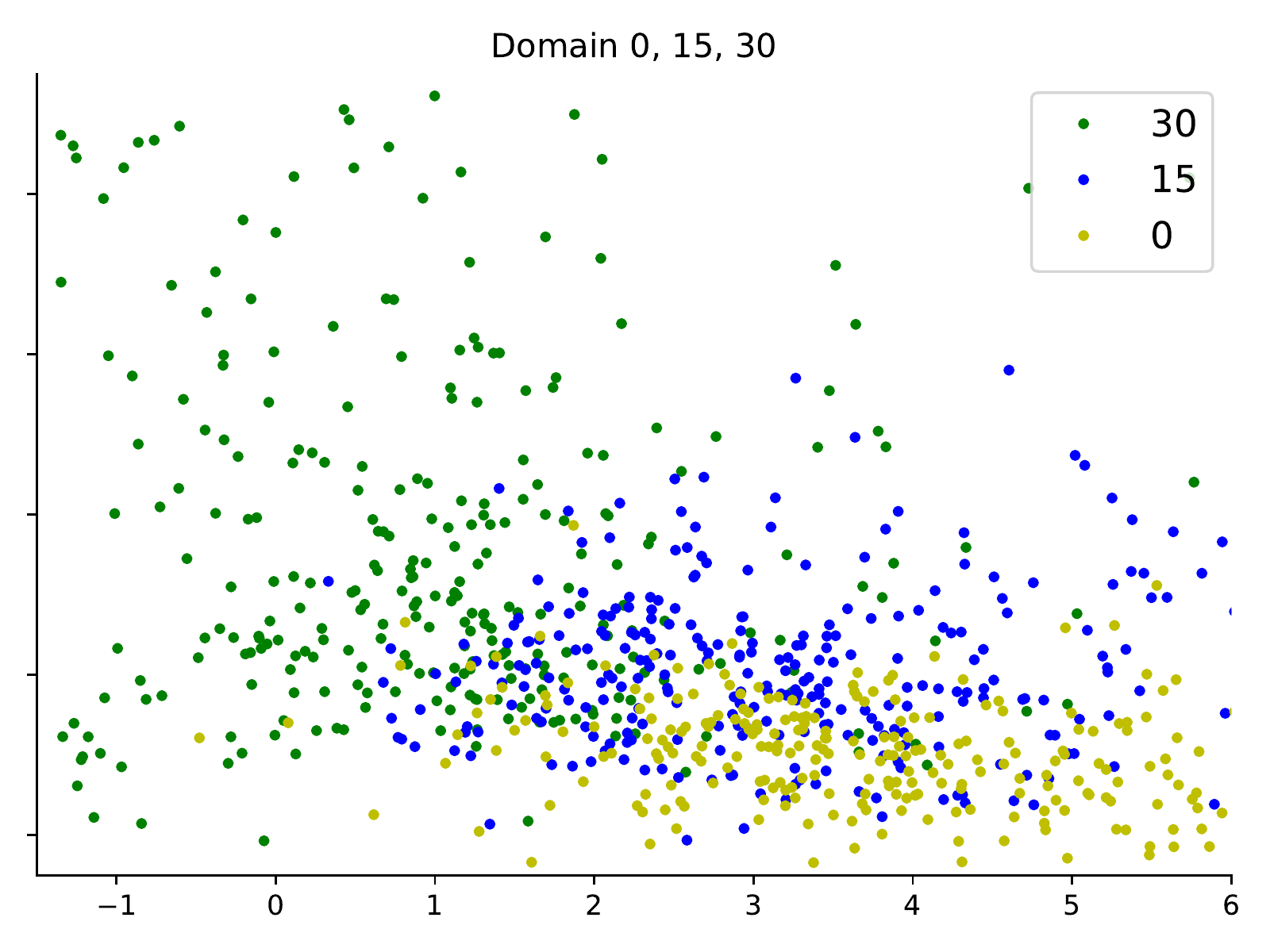}
  \caption{$\vg$ of domain \textit{15} lies between $\vg$ of \textit{0} and \textit{30}.}
  \label{fig:crossgrad:01530}
  \end{center}
 \end{subfigure}
 \begin{subfigure}[b]{0.5\textwidth}
 \begin{center}
  \includegraphics[width=200pt]{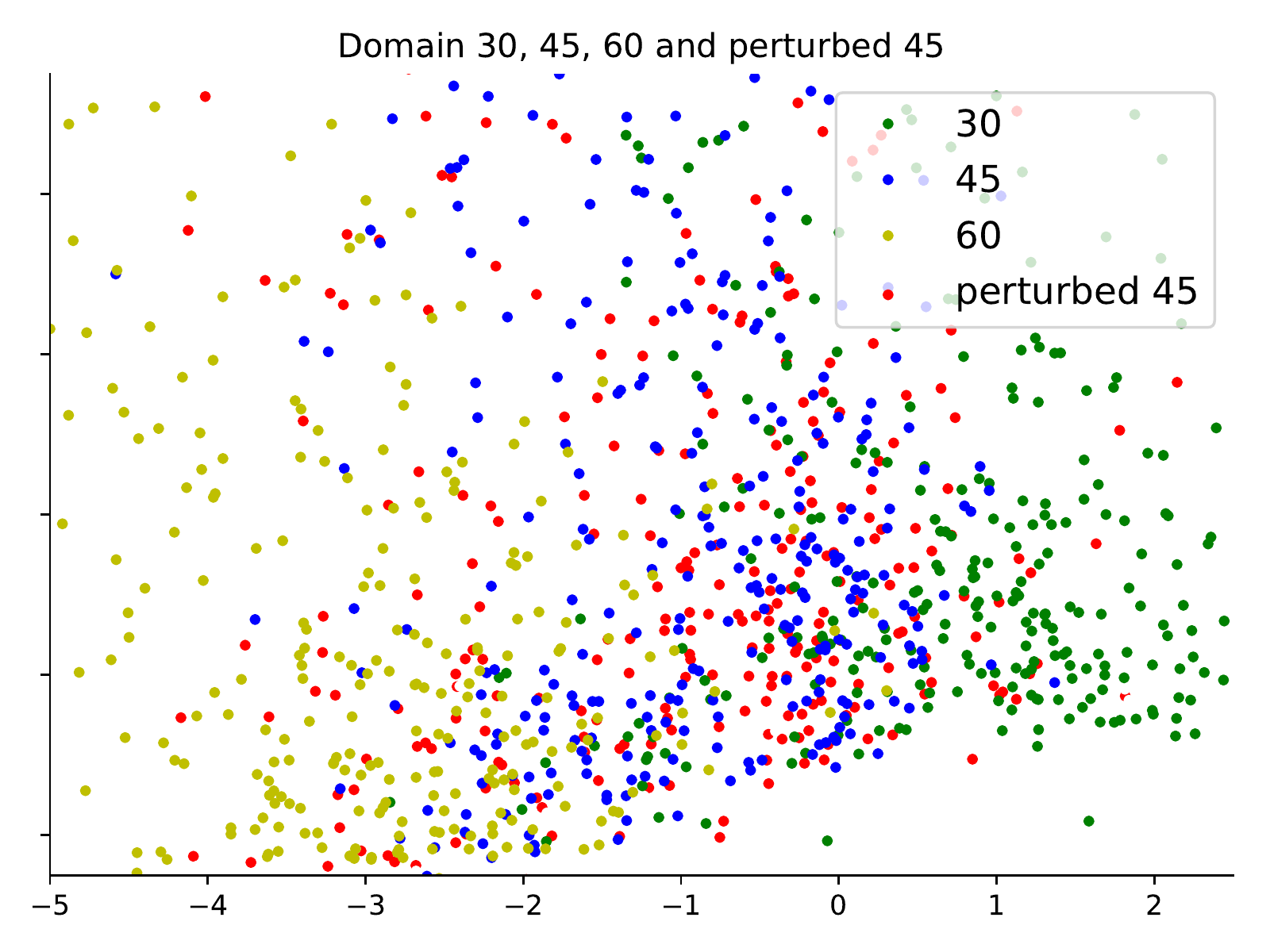}
  \caption{Adding $\vg$ of perturbed domain \textit{45} to above.}
  \label{fig:crossgrad:304560p}
  \end{center}
 \end{subfigure}
 \begin{subfigure}[b]{0.5\textwidth}
 \begin{center}
  \includegraphics[width=200pt]{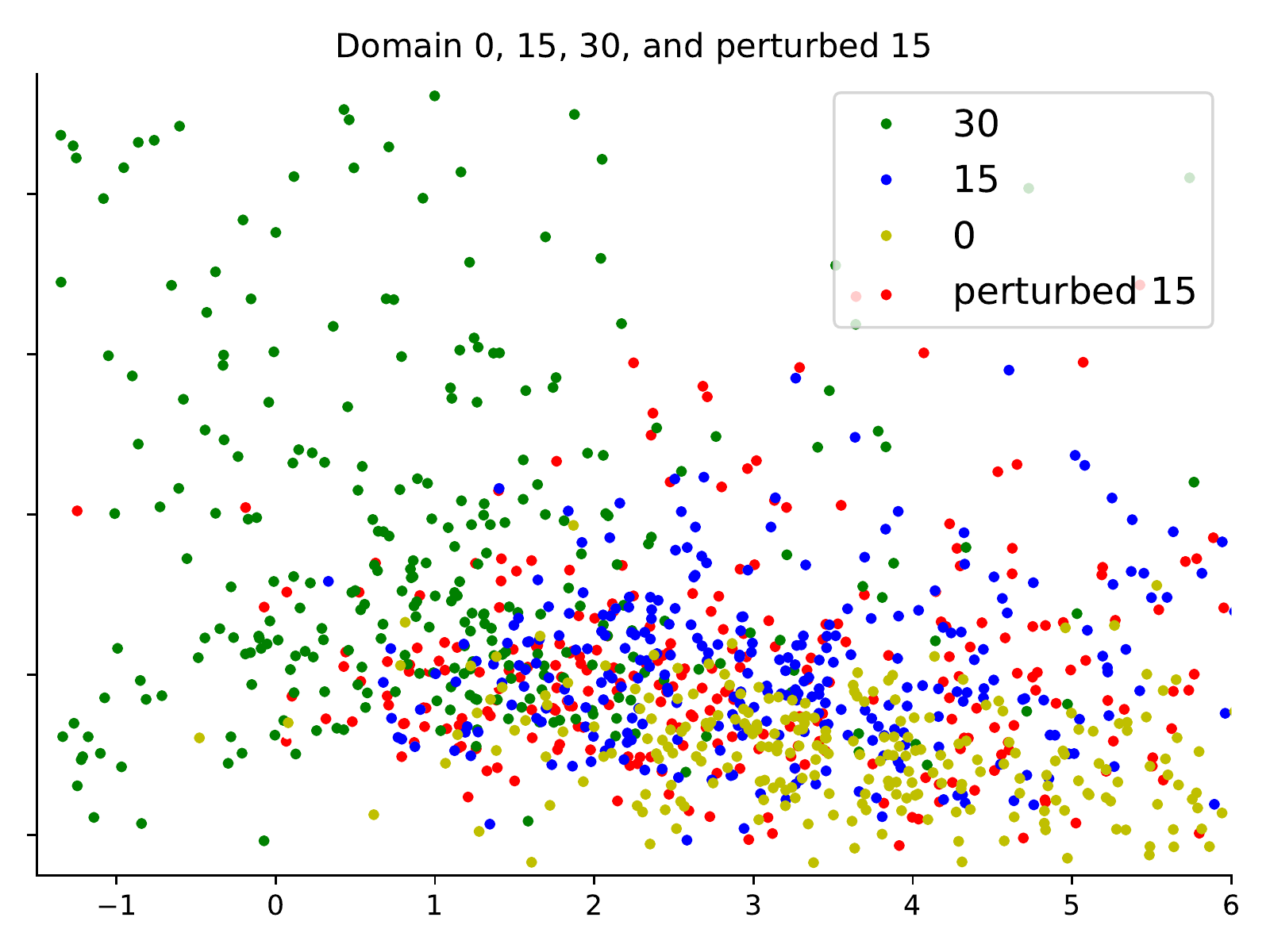}
  \caption{Adding $\vg$ of perturbed domain \textit{15} to above.}
  \label{fig:crossgrad:01530p}
  \end{center}
 \end{subfigure}
\caption{Comparing domain embeddings ($\vg$) across domains. Each color denotes a domain.}
\label{fig:crossgrad:embed_comparison}
\end{figure}

\begin{figure}[htb]
\begin{subfigure}[b]{0.48\textwidth}
\begin{center}
  \includegraphics[width=210pt]{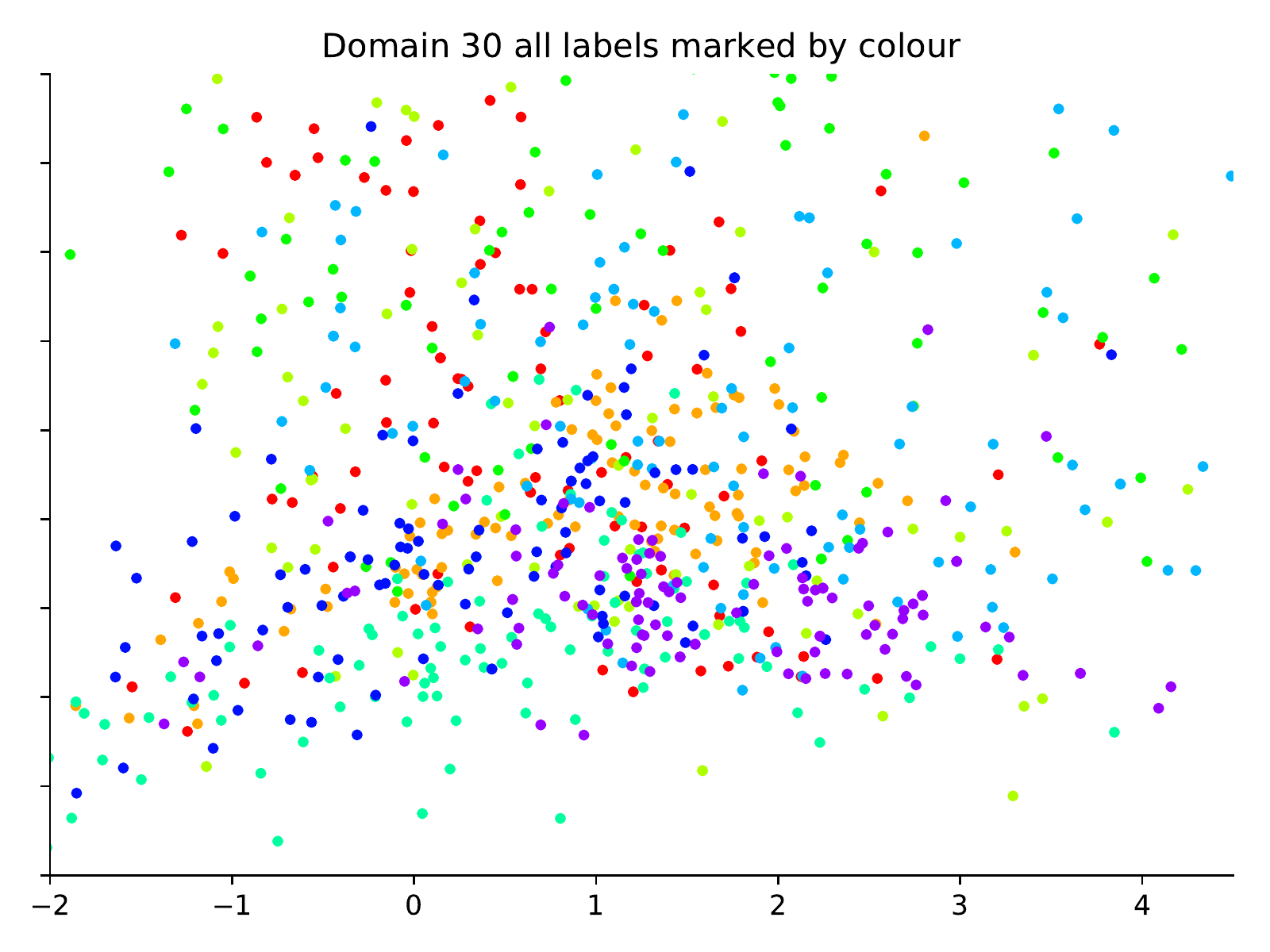}
  \caption{Domain M30}
  \label{fig:crossgrad:30r}
  \end{center}
 \end{subfigure}
 \begin{subfigure}[b]{0.48\textwidth}
 \begin{center}
  \includegraphics[width=210pt]{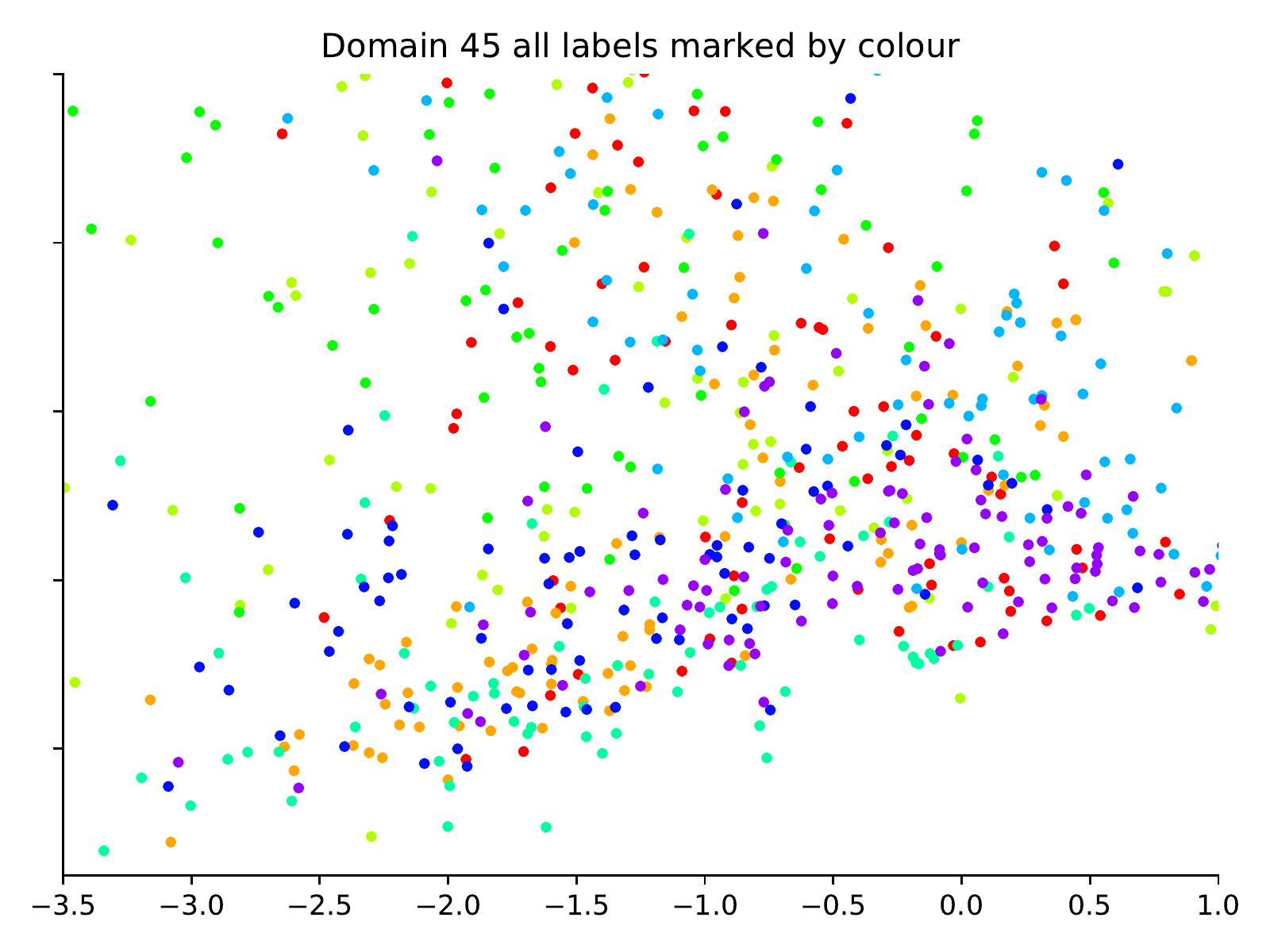}
  \caption{Domain M45}
  \label{fig:crossgrad:45r}
  \end{center}
 \end{subfigure}
\caption{Comparing domain embeddings ($\vg$) across labels. Each color denotes a label. Unclustered colors indicates that label information is mostly absent from $\vg$.}
\label{fig:crossgrad:embed_comparison_label}
\end{figure}

\section{CSD: Common-specific Decomposition}
\label{sec:csd}
In order to generate examples that only perturb the domain but not its label, \crossgrad{} searches in a small neighbourhood around the original example. And as a consequence, it cannot generalize far beyond the training domains. In this section, we will discuss an alternate algorithm that recovers domain generalizing solution by decomposing parameters in to common and specific components. 

Recall the simple setting from~\eqref{eqn:dg:simple} where m-dimensional input ($\vx\in \mathbb{R}^m$) is a combination of common features ($\vx_c$) whose correlation with label is constant across domains, domain-specific  ($\vx_s$) features whose correlation with the label varies from domain to domain, and a noise component as shown below for domain $d$:
\begin{align}\label{eqn:csd:synth-example}
    \vx = y\ve_c + y\beta_d \ve_s + \mathcal{N}(0,\Sigma_d) \in \mathbb{R}^m
\end{align}
where $y = \pm 1$ with equal probability, $e_c \in \R^m, e_s\perp e_c\in \R^m$. 
Suppose for each train domain $d$, $\beta_d \sim \textrm{Unif}\left[-1,2\right]$. Note that though $\beta_{d}$ vary from positive to negative across various domains, there is a net positive correlation between $\vx$ and the label $y$. 
$\mathcal{N}(0,\Sigma_d)$ denotes a  standard normal random variable with mean zero and covariance matrix $\Sigma_d$. Since $\Sigma_d$ varies across domains, every feature  
captures some domain information. 
Since during the test time, there is the possibility of seeing $\beta_d \notin [-1,2]$, the only classifier that generalizes to new domains is the one that depends solely on $\vx_c$.\footnote{Note that this last statement relies on the assumption that $e_c \perp e_s$. 
If this is not the case, the correct domain generalizing classifier is the component of $e_c$ that is orthogonal to $e_s$ i.e., $e_c - \frac{\iprod{e_c}{e_s}}{\norm{e_s}^2} \cdot e_s$. See~\eqref{eq:proj}.} 
Consider training a linear classifier on this dataset. We will describe the issues faced by existing domain generalization methods.

\begin{itemize}
    \item {\textbf{Empirical risk minimization (ERM)}: When we empirically train a linear classifier using ERM with cross entropy loss on all of the training data,
    the resulting classifier puts significant nonzero weight on the domain-specific  component $e_s$. The reason for this is that there is a bias in the training data which gives an overall positive correlation between $e_s$ and the label. Figure~\ref{fig:csd:dg_example} illustrates the same with an example.}
    \item {\textbf{Domain erasure}~\citep{Ganin16}: Domain erasure methods seek to extract features that have the \emph{same distribution} across different domains and construct a classifier using those features. However, since the noise covariance matrix in~\eqref{eqn:csd:synth-example} varies across domains, none of the features have the \emph{same distribution} across domains. We further empirically verified that domain erasure methods do not improve much over ERM.
    }
    \item {\textbf{Meta-learning}:
    Meta-learning based DG approaches such as \citep{DouCK19} work with pairs of domains. Parameters updated using gradients on loss of one domain, when applied on samples of both domains in the pair should lead to similar class distributions.  If the method used to detect similarity is robust to domain-specific noise, meta-learning methods could work well in this setting.  
    But meta-learning methods could be expensive to implement in practice either because they require the second order gradient updates or because they have a quadratic dependence on the number of domains and thereby not scale well to a large number of train domains. 
    }
    \item {\textbf{decomposition-based approaches}~\citep{ECCV12_Khosla,LiYSH17} assume the domain-specific  classifier for each domain is composed of common component that generalizes to any domain and a specific component that transfers poorly to other domains. They then isolate the common generalizing component as the domain generalizing solution. We shall discuss these in more detail later.} 
\end{itemize}

\noindent
Existing decomposition-based approaches~\citep{ECCV12_Khosla,LiYSH17} rely on the observation that for problems like~\eqref{eqn:csd:synth-example}, there exist a good domain-specific  classifiers $w_d$, one for each domain $d$, such that:
    $\tilde{w}_d = e_c + \gamma_d e_s$,
where $\gamma_d$ is a function of $\beta_d$. Note that all these domain-specific  classifiers share the common component $e_c$ which is the domain generalizing classifier that we are looking for! If we are able to find domain-specific  classifiers of this form,
we can extract $e_c$ from them. This idea can be extended to a generalized version of~\eqref{eqn:csd:synth-example}, where the latent dimension of the domain space is $k$ i.e., say
\begin{align}\label{eqn:csd:synth-example-k}
    \vx = y(\ve_c + \sum_{j=1}^k \beta_{d,j} \ve_{s_j}) + \mathcal{N}(0,\Sigma_d).
\end{align}
$e_{s_j} \perp e_c \in \R^m$, {and $e_{s_j} \perp e_{s_\ell}$ for $j,\ell \in [k], j \neq \ell$} are domain-specific  features whose correlation with the label, given by the coefficients $\beta_{d,j}$, varies from domain to domain. In this setting, there exist good domain-specific  classifiers $\tilde{w}_d$ such that:
    $\tilde{w}_d = \ve_c + E_s \gamma_d$,
where $\ve_c \in \R^m$ is a domain generalizing classifier, $E_s = \begin{bmatrix} \ve_{s_1} & \ve_{s_2} & \cdots & \ve_{s_k} \end{bmatrix} \in \R^{m \times k}$ consists of domain-specific  components
and $\gamma_d \in \R^k$ is a domain-specific  combination of the domain-specific  components that depends on $\beta_{d,j}$ for $j=1,\cdots,k$. 
With this observation, the algorithm is simple to state: train domain-specific  classifiers $\tilde{w}_d$ that can be represented as
    $\tilde{w}_d = w_c + W_s \gamma_i \in \R^m$.
Here the training variables are $w_c \in \R^m, W_s \in \R^{m \times k}$ and $\gamma_d \in \R^k$. After training, discard all the domain-specific  components $W_s$ and $\gamma_d$ and return the common classifier $w_c$.
Note that this can equivalently be written as
\begin{align}\label{eqn:csd:classifier-decomp}
    W = w_c \trans{\ones} + W_s \trans{\Gamma},
\end{align}
where $W \defeq \begin{bmatrix} \tilde{w}_1 & \tilde{w}_2 & \cdots & \tilde{w}_D \end{bmatrix}$, $\ones \in \R^D$ is the all ones vector and $\trans{\Gamma} \defeq \begin{bmatrix} \gamma_1 & \gamma_2 & \cdots & \gamma_D \end{bmatrix}$. With a slight abuse of notation, we use $D$ to denote the number of training domains hereafter. 

This framing of the decomposition approach, in the context of simple examples as in~\eqref{eqn:csd:synth-example} and~\eqref{eqn:csd:synth-example-k}, lets us understand the three main aspects that are not properly addressed by prior works: 1) identifiability of $w_c$, 2) choice of $k$ and 3) extension to non-linear models such as neural networks.

\noindent
\textbf{Identifiability of the common component $w_c$}: None of the prior decomposition-based approaches investigate identifiability of $w_c$. In fact, given a general matrix $W$ which can be written as $w_c \trans{\ones} + W_s \trans{\Gamma}$, there are multiple ways of decomposing $W$ into this form, so $w_c$ cannot be uniquely determined by this decomposition alone. For example, given a decomposition~\eqref{eqn:csd:classifier-decomp}, for any $(k+1) \times (k+1)$ invertible matrix $R$, we can write $W = \begin{bmatrix} w_c & W_s \end{bmatrix} R^{-1} R \trans{\begin{bmatrix} \ones & \Gamma \end{bmatrix}}$. As long as the first row of $R$ is equal to $\begin{bmatrix}1 & 0 & \cdots & 0 \end{bmatrix}$, the structure of the decomposition~\eqref{eqn:csd:classifier-decomp} is preserved while $w_c$ might no longer be the same. Out of all the different $w_c$ that can be obtained this way, which one is the \emph{correct domain generalizing classifier}? In the setting of~\eqref{eqn:csd:synth-example-k}, where $e_c \perp E_s$, we saw that the correct domain generalizing classifier is $w_c = e_c$. In the setting where $e_c \not\perp E_s$,
the correct domain generalizing classifier is the projection of $e_c$ onto the space orthogonal to $\Span{E_s}$ i.e.,
\begin{align}
\label{eq:proj}
    w_c = e_c - P_{E_s} e_c,
\end{align}
where $P_{E_s}$ is the projection matrix onto the span of the domain-specific  vectors $e_s$. The next lemma shows that~\eqref{eq:proj} is equivalent to $w_c \perp \Span{W_s}$ and so we train classifiers~\eqref{eqn:csd:classifier-decomp} satisfying this condition.
\begin{lemma}\label{lem:char}
Suppose $W \defeq e_c \trans{\ones} + E_s \trans{\hat{\Gamma}} = w_c \trans{\ones} + W_s \trans{\Gamma}$ is a rank-$(k+1)$ matrix, where $E_s \in \R^{m \times k}, \hat{\Gamma} \in \R^{D \times k}, W_s \in \R^{m \times k}$ and $\Gamma \in \R^{D \times k}$ are all rank-$k$ matrices with $k < m, D$. Then, $w_c = e_c - P_{E_s} e_c$ if and only if $w_c \perp \Span{W_s}$.
\begin{proof}[Proof of Lemma~\ref{lem:char}]
\textbf{If direction}: 
Suppose $w_c \perp \Span{W_s}$. Then, $\trans{W} w_c = \iprod{e_c}{w_c} \cdot \ones + \hat{\Gamma} \cdot \left( \trans{E_s} w_c\right) = \norm{w_c}^2 \cdot \ones$. Since $W$ is a rank-$(k+1)$ matrix, we know that $\ones \notin \Span{\hat{\Gamma}}$ and so it has to be the case that $\iprod{e_c}{w_c} = \norm{w_c}^2$ and $\trans{E_s} w_c=0$. Both of these together imply that $w_c$ is the projection of $e_c$ onto the space orthogonal to $E_s$ i.e., $w_c = e_c - P_{E_s} e_c$.

\noindent
\textbf{Only if direction}: Let $w_c = e_c - P_{E_s} e_c$. Then $e_c \trans{\ones} - w_c \trans{\ones} + E_s \trans{\hat{\Gamma}} = P_{E_s} e_c \trans{\ones} + E_s \trans{\hat{\Gamma}}$ is a rank-$k$ matrix and can be written as $W_s \trans{\Gamma}$ with $\Span{W_s} = \Span{E_s}$. Since $w_c \perp \Span{E_s}$, we also have $w_c \perp \Span{W_s}$.
\end{proof}
\end{lemma}

\noindent
\textbf{Why low rank?}: An important choice in decomposition approaches is the rank in decomposition~\eqref{eqn:csd:classifier-decomp}, which in prior works was justified heuristically, by appealing to number of parameters.
We prove the following result, which gives us a more principled reason for the choice of low rank parameter $k$.
\begin{theorem}\label{thm:csd:main}
Given any matrix $W \in \R^{m \times D}$, the minimizers of the function $f(w_c, W_s, \Gamma) = \frob{W - w_c \trans{\ones} - W_s \trans{\Gamma}}^2$, where $W_s \in \R^{m \times k}$ and $w_c \perp \textrm{Span}\left(W_s\right)$ can be computed by a simple SVD based algorithm (Algorithm~\ref{alg:csd:svdbased} in Appendix~\ref{app:csd:proof}). In particular, we have:
\begin{itemize}
    \item For $k=0$, $w_c = \frac{1}{D} W \cdot \ones$, and
    \item for $k=D-1$, $w_c = W^+ \ones/\norm{W^+ \ones}^2$.
\end{itemize}
\end{theorem}
The proof of this theorem is similar to that of the classical low rank approximation theorem of Eckart-Young-Mirsky, and is presented in Appendix~\ref{app:csd:proof}.
When $W = w_c \trans{\ones} + W_s \Gamma + N$, where $N$ is a noise matrix (for example due to finite samples), both extremes $k=0$ and $k=D-1$ have different advantages/drawbacks:
\begin{itemize}
    \item \textbf{$k=0$}: Averaging effectively reduces the noise component $N$ but ends up retaining some domain-specific  components if there is net correlation with the label in the training data, similar to ERM.
    \item \textbf{$k=D-1$}: The pseudoinverse effectively removes domain-specific  components and retains only the common component.
    However, the pseudoinverse does not reduce noise to the same extent as a plain averaging would (since empirical mean is often asymptotically the best estimator for mean).
\end{itemize}
In general, the sweet spot for $k$ lies between $0$ and $D-1$ and its precise value depends on the dimension and magnitude of the domain-specific  components as well as the magnitude of noise.
In our implementation, we perform cross validation to choose a good value for $k$ but also note that the performance of our algorithm is relatively stable with respect to this choice in Section~\ref{sec:csd:ablation}.

\noindent
\textbf{Extension to neural networks}:
Finally, prior works extend parameter decomposition to non-linear models such as neural networks by imposing decomposition of the form of~\eqref{eqn:csd:classifier-decomp} for parameters in all layers separately. This increases the size of the model significantly and leads to worse generalization performance. Further, it is not clear whether any of the insights we gained above for linear models continue to hold when we include non-linearities and stack them together. So, we propose the following two simple modifications instead:
\begin{itemize}
    \item enforcing~\eqref{eqn:csd:classifier-decomp} only in the final linear layer, to replace the standard single softmax layer, and
    \item including a loss term for predictions of common component, in addition to the domain-specific  losses,
\end{itemize}
both of which encourage learning of \emph{features} with common-specific structure.

Our experiments (Section~\ref{sec:csd:ablation}) show that these modifications (orthogonality, changing only the final linear layer and including common loss) are instrumental in making decomposition methods state of the art for domain generalization. Our overall training algorithm follows below.

\subsection{Algorithm}
Our method of training neural networks for domain generalization appears as Algorithm~\ref{alg:csd:mos} and is called CSD for {\mosdescr} The analysis above was for the binary setting, but we present the algorithm for the multi-class case with  $\numC=$ \# classes.
The only extra parameters that CSD requires, beyond normal feature parameters $\theta$ and softmax parameters $w_c \in \R^{C \times m}$, are the domain-specific low-rank parameters $W_s \in \R^{C \times m \times k}$  and $\gamma_d \in \R^k$, for $d \in [D]$. Here $m$ is the representation size in the penultimate layer. Thus, $\trans{\Gamma}=[\gamma_1,\ldots,\gamma_D]$ can be viewed as a domain-specific embedding matrix of size $k \times D$.  Note that unlike a standard mixture of softmax, the $\gamma_d$ values are not required to be on the simplex.  
Each training instance consists of an input $x$, true label $y$, and a domain identifier $d$ from 1 to $D$.  Its domain-specific softmax parameter is computed by $w_d = w_c + W_s \gamma_d$.

Instead of first computing the full-rank parameters and then performing SVD (Thm~\ref{thm:csd:main}), we directly compute the low-rank decomposition along with training the network parameters $\theta$. 
For this we add a weighted combination of these three terms in our training objective:

\noindent
(1) Orthonormality regularizers to make $w_c[y]$ orthogonal to domain-specific $W_{s}[y]$ softmax parameters for each label $y$ and to avoid degeneracy by  controlling the norm of each softmax parameter to be close to 1.
    
\noindent
(2) A cross-entropy loss between $y$ and distribution computed from the $w_d$ parameters to train both the common and low-rank domain-specific parameters.
    
\noindent
(3) A cross-entropy loss between $y$ and distribution computed from the $w_c$ parameters. This loss might appear like normal ERM loss but when coupled with the orthogonality regularizer above it achieves domain generalization.

\begin{algorithm}[htb]
\caption{Common-Specific Low-Rank Decomposition (\mos\ ) }
\label{alg:csd:mos}
\begin{algorithmic}[1]
\State {\bf Given:} $D,m,\rank,\numC,\lambda,\kappa$,train-data
\State {Initialize params $w_c \in \R^{\numC\times m}, W_s \in \R^{\numC\times m \times k}$}
\State {Initialize  $\gamma_d \in \R^\rank: d \in [D]$}
\State {Initialize params $\theta$ of feature network $G_\theta:{\mathcal X} \mapsto \R^m$}
\State{$\hat{W} = [w_c^T, W_s^T]^T $}
\State{$\mathcal{R} \gets \sum_{y=1}^\numC \|I_{k+1}-\hat{W}[y]^T\hat{W}[y]\|_F^2$} 
     \Comment{Orthonormality constraint}
\For{$(x,y,d) \in $ train-data}
   \State{$w_d \gets w_c + W_s \gamma_d$}
    \State{loss $ \mathrel{+}=  \mathcal{L}(G_\theta(x), y; w_d) + \lambda  \mathcal{L}(G_\theta(x), y; w_c)$}
\EndFor

\State{Optimize loss$+\kappa \mathcal{R}$ wrt $\theta,w_c,W_s,\gamma_d$}
\State{\bf Return $\theta,w_c$ \Comment{for inference}}
  \end{algorithmic}
\end{algorithm}

\subsection{Why does \mos\ work?}
\begin{figure*}[htb]
  \centering
  \begin{subfigure}[b]{0.42\textwidth}
    \includegraphics[width=\textwidth]{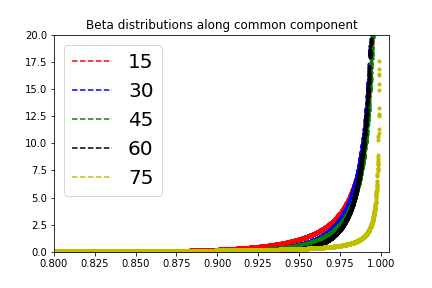}
    \caption{Beta fit on estimated probabilities of correct class using common component.}
    \label{fig:csd:rmnist:1}
 \end{subfigure}
  \qquad
  \begin{subfigure}[b]{0.42\textwidth}
    \includegraphics[width=\textwidth]{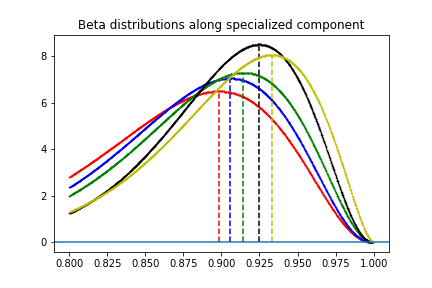}
    \caption{Beta fit on estimated probabilities of correct class using specialized component.}
    \label{fig:csd:rmnist:2}
  \end{subfigure}
  \caption{Distribution of probability assigned to the correct class using common or specialized components alone.}
  \label{fig:csd:rmnist}
\end{figure*}

We provide empirical evidence in this section that \mos\ effectively decomposes common and low-rank specialized components. Consider the Rotated-MNIST task trained on ResNet-18 as discussed in Section~\ref{sec:crossgrad:why}. Since each domain differs only in the amount of rotation, we expect $W_s$ to be of rank 1 and so we chose $k=1$ giving us one common and one specialized component. We are interested in finding out if the common component is agnostic to the domains and see how the specialized component varies across domains. 

We look at the probability assigned to the correct class for all the train instances using only common component $w_c$ and using only specialized component $W_s$. For probabilities assigned to examples in each domain using each component, we fit a Beta distribution. Shown in Figure~\ref{fig:csd:rmnist:1} is fitted beta distribution on probability assigned using $w_c$ and Figure~\ref{fig:csd:rmnist:2} for $w_s$ (the only specialized component). Note how in Figure~\ref{fig:csd:rmnist:1}, the colors are largely overlapping. However in Figure~\ref{fig:csd:rmnist:2}, notice how modes corresponding to each domain are widely spaced, moreover the order of modes and spacing between them cleanly reflects the underlying degree of rotation from 15\degree\ to 75$\degree$. 

These observations support our claims on utility of \mos\ for low-rank decomposition.

\section{Experiments}
\label{sec:crossgrad:Expt}
We compare {\crossgrad}, CSD with contemporary and strong domain generalization methods: (1)~{\bf ERM}, Standard training with the empirical risk that discards domain annotations (2)~{\bf MASF}~\citep{DouCK19} is a recently proposed meta-learning based strategy to learn domain-invariant features, and (3)~{\bf LRD}: the low-rank decomposition approach of \citet{LiYSH17} but only at the last softmax layer. 
We also compare with other strong contemporary methods: JiGen~\citep{CarlucciAS2019} and EpiFCR~\citep{LiZY2019} on some of the datasets. 

We evaluate on six different datasets spanning image and speech data types and varying number of training domains. We assess quality of domain generalization as accuracy on a set of test domains that are disjoint from the set of training domains.

\noindent
\textbf{Experiment setup details:}
We use ResNet-18 to evaluate on rotated image tasks, LeNet for Handwritten Character datasets, and a multi-layer convolution network for speech tasks. 
We added a layer normalization just before the final layer in all these networks since it helped generalization error on all methods, including the baseline. \mos\ is relatively stable to hyper-parameter choice, we set the default rank (k) to 1, and parameters of weighted loss to $\lambda=1$ and $\kappa=1$. These hyper-parameters along with learning rates of all other methods as well as number of meta-train/meta-test domains for MASF and step size of perturbation in \crossgrad{} are all picked using a task-specific development set.  

\subsection{Overall Comparison}
\noindent
\textbf{Handwritten character datasets:}
In these datasets we have characters  written by many different people, where the person writing qualifies a domain and generalizing to new writers is a natural requirement. Handwriting datasets are challenging since it is difficult to disentangle a person's writing style from the character (label), thereby domain-invariant and some of the meta-learning methods that attempt to erase domains are unlikely to succeed.
We have two such datasets.

\noindent
(1) The LipitK dataset\footnote{\url{http://lipitk.sourceforge.net/datasets/dvngchardata.htm}} is a Devanagari Character dataset which has classification over 111 characters (label) collected from 106 people (domain). We train three different models on each of 25, 50, and 76 domains, and test, validate on a disjoint set of 20, 10 domains.

\noindent
(2) Nepali Hand Written Character Dataset (NepaliC)\footnote{\url{https://www.kaggle.com/ashokpant/devanagari-character-dataset}} contains data collected from 41 different people on consonants as the character set which has 36 classes. Since the number of available domains is small, in this case we create a fixed split of 27 domains for training, 5 for validation and remaining 9 for testing. 

\noindent
We use LeNet as the base classifier on both the datasets. 

\begin{table*}[htb]
    \centering
    \begin{tabular}{|l|lll|l|}
    \hline
     & \multicolumn{3}{|c|}{LipitK} & \multicolumn{1}{c|}{NepaliC} \\
    \hline
Method & 25 & 50 & 76 & 27 \\
\hline
ERM~(Baseline) & 74.5 (0.4) & 83.2 (0.8) & 85.5 (0.7) & 83.4 (0.4) \\
LRD~\citep{LiYSH17} & 76.2 (0.7) & 83.2 (0.4) & 84.4 (0.2) & 82.5 (0.5) \\
MASF~\citep{DouCK19} & {\bf 78.5} (0.5) & 84.3 (0.3) & 85.9 (0.3) & 83.3 (1.6) \\
\crossgrad{} & 75.3 (0.5) & 83.8 (0.3) & 85.5 (0.3) & 82.6 (0.5) \\
\mos & 77.6 (0.4) & {\bf 85.1} (0.6) & {\bf 87.3} (0.4) & {\bf 84.1} (0.5) \\

    \hline
    
    \end{tabular} 
    \caption{Comparison of our method on two handwritting datasets: LipitK and NepaliC.  For LipitK, since the number of available training domains is large we also report results with increasing number of domains. The numbers are average (and standard deviation) from three runs. }
    \label{tab:csd:image}
\end{table*}

In Table~\ref{tab:csd:image} we show the accuracy using different methods for different number of training domains on the LipitK dataset, and on the NepaliC dataset. \crossgrad{} improves over ERM somewhat consistently but its gains diminish quickly with increasing number of train domains. We observe that across all four models \mos\ provides significant gains in accuracy over the baseline (ERM), and all three existing methods LRD, \crossgrad{} and MASF.  The gap between prior decomposition-based approach (LRD) and ours, establishes the importance of our orthogonality regularizer and common loss term.   MASF is better than \mos\ only for 25 domains and as the number of domains increases to 76, \mos's accuracy is 87.3 whereas MASF's is 85.9.  

\noindent
\textbf{PACS}:
PACS\footnote{\url{https://domaingeneralization.github.io/}} is a popular domain generalization benchmark. The dataset in aggregate contains around 10,000 images from seven object categories collected from four different sources: Photo, Art, Cartoon and Sketch. Evaluation using this dataset trains on three of the four sources and tests on the left out domain. This setting is challenging since it tests generalization to a radically different target domain. 

We present comparisons in Table~\ref{tab:csd:expt:pacs} of our method with JiGen~\citep{CarlucciAS2019}, EpiFCR~\citep{LiZY2019} and ERM baseline. In order for a fair comparison, we use the implementation of JiGen and made comparisons with only those methods that have provided numbers using a consistent implementation. We report numbers for the case when the baseline network is ResNet-18 and AlexNet. We perform almost the same or slightly better than image specific JiGen. We stay away from comparing with MASF~\citep{DouCK19}, EpiFCR on AlexNet, since their reported numbers from different implementation have different baseline number.

\begin{table}[htb]
    \centering
    \begin{tabular}{|l|r|r|r|r|r|}
    \hline
    Alg. & Photo & Art & Cartoon & Sketch & Avg. \\
    \hline
    & \multicolumn{5}{|c|}{ResNet-18}\\\hline
         ERM & 95.7 & 77.8 & 74.9 & 67.7 & 79.0 \\
         JiGen & 96.0 & 79.4 & 75.2 & 71.3 & 80.5 \\
         EpiFCR & 93.9 & 82.1 & 77.0 & 73.0 & 81.5 \\
         CSD & 94.1 (0.2) & 78.9 (1.1) & 75.8 (1.0) & 76.7 (1.2) & 81.4 (0.3) \\
         \hline
        & \multicolumn{5}{|c|}{AlexNet}\\\hline
        ERM & 90.0 & 66.7 & 69.4 & 60.0 & 71.5 \\
        JiGen & 89.0 & 67.6 & 71.7 & 65.2 & 73.4 \\
        CSD & 90.2 (0.2) & 68.3 (1.2) & 69.7 (0.3) & 63.4 (1.8) & 72.9 (0.6) \\
         \hline
    \end{tabular}\hfill
    \caption{Comparison of CSD with  JiGen~\citep{CarlucciAS2019} and EpiFCR~\citep{LiZY2019} on PACS datset with ResNet-18 and AlexNet architecture. The header of each column identifies the target domain with last column showing average accuracy. Also shown are the standard deviation for CSD in parenthesis.}
    \label{tab:csd:expt:pacs}
\end{table}

\begin{table}[htb]
    \centering
    \begin{tabular}{|l|r|r|r|r|}
    \hline
    Method & 50 & 100 & 200 & 1000 \\
    \hline
    ERM & 72.6 (.1) & 80.0 (.1) & 86.8 (.3) & 90.8 (.2) \\
    \crossgrad{} & 73.3 (.1) & 80.4 (.0) & 86.9 (.4) & 91.2 (.2) \\
    \mos & {\bf 73.7} (.1) & {\bf 81.4} (.4) & {\bf 87.5} (.1) & {\bf 91.3} (.2) \\
    \hline
    \end{tabular}
    \caption{Accuracy comparison on speech utterance data with varying number of training domains. The numbers are average (and standard deviation) from three runs.}
    \label{tab:csd:speech}
\end{table}

\noindent
\textbf{Speech utterances dataset:}
We use the utterance data released by Google that was collected from thousands of speakers\footnote{Can be found at this \href{https://ai.googleblog.com/2017/08/launching-speech-commands-dataset.html}{link}}. The base classifier and the preprocessing pipeline for the utterances are borrowed from the implementation provided in the Tensorflow examples\footnote{\href{https://github.com/tensorflow/tensorflow/tree/r1.15/tensorflow/examples/speech_commands}{link}}. We used the default ten (of the 30 total) classes for classification. We use ten percent of the total number of domains for each of validation and test splits. 

The accuracy comparison for each of the methods on varying number of training domains is shown in Table~\ref{tab:csd:speech}. We could not compare with MASF since their implementation is only made available for image tasks.  Also, we skip comparison with LRD since earlier experiments established that it can be worse than even the baseline. From the results shown in Table~\ref{tab:csd:speech}, both \crossgrad{} and \mos{} improve over the ERM baseline consistently with \mos{} better than \crossgrad{}.  

We can also observe from the table that the domain generalization algorithms are most effective when the number of train domains is small. When the number of domains is very large (for example, 1000 in the table), even standard training can suffice since the training domains could {\it cover} unseen test domains.

\begin{table}[htb]
\centering
    \begin{tabular}{|l|l|l|l|l|}
    \hline
     & \multicolumn{2}{|c|}{MNIST} & \multicolumn{2}{|c|}{Fashion-MNIST} \\
     & in-domain & out-domain & in-domain & out-domain \\
    \hline
    ERM & 98.3 (0.0) & 93.6 (0.7) & 89.5 (0.1) & 75.8 (0.7)\\
    MASF & 98.2 (0.1) & 93.2 (0.2) &  86.9 (0.3) & 72.4 (2.9)\\
    \mos & \textbf{98.4} (0.0) & \textbf{94.7} (0.2) & \textbf{89.7} (0.2) & \textbf{78.0} (1.5)\\
    \hline
    \end{tabular}
    \caption{Performance comparison on rotated MNIST and rotated Fashion-MNIST, shown are the in-domain and out-domain accuracies averaged over three runs along with standard deviation in the brackets.}
    \label{tab:csd:toy:results}
\end{table}
\begin{table}[htb]
\centering
    \begin{tabular}{|l|r|r|}
    \hline
        & \multicolumn{2}{|c|}{MNIST} \\
         & in-domain & out-domain \\\hline
         ERM & 97.7 (0.) & 89.0 (.8) \\ 
         MASF & 97.8 (0.) & 89.5 (.6) \\
         \mos\ & \textbf{97.8} (0.) & \textbf{90.8} (.3) \\\hline
    \end{tabular}
    \caption{In-domain and out-domain accuracies on rotated MNIST without batch augmentations. Shown are average and standard deviation from three runs.}
    \label{tab:csd:expt:rmnist:woaug}
\end{table}

\noindent
{\bf Training time:} MASF is 5--10 times slower than \mos, and \crossgrad{} is 3--4 times slower than \mos. In contrast \mos\ is just 1.1 times slower than ERM.  Thus, the increased generalization of \mos\ incurs little additional overheads in terms of training time compared to existing methods. 

\paragraph{Rotated MNIST and Rotated Fashion-MNIST:}
\label{sec:csd:expt:rotation}
Rotated image datasets are a popular tool for evaluating domain generalization where the angle by which images are rotated is the proxy for domain.
We randomly select\footnote{ The earlier work on this dataset however lacks standardization of splits, train sizes, and baseline network across the various papers~\citep{VihariSSS18}~\citep{WangZZ2019}. Hence we rerun experiments using different methods on our split and baseline network.
} 
a subset of 2000 images for Rotated MNIST and 10,000 images for Rotated Fashion-MNIST, the original set of images is considered to have rotated by 0{\degree} and is denoted as $\mathcal{M}_0$. Each of the images in the data split when  rotated by $\theta$ degrees is denoted $\mathcal{M}_\theta$. The training data is union of all images rotated by 15{\degree} through 75{\degree} in intervals of 15{\degree}, creating a total of 5 domains.  We evaluate on $\mathcal{M}_0, \mathcal{M}_{90}$.  In that sense only in this artificially created domains, are we truly sure of the test domains being outside the span of train domains.
Further, we employ batch augmentations such as flip left-right and random crop since they significantly improve generalization error and are commonly used in practice. We train using the ResNet-18 architecture.  

Table~\ref{tab:csd:toy:results} compares the baseline, MASF, and \mos\ on rotated MNIST and rotated Fashion-MNIST. We show accuracy on test set from the same domains as training (in-domain) and test set from 0{\degree} and 90{\degree} that are outside the training domains. 
Note how the \mos's improvement on in-domain accuracy is insignificant, while gaining substantially on out of domain data.  This shows that CSD specifically targets domain generalization.  Surprisingly MASF does not perform well at all, and is significantly worse than even the baseline.  One possibility could be that the domain-invariance loss introduced by MASF conflicts with the standard data augmentations used on this dataset. To confirm, we repeated the experiment again without using batch augmentations in Table~\ref{tab:csd:expt:rmnist:woaug}. We found that although MASF is now better than ERM, the trend of superior performance with CSD continued. 




\paragraph{Hyper-parameter optimization and Validation set:}
The hyper-parameters in our implementation includes algorithm specific hyper-parameter such as rank of decomposition (k) for \mos{} and optimization specific parameters such as number of training epochs, learning rate, optimizer. We use a validation set to pick the optimal values for these parameters. Construction of the validation set varies slightly between each task. 

In handwritten character and speech tasks, we construct validation set from test data of unseen users. The set of users considered for validation are mutually exclusive with users (or domains) of train and test splits. 

In PACS and rotated experiments, the validation set is the validation split of train data thereby the validation set only reflects the performance on the train domains.

\subsection{Ablation Study}
In this section, we evaluate the stability and the effect of model design choices for our algorithms. 
\subsubsection{\crossgrad}
In order to make sure that \crossgrad{} improvements carry to other model architectures, we compare different methods with 2-block ResNet~\citep{he16deepresidual} and LeNet \citep{lenet97} for the Fonts and LipitK dataset in  Table~\ref{tab:crossgrad:arch}. 
Fonts too is a character recognition dataset where different domains or users are simulated using different system fonts. 

For both the datasets, the ResNet model is significantly better than the LeNet model.  But even for the higher capacity ResNet model, \crossgrad\ surpasses the baseline accuracy as well as other methods like \goodfellow{}~\citep{goodfellow14}, DAN~\citep{Ganin16}.
Note that the numbers shown for LipitK dataset in Table~\ref{tab:crossgrad:arch} differ from what is shown in Table~\ref{tab:csd:image} due to implementation differences. Nevertheless Table~\ref{tab:crossgrad:arch} serves to understand the effect of model architecture on \crossgrad{}'s performance. 

\begin{table}[htb]
\begin{center}
\begin{tabular}{|l|r|r|r|r|} \hline
 & \multicolumn{2}{|c|}{Fonts} & \multicolumn{2}{|c|}{LipitK}  \\ \hline
  Method Name          & LeNet & ResNet & LeNet & ResNet \\ \hline
Baseline  &  68.5 &    80.2       & 82.5& 91.5 \\
\dan &  68.9 &      81.1         & 83.8 & 88.5 \\
\goodfellow & 71.4 &   80.5      & 86.3 & 91.8 \\
\crossgrad  & \textbf{72.6} &\textbf{82.4} & \textbf{88.6}  &   \textbf{92.1}\\ \hline
\end{tabular}
\end{center}
\caption{\label{tab:crossgrad:arch} Accuracy with varying model architectures.}
\end{table}

\subsubsection{CSD}\label{sec:csd:ablation}

In this section we study the importance of each of the three terms in CSD's final loss: common loss computed from $w_c$ ($\mathcal{L}_c$), specialized loss  ($\mathcal{L}_s$) computed from $w_i$ that sums common ($w_c$) and domain-specific parameters ($W_s, \Gamma$), orthonormal loss ($\mathcal{R}$) that makes $w_c$ orthogonal to domain-specific softmax (Refer: Algorithm~1). In Table~\ref{tab:csd:ablation}, we demonstrate the contribution of each term to \mos\ loss by comparing accuracy on LipitK with 76 domains.

\begin{table}[htb]
    \centering
    \begin{tabular}{|c|c|c|r|}
    \hline
    Common  & Specialized  & Orthonormality & Accuracy \\
    loss $\mathcal{L}_c$ & loss $\mathcal{L}_s$ & regularizer $\mathcal{R}$ &  \\
    \hline
    Y & N & N & 85.5 (.7) \\
    N & Y & N & 84.4 (.2) \\
    N & Y & Y & 85.3 (.1) \\
    Y & N & Y & 85.7 (.4) \\
    Y & Y & N & 85.8 (.6) \\
    Y & Y & Y & 87.3 (.3) \\
    \hline
    \end{tabular}
    \caption{Ablation analysis on \mos\ loss using LipitK (76)}
    \label{tab:csd:ablation}
\end{table}

The first row is the baseline with only the common loss.  The second row shows prior
decomposition methods that imposed only the specialized loss without any orthogonality or a separate common loss. This is worse than even the baseline (first row).  
This can be attributed to decomposition without identifiability guarantees thereby losing part of $w_c$ when the specialized $W_s$ is discarded. Using orthogonal constraint, third row, fixes this ill-posed decomposition however, understandably, just fixing the last layer does not gain big over baseline. Using both common and specialized loss even without orthogonal constraint showed some merit, perhaps because feature sharing from common loss covered up for bad decomposition. Finally, fixing this bad decomposition with orthogonality constraint and using both common and specialized loss constitutes our \mos\ algorithm and is significantly better than any other variant.

This empirical study goes on to show that both $\mathcal{L}_c$ and $\mathcal{R}$ are important. Imposing $\mathcal{L}_c$ with $w_c$ does not help feature sharing if it is not devoid of specialized components from bad decomposition. A good decomposition on final layer without $\mathcal{L}_c$ does not help generalize much. 

\subsection*{Importance of Low-Rank}
Table~\ref{tab:csd:speech:k} shows accuracy on various speech tasks with increasing k controlling the rank of the domain-specific component.  Rank-0 corresponds to the baseline ERM without any domain-specific part.  We observe that accuracy drops with increasing rank beyond 1 and the best is with $k=1$ when number of domains $D \le 100$.  As we increase D to 200 domains, a higher rank (4) becomes optimal and the results stay stable for a large range of rank values.  This matches our analytical understanding resulting from Theorem~\ref{thm:main} that we will be able to successfully disentangle only those domain-specific components which have been observed in the training domains, and using a higher rank will increase noise in the estimation of $w_c$.


\begin{table}[htb]
    \centering
    \begin{tabular}{|l|r|r|r|}
    \hline
    Rank $k$ & 50 & 100 & 200 \\
    \hline
    0  &   72.6 (.1) & 80.0 (.1) & 86.8 (.3)  \\           
    1 & \textbf{74.1} (.3) &  \textbf{81.4} (.4) & 87.3 (.5) \\
    4 & 73.7 (.1) & 80.6 (.7) &  \textbf{ 87.5} (.1) \\
    9 & 73.0 (.6) & 80.1 (.5) &  \textbf{  87.5} (.2) \\
    24 & 72.3 (.2) & 80.5 (.4) & 87.4 (.3) \\
    \hline
    \end{tabular}
    \caption{Effect of rank constraint (k) on test accuracy for Speech task with varying number of train domains.}
    \label{tab:csd:speech:k}
\end{table}

\section{Related Work}
The work on Domain Generalization is broadly characterized by four major themes:

\paragraph{Domain Erasure}
Many early approaches attempted to repair the feature representations so as to reduce divergence between representations of different training domains. \citet{MuandetBS13} learns a kernel-based domain-invariant representation. \citet{GhifaryBZB15} estimates shared features by jointly learning multiple data-reconstruction tasks.   \citet{Li2018DomainGW} uses MMD to maximize the match in the feature distribution of two different domains. The  idea of domain erasure is further specialized in \citet{WangZZ2019} by trying to project superficial (say textural) features out using image specific kernels. 
Domain erasure is also the founding idea behind many domain adaptation approaches, example~\citep{Ganin16,Ben-David:2006:ARD:2976456.2976474,HoffmanMN18} to name a few.

\paragraph{Augmentation} The idea behind these approaches is to train the classifier with instances obtained by domains hallucinated from the training domains, and thus make the network `ready' for these neighboring domains. \citet{VolpiNSDM2018} extends \crossgrad{} based domain augmentations with only a single domain data.  Another type of augmentation is to simultaneously solve for an auxiliary task.  For example, \citet{CarlucciAS2019} (JiGen) achieves domain generalization for images by solving an auxiliary unsupervised jig-saw puzzle on the side.  

\paragraph{Meta-Learning/Meta-Training}
A recent popular approach is to pose the problem as a meta-learning task, whereby we update parameters using meta-train loss but simultaneously minimizing meta-test loss~\citep{liYY2018},~\citep{BalajiSR2018} or learn discriminative features that will allow for semantic coherence across meta-train and meta-test domains~\citep{DouCK19}.  More recently, this problem is being pursued in the spirit of estimating an invariant optimizer across different domains~ and solved by a form of meta-learning in \citet{ArjovskyLID19}. Meta-learning approaches are complicated to implement, and slow to train.  

\paragraph{Decomposition} In these approaches the parameters of the network are expressed as the sum of a common parameter and  domain-specific parameters during training.  \citet{Daume2007} first applied this idea for domain adaptation.  \citet{ECCV12_Khosla} applied decomposition to DG by retaining only the common parameter for inference. \citet{LiYSH17} extended this work to CNNs where each layer of the network was decomposed into common and specific low-rank components.  
Our work provides a principled understanding of when and why these methods might work and uses this understanding to design an improved algorithm \mos.  Three key differences are:  \mos\ decomposes only the last layer, imposes loss on both the common and domain-specific parameters, and constrains the two parts to be orthogonal.  We show that orthogonality is required for theoretically proving identifiability. 
As a result, this newer avatar of an old decomposition-based approach surpasses recent, more involved augmentation and meta-learning approaches. 

\paragraph{Other} Distributionally robust optimization ~\citep{Sagawa20} techniques deliver robustness to any mixture of the training distributions. This problem can be seen as a specific version of Domain Generalization, whose objective is to provide robustness to any distribution shift including the shift in population of the train domains. \\
There has been an increased interest in learning invariant predictors through a causal viewpoint ~\citep{ArjovskyLID19,AhujaSV20} for better out-of-domain generalization. 

\section{Discussion}
We considered a natural multi-domain setting and looked at how standard classifier could overfit on domain signals and delved on efficacy of several other existing solutions to the domain generalization problem. 
We proposed \crossgrad{}, which provides a new data augmentation scheme based on the label (respectively, domain) predictor using the gradient of the domain (respectively, label) predictor over the input space, to generate perturbations.  \crossgrad\ is most useful when number of training domains is small and do not directly cover test domains well. 

Domain generalization through data augmentation of \crossgrad{}, however, can be limiting because new examples are sampled from a small neighbourhood of original examples. 
Instead of regularizing domain overfitting component implicitly through data augmentation, we developed a new algorithm called CSD that explicitly recovers the domain generalizing classifier. 
CSD decomposes classifier parameters into a common part and a low-rank domain-specific part.  
We presented a principled analysis to provide identifiability results of CSD and analytically studied the effect of rank in trading off domain-specific noise suppression and domain generalization, which in earlier work was largely heuristics-driven. Apart from the proposed algorithms, our contribution also lies in understanding of lack of domain generalization through simple synthetic settings.  
 
While both our algorithms improved domain generalization over standard methods, the problem is far from solved given the large in-domain and out-of-domain performance disparity. We will discuss subsequent research and potential future directions in the next section. 


\subsection*{Subsequent work}
\vspace{-20pt}

\noindent
{\bf Troubling interpretations of domain generalization.}
A significant fraction of the research community interprets the domain generalization problem as to train models that generalize to any domain without qualifying the relation between train and test domains.
For instance, popular datasets such as PACS, VLCS, DomainNet, OfficeHome contain collection of examples from unrelated and distant domains; PACS, DomainNet contain examples from different renditions of an object: sketch, clipart, photo, cartoon etc.; VLCS, OfficeHome contain examples from four different datasets. Since algorithms are evaluated using leave-one-domain-out splits, the nature of domain shifts among the train domains is unrelated to shifts between the train and the test domains. It is unclear though if we can generalize to domain shifts not seen during training. For example, a model trained on data with rotation shifts need not generalize to blur or noise shifts. Unjustified objectives of popular benchmarks may have contributed to perceived stagnation of progress on this problem~\citep{Gulrajani20}.



\noindent
{\bf New benchmarks.}
Several new datasets that better encapsulate the problem goals in the real-world are being released.
Many new datasets evaluate robustness to shifts of a certain kind: common image corruptions~\citep{Hendrycks19}, image renditions~\citep{Hendrycks21}, abstract forms of objects~\citep{Rusak21}, independent and (almost) identically sampled ImageNet test split~\citep{Recht19}. \citet{WILDS20} released WILDS that contain multiple datasets from the real-world, such as dataset for prediction of disease using stained micrographs where color, shape, scale of the micrograph can vary between hospitals. Further work in developing more datasets that are representative of the domain shifts in the wild is much needed. 

\noindent
{\bf Test-time training.} A complementary new paradigm of zero-shot generalization has emerged called test-time training~\citep{tent,arm,ttt++}. These methods, which are shown to be somewhat effective for certain image tasks, adapt parameters to minimize a self-supervised loss on a single example or a batch of examples drawn from an unknown distribution. Albeit, more research is needed on how or why such self-supervised losses can improve robustness. 

\noindent
{\bf Future work.}
Most work in domain generalization focused on loss engineering and data augmentation, but can we solve the problem by continuing to research in this direction?
Can end-to-end training on multi-domain data learn models robust to domain shifts? We will evaluate this question with a historical parallel of training models that are robust to translation shifts in images. Existing approaches to domain generalization would simply characterize the ideal solution in the space of all possible feed-forward networks. Nevertheless, there could be multiple dense networks that explain training data and satisfy any additional solution constraints. Although, the solution with convolutional neural network (CNN) like symmetries is realizable with dense networks of sufficient width, the amount of data or model constraints required to recover it could be impractically large. The search for CNN like solution from multiple zero training error solutions~\citep{zhang21} is like trying to find a needle in a haystack.
In the same spirit, domain generalization problem also attempts to find the ideal hypothesis that is robust to domain shifts from multi-domain data while many approaches only qualify desired attributes of the ideal solution: feature invariance~\citep{Ganin16}, classifier invariance~\citep{ArjovskyIRM19}, decomposition~\citep{LiYSH17,VihariNS2020}. This argument leads us to the sober takeaway that simply engineering loss objectives on multi-domain training data in itself may not realize human-level robustness. We also need innovations on alternate forms of human supervision, model architectures, optimization algorithms, representation learning etc. to train models that are robust to domain shifts.


\chapter{Subpopulation Shift}
\label{chap:cgd}
In the preceding chapter, we studied training algorithms for generalization to any domain drawn from a latent distribution of domains. In this chapter, we consider generalization to any subpopulation of the training distribution. For instance, if the training data contains examples drawn from multiple domains, we wish to generalize to any mixture of the train domains. 
For example, a training dataset could contain portraits collected from different demographics (domains) in proportion (0.9, 0.1), and we wish to generalize to test datasets obtained from any other proportions, say (0.2, 0.8). 

Standard methods based on empirical risk minimization (ERM) could sacrifice generalization on minority training domains in order to achieve high overall (average) generalization performance. 
One of the key reasons for this behavior of ERM is the existence of spurious correlations between labels and some features on majority domains that are either nonexistent, or worse oppositely correlated, on the minority domains~\citep{Sagawa19}. In such cases, ERM exploits these spurious correlations to achieve high accuracy on majority domains, thereby suffering from poor performance on minority domains. Consequently, this is a form of unfairness, where accuracy on minority subpopulation or domain is being sacrificed to achieve high accuracy on majority domains. 
Inspired by ideas from the closely related domain generalization problem, we present a simple training algorithm for subpopulation shift robustness, which explicitly encourages learning of features that are shared across various domains.
Work in this chapter is based on~\citet{Piratla22}.

The subpopulation shift problem, was popularized by~\citet{Sagawa19}, and is well-studied in the literature~\citep{Sagawa19, SagawaExacerbate20, MenonOverparam21, DORO21, Goel20, Bao21}. Among all algorithms proposed for this problem, Group Distributionally Robust Optimization (\dro) is a popular method~\citep{Sagawa19}, which at every update step focuses on the domain with the highest regularized loss. 
While \dro~has been shown to obtain reduction in worst-domain error over ERM on \emph{some} benchmark datasets, it has obvious but practically relevant failure modes such as when different domains have different amounts of label noise. In such a case, \dro~ends up focusing on the domain(s) with the highest amount of label noise, thereby obtaining worse performance on all domains. In fact, on \emph{several other} datasets with subpopulation shift, \dro~performs poorly compared to ERM~\citep{KohWilds20}. 

The key issue with \dro's strategy of just focusing on the domain with the highest training loss is that, it is uninformed of the inter-domain interactions. Consequently, the parameter update using the domain with the highest training loss may increase the loss on other domains. 
In fact, this observation also reveals that {\dro} is not properly addressing the issue of spurious correlations, which have been identified by prior work as a key reason behind the poor performance of ERM~\citep{Sagawa19}.
Inspired by ideas from the closely related problem of \emph{domain generalization}~\citep{CORAL,PiratlaCG,ArjovskyIRM19}, we hypothesize that modeling inter-domain interactions is essential for properly addressing spurious correlations and subpopulation shift problem. More concretely, in each update step, we propose the following simple strategy to decide which domain to train on:

\noindent
{\it Train on that domain whose gradient leads to largest decrease in average training loss over all domains.}

In other words, our strategy focuses on the domain that contributes to the common good.
We call the resulting algorithm \emph{Common Gradient Descent} (\cg). 
We show that \cg~is a sound optimization algorithm as it monotonically decreases the macro/domain-average loss, and consequently finds first order stationary points. \cg's emphasis on common gradients that lead to improvements across all domains, makes it robust to the presence of spurious features, enabling it to achieve better generalization across all domains compared to usual gradient descent.
We present insights on why \cg\ may perform better than \dro\ by studying simple synthetic settings. Subsequently, 
through empirical evaluation on seven real-world datasets---which include two text and five image tasks with a mix of subpopulation and domain shifts---we demonstrate that {\cg} either matches or obtains improved performance over several existing algorithms for robustness that include: ERM, {\pgi}~\citep{Ahmed21}, IRM~\citep{ArjovskyIRM19}, and {\dro}. 
\paragraph{Problem statement}
\label{sec:cgd:problem}
Let $\mathcal{X}$ and $\mathcal{Y}$ denote the input and label spaces respectively.
We assume that the training data comprises of $k$ domains from a set $\mathcal{G}$ where each  $i \in \mathcal{G}$ include $n_i$ instances from a probability distribution $P_i(\mathcal{X}, \mathcal{Y})$. In addition to the label $y_j \in \mathcal{Y}$, each training example $x_j \in \mathcal{X}$ is also annotated by the domain/subpopulation $i \in \mathcal{G}$ from which it comes. 
The domain annotations are available only during training and not during testing time, hence the learned model is required to be domain-agnostic, i.e., it may not use the domain label at test time.
The number of examples $n_i$ could vary in each domain in the train data.
We refer to domains $\{i\}$ with (relatively) large $n_i$ as majority domains and those with small $n_i$ as minority domains. 
We use the terms domain and subpopulation interchangeably. 
%
%
Our goal is to learn a model that performs well on all the domains in $\mathcal{G}$.

Following prior work~\citep{Sagawa19}, we use two metrics to quantify performance: \emph{worst domain accuracy} denoting the minimum test accuracy across all domains $i \in \mathcal{G}$ and \emph{average accuracy} denoting micro averaged test accuracy across \emph{all examples belonging to all domains}. 

We denote with $\ell_i$, the  average loss over examples of domain $i$ using a predictor $f_\theta$ (with parameters $\theta$), $\ell_i(\theta)=\mathbb{E}_{(x, y)\sim P_i(\mathcal{X}, \mathcal{Y})}\mathcal{L}(x, y; f_\theta)$, for an appropriate classification loss $\mathcal{L}$. We refer to the parameters at step `t' of training by $\theta^t$.

\section{CGD: Common Gradient Descent}
\label{sec:cgd:method}
The {\erm} training algorithm learns parameters over the joint risk, which is the population weighted sum of the losses: $\sum_i n_i\ell_i/\sum_i n_i$. However, when there are spurious correlations of certain features in the majority domains, {\erm}'s loss may learn these spurious features. This leads to poor performance on minority domains if these features are either nonexistent or exhibit opposite behavior on the minority domains. A simple alternative, called {\ermuw}, is to reshape the risk function so that the minority domains are up-weighted, for some predefined domain-specific weights $\alpha$.
{\ermuw}, however, could overfit to the up-weighted minority domains. The {\dro} algorithm of \citet{Sagawa19} at any update step trains on the domain with the highest loss, as shown below. 
\begin{align*}
\mbox{\dro~update step:} \quad j^* &= \argmax_j \ell_j(\theta) \quad \mbox{ and } \quad 
\theta^{t+1} = \theta^{t} - \eta\nabla_\theta\ell_{j^*}.
\end{align*}

Training on only high loss domains avoids (quickly) overfitting on the minority domain since we avoid zero training loss on any domain while the average training loss is non-zero. However, this approach of focusing on the worst domain is subject to failure when domains have heterogeneous levels of noise and transfer. In Section~\ref{sec:cgd:qua}, we will illustrate these failure modes. 

Instead, we consider an alternate strategy for picking the training domain: the domain when trained on, most minimizes the overall loss across {\emph{all}} domains. Let $g_i=\nabla_\theta \ell_i(\theta^t)$ (although $g_i$ also depends on `t', we drop it to avoid clutter) represent the gradient of the domain $i$, we pick the desired domain as:
\begin{align}
    j^*=&\argmin_j \sum_i \ell_i(\theta^t - \eta\nabla_\theta\ell_j(\theta^t))
    &\approx \argmax_j \sum_i g_i^Tg_j \quad \text{[First-order Taylor]} \label{eqn:cgd:firstorder}
\end{align}
The choice of the training domain based on gradient inner-product can be noisy. To counteract, we incorporate three strategies to reduce variance of the above selection:
\begin{itemize}[noitemsep]
    \item Replace a hard selection with regularized soft selection. 
    \item Replace inner-product with a scaled cosine similarity.
    \item Choice adjustment to bias selection choice toward prespecified domains.
\end{itemize}
We describe each of them below. 

\noindent
{\bf Regularized soft selection.} We smooth the choice of domain with a weight vector $\alpha^t\in \Delta^{k-1}$, at step t, where $\alpha^t$ is regularized between consecutive time steps. The amount of regularization is controlled by a hyperparameter: $\eta_\alpha$.
The weight vector at step $t+1$ $\alpha^{t+1}$, therefore, maximizes gradient inner product while being similar to previous time step and takes the following form. 

\begin{align}
  \alpha^{t+1} &= \argmax_{\alpha \in \Delta^{k-1}} \sum_i \alpha_i \langle g_i, \sum_j g_j \rangle - \frac{1}{\eta_\alpha} KL(\alpha, \alpha^t)\label{eq:cgd:alpha_opt}\\ 
  \theta^{t+1} &= \theta^{t} - \eta\sum_j \alpha_j^{t+1}g_j(\theta_t). \nonumber
\end{align}

The update of $\alpha$ in \eqref{eq:cgd:alpha_opt} can be solved in closed form using KKT first order optimality conditions and rearranged to get:
\begin{align}
    \alpha^{t+1}_i = \frac{\alpha^t_i \cdot \exp\left(\eta_\alpha \langle g_i, \sum_j g_j \rangle \right)}{\sum_s \alpha^t_s \cdot \exp\left(\eta_\alpha \langle g_s, \sum_j g_j \rangle \right)}.
    \label{eqn:cgd:alpha_update}
\end{align}

\subsection*{Algorithm}
Pseudocode for the overall algorithm is shown in Algorithm~\ref{alg:cgd:cg}.
\begin{algorithm}
\caption{{\cg} Algorithm}\label{alg:cgd:cg}
\begin{algorithmic}[1]
\State{{\bf Input:} Number of domains: $k$, Training data: $\left\{(x_j, y_j, i): i \in [k], j \in [n_i]\right\}$, Step sizes: $\eta_\alpha, \eta$}
\State{Initialize $\theta^0$, $\alpha^0 = \left(\frac{1}{k}, \cdots, \frac{1}{k}\right)$}
\For{$t=1,2,\cdots,$}
\For{$i \in \{1,\cdots,k\}$}
\State{$ \alpha_i^{t+1}\leftarrow \alpha_i^t \exp(\eta_\alpha \nabla \ell_i(\theta^t)^\top \sum_{s \in [k]} \nabla \ell_s(\theta^t))$} \label{alg:cgd:cg:loss} 
\EndFor
\State{$\alpha_i^{t+1} \leftarrow \alpha_i^{t+1}/\|\alpha^{t+1}\|_1 \quad\forall i\in [1\ldots k]$} \Comment{Normalize}
\State{$\theta^{t+1} \leftarrow \theta^t - \eta\sum_{i\in \{1,\cdots,k\}} \alpha_i^{t+1} \nabla \ell_i(\theta^t)$} \Comment{Update parameters} 
\EndFor
\end{algorithmic}
\end{algorithm}

\paragraph{Scaled cosine similarity.}
The scale of the gradients can vary widely depending on the task, architecture and the loss landscape. As a result, the $\alpha$ update through ~\eqref{eqn:cgd:alpha_update} can be unstable and tuning of the $\eta_\alpha$ hyperparameter tedious. Since we are only interested in capturing if a domain transfers positively or negatively to others, we retain the cosine-similarity of the gradient dot-product but control their scale through $\ell(\theta)^p$ for some $p>0$.
That is, we set the gradient: $\nabla \ell_i(\theta^t)$ to $\frac{\nabla \ell_i(\theta^t)}{\norm{\nabla \ell_i(\theta^t)}}\ell_i(\theta)^p$. 
In our implementation, we use $p=1/2$. In Appendix~\ref{appendix:cgd:grad_approx}, we discuss in more detail the range of gradient norm, why we scale the gradient by loss, and how we pick the value of p.
Finally, the $\alpha^{t+1}$ update of~\eqref{eqn:cgd:alpha_update} is replaced with~\eqref{eqn:cgd:final_alpha_update}. When we assume the domains do not interact, i.e. $\langle g_i, g_j \rangle=0\quad\forall i\neq j$, then our update (\eqref{eqn:cgd:final_alpha_update}) matches {\dro}.
\begin{align}
  \alpha_i^{t+1} = \frac{\alpha_i^t\exp(\eta_\alpha\sum_j\sqrt{\ell_i\ell_j}\cos(g_i, g_j))}{\sum_s\alpha_s^t\exp(\eta_\alpha\sum_j\sqrt{\ell_s\ell_j}\cos(g_s, g_j))}.
  \label{eqn:cgd:final_alpha_update}
\end{align}
%
%

\noindent
\paragraph{Choice Adjustment.} The empirical train loss underestimates the true loss by an amount that is inversely proportional to the population size. Owing to the large domain population differences, we expect varying generalization gap per domain. \citet{Sagawa19} adjusts the loss value of the domains to correct for these generalization gap differences as  $   \ell_i = \ell_i + C/\sqrt{n_i}, \quad \forall i.$, where $C>0$ is a hyper-parameter. 
We apply similar 
choice adjustments for {\cg} as well. The corrected loss values simply replace the $\ell_i$ of Line~\ref{alg:cgd:cg:loss} in Algorithm~\ref{alg:cgd:cg}. choice adjustments improved results on only a small subset of the datasets (WaterBirds, CelebA), albeit we tuned this hyperparameter over all standard subpopulation shift datasets for a fair comparison with \dro{}. 

\section{Convergence Analysis: Descent Nature of CGD}
\label{sec:cgd:thm}
In this section, we will show that \cg~is a sound optimization algorithm by proving that it monotonically decreases the function value and finds first order stationary points (FOSP) for bounded, Lipschitz and smooth loss functions.
We now define the notion of $\epsilon$-FOSP which is the most common notion of optimality for general smooth nonconvex functions.
\begin{defn}\label{defn:fosp}
A point $\theta$ is said to be an $\epsilon$-FOSP of a differentiable function $f(\cdot)$ if $\norm{\nabla f(\theta)}\leq \epsilon$.
\end{defn}
In the context of our paper, we consider the following cumulative loss function: $\Rcal(\theta)\defeq \frac{1}{k}\sum_i \ell_i(\theta)$. This is called the macro/domain-average loss.
We are now ready to state the convergence guarantee for our algorithm.
\begin{theorem}\label{thm:main}
Suppose that (i) each $\ell_i(\cdot)$ is $G$-Lipschitz, (ii) $\Rcal(\cdot)$ is $L$-smooth and (iii) $\Rcal(\cdot)$ is bounded between $-B$ and $B$. Suppose further that Algorithm~\ref{alg:cgd:cg} is run with $\eta = 2 \sqrt{\frac{B}{LG^2T}}$ and $\eta_\alpha = \sqrt{\frac{BL}{G^6T}}$. Then, Algorithm~\ref{alg:cgd:cg} will find an $\epsilon$-FOSP of $\Rcal(\theta)$ in $\order{\frac{BLG^2}{\epsilon^4}}$ iterations.
\end{theorem}
In other words, the above result shows that Algorithm~\ref{alg:cgd:cg} finds an $\epsilon$-FOSP of the cumulative loss in $\order{\epsilon^{-4}}$ iterations. We present a high level outline of the proof here and present the complete proof of Theorem~\ref{thm:main} in Appendix~\ref{app:cgd:proof}.
\begin{proof}[Proof outline of Theorem~\ref{thm:main}]
Considering the $t^\textrm{th}$ iteration where the iterate is $\theta^t$ and mixing weights are $\alpha^t$.
Using the notation in Section~\ref{sec:cgd:method}, let us denote $\Rcal(\theta)= \frac{1}{k} \sum_i \ell_i(\theta)$, $g_i = \nabla \ell_i(\theta^t)$ and $g = \frac{1}{k} \sum_i g_i$. The update of our algorithm is given by:
\begin{align}
    \alpha_i^{t+1} &= \frac{\alpha_i^t \cdot \exp\left(\eta_\alpha \iprod{g_i}{g}\right)}{Z} \label{eqn:cgd:algo1-main} \\
    \theta^{t+1} &= \theta^t - \eta \sum_i \alpha_i^{t+1} g_i, \label{eqn:cgd:algo2-main}
\end{align}
where $Z = \sum_j \alpha_j^t \cdot \exp\left(\eta_\alpha \iprod{g_j}{g}\right)$.
Let us fix $\alpha^* = \left(1/k,\cdots,1/k\right) \in \R^k$ and use $KL(p,q) = \sum_i p_i \log \frac{p_i}{q_i}$ to denote the KL-divergence between $p$ and $q$. Noting that the update~\eqref{eqn:cgd:algo1-main} on $\alpha$ corresponds to mirror descent steps on the function $\iprod{g}{\sum_i \alpha_i g_i}$ and using mirror descent analysis, we obtain:
\begin{align}
    KL(\alpha^*, \alpha^{t+1}) &\leq KL(\alpha^*, \alpha^{t}) + \left(\sum_i \alpha_i^{t} \eta_\alpha \iprod{g_i}{g}\right) + \left(\eta_\alpha G^2\right)^2 - \eta_\alpha \norm{g}^2 \nonumber \\
    \Rightarrow - \sum_i \alpha_i^{t} \iprod{g_i}{g} &\leq - \norm{g}^2 + \frac{KL(\alpha^*, \alpha^{t}) - KL(\alpha^*, \alpha^{t+1})}{\eta_\alpha} + \eta_\alpha G^4.\nonumber
\end{align}
Using monotonicity of the $\exp$ function, we further show that $\sum_i \alpha_i^{t} \iprod{g_i}{g} \leq \sum_i \alpha_i^{t+1} \iprod{g_i}{g}$.

Using smoothness of $\Rcal$ and update~\eqref{eqn:cgd:algo2-main}, we then show:
\begin{align*}
    \Rcal(\theta^{t+1}) &\leq \Rcal(\theta^t) + \eta \left( - \norm{g}^2 + \frac{KL(\alpha^*, \alpha^{t}) - KL(\alpha^*, \alpha^{t+1})}{\eta_\alpha} + \eta_\alpha G^4 \right)+ \frac{\eta^2 L G^2}{2} \\
    \Rightarrow \norm{g}^2 &\leq \frac{\Rcal(\theta^t) - \Rcal(\theta^{t+1})}{\eta} + \frac{KL(\alpha^*, \alpha^{t}) - KL(\alpha^*, \alpha^{t+1})}{ \eta_\alpha} + {\eta_\alpha G^4}{} + \frac{\eta LG^2}{2}.
\end{align*}
Summing the above inequality over timesteps $t=1,\cdots,T$, and using the parameter choices for $\eta$ and $\eta_\alpha$ proves the theorem.
\end{proof}

\section{Qualitative Analysis}
\label{sec:cgd:qua}
In this section, we present simple multi domain training scenarios to derive insights on the difference between \dro{} and \cg{}.
For all the settings, we carry two sets of experiments: (a) with linear models on simple data (b) deep models on MNIST images.\\
{\it For experiments using toy data:} the number of input dimensions if not otherwise mentioned is two and the two features $x_1,x_2$ are sampled from a standard normal, with the binary label  $y=\mathbb{I}[x_1+x_2>0]$, and the number of domains is three. We train a linear binary classifier model for 400 epochs with batch SGD optimizer and learning rate 0.1.\\
{\it For experiments using MNIST examples:} we train a Resnet-18 model with standard SGD, pick the best model based on the validation set, and report the performance on the test of the best model.  We binarized the label space by grouping the first five digits in to one class and the next five in to another. We create three domains of examples by randomly partitioning the training set into 4,900, 4,900, and 200 examples each.

\noindent
In each setting, we inspect the training weight ($\alpha_i^t$) assigned by each algorithm per domain $i$ at epoch $t$ when training on the toy data (Figure~\ref{fig:cgd:noisesimple},~\ref{fig:cgd:rotsimple},~\ref{fig:cgd:spusimple}). The plots are then used to draw qualitative judgements.
See Appendix~\ref{appendix:cgd:synth} for descriptive plots.

\subsection{Label noise}
\label{sec:cgd:qua:noise}
\vspace{-20pt}
\begin{wrapfigure}{r}{0.5\textwidth}
  \centering
  \vspace{-20pt}
  \begin{minipage}{0.47\linewidth}
    \includegraphics[width=\linewidth]{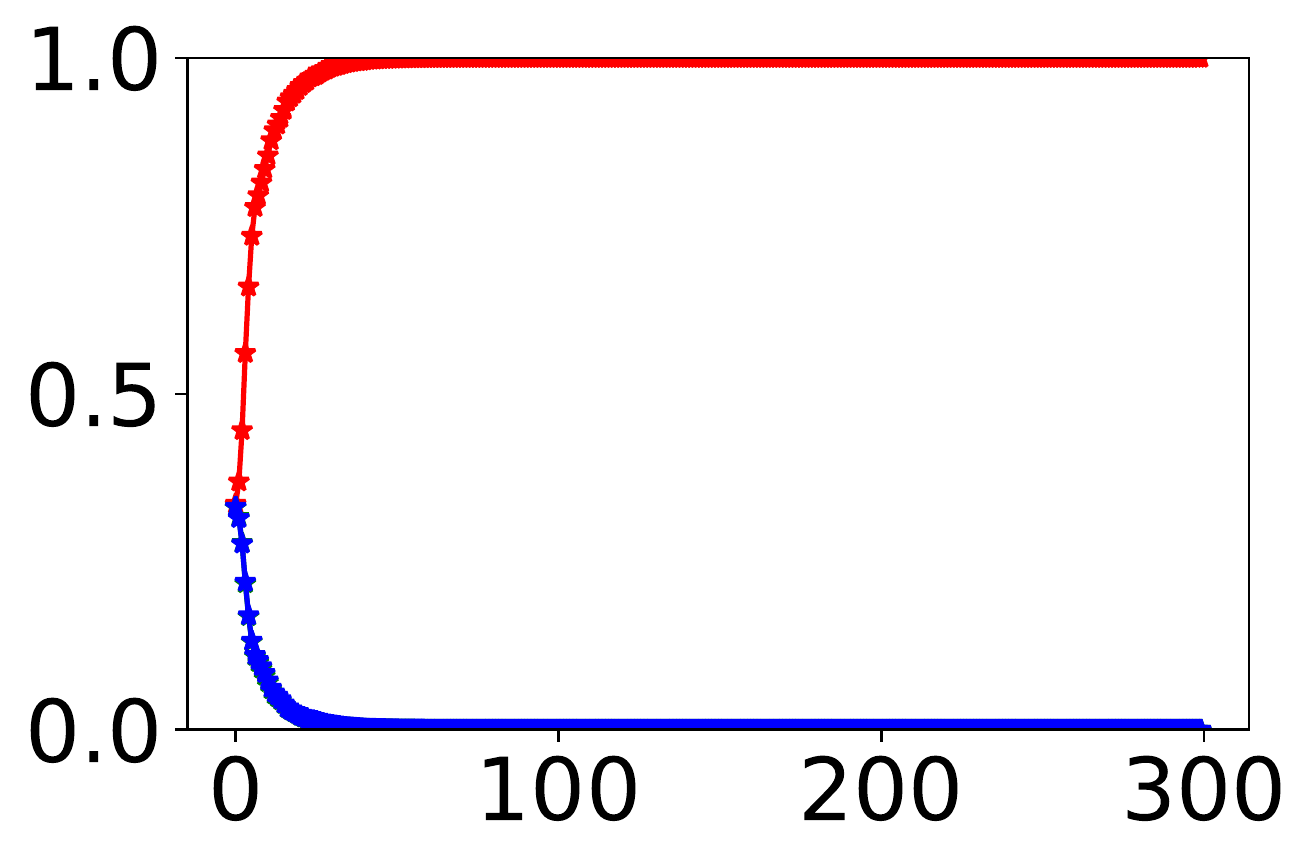}
    \subcaption{{\dro}}
    \label{fig:cgd:noisesimple:1}
  \end{minipage}\hfill
  \begin{minipage}{0.47\linewidth}
    \centering
    \includegraphics[width=\linewidth]{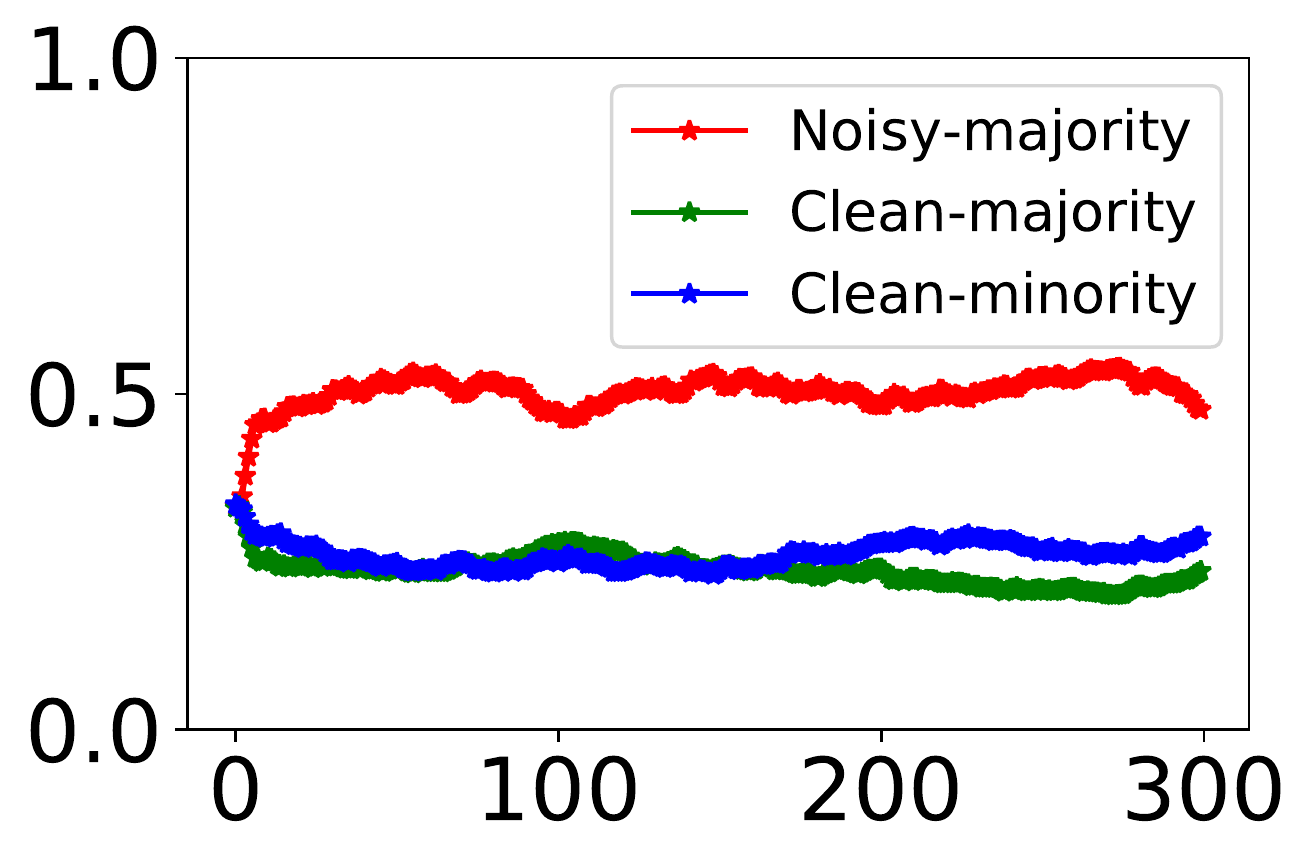}
    \subcaption{\cg}
    \label{fig:cgd:noisesimple:2}
  \end{minipage}
  \caption{Noise-Simple. Training weight ($\alpha$) of domains vs number of epochs.}
  \label{fig:cgd:noisesimple}
  \vspace{-20pt}
\end{wrapfigure}
\emph{{\dro} focuses on noisy domains.}
We induced label noise in the first domain (only during training) by flipping labels for randomly picked 20\% examples. The first and second domain formed the majority with a population of 450 examples each and the minority last
domain had 100 examples. The noisy first domain and the two subsequent domains are referred to as Noisy-Majority, Clean-Majority, and Clean-Minority. 
Due to noise in the first domain, and since we cannot overfit in this setup, the loss on the first domain was consistently higher than the other two. As a result, {\dro} trained only on the first domain (Fig~\ref{fig:cgd:noisesimple:1}), while {\cg} avoided overly training on the noisy first domain (Fig~\ref{fig:cgd:noisesimple:2}). This results in not only lower worst-case error but more robust training (Noise-Simple column, Table~\ref{tab:cgd:simple}). The variance of the learned parameters across six runs drops from 1.88 to 0.32 using {\cg}.

Table~\ref{tab:cgd:mnist} (Noise-MNIST column) shows the same trend under similar setup on the MNIST dataset where we randomly flip labels for all the examples in the training split of the first domain.
These experiments demonstrate that \dro's method of focusing on the domain with the worst training loss is sub-optimal  when domains have heterogeneous label noise or training difficulty. \cg's weights that are informed by inter-domain interaction provide greater in-built robustness to domain-specific noise.

\subsection{Uneven inter-domain similarity}
\label{sec:cgd:qua:rot}
\vspace{-20pt}
\emph{{\cg} focuses on central domains.}
Here we simulated a setting such that the optimal classifier for the first, third domain is closest to the second. 
The label (y) was set to a domain-specific deterministic function of the input ($x_1, x_2$); for the first domain the mapping function was $y=\mathbb{I}[x_1>0]$ (which was rotated by 30 degrees for the two subsequent domains), for the second domain it was $y=\mathbb{I}[0.87x_1 + 0.5x_2>0]$, and for the third domain it was $y=\mathbb{I}[0.5x_1+0.87x_2>0]$. The angle between the classifiers is a proxy for the distance between the
\begin{wrapfigure}{r}{0.5\textwidth}
    \begin{subfigure}[b]{0.47\linewidth}
       \centering
       \includegraphics[width=\linewidth]{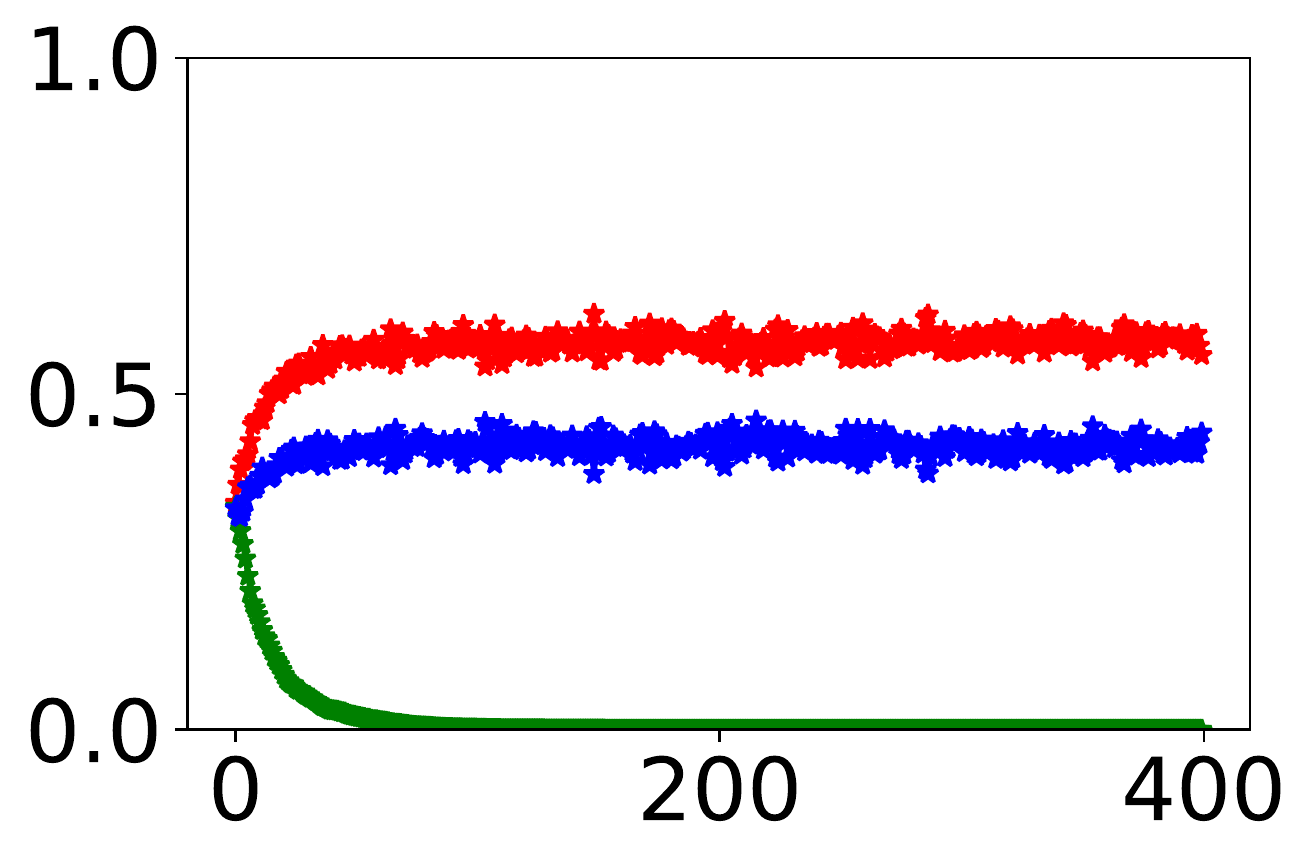}
       \caption{\dro}   
       \label{fig:cgd:rotsimple:1}
     \end{subfigure}
     \begin{subfigure}[b]{0.47\linewidth}
       \centering
       \includegraphics[width=\linewidth]{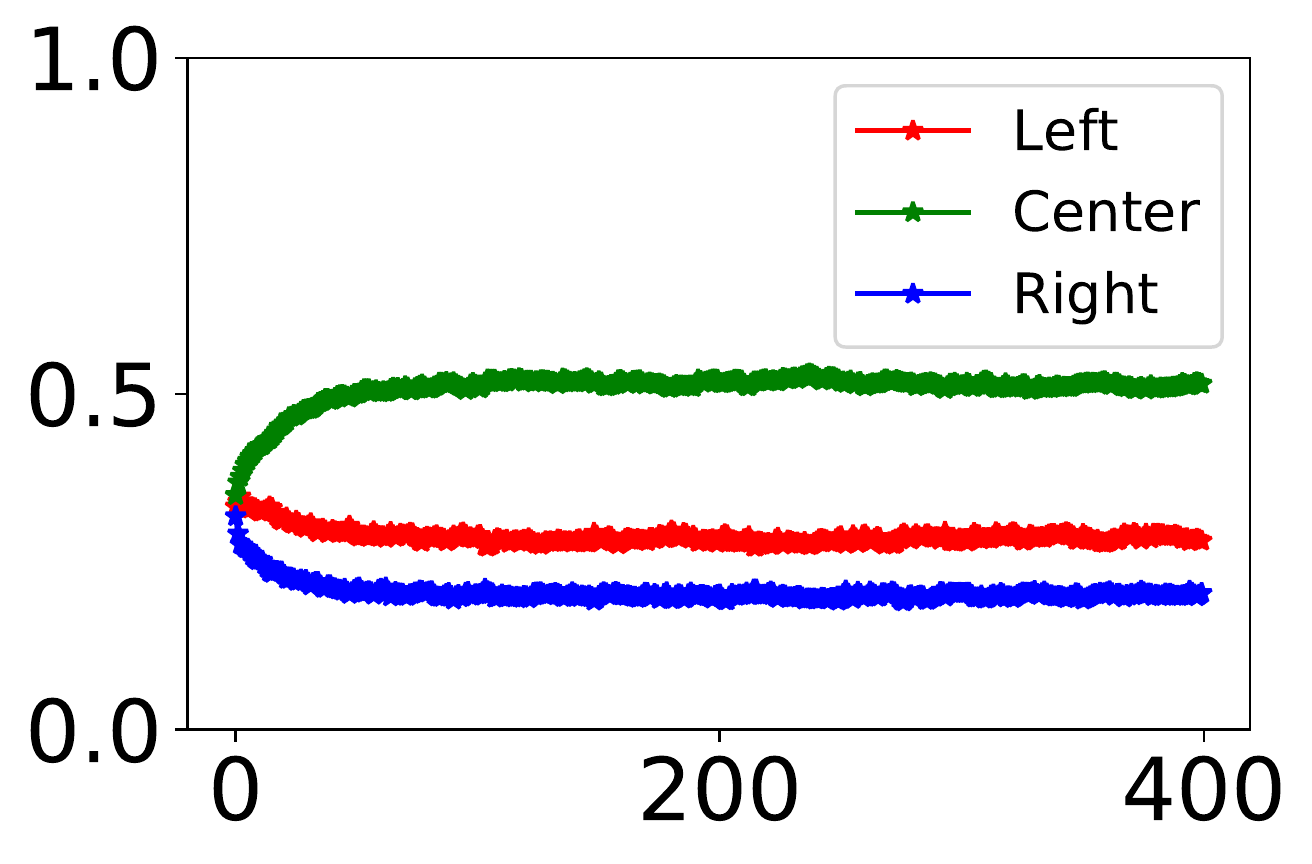}
       \subcaption{\cg}
       \label{fig:cgd:rotsimple:2}
     \end{subfigure}
     \captionof{figure}{Rot-Simple. Training weight ($\alpha$) of domains vs number of epochs.}
     \label{fig:cgd:rotsimple}
\end{wrapfigure}
domains. We refer to the domains in the order as {\it Left, Center, Right}, and their respective training population was: 499, 499, 2.

The optimal classifier that generalizes equally well to all the domains is the center classifier. 
{\dro}, {\cg}, differed in how they arrive at the center classifier: {\dro} assigned all the training mass equally to the Left and Right domains (Figure~\ref{fig:cgd:rotsimple:1}), while {\cg} trained, unsurprisingly, mostly on the Center domain that has the highest transfer (Figure~\ref{fig:cgd:rotsimple:2}). {\dro} up-weighted the noisy minority domain (2 examples) and is inferior when compared with {\cg}'s strategy, which is reflected by the superior performance on the test set shown in Rotation-Simple column of Table~\ref{tab:cgd:simple}.  
On the MNIST dataset we repeated a similar experiment by rotating digits by 0, 30, and 60 degrees, and found \cg\ to provide significant gains over \dro\ (Rotation-MNIST column of Table~\ref{tab:cgd:mnist}). 
These experiments demonstrate the benefit of focusing on the central domain even for maximizing worst case accuracy.

\subsection{Spurious correlations} 
\label{sec:cgd:qua:neg}
\vspace{-20pt}
\emph{{\cg} focuses on domains without spurious correlations.}
Here we create spurious correlations by adding a third feature that takes different values across the three different domains whose sizes are  490, 490, and 20 respectively.
A third feature is introduced and was set to the value of the label y on the first domain, 1-y on the third domain, and set to y or 1-y w.p. 0.6 on the second domain. We 
\begin{wrapfigure}{r}{0.5\textwidth}
  \begin{minipage}{0.47\textwidth}
     \centering
     \hspace{5pt}
     \begin{subfigure}[b]{0.47\linewidth}
       \centering
       \includegraphics[width=\linewidth]{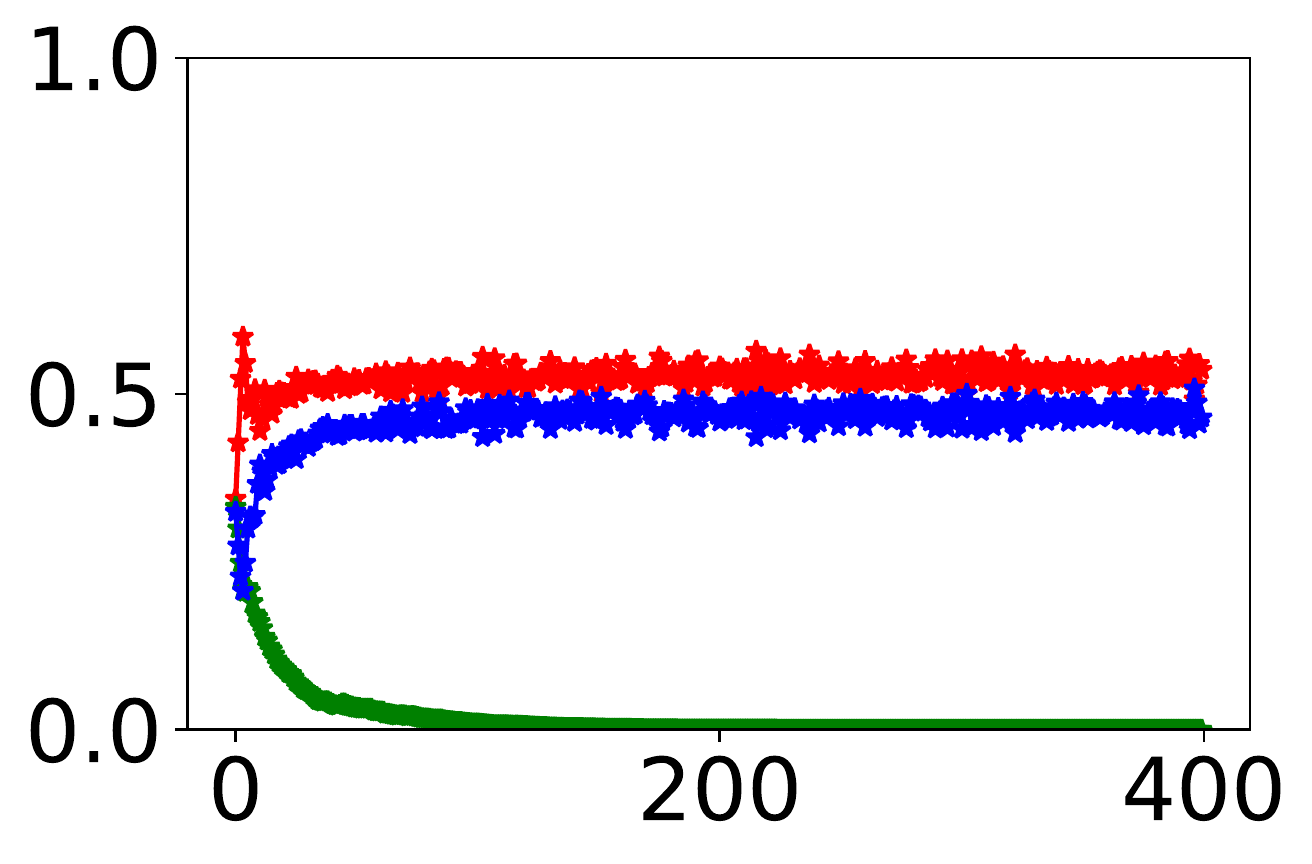}
       \subcaption{\dro}
       \label{fig:cgd:spusimple:1}
     \end{subfigure}
     \begin{subfigure}[b]{0.47\linewidth}
       \centering
       \includegraphics[width=\linewidth]{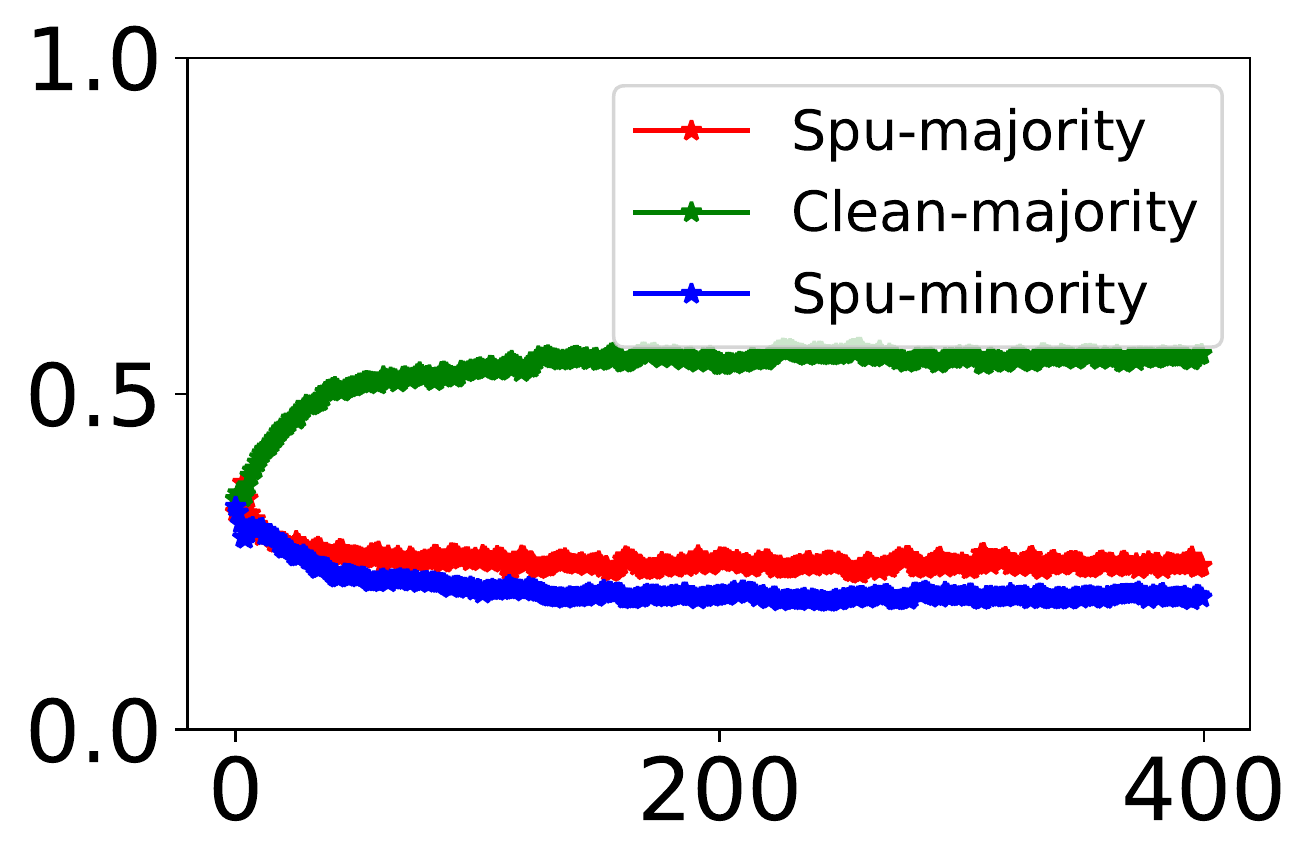}
       \subcaption{\cg}
       \label{fig:cgd:spusimple:2}
     \end{subfigure}
     \captionof{figure}{Spurious-Simple. Training weight ($\alpha$) of domains vs number of epochs.}
     \label{fig:cgd:spusimple}
  \end{minipage}
\end{wrapfigure}
perturbed the first two features of the first domain examples such that they predict label correctly only for 60\% of the examples. Therefore, the net label correlation of the first two and the third feature is about the same: 0.8. We refer to the first and last domains as Spurious-majority, Spurious-minority---since they both contain an active spurious third feature---and the second as the clean majority. 

{\dro} balances training between the two spurious domains (Figure~\ref{fig:cgd:spusimple:1}), while {\cg} focuses mostly on the clean domain (Figure~\ref{fig:cgd:spusimple:2}). Both the training strategies could avoid the spurious third feature, however, the training strategy of {\dro} is sub-optimal since we up-weight the small, hence noisy, spurious minority domain. The better test loss of {\cg} in Table~\ref{tab:cgd:simple} (Spurious-Simple column) further verifies the drawback of {\dro}'s training strategy.  On the MNIST dataset too we observed similar trends by creating spurious correlation using colors of the digit (Spurious-MNIST column of Table~\ref{tab:cgd:mnist}); digits in the first domain are colored red for label 0 and blue for label 1,  digits are randomly colored blue or red in the second domain, and colored blue for label 0 and red for label 1 in the third domain. 

More generally, whenever a relatively clean domain with low spurious correlation strength is present, and when the train domains represent both positive and negative spurious correlations, we expect {\cg} to focus on the domain with the least spurious correlation since it is likely to have greater inter-domain similarity. 

It is worth noting that for all the experiments shown in this section (including linear classifier), the train loss of \cg\ is comparable with \dro{} and \ermuw{} confirming our convergence result of Section~\ref{sec:cgd:thm}. All the algorithms converged to zero training loss when using Resnet-18 on MNIST digits, and approximately similar train loss for toy experiments as shown in Table~\ref{tab:cgd:simple:train_loss}. However, worst domain test loss varies by a large amount between the different algorithms as shown in Table~\ref{tab:cgd:simple},~\ref{tab:cgd:mnist}. 

\begin{table}[htb]
    \centering
    \begin{tabular}{|c|c|c|c|}
        \hline
         Task $\rightarrow$ & Spurious-Simple & Rotation-Simple & Noise-Simple \\\hline
         \cg{} & 0.43 (0.01) & 0.24 (0.04) & 0.36 (0.01) \\
         \dro{} & 0.45 (0.01) & 0.25 (0.04) & 0.41 (0.02) \\
         \erm{} & 0.42 (0.01) & 0.23 (0.05) & 0.34 (0.02) \\
        \hline
    \end{tabular}
    \caption{Macro averaged train loss for different tasks and algorithms in the simple toy setting with a linear classifier model. All algorithms converged almost to the same loss value but to different solutions.}
    \label{tab:cgd:simple:train_loss}
\end{table}

\begin{table}[htb]
    \centering
    \begin{tabular}{|c|c|c|c|c|c|c|}
    \hline
     {\bf Task} $\rightarrow$ &
     \multicolumn{2}{|c|}{\bf Noise-Simple} & \multicolumn{2}{|c|}{\bf Rotation-Simple} &
     \multicolumn{2}{|c|}{\bf Spurious-Simple}\\
     Alg$\downarrow$ & Variance & Worst Loss & Variance & Worst Loss & Variance & Worst Loss\\\hline
    {\dro} & 1.88 & 0.35 (0.03) & 0.41 & 0.77 (0.14) & 0.17 & 0.70 (0.16)\\
    \rowcolor{LightGreen} {\cg} & 0.32 & 0.25 (0.02) & 0.08 & 0.59 (0.05) & 0.04 & 0.43 (0.06) \\\hline
    \end{tabular}
    \caption{Worst loss for different tasks and algorithms in the simple toy setting with a linear classifier model. Worst loss is the worst domain binary cross entropy loss on the test set, and averaged over six seeds, shown in parenthesis is the standard deviation. Variance column shows the variance of $L_\infty$ normalized solution across the six runs. {\cg} has lower variance and test loss when compared with {\dro}.}
    \label{tab:cgd:simple}
\end{table}
\begin{table}[tbh]
    \centering
    \begin{tabular}{|c|c|c|c|}
    \hline
         Alg. $\downarrow$ & Noise-MNIST & Rotation-MNIST & Spurious-MNIST \\\hline
         {\dro} & 86.0 (1.0), 85.9 (1.0) &  90.5 (0.2), 80.6 (1.1) & 95.4 (0.2), 95.0 (0.4) \\
         \rowcolor{LightGreen} {\cg} & 88.9 (1.0), 88.8 (1.0) & 92.1 (0.6), 85.0 (0.7) & 96.7 (0.1), 96.4 (0.3) \\\hline
    \end{tabular}
    \captionof{table}{Average, worst domain accuracy on the test split. Shown in parenthesis is the standard deviation. All numbers are aggregated over three runs.}
    \label{tab:cgd:mnist}
\end{table}

\section{Experiments}
\label{sec:cgd:expt}
In this section, we enumerate datasets, discuss baselines and present implementation details, and finally present our results in Section~\ref{sec:cgd:results}.

\noindent
{\bf \large Datasets}\\
We evaluated on eight datasets, which include two synthetic datasets with induced spurious correlations: ColoredMNIST, WaterBirds; two real-world datasets with known spurious correlations: CelebA, MultiNLI; four WILDS~\citep{KohWilds20} datasets with a mix of subpopulation and domain shift. Table~\ref{tab:cgd:dataset:all} summarises all the eight datasets.  We describe all the datasets in more detail below. For more details on WaterBirds, CelebA, MultiNLI, we point the reader to the original paper~\citep{Sagawa19}, and~\citet{KohWilds20} for details on the WILDS datasets. We discuss how some of the real-world datasets relate to our synthetic setting in the next section.

\begin{table}[htb]
    \centering
    \setlength{\tabcolsep}{3pt}
    \begin{tabular}{|l|c|c|c|c|c|c|}
    \hline
         \thead{Dataset} & \thead{\# labels} & \thead{\# domains} & \thead{Grouping type} & \thead{Type} & \thead{Size} & \thead{Worst ratio} \\\hline
         Colored MNIST &  2 & 2 & Label$\times$domain & Image & 50K & 1000 \\
         WaterBirds & 2 & 2 & Label$\times$domain & Image & 4.8K & 62.5\\
         CelebA & 2 & 2 & Label$\times$domain & Image & 162K & 51.6\\
         MultiNLI & 3 & 2 & Label$\times$domain & Text & 200K & 44.3 \\
         CivilComments-WILDS & 2 & 2 & Label$\times$domain & Text & 270K & 74.5 \\
         PovertyMap-WILDS & real & 13 & domain & Image & 10K & 5.9 \\
         FMoW-WILDS & 62 & 11 & domain & Image & 80K & 16.0 \\
         Camelyon17-WILDS & 2 & 3 & domain & Image & 3K & 2.5 \\\hline
    \end{tabular}
    \caption{Summary of the datasets we used for evaluation. The Size column shows the number of training instances. The Worst ratio column shows the ratio of the size of the largest train domain to the smallest; worst ratio reflects loosely the risk of minority domain overfitting.}
    \label{tab:cgd:dataset:all}
\end{table}

\noindent
{\bf Colored-MNIST:} Using MNIST~\citep{MNIST} digits, we create a dataset where the foreground color  is spuriously correlated with the label. We split 50,000 examples into majority, minority with 1000:1 domain size ratio, and binarize the label space to predict digits 0-4, 5-9 as the two classes. Test data has equal proportion of each domain.
In the majority domain, the foreground is red for examples of label 0 and blue for label 1; the minority domain examples were colored in reverse. Color based prediction of the label cannot generalize, however, color is spuriously and strongly correlated with the label in the training data.  Many previous work~\citep{ArjovskyIRM19,Ahmed21} adopted Colored MNIST for probing generalization.  

\noindent
{\bf WaterBirds} task is to predict if an image contains water-bird or a land-bird and contains two domains: majority, minority in ratio 62:1 in training data. The minority domain has reverse background-foreground coupling compared to the majority. Test/validation data has equal proportion of each domain. 

\noindent
{\bf CelebA}~\citep{CelebA} task is to classify a portrait image of a celebrity as blonde/non-blonde, examples are grouped based on the gender of the portrait's subject. The training data has only 1,387 male blonde examples compared to 200K total examples. Test/validation data follows same domain-class distribution as train. 

\noindent
{\bf MultiNLI}~\citep{MultiNLI} task is to classify a pair of sentences as one of {\it entailment, neutral, contradiction}. \citet{Gururangan18} identified that negation words in the second sentence are spuriously correlated with the contradiction labels. Accordingly, examples are grouped based on if the second sentence contains any negation word. Test/validation data follows same domain-class distribution as train.


\noindent
{\bf CivilComments-WILDS}~\citep{CivilComments} task is to classify comments as toxic/non-toxic. Examples come with eight demographic annotations: {\it white, black, gay, muslim}, based on if the comment mentioned terms that are related to the demographic. Label distribution varies across demographics. In the bechmark train examples are grouped on {\it black} demographic, and testing is on the worst domain accuracy among all 16 combinations of the binary class label and eight demographics.  

\noindent
{\bf PovertyMap-WILDS}~\citep{PoverertMap} task is to classify satellite images in to a real valued wealth index of the region. 
The rural and urban sub-population from different countries make the different train domains. Bechmark evaluates on worst-region Pearson correlation between predicted and true wealth index in two settings: An in-domain setting that measures sub-population shift to seen regions, and an out-domain setting that evaluates generalization to new countries.

\noindent
{\bf FMoW-WILDS}~\citep{FMoW} task is to classify RGB satellite images to one of 62 land use categories.  Land usage differs across countries and evolves over years. Training examples are stratified in to eleven domains based on the year of satellite image acquisition. We report out-domain evaluation using test data from regions of seen countries but from later years, and measures worst accuracy among the five geographical regions: Africa, Americas, Oceania, Asia, and Europe. The in-domain evaluation with worst accuracy on the eleven years is unavailable for other algorithms, so we stick to only out-domain evaluation. 

\noindent
{\bf Camelyon17-WILDS}~\citep{Camelyon17} is a binary classification task of predicting from a microscope image of a tissue if it contains a tumour. The training data contains scans from three hospitals, and the test data contains scans from multiple unseen hospitals. The evaluation metric is average accuracy on the test set hospitals. 


\noindent
{\bf \large Experiment Details}\\
{\bf Baselines}\\
{\it \erm}: Simple descent on the joint risk.\\
{\it \ermuw}, {\it \dro}: risk reshaping baselines described in Section~\ref{sec:cgd:method}.
With {\ermuw}, instead of simply up-weighting all the domains equally, we set the domain weighting parameter $\alpha_i$ of a domain $i$ as a function of its size. For some non-negative constant C, $\alpha_i\propto\exp(C/\sqrt{n_i})$. \\
{\it \pgi}~\citep{Ahmed21}, predictive domain invariance, penalizes the divergence between predictive probability distributions among the training domains, and was shown to be effective for in-domain as well as out-domain generalization. We provide more details about the algorithm in Appendix~\ref{appendix:cgd:technical}.

On the WILDS dataset, we also compare with two algorithms from Domain Generalization literature: {\it CORAL}~\citep{CORAL}, {\it IRM}~\citep{ArjovskyIRM19}. Domain Generalization methods are relevant since they could generalize to any domain including the seen train domains.

Yet another simple baseline is to train a a classifier to recognize the domain an example belongs and use its corresponding domain-specific  classifier. For instance, in the waterbirds dataset, we could first determine if an example belongs to majority/minority domain, i.e. if the background is positively or negatively correlated with the foreground, and use majority or minority domain-specific  classifier that could potentially exploit spurious correlations. 
This algorithm requires a good domain classifier and a good domain-specific label classifier. We eliminated this algorithm without further comparison due to the following reasons:
(a) On the Waterbirds dataset, majority/minority specific label classifiers perform extremely well, but we could not train a majority/minority classifier (for classifying if an example belongs to majority or minority) that is better than a random predictor.
(b) The CelebA dataset, similar to waterbirds, has two domains: male/female. domain-specific  classifiers that are trained exclusively on male or female subjects do not improve the minority (male blond) performance over a model that is trained on all genders using ERM, violating the second requirement: good domain-specific classifier.


\noindent
{\bf Implementation Details and Evaluation Metric: }
We use the codestack\footnote{\url{https://github.com/p-lambda/wilds}} released with the WILDS~\citep{KohWilds20} dataset as our base implementation. We report the average and standard deviation of the evaluation metric from at least three runs. If not otherwise stated explicitly, we report worst accuracy over the domains evaluated on the test split. 

\noindent
{\bf Hyperparameters: }
We search $C$: the choice adjustment parameter for {\dro}, {\cg}, and domain weighting parameter for {\ermuw}, over the range $[0, 20]$. We tune $C$ only for WaterBirds, CelebA and MultiNLI datasets for a fair comparison with {\dro}. For other dataset, we set C to 0 because choice adjustment tuning did not yield large improvements on many datasets. 
The step size parameter of {\dro}, {\cg}, and $\lambda$ penalty control parameter of {\pgi}, is picked from \{1, 0.1, 1e-2, 1e-3\}. We follow the same learning procedure of \citet{Sagawa19} on Colored-MNIST, WaterBirds, CelebA, MultiNLI, datasets, we pick the best learning rate parameter from \{1e-3, 1e-5\}, weight decay from \{1e-4, .1, 1\}, use SGD optimizer, and set the batch size to 1024, 128, 64, 32 respectively.
On WILDS datasets, we adopt the default optimization parameters\footnote{\href{https://github.com/p-lambda/wilds/blob/747b774a8d7a89ae3bde3bc09f3998807dfbfea5/examples/configs/datasets.py}{WILDS dataset parameter configuration Github URL.}}, and only tune the step-size parameter.\\
We report the test set's performance corresponding to the best hyperparameter(s) and epoch on the validation split. We use standard train-validation-test splits for all the datasets when available.

\noindent
{\bf Base Model: }
We followed closely the training setup of~\citet{Sagawa19,KohWilds20}. 
On WaterBirds, CelebA dataset, we used Resnet-50; on Colored-MNIST, PovertyMap, we used Resnet-18; on Cemlyon17, FMoW, we used DenseNet-121. All the image models except Colored-MNIST (which trains on $28\times 28$ images), PovertyMap (which deals with multi-spectral version of Resnet-18), are pretrained on ImageNet. MultiNLI, CivilComments, use a pretrained uncased DistilBERT-base model. 

\noindent
{\bf Additional Details:}
We model inter-domain transfer characteristics using inner products of first order gradients (\eqref{eqn:cgd:final_alpha_update}). Since per-domain gradient computation for all the parameters is expensive, we use only a subset of parameters. We use only the last three layers for ResNet-50, Densenet-101 and DistilBERT, and all the parameters for gradient computation for any other network. 

\noindent
\subsection*{Results}
\label{sec:cgd:results}
Table~\ref{tab:cgd:results} shows the worst domain test split accuracy for the four standard (non-WILDS) subpopulation shift datasets. For all the tasks shown, {\erm} performs worse than a random baseline on the worst domain, although the average accuracy is high. {\ermuw} is a strong baseline, and improves the worst domain accuracy of {\erm} on three of the four tasks, without hurting much the average accuracy. {\pgi} is no better than {\erm} or {\ermuw} on all the tasks except on Colored-MNIST. {\dro} improves worst accuracy on most datasets, however, {\cg} fares even better. {\cg} improves worst-domain accuracy on all the datasets over {\erm}. Except on MultiNLI text task, the gains of {\cg} are significant over other methods, and on MultiNLI task the worst domain accuracy of {\cg} is at least as good as {\dro}. On Colored-MNIST, which has the highest ratio of majority to minority domain size, the gains of {\cg} are particularly large. 

We report comparisons using the four WILDS datasets in Table~\ref{tab:cgd:results:wilds}. We show both in-domain (ID) and out-domain (OOD) generalization performance when appropriate, they are marked in the third row. All the results are averaged over multiple runs; FMoW numbers are averaged over three seeds, CivilComments over five seeds, Camelyon17 over ten seeds, and PovertyMap over the five data folds. Confirming with the WILDS standard, we report worst domain accuracy for FMoW, CivilComments, worst region Pearson correlation for PovertyMap, average out-of-domain accuracy for Camelyon17. We make the following observations from the results.
{\erm} is surprisingly strong on all the tasks except CivilComments. Strikingly, {\dro} is worse than {\erm} on four of the five tasks shown in the table, including the in-domain (subpopulation shift) evaluation on PovertyMap task. {\cg} is the only algorithm that performs consistently well across all the tasks.
The results suggest {\cg} is significantly robust to subpopulation shift, and performs no worse than {\erm} on domain shifts. Further, we study Colored-MNIST dataset under varying ratio of majority to minority domain sizes in Appendix~\ref{appendix:cgd:hetero} and demonstrate that {\cg} is robust to subpopulation shifts even under extreme training population disparity. 

\begin{table}[tbh]
\centering
\setlength{\tabcolsep}{5pt}
\begin{tabular}{|l|c|c|c|c|}
\hline
Alg. $\downarrow$ & CMNIST & WaterBirds & CelebA & MultiNLI\\
\hline
{\erm} & 50/0.0, 0/0.0 & 85.2/0.8, 61.2/1.4 & {\bf 95.2/0.2)}, 45.2/0.9 & {\bf 81.8/0.4}, 69.0/1.2\\
{\ermuw} & 53.6/0.1, 7.2/2.1 & {\bf 91.8/0.2}, 86.9/0.5 & 92.8/0.1, 84.3/0.7 & 81.2/0.1, 64.8/1.6 \\
{\pgi} & 52.5/1.7, 44.4/1.5 & 89.0/1.8, 85.8/1.8 & 92.9/0.2, 83.0/1.3 & 80.8/0.8, 69.0/3.3 \\
G-DRO & 73.5/5.3, 48.5/11.5 & 90.1/0.5, 85.0/1.5 & 92.9/0.3, 86.1/0.9 & 81.5/0.1, {\bf 76.6/0.5} \\
{\cg} & {78.6/0.5, \bf 65.6/5.9} & 91.3/0.6, {\bf 88.9/0.8} & 92.5/0.2, {\bf 90.0/0.8} & 81.3/0.2, {\bf 76.1/1.5} \\\hline
\end{tabular}
\caption{Average and worst-domain accuracy in that order. All numbers are averaged over 3 seeds and presented in the format `average/std. dev.'. G-DRO is abbreviation for {\dro}.}
\label{tab:cgd:results}
\end{table}

\begin{table}[htb]
    \centering
    \setlength{\tabcolsep}{3pt}
    \begin{tabular}{|l|c|c|c|c|c|}
    \hline
    Algorithm$\downarrow$ & Camelyon17 & \multicolumn{2}{c|}{\bf PovertyMap} & FMoW & CivilComments \\
    Metric $\rightarrow$ & Avg. Acc. & \multicolumn{2}{c|}{Worst Pearson r} & Worst-region Acc & Worst-domain Acc. \\
    Eval. type $\rightarrow$ & OOD & ID & OOD & OOD & ID \\
    \hline
         CORAL & 59.5 (7.7) & {\bf 0.59 (0.03)} & {\bf 0.44 (0.06)} & 31.7 (1.2) & 65.6 (1.3) \\
         IRM & 64.2 (8.1) & 0.57 (0.08) & 0.43 (0.07) & 30.0 (1.4) & 66.3 (2.1) \\
         ERM & {\bf 70.3 (6.4)} & 0.57 (0.07) & {\bf 0.45 (0.06)} & {\bf 32.3 (1.2)} & 56.0 (3.6) \\
         \dro & 68.4 (7.3) & 0.54 (0.11) & 0.39 (0.06) & 30.8 (0.8) & {\bf 70.0 (2.0)} \\
         \cg  & {\bf 69.4 (7.8)} & {\bf 0.58 (0.05)} & 0.43 (0.03) & {\bf 32.0 (2.2)} & {\bf 69.1 (1.9)}\\ \hline
    \end{tabular}
    \caption{Evaluation on WILDS datasets: All numbers averaged over multiple runs, standard deviation is shown in parenthesis. Second row shows the evaluation metric, and the third shows the evaluation type: in-domain (ID) or out-of-domain (OOD). Two highest absolute performance numbers are marked in bold in each column.}
    \label{tab:cgd:results:wilds}
\end{table}

\section*{How do synthetic experiments relate to real-world datasets?}
\label{appendix:cgd:synth_to_real}
In this section we discuss the similarities between our synthetic settings of Section~\ref{sec:cgd:qua} and the real-world datasets~\ref{sec:cgd:expt}. 

The text benchmarks (MultiNLI, CivilComments-WILDS) and CelebA resemble our toy setup of Section~\ref{sec:cgd:qua:neg}. 
In MultiNLI, the examples are grouped based on whether or not they contain negation words. The examples from the domain with no negation words, therefore, do not contain any spuriously correlated features. 
Similarly, in CivilComments-WILDS the examples from black demographic (domain) contain spurious correlation (they contain tokens that identify the demographic, which can be easily exploited to classify examples from black domain as mostly toxic), while such spurious features are absent in the non-black demographic (domain). 
In the CelebA dataset too, male non-blonde (majority) negatively transfers to the male blonde (minority) since the classifier may learn to interpret short hair (male) to be non-blonde. On the other hand, the female blond/non-blond domains do not contain any known spurious correlation.
In all the cases, CGD could avoid learning spurious features by focussing training on domains with no or relatively low spurious correlation similar to what was demonstrated in Section~\ref{sec:cgd:qua:neg}, thereby learning a more robust solution with dampened strength of spurious features. 

FMoW-WILDS, PovertyMap-WILDS are similar to our label noise simple setup of Section~\ref{sec:cgd:qua:noise}.
FMoW-WILDS task is to classify a satellite image into one of 62 land-use categories. The dataset is annotated with human curators labeling if a marked region in an image contains, say a ``police station''~\citep{PoverertMap}. Depending on the demographic spread of the human curators, the label correctness is expected to vary from one region to another. Also, some land-use categories are far easier to classify than others (for eg. ``police station'' vs ``helipad''). Similarly, PovertyMap-WILDS task is to map a satellite image to its poverty index. Poverty index per region (urban/rural settlement) ground-truth was acquired through secondary sources such as asset index of the region from the national demographic surveys~\citep{FMoW}. The asset to wealth index per region was found to vary per country and hence the quality of the label. The non-uniform label noise of the two datasets is similar to our setup in Section~\ref{sec:cgd:qua:noise}. CGD focuses only on the difficult domains that transfer well to the rest, unlike group-DRO that only pursue the maximum loss domains.
 
\section{Subpopulation Shift vs Domain Generalization}
Both \crossgrad\ and CSD fail to perform in the subpopulation shift setting owing largely to data imbalance. Majority of domain generalization algorithms cannot handle data imbalance because their design was guided by standard benchmarks, which contain almost uniform domain population. We will elaborate on their limitations for subpopulation shift robustness in this section. 

Subpopulation shift can be viewed as a special case of the domain generalization problem where the distribution of domains is defined as mixture of train distributions. Albeit with the crucial difference that the objective of subpopulation shift problem is to alleviate worst domain generalization while that of domain generalization is average generalization. Given their similarities, it is worth discussing if domain generalization methods also render subpopulation shift robustness; we already addressed this question in the experiments section to some extent. In this section, we specifically comment on efficacy of the domain generalization algorithms: \crossgrad{}, CSD, proposed in Chapter~\ref{chap:dg} for subpopulation shift.  

\crossgrad\ depends on the domain classifier for augmenting examples from new domains. Training a good domain classifier is challenging for highly imbalanced subpopulation shift datasets. In this setting, domain classifier can be prone to overfitting just as the label classifier. In fact if we can predict domain well from an example, we can learn domain-specific  classifiers and address subpopulation shift problem directly. We also empirically evaluated domain classification performance on Waterbirds and CelebA dataset, and found them to have performance close random baseline. Owing to challenges in training a domain classifier, \crossgrad{} augmentations cannot address the minority domain overfitting of the subpopulation shift problem.

Empirical evaluation revealed that CSD too fails to learn models robust to subpopulation shifts. Due to the presence or absence of spurious correlations and population skew, the pace of learning varies between different domains. CSD assumes that the optimal classifier for each domain is a combination of a common component and a specific component, which is violated when domains learn at different paces. The following simple exercise illustrates this CSD limitation. Consider a linear classification example with three domains and the optimal classifier per domain is proportional to $[1, 1, -0.5]^T, [1, 1, 1]^T, [1, 1, 1]^T$ respectively. Since the last feature has unstable label correlation (spurious feature), the ideal generalizing classifier is $[1, 1, 0]^T$. When the pace of learning is uniform across domains, the analytically derived decomposition of per-domain classifiers into $w_c, w_s$ is as shown in~\eqref{eqn:cgd:csd_ideal_decomposition}. In the expression below, we draw on the notation of Section~\ref{sec:csd} to decompose the per domain classifier W as the sum of common and specific components: $W = \ones{}w_c^T + \gamma w_s^T, w_c\perp w_s$. 

\begin{align}
&\underbrace{\begin{bmatrix}
1 & 1 & -0.5\\
1 & 1 & 1 \\
1 & 1 & 1
\end{bmatrix}}_{W} 
= \begin{bmatrix}1\\1\\1\end{bmatrix}\underbrace{\begin{bmatrix}1 & 1 & 0\end{bmatrix}}_{w_c^T} 
+ \begin{bmatrix}-0.5\\1\\1\end{bmatrix}\underbrace{\begin{bmatrix}0 & 0 & 1\end{bmatrix}}_{w_s^T}
\label{eqn:cgd:csd_ideal_decomposition}
\end{align}
The decomposition shown above correctly recovered the ideal common and specific components. However in the subpopulation shift setting, domains could learn at a different pace. For instance, if the first domain learns twice as fast as the other two domains, the analytically obtained solution to per-domain classifier decomposition is as shown below. 
\begin{align}
\underbrace{\begin{bmatrix}
2 & 0 & 0\\
0 & 1 & 0 \\
0 & 0 & 1
\end{bmatrix}  
\begin{bmatrix}
1 & 1 & -0.5\\
1 & 1 & 1 \\
1 & 1 & 1
\end{bmatrix}}_{W} 
= \begin{bmatrix}1\\1\\1\end{bmatrix}\underbrace{\begin{bmatrix}1 & 1 & 1\end{bmatrix}}_{w_c^T} 
+ \begin{bmatrix}2\\0\\0\end{bmatrix}\underbrace{\begin{bmatrix}0.5 & 0.5 & -1\end{bmatrix}}_{w_s^T}
\label{eqn:cgd:csd_practice_decomposition}
\end{align}

From the values shown above, observe that we can no longer recover the true common ($w_c$) and specific component ($w_s$) as desired. This simple exercise illustrates why CSD fails in subpopulation shift setting.
It is worth noting that the scale issue of decomposition cannot be simply handled by normalizing. Since we do not know stable features a priori, we cannot normalize such that stable features remain stable post normalization.

\section{Related Work}
Distributionally robust optimization (DRO) methods~\citet{duchi18} seek to provide uniform performance across all examples, through focus on high loss domains. As a result, they are not robust to the presence of outliers~\citep{Hashimoto18,Hu18,DORO21}. \citet{DORO21} handles the outlier problem of DRO by isolating examples with high train loss.

\citet{Sagawa19} extend DRO to the case where training data contains demarcated domains like in our setting.
\citet{Dagaev21,Liu21,Creager21,Ahmed21} extend the sub-population shift problem to the case when the domain annotations are unknown. They proceed by first inferring the latent domains of examples that negatively interfere in learning followed by robust optimization with the identified domains. In the same vein,~\citet{Zhou21,Bao21}  build on {\dro} for the case when the supervised domain annotations cannot recover the ideal distribution with no spurious features in the family of training distributions they represent. \citet{Zhou21} maintains per-domain and per-example learning weights, in order to handle noisy domain annotations. \citet{Bao21} uses environment specific classifiers to identify example groupings with varying spurious feature correlation, followed by {\dro}'s worst-case risk minimization. 

\citet{Goel20} augment the minority domain with generated examples. However, generating representative examples may not be easy for all tasks. 
In contrast, \citet{SagawaExacerbate20} propose to sub-sample the majority domains to match the size of minority domains. In the general case, with more than two domains and when domain skew is large, such sub-sampling could lead to severe data loss and poor performance for the majority domains. \citet{MenonOverparam21} partially get around this limitation by pre-training on the majority domain first.   \citet{Ahmed21} considered as a baseline a modified version of {\dro} such that all the label classes have equal training representation.  However, these methods have only been applied on a restricted set of image datasets with two domains, and do not consistently outperform {\dro}. In contrast, we show consistent gains over DRO and ERM up weighing, and our experiments are over eight datasets spanning image and text modalities.


It is well established that example reweighting methods converge to the same solution as with \erm\ loss for overparameterized deep learning models~\citep{Byrd19, Sagawa19, Zhai22}. However, these empirical or theoretical results do not hold when regularizing or early stopping the model training. In all our experiments, we picked the best model through early stopping based on validation split performance. \ermuw\, \dro\, \cg\ too may converge to the same solution as \erm{}, however the path toward the solution, and as a result the best model (according to the validation performance) along the path could vary.  

\noindent
{\bf Domain Generalization} (DG)  algorithms~\citep{CORAL,PiratlaCG,ArjovskyIRM19,Ahmed21} train such that they generalize to all domains including the seen domains that is of interest in the sub-population shift problem. However, due to the nature of DG benchmarks, the DG algorithms are not evaluated on cases when the domain sizes are heavily disproportionate such as in the case of sub-population shift benchmarks (Table~\ref{tab:cgd:dataset:all}). We compare with two popular DG methods: CORAL~\citet{CORAL} and IRM~\citep{ArjovskyIRM19} on the WILDS benchmark, and obtain significantly better results than both.

\section{Discussion}
We present a simple, new algorithm {\cg} for training with domain annotations that models inter-domain interactions for minority domain generalization. 
We demonstrated the qualitative as well as empirical effectiveness of {\cg} over existing and relevant baselines through extensive evaluation using simple and real-world datasets. We also prove that {\cg} converges to first order stationary point of ERM objective even though it is not performing gradient descent on ERM. As part of future work, we could extend \cg\ to work in settings with noisy or unavailable domain annotations following recent work that extends {\dro} to  such settings. The merits of {\cg}, rooted in modeling inter-domain interactions, could be more generally applied to example weighting through modeling of inter-example interactions, which could train robust models without requiring domain annotations.

\subsection*{Subsequent work} 
\vspace{-20pt}
Identifying example groups in the train data on which standard training under-performs is a nontrivial task. Subpopulations could be naturally defined by grouping on example properties, such as portraits of male vs female, blond vs brunette, colored vs white, and are combinatorially large. Nevertheless, the subpopulations that require our attention due to divergent model behaviour may only be handful. 
A large chunk of subsequent work~\citep{Liu21,Creager21} is devoted to automatically identifying subpopulations with spurious correlations given a training dataset and a training algorithm. Few others~\citep{Sohoni20,Zhang22} proposed training methods for subpopulation shift robustness without any domain annotations.  




\chapter{Accuracy surfaces over Attribute spaces}
\label{chap:aaa}
\def\mlaas{MLaaS}
\def\match{\text{Agree}}
\newcommand{\service}{S}

\newcommand{\calFull}{Cal:Full}
\newcommand{\calRaw}{Cal:Raw}
\newcommand{\calGold}{Cal:Gold}
\newcommand{\calTemp}{Cal:Temp}
\newcommand{\mx}{$\sigma$}
\def\cX{\mathcal{X}}
\def\cY{\mathcal{Y}}
\def\hy{y}
\def\E{\mathbb{E}}
\def\bmean{{\phi}}
\def\bscale{{\psi}}

\newcommand{\dseed}{{D}}
\def\accgold{\rho}
\def\attspace{\mathcal{A}}
\newcommand{\arms}{$\mathcal{A}$}
\def\attsubspace{\bar{\attspace}}
\def\attpred{M}
\def\attlist{A}
\def\attrlist{\attlist}
\def\attarm{{\bm{a}}}
\def\accpred{\Pr}
\def\budget{B}
\def\xv{\bm{x}}
\def\acc{\rho}
\def\iacc{c}
\def\thresh{Thresh}

\def\Beta{\mathfrak{B}}
\newcommand{\f}{f}
\newcommand{\g}{g}
\def\fpost{\Pr(\f|\dseed)}
\def\gpost{\Pr(\g|\dseed)}
\def\bmeanpost{\bmean_\dseed}
\def\bscalepost{\bscale_\dseed}
\def\km{K_1}
\def\ks{K_2}
\def\Emb{\mathcal{V}}

\def\shortname{AAA}
\def\longname{Attributed Accuracy Assay}

\newcommand{\imdbmf}{MF-IMDB}
\newcommand{\celebamf}{MF-CelebA}

\def\PerArmBeta{Beta-I}  
\def\BetaOnly{Beta}  

\def\BernGP{BernGP}
\def\BetaGP{BetaGP}
\def\BetaAB{\mbox{BetaGP{$\alpha$}{$\beta$}}}
\def\cpred{CPredictor}
\def\DirGP{BetaGP-SL}
\def\DirGPR{BetaGP-SLP}

In the previous chapters, we discussed challenges in training models that are robust to domain-shifts. 
In this chapter, we introduce a new paradigm of evaluating deployed models. Our starting observation is that the ML model's performance varies under domain shift. Yet, deployed models follow the standard evaluation practice of reporting performance on benchmarks, which need not inform performance on diverse data regions of user's interest. The user, therefore, finds it difficult to choose the best model without extensive pilot trials~\citep{frugalml20}.
Different users may need to deploy the model on very different data distributions, with possibly widely different performance. Several recent studies have highlighted the variability in accuracy across data sub-populations \citep{Subbaswamy21eval,Sagawa19}, specifically in the context of fairness \citep{gendershades,modelcards19,JiSS20neurips}, and also proposed active estimation techniques of sub-population accuracy~\citep{JiLR20,Miller21}. Setting the performance expectations is crucial for building trust and addressing reliability challenge when deploying a model.

In this chapter, we discuss an evaluation approach suitable for models deployed in the wild. Our objective is to enable the end-user to query the performance on their data without any data acquisition effort. 
We propose that a model provider, or a model standardization agency, publish the accuracy of the classification model, not as one or few aggregate numbers, but as a \emph{surface} defined on a space of input instance \emph{attributes} that capture the variability of consumer expectations.  Indoor/outdoor, day/night, urban/rural may be attributes of input images for visual object recognition tasks. Speaker age, gender, ethnicity/accent may be attributes of input audio for speech recognition tasks.  We call a combination of attributes in their Cartesian space an \emph{arm} (borrowing from bandit terminology)\footnote{\figurename~\ref{fig:aaa:task:stats1c} shows an example of diverse accuracy over arms.}. 
The accuracy surface estimation is a trivial task if we have access to labeled data with sufficient support per arm.  However, assuming access to labeled data spanning combinatorially large number of arms is impractical.
We are interested in estimating the accuracy surface up to a certain level of confidence while minimizing the amount of data to be labeled. 
Estimation over arms that grow combinatorially inevitably leads to extreme sparsity of labeled instances for many arms. A central challenge is how to smooth the estimate across related arms while faithfully representing the uncertainty for active exploration. 

\noindent
{\bf Summary of our contributions}

In Section~\ref{sec:aaa:problem}, we describe the estimation problem. In Section~\ref{sec:aaa}, we present \longname{} (\shortname) --- a practical system that estimates accuracy, together with the uncertainty of the estimate, as a function of the attribute space.  \shortname{} uses these estimates to drive the sampling policy for each attribute combination. Gaussian Process (GP) regression is a natural choice to obtain smooth probabilistic accuracy estimates over arm attributes. However, a straightforward GP model fails to address the challenge of heteroscedasticity that we face with uneven and sparse supervision across arms.  We model arm-specific accuracy as a Beta density that is characterized by mean and scale parameters, which are sampled from two GPs that are informed by suitable trained kernels over the attribute space.  The second GP allows us to estimate the variance for each arm independently to handle heteroscedasticity. We propose two further enhancements to the training of this model.
First, we recognize an over-smoothing problem with GP's estimation of the Beta scale parameters, and propose  
a Dirichlet likelihood to supervise the relative values of scale across arms. Second, we recognize that arms with very low support interfere with learning the kernel parameters of the GPs.  We mitigate this by pooling observations across related arms.
With these fixes, \shortname{} achieves the best estimation performance among competitive alternatives.

Another practical challenge in our setting is that some attributes of instances are not known exactly. For example, attributes, such as camera shutter speed or speaker gender, may be explicitly provided as meta information attached with instances.  
But other attributes, such as indoor/outdoor, or speaker age, may have to be estimated noisily via another (attribute) classifier, because accurate human-based acquisition of attributes would be burdensome.   \shortname{} also tackles uncertain attribute inference.  Its attribute classifiers are trained on a small amount of labeled data and their error rates are modeled in a probabilistic framework.

In Section~\ref{sec:aaa:Expt}, we report on extensive experiments using several real data sets. Comparison with several estimators based on Bernoulli arm parameters, Beta densities per arm, and even simpler forms of GPs on the arm Beta distributions, shows that \shortname{} is superior at quickly cutting down arm accuracy uncertainty.

\noindent
Work in this chapter is based on~\citet{PiratlaAAA21}.

\section{Setup}
\label{sec:aaa:problem}
Our goal is to evaluate a given pretrained model $S$ used by a diverse set of consumers. 
The model $S:\cX\mapsto\cY$ could be any predictive model that, for an input instance $\vx \in \cX$, assigns an output label $\hat{y} \in \cY$, where $\cY$ is a discrete label space. Let $y$ denote the true label of $\vx$ and $\match(y,\hat{y})$ denote the match between the two labels.  For scalar classification, $\match(y,\hat{y})$ is binary \{0,1\}. For structured outputs, e.g., sequences, we could use measures like BLEU scores in [0,1].   Classifiers are routinely evaluated on their expected accuracy on a data distribution $\Pr(\mathcal{X}, \mathcal{Y})$: 
\begin{align}
\accgold = \E_{\Pr(\vx,y)}[\match(y, S(\vx)] \label{eq:OneGamma}
\end{align}
We propose to go beyond this single measure and define accuracy as a surface over a space of attributes of the input instances.  
Let $\attrlist$ denote a list of $K$ attributes that capture the variability of consumer expectation.  For instance, visual object recognition is affected by the background scene, and facial recognition is affected by demographic attributes.  We use $\attrlist(\vx) \in \attspace$ to denote the vector of values of attributes of input $\vx$ and $\attspace$ to denote the Cartesian product of the domains of all attributes.  An attribute could be discrete, e.g., the ethnicity of a speaker; Boolean, e.g., whether a scene is outdoors/indoors; or continuous, e.g., the age of the speaker in speech recognition. Some of the attributes of $\vx$, for example the camera settings of an image, may be known exactly, and others may only be available as a distribution $\attpred_k(a_k|\vx)$ for an attribute $a_k \in \attrlist$, obtained from a pre-trained probabilistic classifier.

Generalizing from a single global expected accuracy \eqref{eq:OneGamma}, we define the accuracy surface  $\accgold:\attspace \rightarrow [0,1]$ of a pretrained model $S$ at each attribute combination $\attarm \in \attspace$, given a data distribution $\Pr(\cX,\cY)$, as
\begin{equation}
\accgold(\attarm) = \E_{\Pr(\vx,y|A(\vx)=\attarm)}[\match(y, S(\vx)]
\end{equation}
Our goal is to provide an estimate of $\accgold(\attarm)$ given two kind of  data sampled from $\Pr(\cX,\cY)$: a 
small labeled sample $\dseed$, and a large unlabeled sample~$U$.  
In addition, we are given a budget of $\budget$ instances for which we can seek labels $y$ from a human by selecting them from~$U$.  Applying $M_k$ to all of $U$ is, however, free of cost.

We aim to design a probabilistic estimator for~$\accgold(\attarm)$, which we denote as $\accpred(\acc|\attarm)$ where $\acc \in [0,1]$ and $\attarm \in \attspace$.  This is distinct from active learning, which selects instances to train the learner toward greater accuracy, and also active accuracy estimation~\citep{JiLR20}, which does not involve a surface over the~$\attarm$s.  We also show that standard tools to regress from $\attarm$ to $\accgold$ are worse than our proposal.

We measure the quality of our estimate as the square error between the gold accuracy $\accgold(\attarm)$ and the mean of the estimated accuracy distribution $\accpred(\acc|\attarm)$.  Our estimator distribution naturally gives an idea of the posterior variance of accuracy estimate of each attribute combination, which we use for uncertainty-based exploration.  

The overall setup of the estimator pipeline is shown in Figure~\ref{fig:aaa:setup:a}. The accuracy surface estimating agent takes as input the four resources shown in orange: (a) human specified interpretable attributes $\attrlist$ (b) attribute classifiers $\{M_k, k\in \attrlist\}$ (c) (un)labeled data $U, \dseed$  (d) access to model/service under scrutiny $S$. At any moment, the estimator picks an example and seeks human feedback of its correctness $\text{Agree}(S(\vx), y)$. Finally, upon reaching the desired confidence level, the accuracy surface is emitted. Figure~\ref{fig:aaa:setup:b} shows the estimation pipeline. The example~$\vx$ and its observation (1/0---correctly predicted or not) are mapped to its corresponding arm $A(\vx)$, which is discussed in Section~\ref{sec:aaa:AttribNoise}. The accuracy surface table is then updated with the arm, observation pair as presented broadly in Section~\ref{sec:aaa}. We then pick the most uncertain arm, which further picks the next example $\vx^\prime$ as explained in Section~\ref{sec:aaa:Explore}.

\begin{figure}
    \centering
    \begin{subfigure}[b]{0.48\textwidth}
        \includegraphics[width=\textwidth]{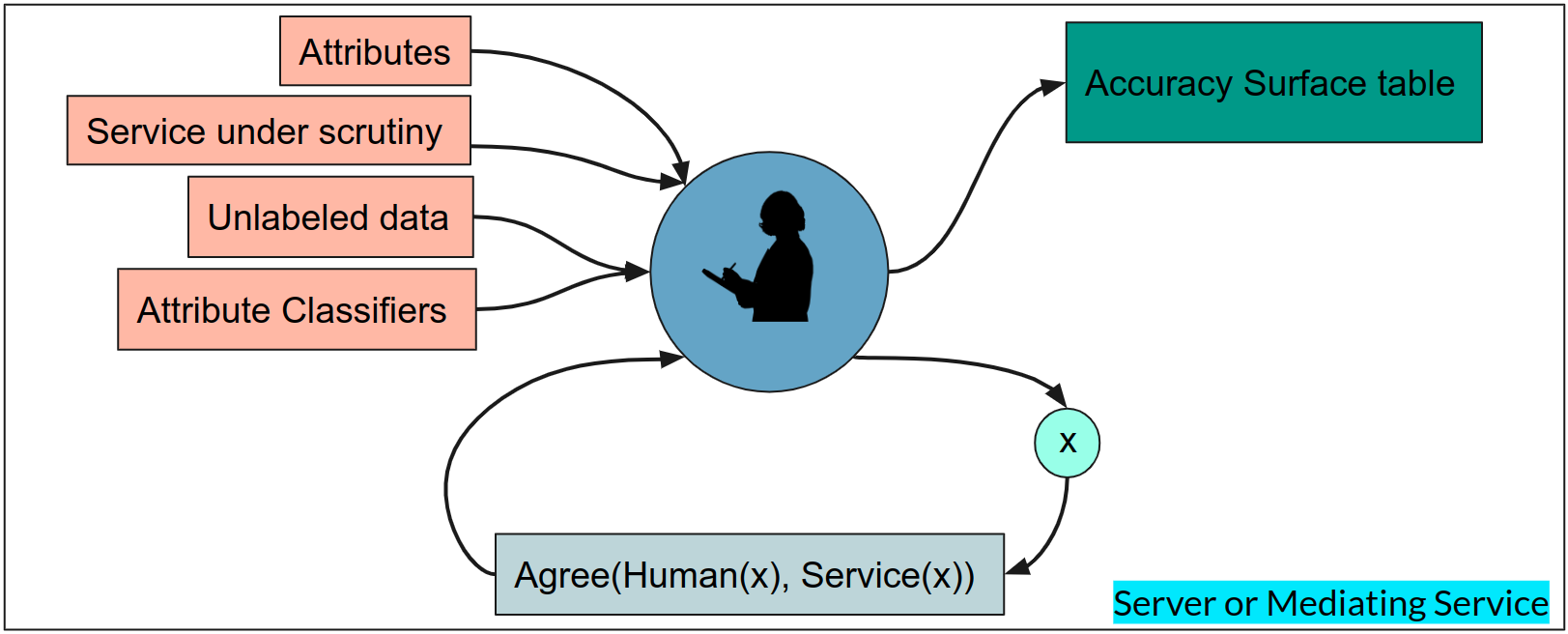}
        \caption{Overall setup.}
        \label{fig:aaa:setup:a}
    \end{subfigure}
    \hfill
    \begin{subfigure}[b]{0.48\textwidth}
        \includegraphics[width=\textwidth]{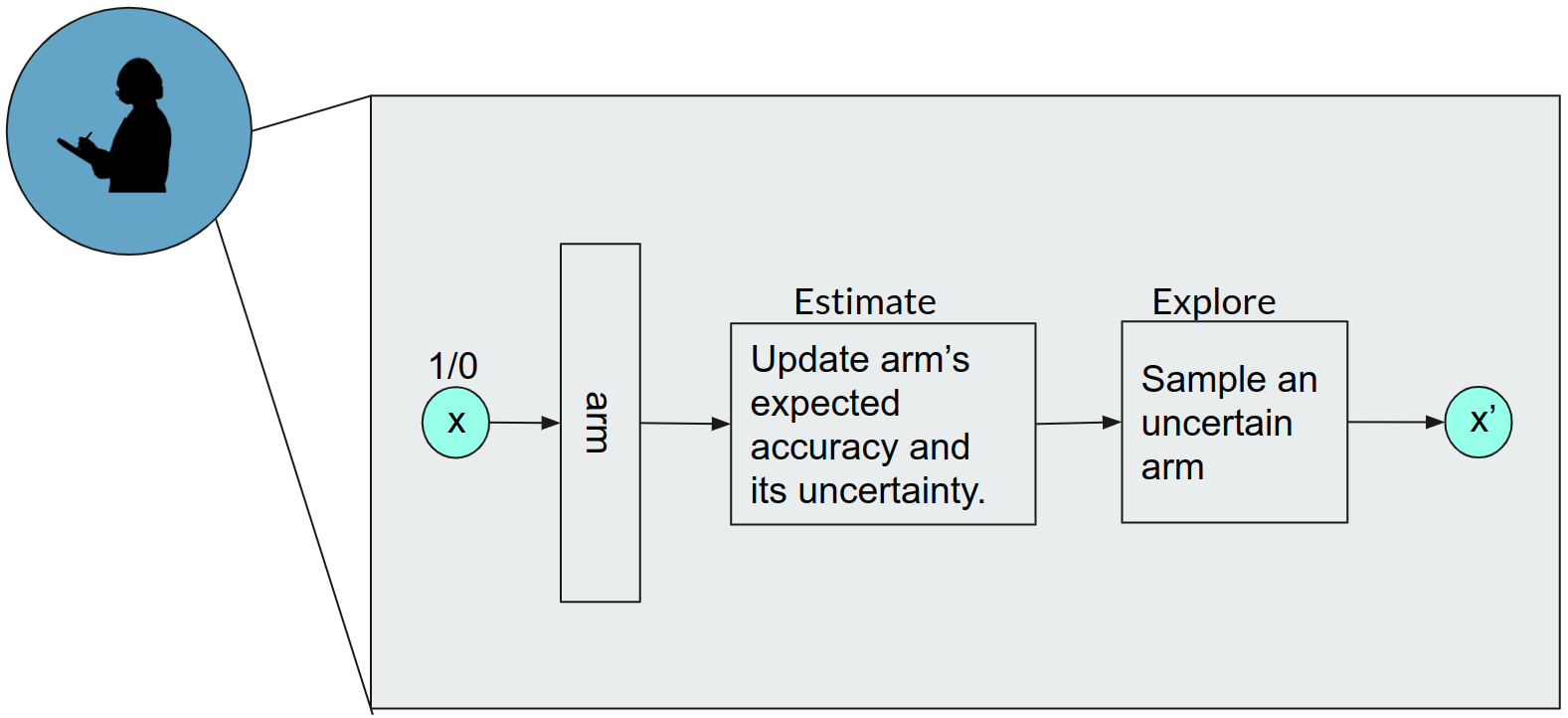}
        \caption{Estimation pipeline. }
        \label{fig:aaa:setup:b}
    \end{subfigure}
    \caption{See text for description.}
    \label{fig:aaa:setup}
\end{figure}

\section{Proposed Estimator}
\label{sec:aaa}

We will first review recent work that leads to candidate solutions to our problem, discuss their limitations, and finally present our solution.
Initially, to keep the treatment simple, we assume arm mapping for an example $\vx$: $\attlist(\vx)\in \attspace$ and gold $y$ (hence $c=\match(S(x), y)$, the model correctness bit)
is known for all instances.  Later in this section, we remove these assumptions.

The simplest option is to ignore any relationship between arms, and, for each arm $\attarm$, fit a suitable density over $\accgold(\attarm)$.  When this density is sampled, we get a number in $[0,1]$, which is like a coin head probability used to sample correctness bits~$c$.
For representing uncertainty of accuracy values (which are ratios between two counts), the \href{https://en.wikipedia.org/wiki/Beta_distribution}{Beta distribution} $\Beta(\cdot,\cdot)$ is a natural choice.  
We call this baseline method \textbf{\PerArmBeta}.

The variance of the estimated Beta density can be used for actively sampling arms.  \citet{JiLR20} describe a related scenario, stressing on active sampling.  However, this approach cannot share observations or smooth the estimated density at a sparsely-populated arm with information from similar arms.  In our real-life scenario, we expect accuracy surface smoother and the number of arms to be large enough that many arms will get very few, if any, instances.

The second baseline method, which we call \textbf{\BernGP}, is to view the $(\attarm,\iacc)$ instances in $\dseed$ as a standard classification data set with the binary $\iacc$ values as class label and $\attarm$ as input features.  Given the limited data, we can use the well-known GP classification approach \citep{HensmanMG15} for fitting smooth values $\acc$ as a function of~$\attarm$. %
Suppose the arms $\attarm$ can be embedded to $\Emb(\attarm)$ in a suitable space induced by some similarity kernel.  In this embedding space, we expect the accuracy of $S$ to vary smoothly.
Given a kernel $K_1(\attarm, \attarm')$ to guide the extent of sharing of information across arms, a standard form of this GP would be
\begin{align}
\Pr(\iacc|\attarm) &= 
\text{Bernoulli}(\iacc; \operatorname{sigmoid}(f_\attarm));
\quad f \sim GP(0, K_1).
\label{eq:bernLL}\\
\E(\rho|\attarm) &= \E_{f_\attarm\sim GP(0, K_1)} 
\left[ \textstyle \operatorname{sigmoid}(f_\attarm)
       \right] \nonumber\\
\mathbb{V}(\rho|\attarm) &= \E_{f_\attarm\sim GP(0, K_1)} 
\left[ \textstyle  (\operatorname{sigmoid}(f_\attarm)-\E(\rho|\attarm))^2
      \right]
\label{eqn:bernLL:var}
\end{align}
Where $f_\attarm, g_\attarm$ denotes the mean and variance at $\attarm$ as modeled by GP model. The GP can give estimates of uncertainty of $\accgold(\attarm)$ (\eqref{eqn:bernLL:var}), which may be used for active sampling of arms.

As we will demonstrate, such GP-imposed estimate of uncertainty of $\accgold(\attarm)$ is inadequate, because it loses sight of the number of supporting observations at each arm, which could be very diverse.
This is because the standard GP assumption of homoscedasticity, that is, identical noise around each arm is violated when observations per arm differ significantly.  We therefore need a mechanism to separately account for the uncertainty at each arm, even the unexplored ones, to guide the strategy for actively collecting more labeled data. 

\subsection{Heteroscedastic Modeling}
\label{sec:aaa:betagp}
We model arm-specific noise by allowing each arm to represent the uncertainty of $\rho_a$, not just by an underlying GP as in \BernGP\ above, but also by a separate scale parameter.  Further, the scale parameter is smoothed over neighboring arms using another GP. The influence of this scale on the uncertainty of $\rho_a$ is expressed by a Beta distribution as follows: 
\begin{align}
\accpred(\acc|\attarm) &\sim \Beta(\acc; \bmean(\f_\attarm), \bscale(\g_\attarm)) \label{eq:model}\\
\bmean(\f_\attarm) &= \operatorname{sigmoid}(\f_\attarm),
\qquad \f \sim GP(0, \km),  \label{eq:priorm}  \\
\bscale(\g_\attarm) &= \log(1+e^{\g_\attarm}),
\qquad \g \sim GP(0, \ks), \label{eq:priors}
\end{align}
where we use $\bmean(f_\attarm), \bscale(g_\attarm)$ to denote the parameters of the Beta distribution at arm $\attarm$.
The Beta distribution is commonly represented via $\alpha,\beta$ parameters whereas we chose the less popular mean ($\bmean$) and scale ($\bscale$) parameters.  While these two forms are functionally equivalent with $\bmean = \frac{\alpha}{\alpha+\beta}, \bscale=\alpha+\beta$, we preferred the second form because imposing GP smoothness across arms on the mean accuracy and scale seemed more meaningful.  We validate this empirically in the Appendix~\ref{sec:aaa:appendix:betaab}.  

Two kernel functions $\km(\attarm,\attarm')$, $\ks(\attarm,\attarm')$ defined over pairs of
arms $\attarm, \attarm' \in \attspace$ 
control the degree of smoothness among the Beta parameters across the arms.  We use an RBF kernel defined over learned shared embeddings $\Emb(\attarm)$:
\begin{align}
    \km(\attarm, \attarm') = s_1\exp\left[-\tfrac{\|\Emb(\attarm) - \Emb(\attarm')\|^2}{l_1}\right],
    \qquad
     \ks(\attarm, \attarm') = s_2\exp\left[-\tfrac{\|\Emb(\attarm) - \Emb(\attarm')\|^2}{l_2}\right]
    \label{eq:kernel}
\end{align}
where $s_1,s_2,l_1,l_2$ denote the scale and length parameters of the two kernels.  The scale and length parameters are learned along with the parameters of embeddings $\Emb(\bullet)$ during training.

Initially, we assume we are given a labeled dataset $\dseed = \{(\vx_i, \attarm_i, y_i):i=1\ldots,I\}$ with attribute information available.  Using predictions from the pretrained model $\service$, we associate a 0/1 accuracy $\iacc_i=\match(y_i, \service(\vx_i))$.  We can thus extend $\dseed$ to $\{(\vx_i, \attarm_i, y_i, c_i): i\in[I]\}$.

Let $\iacc_\attarm=\sum_{i: A(\vx_i)=\attarm} \iacc_i$ denote the total agree score in arm $\attarm$.  Let $n_\attarm$ denote the total number of labeled examples in arm~$\attarm$. 
The likelihood of all observations given functions $\f,\g$ decomposes as a product of Beta-binomial\footnote{The $\binom{n_a}{c_a}$ term does not apply since we are given not just counts but accuracy $c_i$ of individual points.} distributions at each arm as follows:
\begin{align}
\Pr(\dseed|\f,\g) &= \prod_\attarm 
\int_\rho \rho^{\iacc_a}(1-\rho)^{n_\attarm-\iacc_\attarm}\, \Beta(\rho|\bmean(\f_{\attarm}), \bscale(\g_{\attarm}))) \text{d}\rho.
\label{eq:pointLL} \\
&=  \prod_\attarm  \frac{\text{B}(\bmean(\f_{\attarm})\bscale(\g_\attarm) + \iacc_a,   (1-\bmean(\f_{\attarm}))\bscale(\g_\attarm)  + n_\attarm -\iacc_a) }{
\text{B}(\bmean(\f_{\attarm})\bscale(\g_\attarm),   (1-\bmean(\f_{\attarm}))\bscale(\g_\attarm))
},
\label{eq:pointLLSimple}
\end{align}
where \text{B} is the Beta function, and the second expression is a rewrite of the \href{https://en.wikipedia.org/wiki/Beta-binomial_distribution}{Beta-binomial likelihood}.

During training we calculate the posterior distribution of functions $f,g$ using the above data likelihood $\Pr(D|f,g)$ and GP priors given in~\eqref{eq:priorm} and~\eqref{eq:priors}.  The posterior cannot be computed analytically given our likelihood, so we use variational methods. Further, we reduce the $\mathcal{O}(|\attspace|^3)$ complexity of posterior computation, using the inducing point method of \citet{HensmanMG15}, where we approximate posterior distribution with $m$ inducing points. Thereby, jointly learning $m$ locations $\mathbf{u} \in \mathbb{R}^{d\times m}$, mean $\mu \in \mathbb{R}^m$, and covariance $\Sigma \in \mathbb{R}^{m\times m}$ of inducing points. Doing so brings down the complexity to $\mathcal{O}(m^2|\attspace|), m \ll |\attspace|$. These parameters are learned end to end with the parameters of the neural network used to extract embeddings $\Emb(\bullet)$, and kernel parameters $s_1,s_2,\ell_1,\ell_2$. 
We used off-the-shelf Gaussian process library: GPyTorch~\citep{gpytorch} to train the above likelihood with variational methods.
We provide further background on GPs and implementation details in Appendix~\ref{sec:aaa:appendix:gp}. We denote the posterior functions as $\fpost, \gpost$. 
Thereafter, the mean estimated accuracy for an arm $\attarm$ is computed as
\begin{equation}
\label{eq:our:est}
\E(\rho|\attarm) = \E_{\f\sim\fpost} [\bmean(\f_\attarm)].
\end{equation}
We call this setup \textbf{\BetaGP}.
Next, we will argue why \BetaGP{} still has serious limitations, and offer mitigation measures. The limitations are in part due to the well-known underestimation of posterior uncertainty when using variational inference methods to approximate the posterior~\citep{murphyml}(chap. 21). The zero-forcing nature of the order of KL divergence in variational inference causes the underestimation problem. Although there is some work addressing the underestimation problem of variational inference, we could mitigate the issue with simple additional supervision/regularization that we describe next. 




\subsection{Supervision for scale parameters}
\label{sec:aaa:sl}

We had introduced the second GP $g_\attarm$ to model arm-specific noise, and similar techniques have been proposed earlier by \citet{LzaroGredilla2011VariationalHG, Kersting2007MostLH,Goldberg1997}, but for heteroscedasticity in Gaussian observations.  
However, we found the posterior distribution of scale values $\psi(g_\attarm)$ at each arm tended to converge to similar values, even across arms with orders of magnitude difference in number of observations $n_\attarm$.  On hindsight, that was to be expected, because the data likelihood~\eqref{eq:pointLL} increases monotonically with scale $\psi_\attarm$. The only control over its converging to $\infty$ is the GP prior $g \sim GP(0,K_2)$. 
In Section~\ref{sec:aaa:appendix:synth}, we illustrate this phenomenon with an example. 
%
%
We propose a simple fix to the scale supervision problem. We expect the relative values of scale across arms to reflect the distribution of the proportion of observations $\frac{n_a}{n}$ across arms (with $n=\sum_\attarm n_\attarm$).  We impose a joint Dirichlet distribution using the scale of arms $\psi(g_\attarm)$ as parameters, and write the likelihood of the observed proportions as (with $\Gamma$ denoting \href{https://en.wikipedia.org/wiki/Gamma_function}{Gamma function}):
\begin{equation}
\log\Pr(\{n_\attarm\}|g) = \sum_\attarm ((\psi(g_\attarm)-1) \log \frac{n_a}{n} - \log \Gamma(\psi(g_\attarm)) + \log\Gamma(\textstyle\sum_\attarm \psi(g_\attarm))
\end{equation}
We call this \textbf{\DirGP}. With this as an additional term in the data likelihood, we obtained significantly improved uncertainty estimates at each arm, as we will show in the experiment section. 


\subsection{Pooling for sparse observations}
\label{sec:aaa:pool}

Recall that the observations are accumulation of 1/0 agreement scores for all instances that belong to an arm. Given the nature of our problem, arms have varying levels of supervision. 
Even when the available labeled data is large, many arms will continue to have sparse supervision because they represent rare attribute combinations.
%
The combination of high variance observations and sparse supervision could lead to learning of non-smooth kernel parameters. 
In Section~\ref{sec:aaa:appendix:synth}, we demonstrate with a simple setting that GP parameters under-represent the smoothness of the surface due to overfitting on noisy observations.
The situation is further aggravated when learning a deep kernel. 
This problem has resemblance to ``collapsing variance problem''~\citep{murphyml}(chap. 11) such as when Gaussian mixture models overfit on outliers or when topic models overfit a noisy document in the corpus.
Instead of depending purely on GP priors to smooth over these noisy observations, we found it helpful to also externally smooth noisy observations.  For each arm $\attarm$ with observations below a threshold, we mean-pool observations from some number of nearest neighbors, weighted by their kernel similarity with~$\attarm$.  We will see that such external smoothing resulted in significantly more accurate estimates particularly for arms with extreme accuracy values.  We call this method \textbf{\DirGPR} (note that this also includes the scale supervision objective described in the previous section).
Two other mechanisms take us to the full form of the \textbf{\shortname} system, which we describe next.

\subsection{Exploration}
\label{sec:aaa:Explore}
The variance estimate of an arm informs its uncertainty and is commonly used for efficient exploration~\citep{Schulz18tutorial}. 
Let $\fpost,\gpost$ denote the learned posterior distribution of the GPs.  Using these, the estimated variance at an arm is given as:
\begin{align}
\mathbb{V}(\rho|\attarm) = \E_{\f\sim\fpost,\g\sim\gpost} 
\left[ \textstyle \int_\rho (\rho-\E(\rho|\attarm))^2
      \Beta(\rho; \bmean(\f_\attarm ), \bscale(\g_\attarm ))
      \text{d}\rho \right]
\label{eqn:mean_variance}
\end{align}
where the expected value is given in~\eqref{eq:our:est}.  We use sampling to estimate the above expectation.  The arm to be sampled next is chosen as the one with the highest variance among unexplored arms. We then sample an unexplored example with highest affiliation ($\Pr(\va\mid\vx)$) with the chosen arm. 





\subsection{Modeling Attribute Uncertainty}
\label{sec:aaa:AttribNoise}
Recall that attributes of an instance $\vx$ are obtained from  models $M_k(a_k|\vx),~~k \in [K]$, which may be highly noisy for some attributes. Thus, we cannot assume a fixed attribute vector $A(\vx)$ for an instance $\vx$.  We address this by designing a model that can combine these noisy estimates into a joint distribution $\Pr(\va|\vx)$ using which, we can fractionally assign each instance $\vx_i$ across arms.   A baseline model for $\Pr(\va|\vx)$ would be just the product $\prod_{k=1}^K M_k(a_k|\vx)$. However, 
we expect values of attributes to be correlated (e.g. attribute `high-pitch' is likely to be correlated with gender `female'). Also, the probabilities  $M_k(a_k|\vx)$ may not be well-calibrated.  

We therefore propose an alternative joint model that can both recalibrate individual classifiers via temperature scaling~\citep{GuoPSW17}, and model their correlation.   We have a small seed labeled dataset $\dseed$ with gold attribute labels, independent noisy distributions from each attribute model $M_k(a_k|\vx)$, and an unlabeled dataset $U$. We prefer simple factorized models.
We factorize $\log\Pr(\va|\vx)$ as a sum of temperature-weighted logits and a joint (log) potential as shown in expression~\eqref{eqn:factor:assume} below. 
\begin{align}
\Pr(\va|\vx) = \Pr(a_1, a_2, \cdots ,a_K|\vx) \propto N(a_1, a_2, \cdots, a_K)\prod_{k=1}^K t_k M_k(a_k|\vx)
 \label{eqn:factor:assume}
\end{align}
Here $N$ denotes a dense network to model 
the correlation between attributes, and  $t_1,\ldots,t_K$ denote the temperature parameters used to rescale noisy attribute probabilities.
%
The maximum likelihood over $\dseed$ is
$\max_{t,N} \sum_{(\vx_i,\attarm_i) \in \dseed} \log \Pr(\va_i|\vx_i)$ 
%
%
\begin{align}
=\max_{t,N} \sum_{(\vx_i, \attarm_i) \in \dseed} \big\{ \textstyle
\sum_{k=1}^K t_k \log M_{k}(a_{ik}|\vx_i)+\log N(a_{i1}, \ldots a_{iK}) - \log (Z_i) \big\}    
  \label{eqn:max_likelihood}
\end{align}
$Z_i$ denotes the partition function for an example $\vx_i$ which requires summation over~$\attspace$.  Exact computation of $Z_i$ could be intractable especially when \arms\ is large. In such cases, $Z_i$ can be approximated by sampling. In our experiments, we could get exact estimates since |{\arms}| is not too large.

In addition to $\dseed$, we use the unlabeled instances~$U$ with predictions from attribute predictors filling the role of gold-attributes. Details on how we train the parameters on large but noisy $U$ and small but correct $\dseed$ are described in the next section. 

\subsection*{Calibration Training Details}


We use both labeled $\dseed$ and unlabeled $U$ for training calibration parameters that are expressed in the objective~\eqref{eqn:max_likelihood}.
On the unlabeled data $U$, we use attribute values predicted using the predictors: $\{\attpred_k \mid k \in \attlist\}$ as a proxy for true values. The use of predicted values as the replacement for true value under-represents the attribute prediction error rate and interferes in the estimation of temperature parameters $t$. However, if we use $U$ for training, we see a lot more attribute combinations and this can help identify more natural attribute combinations aiding in the learning of joint potential parameters of $N$. 
We mitigate this problem by up-sampling instances in $\dseed$ such that the loss in every batch contains equal contribution from $\dseed$ and~$U$. 
 
Recall that the MLE objective~\eqref{eqn:max_likelihood}, contains contribution from two terms: (a)~temperature scaled logits (b)~attribute combination potential. In practice, we found that the second term (b) dominates the first, this causes under-training of the temperature parameters. Ideally, the two terms should be comparable and replaceable. We address this issue by dropping the second term corresponding to the network-assigned edge potential term in the objective half the times, which estimates better the temperature parameters. Further, we use a small held out fraction of $\dseed$ for network architecture search on~$N$, and for early stopping.

\noindent
{\bf Pseudo-code.} The calibration training procedure is summarized in Alg.~\ref{alg:calibrate}.
The estimation method of \DirGPR{} with variance based exploration and calibration described here constitute our proposed estimator: \shortname{}. 

\begin{algorithm}[htb]
\caption{Attribute Model Calibration}
\begin{algorithmic}[1]
\Require $\dseed , U, \{\attpred_k, k\in \attlist\}, \eta$
\State Initialize $t, N$
\State converged = False, $\tau$ = $10^{-3}$
\While{not converged}
\State d, u$^\prime$ $\leftarrow$ batch(D), batch(U) \Comment{sample a subset for batch processing}
\State u = \{($\vx$, \{$M_k(\vx)\mid k\in A$\}) for $\vx$ in u$^\prime$\}
\State LL = \eqref{eqn:max_likelihood} using d, u
\State t, N = optimizer-update($\eta, \nabla_t LL, \nabla_N LL$)
\State converged = True if LL < $\tau$
\EndWhile
\State \Return $t, N$
\end{algorithmic}
\label{alg:calibrate}
\end{algorithm}

\section{Simple Setting}
\label{sec:aaa:appendix:synth}

In Section~\ref{sec:aaa:sl}, we describe how the objective of \BetaGP{} does not supervise the scale parameter. Further, in Section~\ref{sec:aaa:pool}, we posit that the presence of sparse observations leads to learning a non-smooth kernel. In this section, we illustrate these two observations using a simple setting. 

We consider a simple estimation problem with 10 arms, their true accuracies go from 0.1 to a large value of 0.9 and then back to a small value of 0.2 as shown in the Table~\ref{tab:aaa:simple:fit}. In Table~\ref{tab:aaa:simple:fit}, we also show the number of observations per arm; observe that the first three and the last three arms are sparsely observed.

\begin{table}[th]
    \centering
    \setlength{\tabcolsep}{4pt}
    \begin{tabular}{|l|r|r|r|r|r|r|r|r|r|r|}
        \hline
         \thead{Arm Index} &  \thead{1} & \thead{2} & \thead{3} & \thead{4} & \thead{5} & \thead{6} & \thead{7} & \thead{8} & \thead{9} & \thead{10}\\\hline
         Accuracy & 0.1 & 0.3 & 0.5 & 0.7 & 0.8 & 0.9 & 0.6 & 0.4 & 0.3 & 0.2 \\\hline
         N & 1 & 1 & 1 & 20 & 20 & 20 & 20 & 1 & 1 & 1 \\\hline
         \multicolumn{11}{c}{Estimated Scale Value}\\\hline
         \BetaGP{}  & 10.33 & 10.88 & 11.37 & 11.73 & 11.92 & 11.91 & 11.72 & 11.34 & 10.85 & 10.29 \\
         \DirGP{} &  1.49 & 1.42 & 1.59 & 9.64 & 10.40 & 10.42 & 9.58 & 1.59 & 1.42 & 1.49\\ 
         \DirGPR{} &  1.72 & 1.63 & 1.92 & 9.63 & 10.38 & 10.27 & 9.48 & 1.93 & 1.53 & 1.62\\\hline
    \end{tabular}
    \caption{Arms, their indices, accuracies and number of observations (N) in the simple setting are shown in first three columns in that order. Estimated scale parameter for each algorithm is shown in the last three columns, for each arm. Notice that \BetaGP{} fitted scale parameter does not reflect the underlying observation sparsity for the first and last three arms.}
    \label{tab:aaa:simple:fit}
\end{table}

We now present the fitted values by some of the methods we discussed in the previous section. The index of an arm is the input for any estimator with no feature learning. Toward these ends, we evaluate \BetaGP{},~\DirGP{},~\DirGPR{} methods on this setting. The fitted scale parameters for each arm by each of the estimators is shown in 
Table~\ref{tab:aaa:simple:fit}. Observe that \BetaGP{} fitted scale parameter does not reflect the underlying observation sparsity of the first, last three arms. 
\DirGP{},~\DirGPR{} fitted scale values more faithfully reflect the underlying number of observations. All the numbers reported here are averaged over 20 seed runs. 

\begin{table}[htb]
    \centering
    \begin{tabular}{|l|r|r|r|}
    \hline
         Method & Bias$^2$ & Variance & MSE \\\hline
         \BetaGP{} & 0.052 & 1.011 & 1.063 \\
         \DirGP{} & 0.051 & 1.010 & 1.061 \\
         \DirGPR{} & 0.095 & 0.221 & {\bf 0.316} \\\hline
    \end{tabular}
    \caption{Bias-variance decomposition of MSE in the simple setting. Code can be found at this \href{https://github.com/vihari/AAA/blob/main/notebooks/ToyGP.ipynb}{link}.}
    \label{tab:aaa:simple:bv}
\end{table}
In Table~\ref{tab:aaa:simple:bv}, we show the bias$^2$, variance decomposition of the mean squared error from 20 independent runs. Observe that ~\BetaGP{},~\DirGP{} have low bias but large variance and \DirGPR{} has much lower variance at a slight expense of bias, as a result the overall MSE value for \DirGPR{} is much lower than the other two. Moreover, we look at the fitted kernel length parameter (recall from~\eqref{eq:kernel}) as a proxy for smoothness of the fitted kernel. The average kernel length for \BetaGP{},~\DirGP{},~\DirGPR{} are 0.67, 0.68, 1.87 respectively. \DirGPR{} imposes long range smoothness, as a result, decreases the MSE value more effectively when compared with \DirGP{}.

\section{Experiments}
\label{sec:aaa:Expt}

Our exploration of various methods and data sets is guided by the following research questions.
\begin{itemize}
\item How do various methods for arm accuracy estimation compare?
\item To what extent do \BetaGP, scale supervision and pooled observations help beyond \BernGP?
\item For the best techniques from above, how 
do various active exploration strategies compare?
\item How well does our proposed model of attribute uncertainty work?
\end{itemize}

\paragraph{\large Data sets and tasks:} We experiment with two real data sets and tasks.  Our two tasks are male-female gender classification with two classes and animal classification with 10 classes.

\paragraph{Male-Female classification (MF):} 
CelebA \citep{CelebA} is a popular celebrity faces and attribute data set that identifies the gender of celebrities among 39 other binary attributes.  The label is gender.  The accuracy surface spans various demographic, style, and personality related attributes.  We hand-pick a subset of 12 attributes that include attributes that we deem important for gender classification among some other gender-neural attributes such as if the subject is young or wearing glasses (see Appendix~\ref{sec:aaa:appendix:task} for more details). We used a random subset of 50,000 examples from the dataset for training classifiers on each of the 12 attributes using a pretrained ResNet-50 model. The remaining 150,000 examples in the data set are set as the unlabeled pool from which we actively explore new examples for human feedback. The twelve binary attributes make up for $2^{12}=4,096$ attribute combinations.

\paragraph{Animal classification (AC):} 
COCO-Stuff \citep{COCOS} provides an image collection.  For each image, labels for foreground (cow, camel) and background (sky, snow, water) `stuff' are available.  Visual recognition models often correlate the background scene with the animal label such as camel with deserts and cow with meadows.  Thus, foreground labels are our regular $y$-labels while background stuff labels supply our notion of attributes.

We collapse fine background labels into five coarse labels using the dataset provided label hierarchy. These are: water, ground, sky, structure, furniture.  The Coco dataset has around 90 object (foreground) labels.  Here we use a subset of 10 labels corresponding to animals. We take special care to filter out images with multiple/no animals and adapt the pixel segmentation/classification task to object classification (see the Appendix~\ref{sec:aaa:appendix:task} for more details). The image is further annotated with the five binary labels corresponding to five coarse stuff labels. The scene descriptive five binary labels and ten object labels make up for $32{\times}10=320$ attribute combinations.  

\paragraph{\large Pretrained Models:}

For the MF task, we use two pretrained models~($S$).
\textbf{MF-CelebA} is a gender classification model trained on CelebA. To simulate separate $\dseed$ and $U$, it is trained on a random subset of CelebA with a ResNet50 model.  
\textbf{MF-IMDB} is a publicly available\footnote{\url{https://github.com/yu4u/age-gender-estimation}} classifier trained on IMBD-Wiki dataset, also using the ResNet50 architecture.  The attribute predictors are trained using ResNet50 on a subset of the CelebA dataset for both the models.

\begin{figure}[htb]
    \centering
    \includegraphics[width=0.7\linewidth]{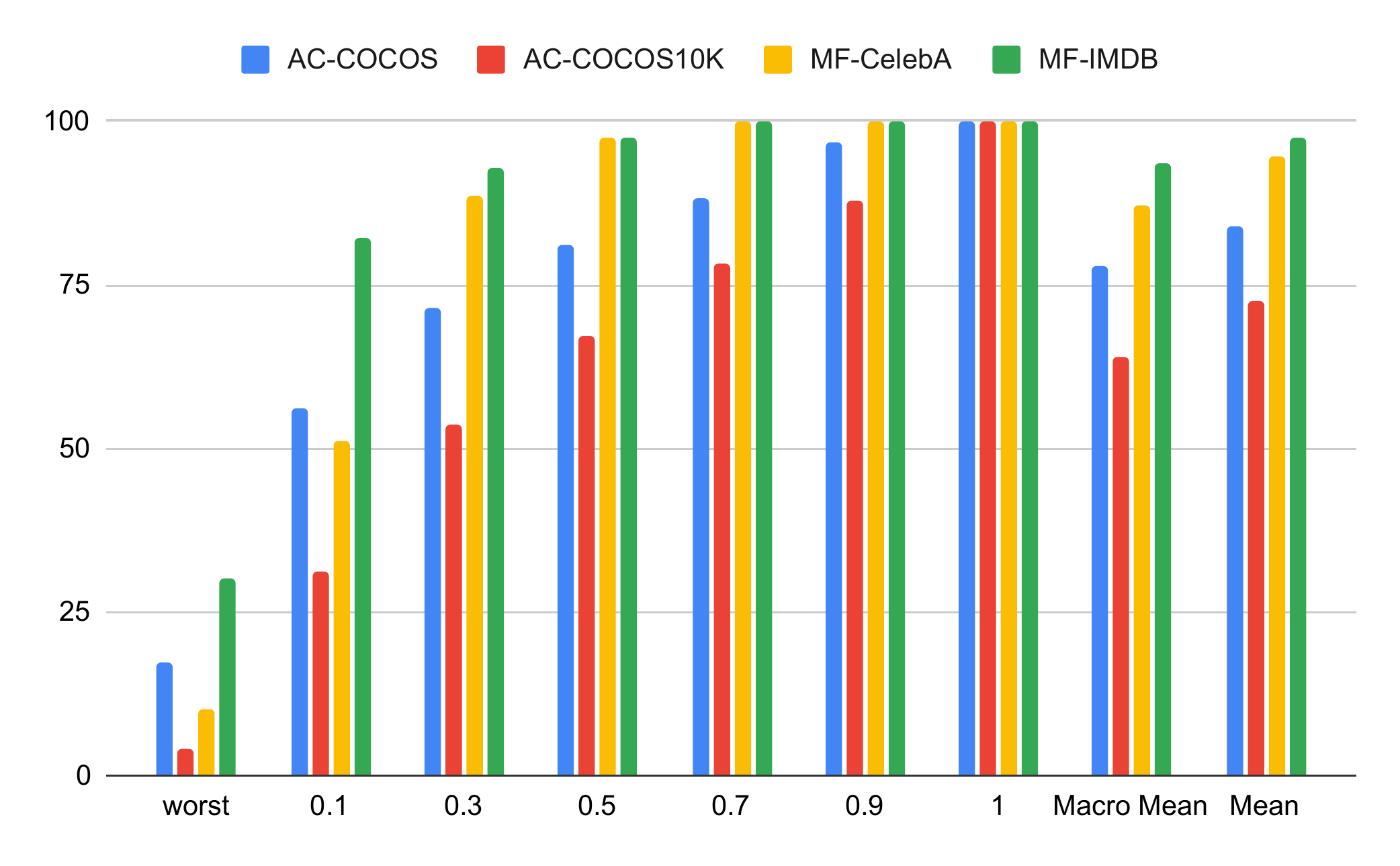}
    \caption{Macro and micro averaged accuracy (right most) and ten quantiles (x-axis) of per-arm accuracy (y-axis). Reporting only standard micro-averaged accuracy hides critical failure on 10\% of the worst arms or data regions.}
    \label{fig:aaa:task:stats1c}
\end{figure}

\begin{table}[htb]
\centering
\begin{tabular}{|l|c|c|c|c|} \hline
ML Models$\rightarrow$ & {\bf AC-COCOS10K} & {\bf AC-COCOS} & {\bf MF-IMDB} & {\bf MF-CelebA}\\ \hline
\cpred & 5.4 / 15.0 & 3.2 / 9.4 & 1.2 / 8.2 & 5.2 / 35.9 \\
\PerArmBeta & 7.0 / 15.6 & 4.3 / 10.0 & 1.6 / 8.4 & 4.7 / 30.3\\
\BernGP & 7.0 / 13.2 & 3.5 / 8.6 & 1.7 / 7.6 & 4.9 / 28.1\\
\BetaGP & 7.1 / 14.3 & 3.3 / 7.9 & 2.2 / 6.6 & 4.6 / 25.9\\
\DirGP & 5.3 / 11.7 & 2.8 / 6.8 & 1.4 / 4.4 & 4.1 / 22.6 \\
\rowcolor{green!10}
\DirGPR & 4.7 / 10.4 & 2.8 / 5.7 & 1.4 / 3.9 & 4.3 / 23.3\\\hline
 \end{tabular}
\caption{Comparing different estimation methods on labeled data  size 2000 across four tasks. No exploration is involved.  Each cell shows two numbers in the format ``macro MSE / worst MSE'' averaged over three runs.  \DirGPR{} generally gives the lowest MSE.}
\label{tab:aaa:est}
\end{table}

For the AC task, we use two publicly available\footnote{\url{https://github.com/kazuto1011/deeplab-pytorch/}} pretrained models~($S$).
\textbf{AC-COCOS} was trained on COCOS data set with 164K examples.
\textbf{AC-COCOS10k} was trained on COCOS10K, an earlier version of COCOS with only 10K instances. We use these architectures for both label and attribute prediction.
%
See Appendix~\ref{sec:aaa:appendix:task},~\ref{sec:aaa:appendix:surface_stats} for more details on accuracy surface, attribute predictor, pretrained models and their architecture. 
In Figure~\ref{fig:aaa:task:stats1c}, we illustrate some statistics of the shape of the accuracy surface for the four dataset-task combinations. Although $S$'s mean accuracy (rightmost bars) is reasonably high, the accuracy of the arms in the 10\% quantile is abysmally low, while arms in the top quantiles have near perfect accuracy. This further motivates the need for an accuracy surface instead of single accuracy estimate.  

\noindent
{\bf \large Methods Compared}\\
We compare the proposed estimation method \shortname{} against natural baselines, alternatives, and ablations.
Some of the methods, such as \textbf{\PerArmBeta}, \textbf{\BernGP} and \textbf{\BetaGP}, we have already defined in Section~\ref{sec:aaa}.  
We train methods \BernGP{} and \BetaGP{} using the default arm-level likelihood.  
We also separately evaluate the impact of our fixes on \BetaGP{} with only scale supervision:~{\bf \DirGP{}} and along with mean pooling:~{\bf \DirGPR{}}.
We also include a trivial baseline: {\bf \cpred{}} that fits all the arms with a global accuracy estimated using gold $\dseed$. We do not try sparse observation pooling with \PerArmBeta{} since there is no notion of per-arm closeness. We also skip it on \BernGP{} since it is worse than \BetaGP{} as we will show below. Recall that {\PerArmBeta} modeling is related to~\citet{JiLR20}.


\noindent
{\bf \large Other experimental settings}\\
\textbf{Gold accuracies~$\accgold(\attarm)$:}
We compute the oracular accuracy per arm using the gold attribute/label values of examples in $U$ which we treat as unlabeled during exploration. For every arm with at least five examples, we set its accuracy to be the empirical estimate obtained through the average correctness of all the examples that belong to the arm. We discard and not evaluate on any arms with fewer than five examples since their true accuracy cannot reliably be estimated. 

\noindent
\textbf{Warm start:} 
We start with 500 examples having gold attributes+labels to warm start all our experiments. The random seed also picks this random subset of 500 labeled examples. We calculate the overall accuracy of the classifier on these warm start examples as  $\hat{\accgold} = ({\sum_i c_i})/({\sum_i 1})$.
For all arms we warm start their observation with $c_\attarm = \lambda\hat{\accgold}, n_\attarm=\lambda$
where $\lambda=0.1$, a randomly picked low value.


Unless otherwise specified, we give equal importance to each arm and report MSE macroaveraged over all arms. Along with MSE, we also sometimes report MSE on the subset of 50 worst (true-)accuracy arms, referred to as worst MSE.  
We report other aggregate errors in the Appendix~\ref{sec:aaa:appendix:metrics}.
All the numbers reported here are averaged over three runs each with different random seed. The initial set of warm-start examples ($\dseed$) is also changed between the runs. In the case of \DirGPR{}, for any arm with observation count below 5, we mean pool from its three closest neighbours. 


In the following Sections:~\ref{sec:aaa:expt:est} and~\ref{sec:aaa:expt:exp}, we compare various estimation and exploration strategies with $\Pr(\va{\mid}\vx)$ noise calibrated as described in Section~\ref{sec:aaa:AttribNoise}. In Section~\ref{sec:aaa:expt:calib}, we study different forms of calibration and demonstrate the superiority of our proposed calibration technique of Equation~\eqref{eqn:factor:assume}.

\begin{figure*}[t]
\centering
\begin{tabular}{cccc} \tabcolsep 1pt
\includegraphics[width=.22\hsize]%
{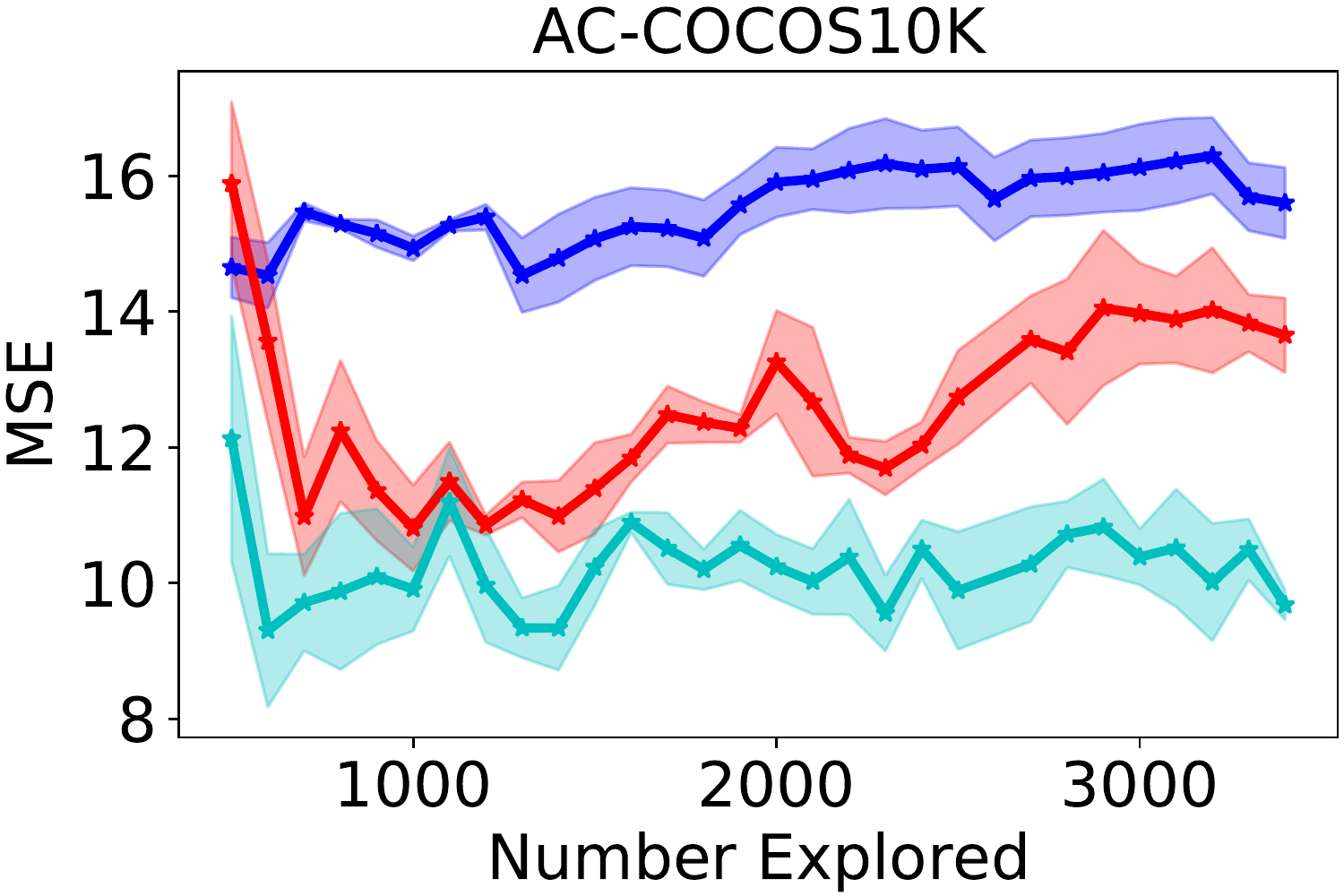} &
\includegraphics[width=.22\hsize]%
{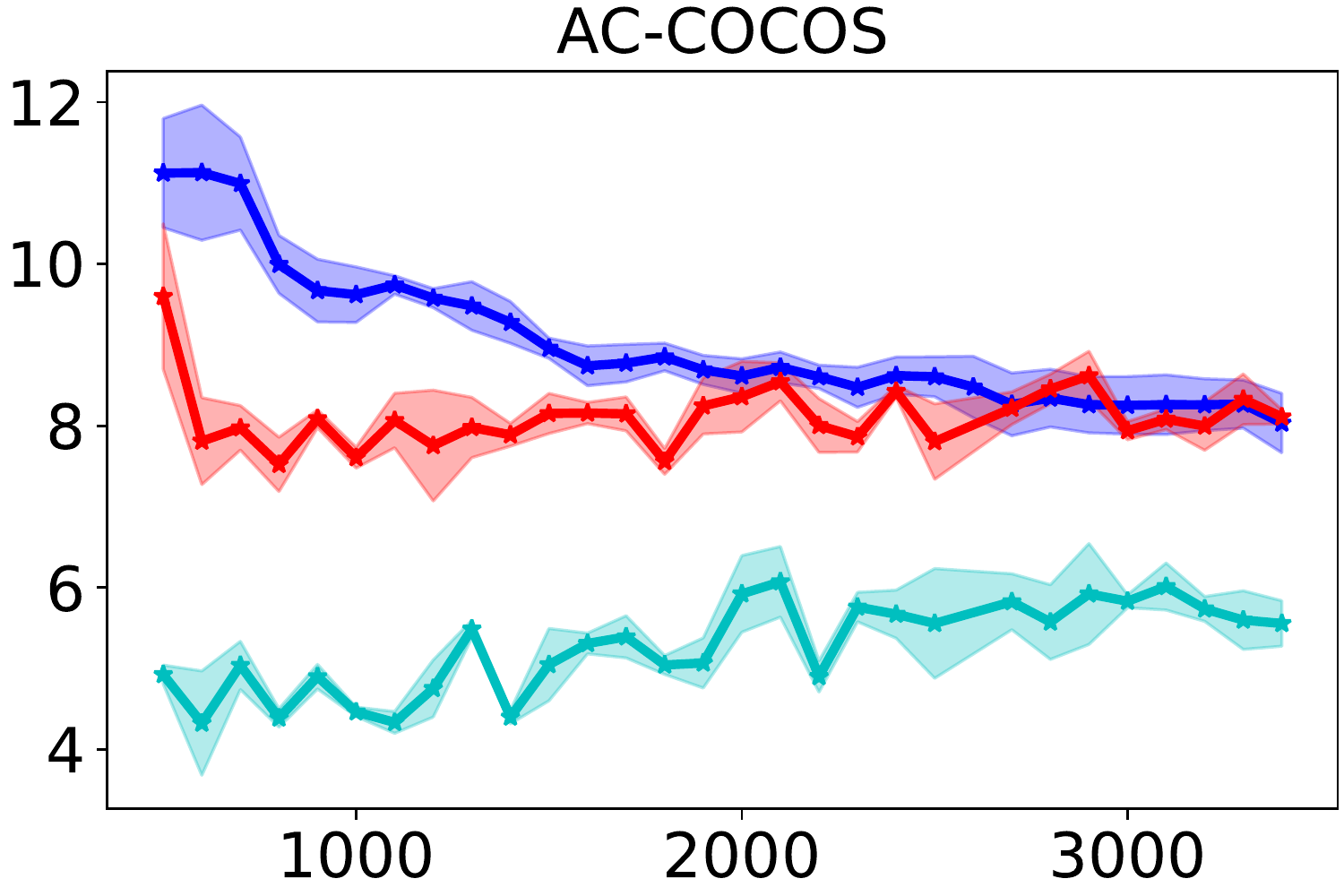} &
\includegraphics[width=.22\hsize]%
{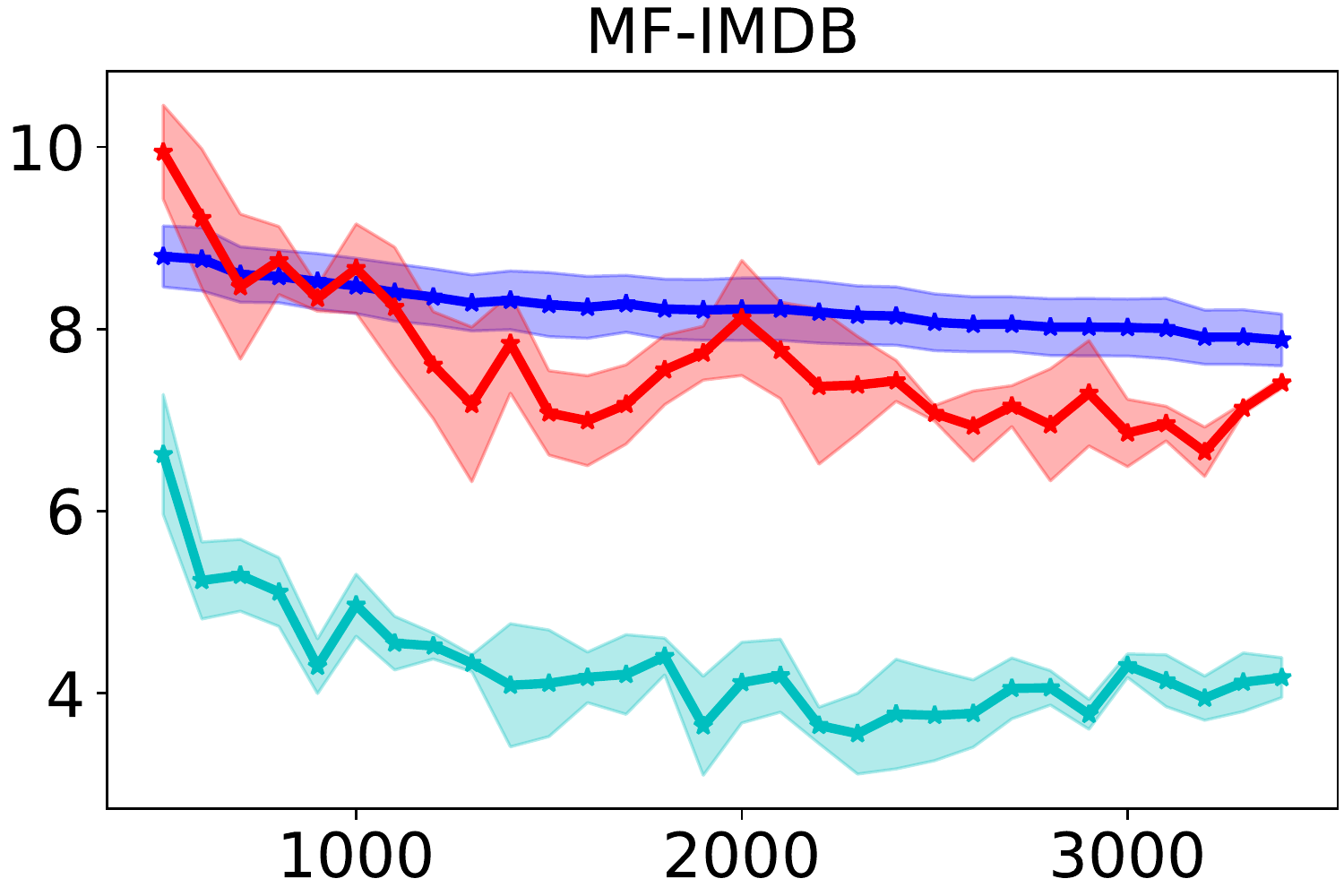} &
\includegraphics[width=.22\hsize]%
{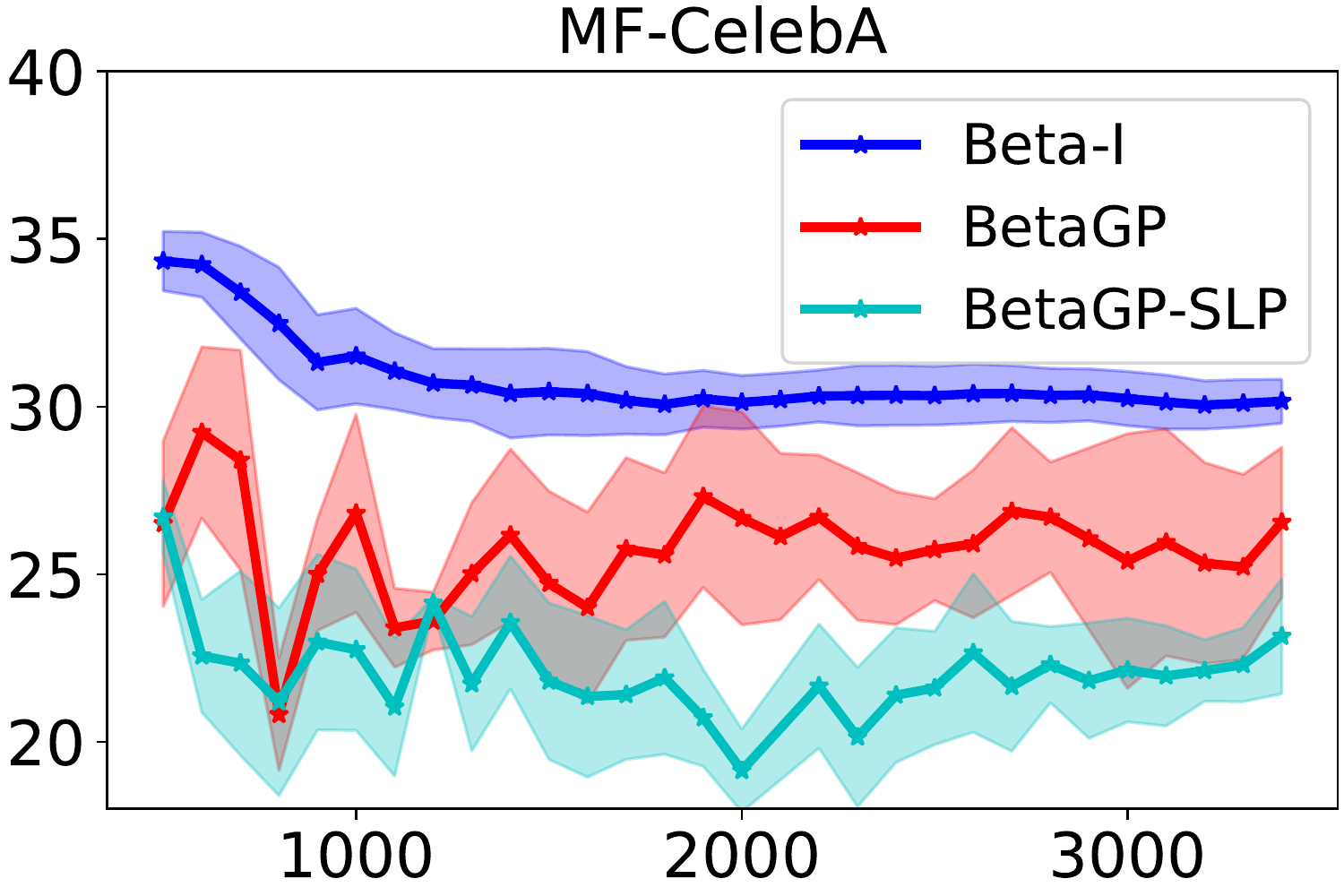}
\end{tabular}
\caption{Comparison of estimation methods using worst MSE metric. The shaded region shows standard error. \DirGPR{} consistently performs better than \BetaGP{}. \PerArmBeta{} is worse than its smoother counterparts.} 
\label{fig:aaa:estimation}
\end{figure*}

\begin{figure*}[t]
\centering
\begin{tabular}{cccc} \tabcolsep 1pt
\includegraphics[width=.22\hsize]%
{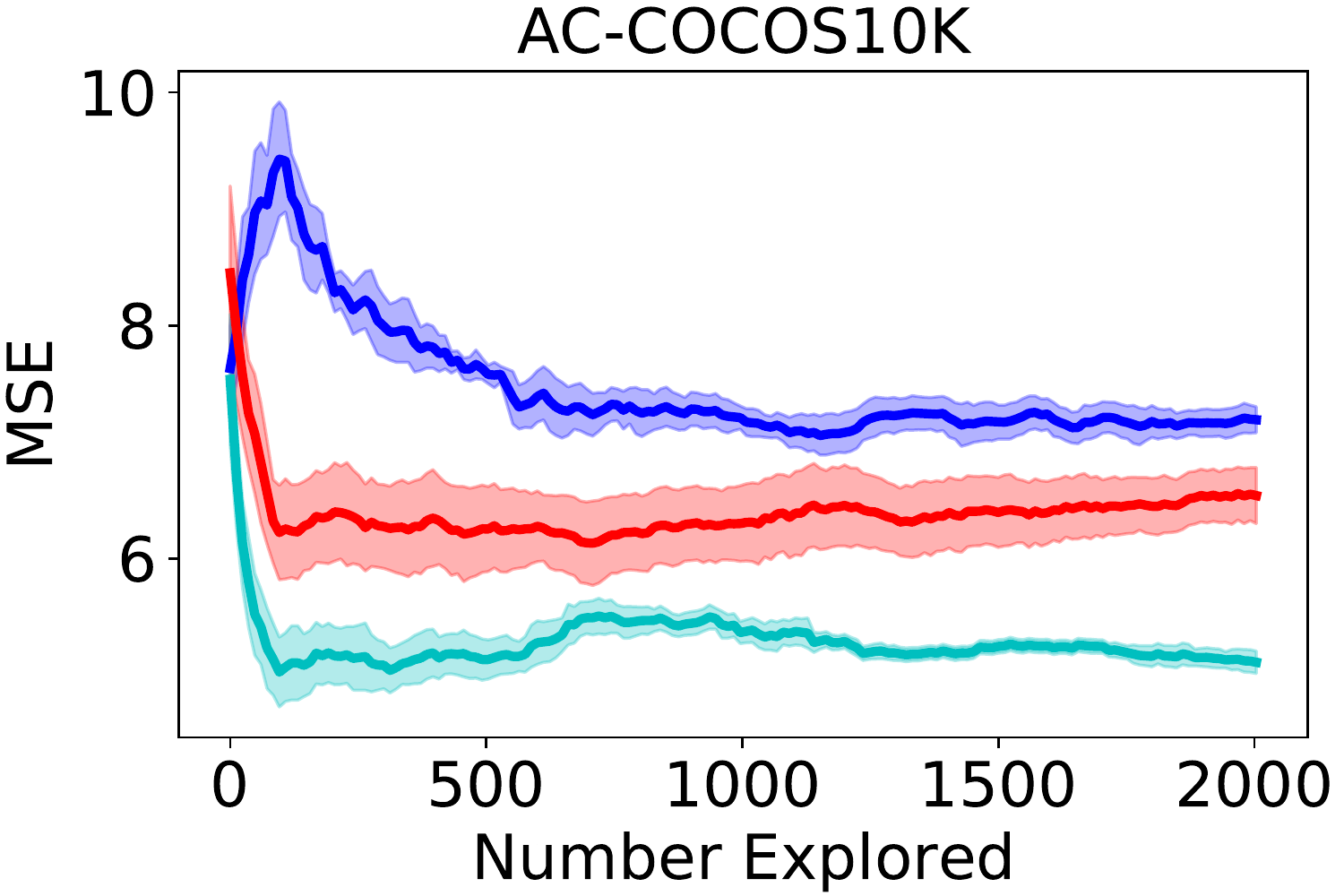} &
\includegraphics[width=.22\hsize]%
{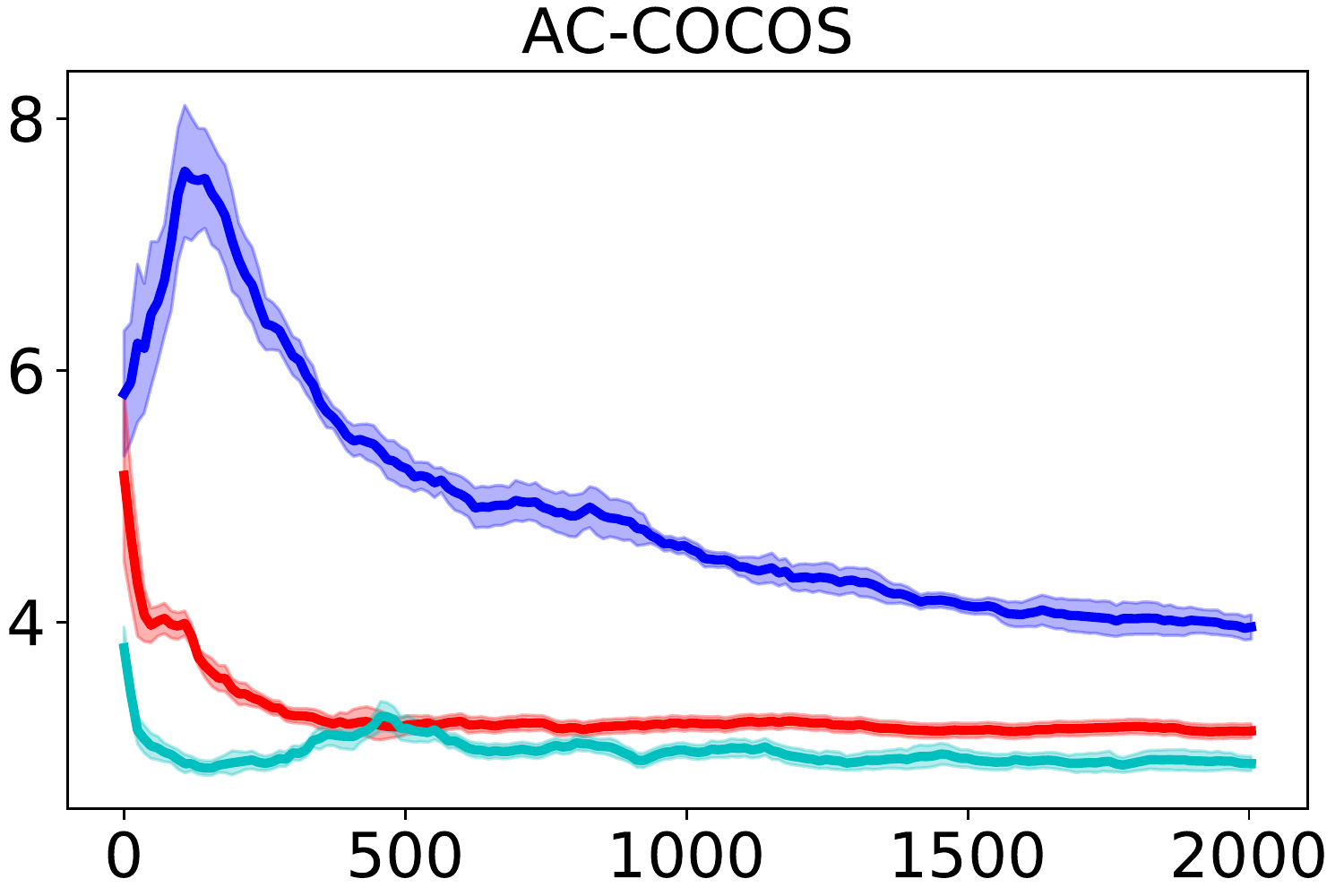} &
\includegraphics[width=.22\hsize]%
{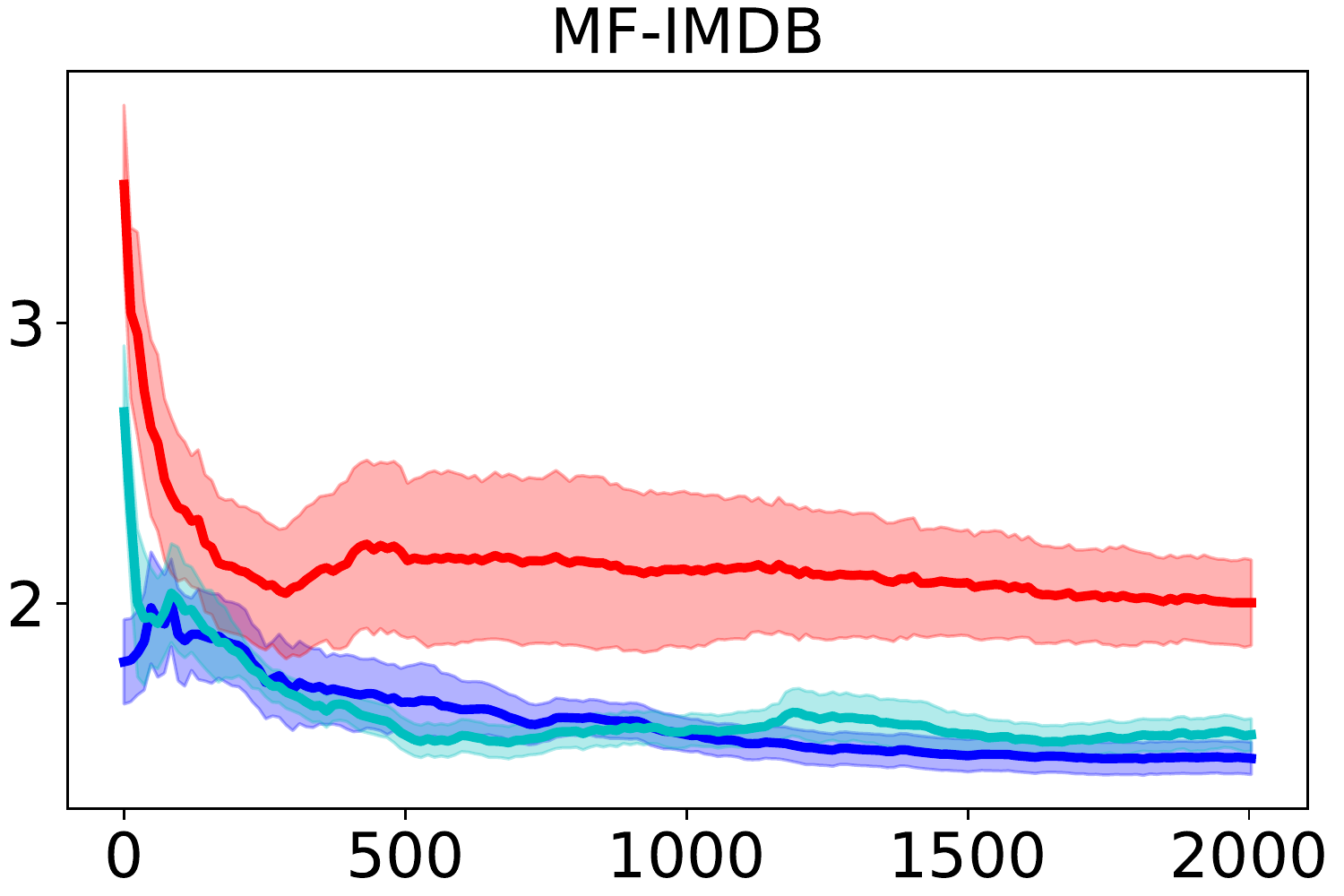} &
\includegraphics[width=.22\hsize]%
{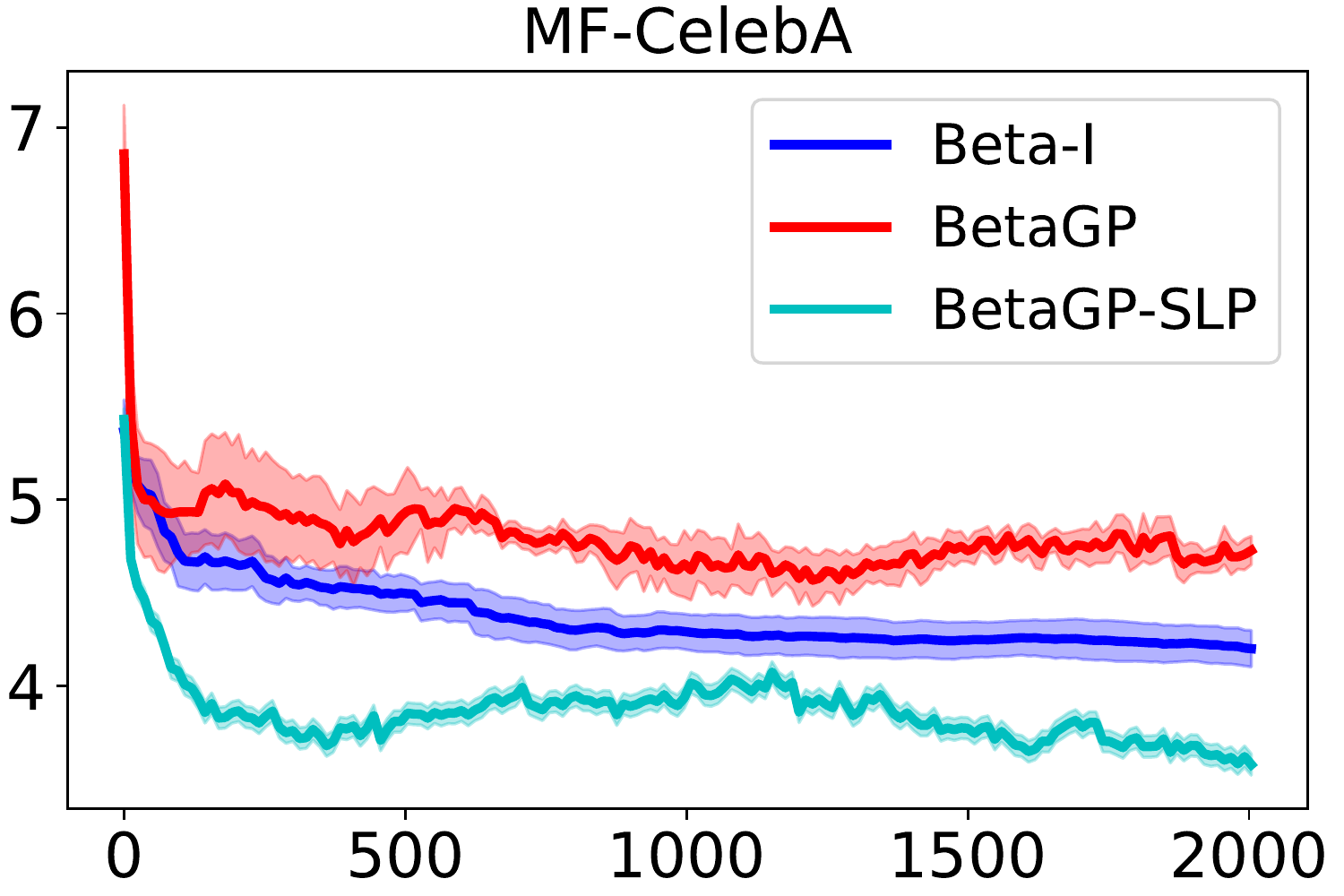}
\end{tabular}
\caption{Comparison of exploration methods. \DirGPR{} reduces macro MSE fastest most of the time. Shaded region shows standard error.} 
\label{fig:aaa:exploration}
\end{figure*}


\begin{figure*}
\centering
\begin{subfigure}{.24\textwidth}
  \centering
  \includegraphics[width=\linewidth]{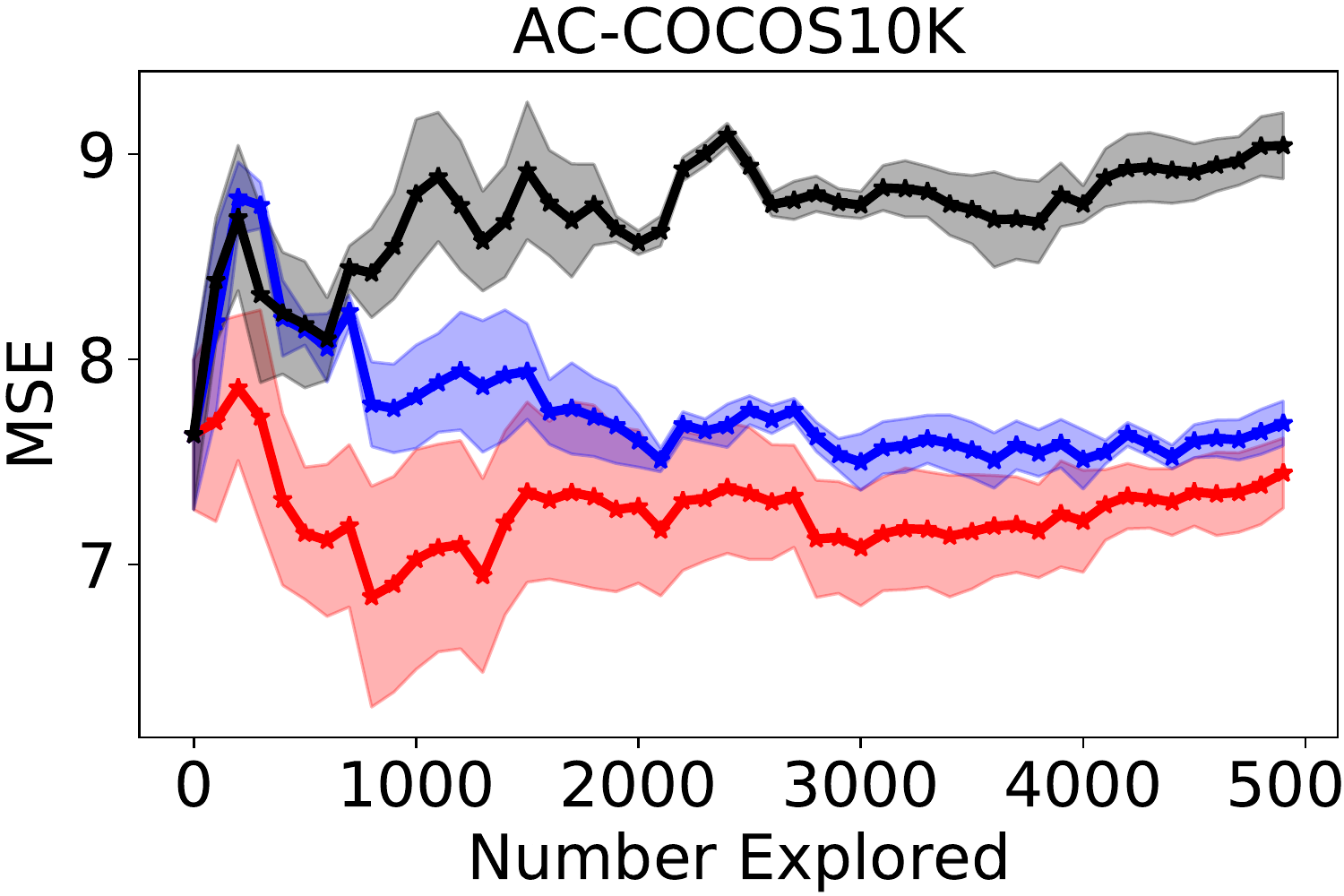}
\end{subfigure}
\begin{subfigure}{.25\textwidth}
  \centering
  \includegraphics[width=0.92\linewidth]{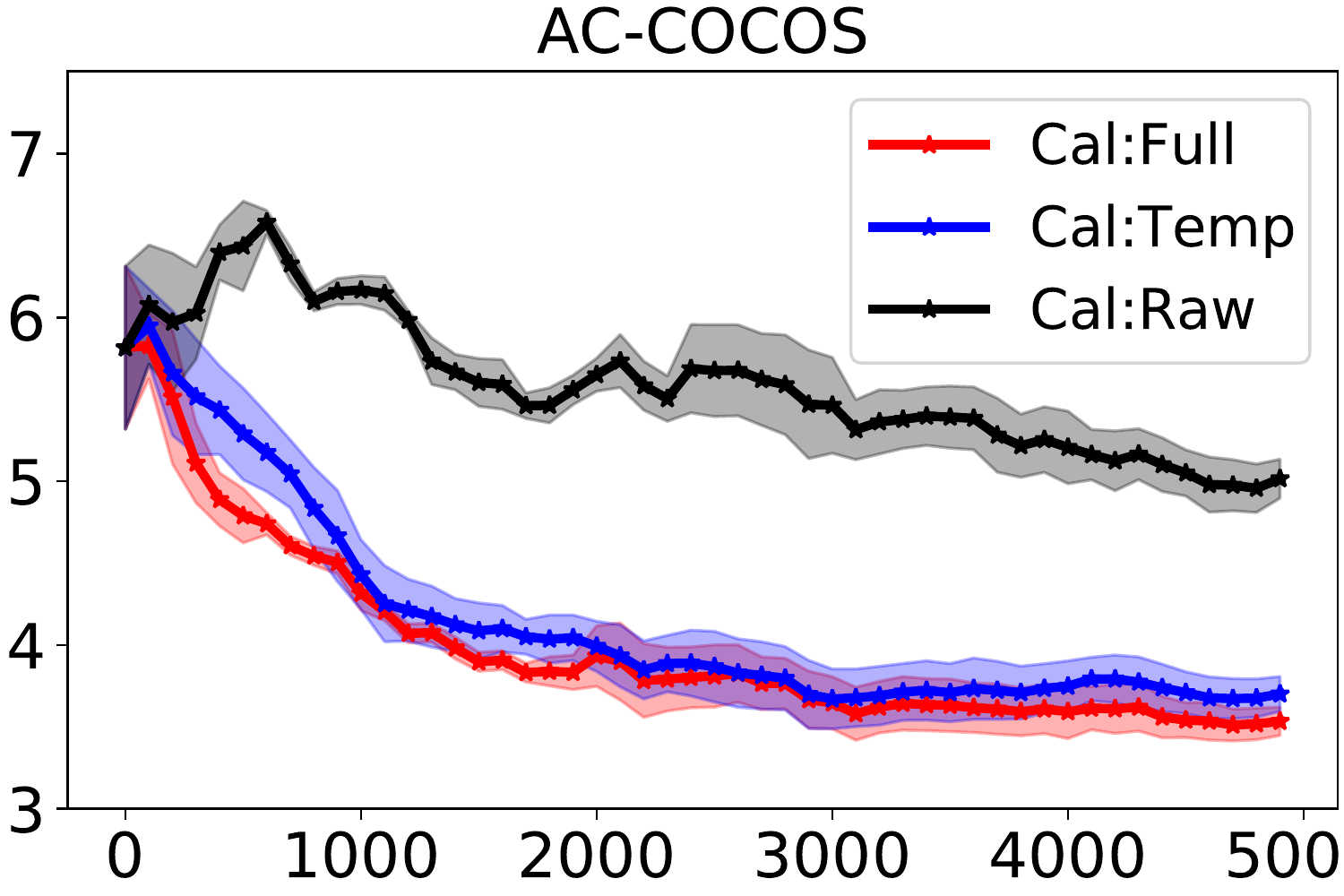}
\end{subfigure}
\begin{subfigure}{.24\textwidth}
  \centering
  \includegraphics[width=\linewidth]{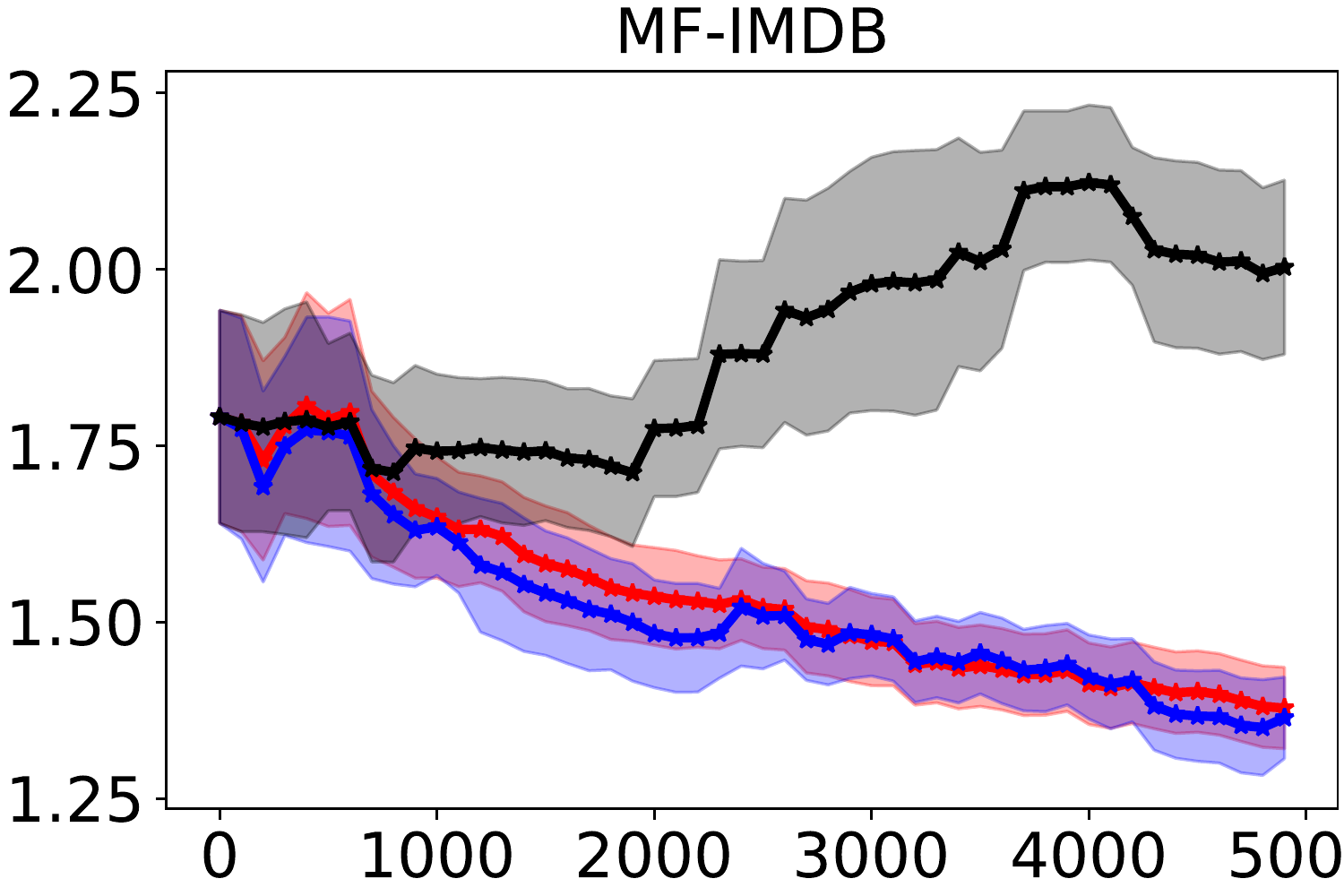}
\end{subfigure}
\begin{subfigure}{.24\textwidth}
  \centering
  \includegraphics[width=\linewidth]{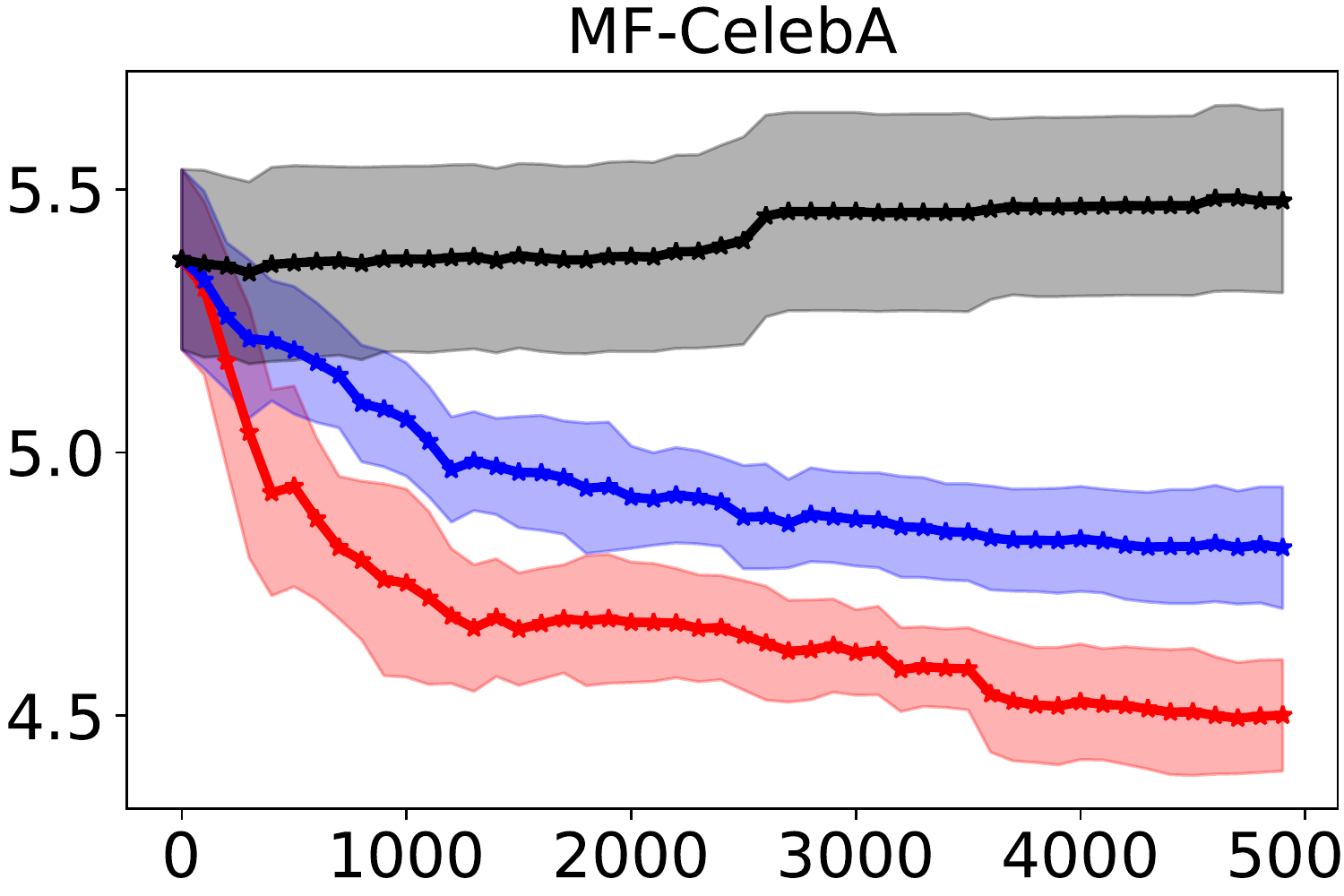}
\end{subfigure}%
\caption{Calibration methods compared on different tasks.  \calFull{} (red) includes temperature-based recalibration and correlation modeling with joint potential and gives the best macro MSE. Shaded region shows standard error.}
\label{fig:aaa:calibration}
\end{figure*}

\subsection{Estimation Quality}
\label{sec:aaa:expt:est}
\vspace{-0.2cm}
We evaluate methods on their estimation quality when each method is provided with exactly the same (randomly chosen) labeled set. We compare the four pretrained models when fitted on labeled data of size 2,000 and the results appear in Table~\ref{tab:aaa:est}. Note that we only have label supervision on $\cY$ in the labeled data. The table shows macro and worst MSE, standard deviation for each metric can be found in Appendix~\ref{sec:aaa:appendix:metrics}. In  Figure~\ref{fig:aaa:estimation}, we show worst MSE for a range of labeled data sizes along with their error bars. We make the following observations.\\
{\bf Smoothing helps:} Since we have a large number of arms, we expect \PerArmBeta{} to fare worse than its smooth counterparts (\BernGP{} and \BetaGP{}), especially on the worst arms. This is confirmed in the table.  In three out of four cases, \PerArmBeta{} method is worse than even the constant predictor \cpred\ on both metrics.\\
{\bf Modeling arm specific noise helps:} \BetaGP{} is better than \BernGP{} on almost all the cases in the table.  \\
{\bf Significant gains when the scale supervision problem of \BetaGP{} is fixed:} \DirGP{} is significantly better than \BetaGP{} in the table and figure. \\
{\bf Our pooling strategy helps:} \DirGPR{} improves \DirGP{} over worst MSE without hurting  macro MSE as seen in the table and figure.

\subsection{Exploration Efficiency}
\label{sec:aaa:expt:exp}
\vspace{-0.2cm}
We compare different methods that use their own estimated variance for exploring instances to label (Section~\ref{sec:aaa:Explore}), as a function of the number of explored examples --- see Figure~\ref{fig:aaa:exploration}. In most cases, \DirGPR{} gives the smallest macro MSE, beating \PerArmBeta{} and \BetaGP{}. 
We observe that \BetaGP\ provides very poor exploration quality, indicating that the uncertainty of arms is not captured well by just using two GPs. In fact, in many cases \BetaGP\ is worse than \PerArmBeta, even though we saw the opposite trend in estimation quality (Figure~\ref{fig:aaa:estimation}). These experiments brings out the significant role of  Dirichlet scale supervision and pooled observations in enhancing the uncertainty estimates at each arm.


\subsection{Impact of Calibration}
\label{sec:aaa:expt:calib}

\vspace{-0.2cm}
We consider two baselines along with our method explained in Section~\ref{sec:aaa:AttribNoise}: \textbf{\calRaw}, which uses the predicted attribute from the attribute models without any calibration and \textbf{\calTemp}, which  calibrates only the temperature parameters shown in eqn.~\eqref{eqn:factor:assume}, i.e., without the joint potential part. We refer to our method of calibration using temperature and joint potentials as \textbf{\calFull}.
%
%
We compare these on the four tasks with estimation method set to {\PerArmBeta} and random exploration strategy. Figure~\ref{fig:aaa:calibration} compares the three methods: \calRaw (Black), \calTemp (Blue), \calFull (Red).  
The X-axis is the number of explored examples beyond $\dseed$, and Y-axis is estimation error. Observe how \calTemp{} and \calFull\ are consistently better than \calRaw, and \calFull\ is better than~\calTemp.

\section{Related Work}
\label{sec:aaa:rel}

Our problem of actively estimating the accuracy {\em surface} of a classifier generalizes the more established problem of estimating a single accuracy {\em score} \citep{sawade10,Sawade2012,katariya12,druck11,bennett10,Karimi2020}.  For that problem, a known solution is stratified sampling, which partitions data into homogeneous strata and then seeks examples from regions with highest uncertainty and support. If we view each arm as a stratum, our method follows similar strategy. A key difference in our setting is that low support arms cannot be ignored. This makes it imperative to calibrate well the uncertainty under limited and skewed support distribution.  The setting of \citet{JiLR20} is the closest to ours. However, their work only considers a single attribute which they fit using independent Beta(\PerArmBeta) method, whereas we focus on the challenges of estimating accuracy over many sparsely populated attribute combinations.

\noindent
{\bf Sub-population performance:} Several recent papers have focused on identifying sub-populations with significantly worse accuracy than aggregated accuracy
\citep{Sagawa19,Oakden20hidden,WILDS20,JiSS20neurips,Miller21,Subbaswamy21eval}. Some of these have also proposed sample-efficient techniques~\citep{JiSS20neurips, Miller21} for estimation of performance on specific sub-groups, such as the ones defined by attributes like gender and race. Our accuracy surface estimation problem can be seen as a generalization where we need to estimate for all sub-groups defined in the Cartesian space of pre-specified semantic attributes. 
\citet{modelcards19} recommend reporting model performance under the influence of various relevant demographic/environmental factors as model cards--similar to the accuracy surface. 

\noindent
{\bf Experiment design:} Another related area is experiment design using active explorations with GPs \citep{SrinivasKK09}.  Their goal is to find the mode of the surface whereas our goal is to estimate the entire surface.  Further, each arm in our setting corresponds to multiple instances, which gives rise to a degree of heteroscedasticity and input-dependent noise that is not modeled in their settings. \citet{LzaroGredilla2011VariationalHG,Kersting2007MostLH} propose to handle heteroscedasticity by using a separate GP to model the variance at each arm. However, we showed the importance of additional terms in our likelihood and observation pooling to reliably represent estimation uncertainty.
\citet{WengerKT20} propose observation pooling for estimating smooth Betas but they assume a fixed kernel.


\noindent
{\bf Model debugging:} Testing deep neural network (DNN) is another related emerging area~\citep{Zhang2020}. 
\citet{PeiCY17DeepXplore,TianPJ18DeepTest,SunWR18Concolic, odena2019tensorfuzz} propose to generate test examples with good coverage over all activations of a DNN.
%
\citet{Ribeiro18Anchors,KimGP20} identify rules that explain the model predictions. 



\noindent
{\bf Performance prediction with target domain resources:} Related to our work is the work on estimating the performance of a pretrained models using labeled/unlabeled data from the target domain of test-user's interest. These methods investigate any correlation of distribution statistic derived from target data with its performance towards exploiting them for performance prediction~\citep{Chen21, Garg22, Deng21, Vedantam21}.

\section {Discussion}
\label{sec:aaa:end}
Evaluation of deployed models, despite its importance, is not well-studied. Knowledge of data regions covered by a pretrained model can help with model debugging, data acquisition, as well as for accuracy prediction on users' data. Some approaches in the past have addressed evaluating accuracy when provided resources from the test distribution. However, proactively declaring accuracy in all possible data regions can help understand better the model behaviour and also when acquiring target (un)labeled data is hard. Accuracy surfaces present a new paradigm of evaluating deployed models. 

We introduced \shortname, a new approach to estimate the accuracy of a pretrained model, not as a single number, but as a {\it surface} over a space of attributes (arms).  \shortname{} models uncertainty with a Beta distribution at each arm and regresses these parameters using two Gaussian Processes to capture smoothness and generalize to unseen arms. We proposed an additional Dirichlet likelihood to mitigate underestimation of GP fitted Beta distributions' scale parameters.
Further, to protect these high-capacity GPs from unreliable accuracy observations at sparsely populated arms, we propose to use  an observation pooling strategy. Finally, we show how to handle noisy attribute labels by an efficient joint recalibration method.
Evaluation on real-life datasets and pretrained models
show the efficacy of \shortname, both in estimation and exploration quality.

Further study of the following aspects can improve evaluation further. (1)~We have evaluated \shortname{} on the order of thousands of arms. Even larger attribute spaces could unearth more challenges. (2)~Identifying relevant attributes for an application can be non-trivial. Future work could devise strategies for attribute selection. 
(3)~Characterizing test-time data shifts could in itself be hard, particularly for text --- there could be subtle changes in word usage, style, or punctuation.  A more expressive attribute space needs to be developed for text applications.


\chapter{Unlabeled Domain Adaptation}
\label{chap:adaptation}
\colorlet{usercolorname}{green!10}
\def\shortname{SrcSel}

In the preceding chapters, we discussed challenges in training and evaluation of deployed models. When a model underperforms on a test domain, we may wish to adapt the model to the test/target domain's context. We could simply train a model from scratch if we have access to large training data from the target domain, which is rarely the case. However, unlabeled data could be more readily available. Nevertheless, state-of-art deep networks are notoriously label data hungry. \emph{How then can we adapt to a target domain with very limited labeled data and access to (potentially large) unlabeled data?} In this chapter, we present two approaches with slightly different motivations for the unlabeled adaptation problem for text applications. 

We will first discuss adaptation of word embeddings---dense representation of words that capture sense/meaning by tracking related words in a low-dimensional space. Word embeddings~\citep{MikolovSCCD2013word2vec,PenningtonSM2014GloVe} map discrete words (in the vocabulary) to a dense vector and benefit many natural language processing (NLP) tasks. We make a detailed study of adaptation of pretrained general-purpose word embeddings to a specific topical target domain when the available unlabeled data is insufficient to train high quality word embeddings from scratch. This work is based on~\citet{Piratla19}. 

We will then motivate on-the-fly or immediate adaptation methods, starting from Section~\ref{sec:kyc}, that can adapt without any labeled data or parameter fine-tuning. With the objective of in-context processing of the input, we learn a predictor model that additionally conditions on the corpus' sketch, from which the example originates, along with the example. Thereby eliminating adaptation on the unlabeled corpus text, this work is based on~\citet{Shah20}.

\section{Adaptation of Word Embeddings}

Often, usage scenario pertains a few focused topics, e.g., discussion boards on Physics, video games, or Unix, or a forum for discussing medical literature, with access to only a limited corpus (unlabeled data) $\DT$. 
Because $\DT$ may be too small to train word embeddings to sufficient quality, a prevalent practice is to harness general-purpose embeddings $\ES$ pretrained on a broad-coverage corpus, not tailored to the topics of interest.
The pretrained embeddings are sometimes used as-is (`pinned').  Even if $\ES$ is trained on a `universal' corpus, considerable sense shift may exist in the meaning of polysemous words and their cooccurrences and similarities with other words.  In a corpus about Unix, `cat' and `print' are more similar than in Wikipedia.  `Charge' and `potential' are more related in a Physics corpus than in Wikipedia.  Thus, pinning can lead to poor target task performance in the case of serious sense mismatch.  
Another popular practice is to initialize the target embeddings to the pretrained vectors, but then ``fine-tune'' using $\DT$ to improve performance in the target~\citep{MouPLXZJ15,min17,Howard2018}.
As we shall see, the number of epochs of fine-tuning is a sensitive knob --- excessive fine-tuning might lead to ``catastrophic forgetting'' \citep{kirkpatrick2017overcoming} of useful word similarities in ~$\ES$, and too little fine-tuning may not adapt to target sense.

Even if we are given development (`dev') sets for target tasks, the best balancing act between a pretrained $\ES$ and a topic-focused $\DT$ is far from clear.
Should we fine-tune (all word vectors) in epochs and stop when dev performance deteriorates?  
Or should we keep some words close to their pretrained embeddings (a form of regularization) and allow others to tune more aggressively?  On what properties of $\ES$ and $\DT$ should the regularization strength of each word depend?  Our first contribution is a new measure of semantic drift of a word from $\ES$ to $\DT$, which can be used to control the regularization strength.  In terms of perplexity, we show that this is superior to both epoch-based tuning, as well as regularization based on simple corpus frequencies of words~\citep{regemnlp}.  
Yet another option is to learn projections to align generic embeddings to the target sense~\citep{BollegalaMK15,Barnes2018,Sarma2018}, or to a shared common space~\citep{YinS16,CoatesB18,BaoB2018}
However, in carefully controlled experiments, none of the proposed approaches to adapting pretrained embeddings consistently beats the trivial baseline of discarding them and training afresh on~$\DT$!

Our second contribution is to explore other techniques beyond adapting generic embeddings $\ES$.  Often, we might additionally have easy access to a broad corpus $\DS$ like Wikipedia.  $\DS$ may span many diverse topics, while $\DT$ focuses on one or few, so there may be  large \emph{overall} drift from $\DS$ to $\DT$ too.  However, a judicious \emph{subset} $\widehat{\DS} \subset \DS$ may exist that would be excellent for augmenting~$\DT$.   The large size of $\DS$ is not a problem: we use an inverted index that we probe with documents from $\DT$ to efficiently identify~$\widehat{\DS}$.  Then we apply a novel perplexity-based joint loss over $\widehat{\DS}\cup\DT$ to fit adapted word embeddings. 
While most of recent research focus has been on designing better methods of adapting pretrained embeddings, we show that retraining with selected source \emph{text} is significantly more accurate than the best of embeddings-only strategy, while runtime overheads are within practical limits.

An important lesson is that non-dominant sense information may be irrevocably obliterated from generic embeddings; it may not be possible to salvage this information by post-facto adaptation. Furthermore, we found non-dominant sense information is lost even when using contextual embeddings~\citep{PetersNIGCKZ2018ELMo, Cer+2018UnivSentEncoder}.

\section{Background: related work and baselines}
\label{sec:srcsel:Intro}

\noindent
{\bf CBOW:}
We review the popular CBOW model for learning unsupervised word representations~\citep{MikolovSCCD2013word2vec}.  As we scan the corpus, we collect a \emph{focus} word $w$ and a set $C$ of \emph{context} words around it, with corresponding embedding vectors $\pmb{u}_w \in \R^n$ and $\pmb{v}_c \in \R^n$, where $c \in C$.  The two embedding matrices $\pmb{U}, \pmb{V}$ are estimated as:
\begin{align}
\max_{\pmb{U}, \pmb{V}} \hspace{-.7em}
\sum_{\langle w, C\rangle \in \mathcal{D}} \hspace{-.7em}
\sigma(\pmb{u}_w \cdot \pmb{v}_C)
+ \sum_{\bar w \sim \mathcal{D}} \sigma(-\pmb{u}_{\bar{w}} \cdot \pmb{v}_C)
\label{eq:srcsel:cbow}
\end{align}
Here $\pmb{v}_C$ is the average of the context vectors in~$C$. $\bar w$ is a negative focus word sampled from a slightly distorted unigram distribution of~$\mathcal{D}$, and $\sigma(\bullet)$ is the sigmoid function.  Usually downstream applications use only the embedding matrix $\pmb{U}$, with each word vector scaled to unit length. Apart from CBOW, \citet{MikolovSCCD2013word2vec} defined the related skipgram model, and \citet{PenningtonSM2014GloVe} proposed the Glove model, which can also be used in our framework.  We found CBOW to work better for our downstream tasks. 

\noindent
{\bf \src, \tgt{}, \concat{} and other simple baselines:}
In the `\src' option, pretrained embeddings $\pmb{u}^S_w$ trained only on a large corpus are used as-is.  The other extreme, called `\tgt', is to train word embeddings from scratch on the limited target corpus~$\DT$. In our experiments we found that \src\ performs much worse than \tgt, indicating the presence of significant drift in prominent word senses.   Two other simple baselines, are `\concat', that concatenates the source and target trained embeddings and let the downstream task figure out their relative roles, and '\avg' that following \citet{CoatesB18} takes their simple average.  Another option is to let the downstream task learn to combine multiple embeddings as in \citet{ZhangRW2016}.

As word embeddings have gained popularity for representing text in learning models, several methods have been proposed for enriching small datasets with pretrained embeddings.

\noindent
{\large \bf Adapting pretrained embeddings}

\noindent
{\bf \srcinit:}
A popular method~\citep{min17, WangHF17,Howard2018} is to use the source embeddings $\pmb{u}^S_w$ to initialize $\pmb{u}_w$ and thereafter train on~$\DT$.  We call this `\srcinit'.  Fine-tuning requires careful control of the number of epochs with which we train on~$\DT$.  Excessive training can wipe out any benefit of the source because of catastrophic forgetting.  Insufficient training may not incorporate target corpus senses in case of polysemous words, and adversely affect target tasks~\citep{MouPLXZJ15}. The number of epochs can be controlled using perplexity on a held-out $\DT$, or using downstream tasks.   \citet{Howard2018} propose to fine-tune a whole language model using careful differential learning rates.  However, epoch-based termination may be inadequate.  Different words may need diverse trade-offs between the source and target topics, which we discuss next.

\noindent
{\bf \yangC~(frequency-based regularization):}
\citet{regemnlp} proposed to train word embeddings using $\DT$, but with a regularizer to prevent a word $w$'s embedding from drifting too far from the source embedding ($\pmb{u}^S_w$).  The weight of the regularizer is meant to be inversely proportional to the concept drift of $w$ across the two corpus.  Their limitation was that corpus frequency was used as a surrogate for stability; high stability was awarded to only words frequent in both corpora.  As a consequence, very few words in a focused $\DT$ about Physics will benefit from a broad coverage corpus like Wikipedia.  Thousands of words like {\em galactic, stars, motion, x-ray,} and {\em momentum} will get low stability, although their prominent sense is the same in the two corpora.  We propose a better regularization scheme in this paper.  Unlike us, \citet{regemnlp} did not compare with fine-tuning.

\noindent
{\bf Projection-based methods:} attempt to project embeddings of one kind to another, or to a shared common space.  \citet{Bollegala2014} and \citet{Barnes2018} proposed to learn a linear transformation between the source and target embeddings.
\citet{YinS16} transform multiple embeddings to a common `meta-embedding' space. Simple averaging are also shown to be effective \citep{CoatesB18}, and a recent~\citet{BaoB2018} auto-encoder based meta-embedder (AEME) is the state of the art. \citet{Sarma2018} proposed CCA to project both embeddings to a common sub-space. Some of these methods designate a subset of the overlapping words as pivots to bridge the target and source parameters in various ways~\citep{BlitzerMP06, Ziser2018, BollegalaMK15}. Many such techniques were proposed in a cross-domain setting, and specifically for the sentiment classification task. Gains are mainly from effective transfer of sentiment representation across domains.  Our challenge arises when a corpus with broad topic coverage pretrains dominant word senses quite different from those needed by tasks associated with narrower topics.






\subsubsection*{Language models for task transfer}
Complementary to the technique of adapting individual word embeddings is the design of deeper sequence models for task-to-task transfer.  \citet{Cer+2018UnivSentEncoder,SubramanianTBP2018SentenceMTL} propose multi-granular transfer of sentence and word representations across tasks using Universal Sentence Encoders. 
ELMo \citep{PetersNIGCKZ2018ELMo} trains a multi-layer sequence model to build  a context-sensitive representation of words in a sentence. 
ULMFiT \citep{Howard2018} present additional tricks such as gradual unfreezing of parameters layer-by-layer, and exponentially more aggressive fine-tuning toward output layers. 
\citet{devlin2018bert} propose a deep bidirectional language model for generic contextual word embeddings.  We show that our topic-sensitive embeddings provide additional benefit even when used with contextual embeddings.

\section{Proposed approaches}
\label{sec:srcsel:Approach}

We explore two families of methods: (1)~those that have access to only pretrained embeddings (Sec~\ref{sec:srcsel:Approach:StabilityReg}), and (2)~those that also have access to a source corpus with broad topic coverage (Sec~\ref{sec:srcsel:Approach:WithSource}). 


\subsection{\srcselReg: Stability-based regularization}
\label{sec:srcsel:Approach:StabilityReg}

Our first contribution is a more robust definition of stability to replace the frequency-based regularizer of \yangC.
We first train word vectors on $\DT$, and assume the pretrained embeddings $\ES$ are available.
Let the focus embeddings of word $w$ in $\ES$ and $\DT$ be 
$\pmb{u}^S_w$ and $\pmb{u}^T_w$.
We overload $\ES \cap \DT$ as words that occur in both.
For each word $w \in \ES \cap \DT$, we
compute $N^{(K)}_S(w, \ES\cap\DT)$, the $K$ nearest neighbors of $w$ with respect to the generic embeddings, i.e.,
with the largest values of $\cos(\pmb{u}^S_w, \pmb{u}^S_n)$ from $\ES\cap\DT$.
Here $K$ is a suitable hyperparameter.
Now we define 
\begin{align}
\stability(w)=\frac{\sum_{n\in N^{(K)}_S(w,\ES\cap\DT)} \cos(\pmb{u}_w^T, \pmb{u}_n^T)}
{|N^{(K)}_S(w, \ES \cap \DT)|}
\label{eqn:srcsel:sim_score}
\end{align}
Where $\cos(\bullet, \bullet)$ is the cosine similarity measure. 
Intuitively, if we consider near neighbors $n$ of $w$ in terms of source embeddings, and most of these $n$'s have target embeddings very similar to the target embedding of $w$, then $w$ is stable across $\ES$ and $\DT$, i.e., has low semantic drift from $\ES$ to~$\DT$.

Finally, the word regularization weight is:
\begin{align}
\wscore(w) &= \max(0, \tanh\bigl(\lambda \, \stability(w))\bigr).
\label{eqn:srcsel:sem_drift}
\end{align}
Here $\lambda$ is a hyperparameter.
$\wscore(w)$ above is a replacement for the regularizer used by \citet{regemnlp}.  
If $\wscore(w)$ is large, it is regularized more heavily toward its source embedding, keeping $\pmb{u}_w$ closer to~$\pmb{u}^S_w$.  The modified CBOW loss is:
\begin{align}
\max_{\pmb{U}, \pmb{V}} \hspace{-.7em}
\sum_{\langle w, C\rangle \in \mathcal{D}} \hspace{-.7em}
\sigma(\pmb{u}_w \cdot \pmb{v}_C) 
+ \sum_{\bar w \sim \mathcal{D}} \sigma(-\pmb{u}_{\bar{w}} \cdot \pmb{v}_C) + \sum_w
\colorbox{green!8}{$\wscore(w)$}
\|\pmb{u}_w - \pmb{u}^S_w \|^2
\label{eq:srcsel:RegSense}
\end{align}
Our $\wscore(w)$ performs better than \citeauthor{regemnlp}'s.

While many other forms of $\stability$ can achieve the same ends, ours seems to be the first formulation that goes beyond mere word frequency and employs the topological stability of near-neighbors in the embedding space.  Here is why this is important.  Going from a  generic corpus like Wikipedia to the very topic-focused StackExchange (Physics) corpus $\DT$, the words {\em x-ray, universe, kilometers, nucleons, absorbs, emits, sqrt, anode, diodes}, and {\em km/h} have large stability per our definition above, but low stability according to \citeauthor{regemnlp}'s frequency method since they are (relatively) rare in source.  Using their method, therefore, these words will not benefit from reliable pretrained embeddings.


\subsection{Source selection and joint perplexity}
\label{sec:srcsel:Approach:WithSource}

To appreciate the limitations of regularization, consider words like {\em potential, charge, law, field, matter, medium,} etc.  These will get small stability ($\wscore(w)$) values because their dominant senses in a universal corpus do not match with those in a Physics corpus~($\DT$), but $\DT$ may be too limited to wipe that dominant sense for a subset of words while preserving the meaning of stable words.
However, there are plenty of high-quality broad-coverage sources like Wikipedia that includes many Physics documents that could gainfully supplement~$\DT$.  Therefore, we seek to include target-relevant documents from a generic source corpus $\DS$, even if the dominant sense of a word in $\DS$ does not match that in~$\DT$.  The goal is to do this without solving the harder problem of unsupervised, expensive and imperfect sense discovery in $\DS$ and sense tagging of~$\DT$, and using per-sense embeddings.

The main steps of the proposed approach, \shortname, are shown in 
\figurename~\ref{fig:srcsel:SrcSel}.  Before describing the steps in detail, we note that preparing and probing a standard inverted index \citep{BaezaYatesR1999MIR} are extremely fast, owing to decades of performance optimization.  Also, index preparation can be amortized over multiple target tasks.  (The granularity of a `document' can be adjusted to the application.)

\begin{figure}[th]
\centering
\begin{tcolorbox}[boxsep=0mm,colback=gray!4,colframe=gray!25,boxrule=.2mm]
\begin{algorithmic}[1] \raggedright
\State Index all source docs $\DS$ in a text retrieval engine.
\State Initialize a score accumulator $a_{s}$ for each source doc $s\in\DS$.
\For{each target doc $t \in \DT$}
\State Get source docs most similar to~$t$.
\State Augment their score accumulators.
\EndFor
\State $\widehat{\DS} \leftarrow \varnothing$
\For{each source doc $s \in \DS$}
\If{$a_{s}$ is ``sufficiently large''}
\State Add $s$ to $\widehat{\DS}$.
\EndIf
\EndFor
\State Fit word embeddings to optimize a joint objective over $\widehat{\DS}\cup\DT$.
\end{algorithmic}
\end{tcolorbox}
\caption{Main steps of \shortname.}
\label{fig:srcsel:SrcSel}
\end{figure}

\noindent
{\bf Selecting source documents to retain:}
Let $s \in \DS, t \in \DT$ be source and target documents.  Let $\text{sim}(s,t)$ be the similarity between them, in terms of the TFIDF cosine score commonly used in Information Retrieval \citep{BaezaYatesR1999MIR}.  The total vote of $\DT$ for $s$ is then $\sum_{t \in \DT}\text{sim}(s, t)$.  We choose a suitable cutoff on this aggregate score, to reduce $\DS$ to $\widehat{\DS}$, as follows.  Intuitively, if we hold out a randomly sampled part of $\DT$, our cutoff should let through a large fraction (we used 90\%) of the held-out part.  Once we find such a cutoff, we apply it to $\DS$ and retain the source documents whose aggregate scores exceed the cutoff.  Beyond mere selection, we design a joint perplexity objective over $\widehat{\DS}\cup\DT$, with a term for the amount of trust we place in a retained source document. This limits damage from less relevant source documents that slipped through the text retrieval filter. Since the retained documents are weighted based on their relevance to the topical target corpus $\DT$, we found it beneficial to also include a percentage (we used 10\%) of randomly selected documents from $\DS$. We refer to the method that only uses documents retained using text retrieval filter as \srcselR{} and only randomly selected documents from $\DS$ as \srcselC. \srcsel{} uses documents both from the retrieval filter and random selection. 

\noindent
{\bf Joint perplexity objective:}
Similar to~\eqref{eq:srcsel:cbow}, we will sample word and context $\langle w, C\rangle$ from $\DT$ and $\widehat{\DS}$.  Given our limited trust in $\widehat{\DS}$, we will give each sample from $\widehat{\DS}$ an alignment score $\sscore(w,C)$.  This should be large when $w$ is used in a context similar to contexts in~$\DT$.  We judge this based on the target embedding~$\pmb{u}^T_w$:
\begin{align}
\sscore(w, C) &=
\max\left\{0, \cos\left(\pmb{u}^T_w, \pmb{v}^T_C \right) \right\}.
\label{eq:srcsel:cscore}
\end{align}
We summarize context words into $\pmb{v}_C$ by averaging $v_{\bullet}$ vectors over all words in the context: $\pmb{v}_C=\sum_{w^\prime\in C}\pmb{v}_{w^\prime}/|C|$ just like in the CBOW algorithm.
Since $\pmb{u}_w$ represents the sense of the word in the target, source contexts $C$ which are similar will get a high score.  Similarity in source embeddings is not used here because our intent is to preserve the target senses.  We tried other forms such as dot-product or its exponential and chose the above form because it is bounded and hence less sensitive to gross noise in inputs. 

\noindent
The word2vec objective~\eqref{eq:srcsel:cbow} is enhanced to
\begin{multline}
\sum_{\langle w, C\rangle \in \DT}
\hspace{-1em} \left[
\sigma(\pmb{u}_w \cdot \pmb{v}_C) + \textstyle
\sum_{\bar{w} \sim \DT} \sigma(-\pmb{u}_{\bar{w}} \cdot \pmb{v}_C)
\right] 
\\\hspace{-2em}
+ \sum_{\langle w, C\rangle \in \widehat{\DS}} \!\!\!\!
\colorbox{red!7}{$\sscore(w,C)$} \, \Bigl[
\sigma(\pmb{u}_w \cdot \pmb{v}_C) +  
\textstyle \sum_{\bar{w} \sim \widehat{\DS}} \! 
\sigma(-\pmb{u}_{\bar{w}} \cdot \pmb{v}_C) \Bigr]. \label{eq:srcsel:SnippetSelect}
\end{multline}
The first sum is the regular word2vec loss over~$\DT$.  Word $\bar{w}$ is sampled from the vocabulary of $\DT$ as usual, according to a suitable distribution.  The second sum is over the retained source documents~$\widehat{\DS}$. Note that $\sscore(w,C)$ is computed using the pretrained target embeddings and does not change during the course of training.


\noindent
{\bf \shortname+\srcselReg{} combo:}
Here we combine objective \eqref{eq:srcsel:SnippetSelect} with the regularization term in \eqref{eq:srcsel:RegSense}, where $\wscore$ uses all of~$\ES$ as in \srcselReg.





\section{Experiments}
\label{sec:srcsel:Expt}

We compare the methods discussed thus far, with the goal of answering these research questions:
\begin{enumerate}[partopsep=0ex,topsep=0ex,leftmargin=*]
\item Can word-based regularization (\yangC\ and \srcselReg) beat careful termination at epoch granularity, after initializing with source embeddings (\srcinit)?
\item How do these compare with just fusing \src\ and \tgt\ via recent meta-embedding methods like AAEME~\citep{BaoB2018}\footnote{We used the implementation available at: \url{https://github.com/CongBao/AutoencodedMetaEmbedding}}?
\item Does \srcsel\ provide sufficient and consistent gains over \srcselReg\ to justify the extra effort of processing a  source corpus?
\item Do contextual embeddings obviate the need for adapting word embeddings?
\end{enumerate}
We also establish that initializing with source embeddings also improves regularization methods. (Curiously, \yangC\ was never combined with source initialization.)

\subsubsection*{Topics and tasks}
\label{sec:srcsel:DedupDataset}
We compare across 15 topic-task pairs spanning 10 topics and 3 task types: an unsupervised language modeling task on five topics, a document classification task on six topics, and a duplicate question detection task on four topics. 
In our setting, $\DT$ covers a small subset of topics in $\DS$, which is the 20160901\footnote{The target corpora in our experiments came from datasets that were created before this time.} version dump of Wikipedia.  Our tasks are different from GLUE-like multi-task learning \citep{WangSMHLB2019glue}, because our focus is on the problems created by the divergence between prominent sense-dominated generic word embeddings and their sense in narrow target topics.  We do not experiment on the cross-domain sentiment classification task popular in domain adaptation papers since they benefit more from sharing sentiment-bearing words, than learning the correct sense of polysemous words, which is our focus here.  All our experiments are on public datasets, and we will publicly release our experiment scripts and code. 


\noindent
{\bf StackExchange topics}\\
We pick four topics (Physics, Gaming, Android and Unix) from the CQADupStack\footnote{\protect\url{http://nlp.cis.unimelb.edu.au/resources/cqadupstack/}} dataset of questions and responses.  For each topic, the available response text is divided into $\DT$, used for training/adapting embeddings, and $\widetilde{\DT}$, the evaluation fold used to measure perplexity.
In each topic, the target corpus $\DT$ has 2000 responses totalling roughly 1 MB. We also report results with changing sizes of $\DT$. Depending on the method we use $\DT, \DS$, or $\pmb{u}^S$ to train topic-specific embeddings and evaluate them as-is on two tasks that train task-specific layers on top of these fixed embeddings.
%
%
\noindent 
The first is an {\bf unsupervised language modeling task} where we train a LSTM\footnote{\url{https://github.com/tensorflow/models/blob/master/tutorials/rnn/ptb/ptb_word_lm.py}} on the adapted embeddings (which are pinned) and report perplexity on~$\widetilde{\DT}$.
The second is a {\bf Duplicate question detection task.}
Available in each topic are human annotated duplicate questions (statistics in \tablename~\ref{tab:srcsel:stats}) which we partition across 
train, test and dev as 50\%, 40\%, 10\%.  For contrastive training, we add four times as much randomly chosen non-duplicate pairs.  The goal is to predict duplicate/not for a question pair, for which we use word mover distance \citep[WMD]{KusnerSKW15} over adapted word embeddings.  We found WMD more accurate than BiMPM~\citep{WangHF17}.  
We use three splits of the target corpus, and for each resultant embedding, measure AUC on three random (train-)dev-test splits of question pairs, for a total of nine runs.  For reporting AUC, WMD does not need the train fold.

\begin{table}[htb]
\centering
\begin{tabular}{|l|r|r|r|} \hline
 & Tokens & Vocab size & \# duplicates\\
\hline
Physics & 542K & 6,026 & 1981\\
Gaming & 302K & 6,748 & 3386\\
Android & 235K & 4,004 & 3190\\
Unix & 262K  & 6,358& 5312\\
\hline
\end{tabular}
\caption{Statistics of the Stack Exchange data used in duplicate question detection.  $\DS$ has a vocabulary of 300,000 distinct words.}
\label{tab:srcsel:stats}
\end{table}

\begin{figure*}[t]
\pgfplotstableread{images/srcsel/physics-ppl.dat} \physicsppl
\pgfplotstableread{images/srcsel/unix-ppl.dat} \unixppl
\pgfplotstableread{images/srcsel/android-ppl.dat} \androidppl
\pgfplotstableread{images/srcsel/gaming-ppl.dat} \gamingppl
\pgfplotstableread{images/srcsel/physics-auc.dat} \physicsauc
\pgfplotstableread{images/srcsel/unix-auc.dat} \unixauc
\pgfplotstableread{images/srcsel/android-auc.dat} \androidauc
\pgfplotstableread{images/srcsel/gaming-auc.dat} \gamingauc
\def\chartheight{50mm}
\centering
\begin{subfigure}[b]{0.48\textwidth}
\begin{tikzpicture}[font=\small]
\begin{axis}[nice, width=\hsize, height=\chartheight, xlabel=Epochs (Android), ylabel=LM Perplexity, legend pos=north east, xmin=5, xmax=100,ymax=125, legend columns=2]
\addplot table [y=tgt, x=epoch] from \androidppl;
\addlegendentry{\tgt}
\addplot table [y=srcinit, x=epoch] from \androidppl;
\addlegendentry{\srcinit}
\addplot table [y=yangC, x=epoch] from \androidppl;
\addlegendentry{\yangC}
\addplot table [y=regour, x=epoch] from \androidppl;
\addlegendentry{\srcselReg}
\addplot[color=green!40!black,mark=x] table [y=retr, x=epoch] from \androidppl;
\addlegendentry{\srcselR}
\end{axis}
\end{tikzpicture}
\end{subfigure}
\begin{subfigure}[b]{0.48\textwidth}
\begin{tikzpicture}[font=\small]
\begin{axis}[nice, width=\hsize, height=\chartheight, xlabel=Epochs (Gaming), legend pos=north east, xmin=5, xmax=100, ymax=160]
\addplot table [y=tgt, x=epoch] from \gamingppl;
\addplot table [y=srcinit, x=epoch] from \gamingppl;
\addplot table [y=yangC, x=epoch] from \gamingppl;
\addplot table [y=regour, x=epoch] from \gamingppl;
\addplot[color=green!40!black,mark=x] table [y=retr, x=epoch] from \gamingppl;
\end{axis}
\end{tikzpicture}
\end{subfigure}
\begin{subfigure}[b]{0.48\textwidth}
\begin{tikzpicture}[font=\small]
\begin{axis}[nice, width=\hsize, height=\chartheight, xlabel=Epochs (Physics), ylabel={\%AUC}, legend pos=south east, xmin=5, xmax=150, legend columns=2]
\addplot table [y=tgt, x=epoch] from \physicsauc;
\addlegendentry{\tgt}
\addplot table [y=srcinit, x=epoch] from \physicsauc;
\addlegendentry{\srcinit}
\addplot table [y=yangC, x=epoch] from \physicsauc;
\addlegendentry{\yangC}
\addplot table [y=regour, x=epoch] from \physicsauc;
\addlegendentry{\srcselReg}
\addplot[color=green!40!black,mark=x] table [y=retr, x=epoch] from \physicsauc;
\addlegendentry{\srcselR}
\end{axis}
\end{tikzpicture}
\end{subfigure}
\begin{subfigure}[b]{0.48\textwidth}
\begin{tikzpicture}[font=\small]
\begin{axis}[nice, width=\hsize, height=\chartheight, xlabel=Epochs (Unix), legend pos=south east, xmin=5, xmax=160]
\addplot table [y=tgt, x=epoch] from \unixauc;
\addplot table [y=srcinit, x=epoch] from \unixauc;
\addplot table [y=regour, x=epoch] from \unixauc;
\addplot table [y=yangC, x=epoch] from \unixauc;
\addplot[color=green!40!black,mark=x] table [y=retr, x=epoch] from \unixauc;
\end{axis}
\end{tikzpicture}
\end{subfigure}
\caption{Language model perplexity (top row) and AUC on duplicate question detection (bottom row).}
\label{fig:srcsel:auc}
\end{figure*}
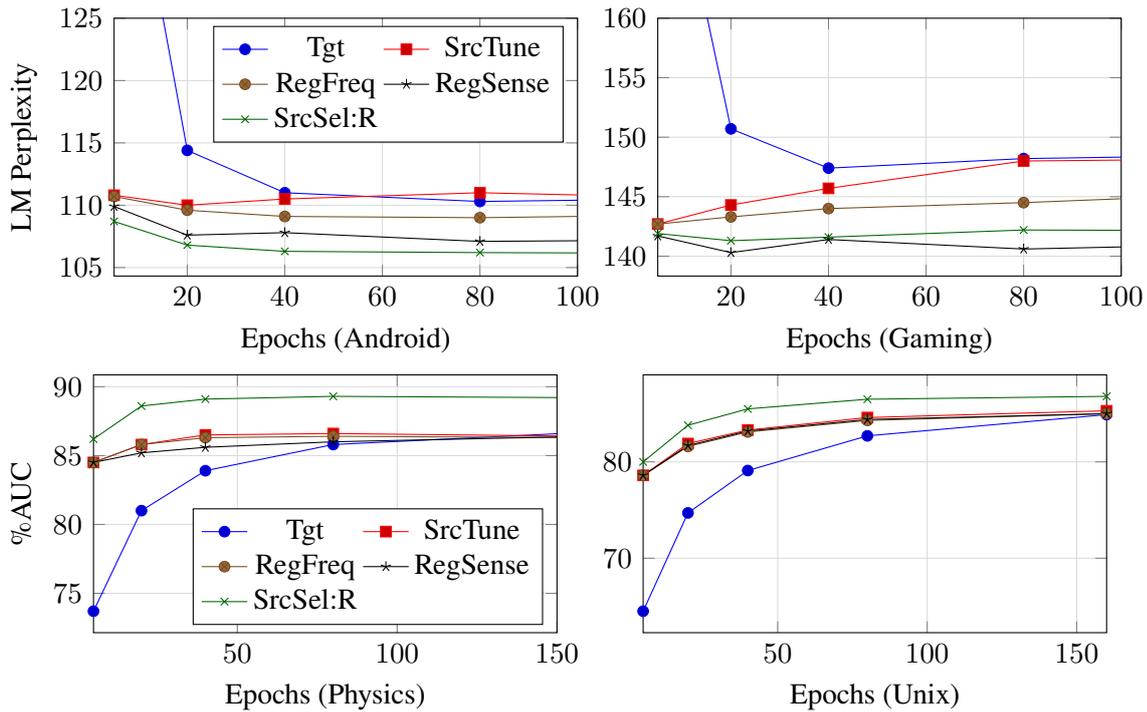

\noindent
{\bf Medical domain:}
This domain from the Ohsumed\footnote{\url{https://www.mat.unical.it/OlexSuite/Datasets/SampleDataSets-about.htm}} dataset has abstracts on cardiovascular diseases.  We sample 1.4\,MB of abstracts as target corpus~$\DT$.  We evaluate embeddings on two tasks: (1)~unsupervised language modeling on remaining abstracts, and (2)~supervised classification on 23 MeSH classes based on title.  We randomly select 10,000 titles with train, test, dev split as 50\%, 40\%, and 10\%.
Following~\citet{JoulinGBM17}, we train a softmax layer on the average of adapted (and pinned) word embeddings.

\noindent
{\bf Topics from 20~newsgroup:}
We choose the five top-level classes in the 20~newsgroup dataset\footnote{\url{http://qwone.com/~jason/20Newsgroups/}} as topics; viz.: \emph{Computer, Recreation, Science, Politics, Religion}.  The corresponding five downstream tasks are text classification over the 3--5 fine-grained classes under each top-level class.
Train, test, dev splits were 50\%, 40\%, 10\%.  We average over nine splits.  
The \texttt{body} text is used as $\DT$ and \texttt{subject} text is used for classification.

\noindent
Pretrained embeddings $\ES$  are trained on  Wikipedia using the default settings of word2vec's CBOW model.

\begin{table}[htb]
\centering
\begin{tabular}{|l|c|c|c|c|}
    \hline 
    Method & Physics & Gaming & Android & Unix \\
    \hline
    \tgt & 121.9 & 185.0 & 142.7 & 159.5 \\
    \tgt (unpinned) & -0.6 & -0.8 & 0.2 & 0.1 \\
    \hline
\end{tabular}
    \caption{Average reduction in perplexity, when embeddings are not pinned,
    on four Stackexchange topics.}
    \label{tab:srcsel:lm_ft}
\end{table}

\subsubsection*{Effect of fine-tuning embeddings on the target task}
We chose to pin embeddings in all our experiments, once adapted to the target corpus, namely the document classification task on medical and 20 newsgroup topics and language model task on five different topics. This is because we did not see any improvements when we unpin the input embeddings. We summarize in Table~\ref{tab:srcsel:lm_ft} the results when the embeddings are not pinned on language model task on the four StackExchange topics.  

\subsubsection*{Epochs vs.\ regularization results}
In Figure~\ref{fig:srcsel:auc} we show perplexity and AUC against training epochs.  Here we focus on four methods: \tgt, \srcinit, \yangC,  and \srcselReg.  First note that \tgt\ continues to improve on both perplexity and AUC metrics beyond five epochs (the default in word2vec code\footnote{\url{https://code.google.com/archive/p/word2vec/}} and left unchanged in \yangC{}\footnote{\url{https://github.com/Victor0118/cross_domain_embedding/}} \citep{regemnlp}).  In contrast, \srcinit, \srcselReg, and \yangC\ are much better than \tgt\ at five epochs, saturating quickly.  With respect to perplexity, \srcinit\ starts getting worse around 20 iterations and becomes identical to \tgt, showing catastrophic forgetting.   Regularizers in \yangC\ and \srcselReg\ are able to reduce such forgetting, with \srcselReg\ being more effective than \yangC.
These experiments show that any comparison that chooses a fixed number of training epochs across all methods is likely to be unfair.  Henceforth we will use a validation set for the stopping criteria.  While this is standard practice for supervised tasks, most word embedding code we downloaded ran for a fixed number of epochs, making comparisons unreliable.  We conclude that validation-based stopping is critical for fair evaluation.

\begin{table}[htb]
\centering
\begin{tabular}{|l|l|l|l|l|l|}
\hline 
Method & Physics & Gaming & Android & Unix & Medical \\
\hline
\tgt & 121.9 & 185.0 & 142.7 & 159.5 & 158.9 \\
\hline
\srcinit & $2.3_{\pm 0.7}$&$6.8_{\pm 0.3}$&$1.1_{\pm 0.3}$&$3.1_{\pm 0.0}$&$5.5_{\pm 0.8}$\\
\yangC &$2.1_{\pm 0.8}$&$7.1_{\pm 0.7}$&$1.8_{\pm 0.4}$&$3.4_{\pm 0.5}$&$6.8_{\pm 0.9}$\\
RegSens &$5.0_{\pm 0.1}$&$13.8_{\pm 0.3}$&$6.7_{\pm 0.8}$&$9.7_{\pm 0.3}$&$14.6_{\pm 1.0}$\\
\srcsel &$5.8_{\pm 0.9}$&$11.7_{\pm 0.6}$&$5.9_{\pm 1.2}$&$6.4_{\pm 0.1}$&$8.6_{\pm 3.0}$\\
\hline
\srcsel + RegSense &$6.2_{\pm 1.3}$&$12.5_{\pm 0.3}$&$7.9_{\pm 1.8}$&$9.3_{\pm 0.2}$&$10.5_{\pm 0.9}$\\
\hline
\end{tabular}
\caption{Average reduction ($\pm$ standard deviation) in language model perplexity over \tgt on five StackExchange Topics.}
\label{tab:srcsel:perplexity_with_sd}
\end{table}

\begin{table}[htb]
\centering
\begin{tabular}{|l|l|l|l|l|}
\hline
 & Physics & Gaming & Android & Unix 
 \\
\hline
\tgt & 86.7 & 82.6 & 86.8 & 85.4 
\\
\hline
\src & -2.3$_{\pm \text{0.5}}$ & 0.8$_{\pm \text{0.5}}$ & -3.7$_{\pm \text{0.5}}$ & -7.1$_{\pm \text{0.3}}$ 
\\
\concat & -1.1$_{\pm \text{0.5}}$ & 1.4$_{\pm \text{0.3}}$ & -2.1$_{\pm \text{0.3}}$ & -4.5$_{\pm \text{0.4}}$ 
\\
AAEME & 1.2$_{\pm \text{0.2}}$ & 4.6$_{\pm \text{0.0}}$ & -0.3$_{\pm \text{0.2}}$ & 0.0$_{\pm \text{0.2}}$ 
\\
\srcinit & -0.3$_{\pm \text{0.3}}$ & 1.9$_{\pm \text{0.2}}$ & 0.6$_{\pm \text{0.2}}$ & -0.0$_{\pm \text{0.2}}$ 
\\
\yangC & -0.4$_{\pm \text{0.2}}$ & 2.4$_{\pm \text{0.2}}$ & -0.5$_{\pm \text{0.5}}$ & -0.5$_{\pm \text{0.2}}$ 
\\
\srcselReg & -0.4$_{\pm \text{0.5}}$ & 2.2$_{\pm \text{0.1}}$ & -0.5$_{\pm \text{0.5}}$ & -0.5$_{\pm \text{0.4}}$ 
\\
\srcsel & 3.6$_{\pm \text{0.2}}$ & 3.0$_{\pm \text{0.2}}$ & 0.8$_{\pm \text{0.3}}$ & 2.1$_{\pm \text{0.2}}$ \\
\rowcolor{green!10}
\srcsel & 3.6$_{\pm \text{0.2}}$ & 3.1$_{\pm \text{0.5}}$ & 0.8$_{\pm \text{0.3}}$ &  2.1$_{\pm \text{0.2}}$ \\ 
\rowcolor{green!10} +\srcselReg & & & & \\
\hline
\end{tabular}
\caption{\label{tab:srcsel:auc}AUC gains over \tgt{} ($\pm$ standard deviation of difference) on duplicate question detection task on various target topics. AAEME is the auto-encoder meta-embedding of \citet{BaoB2018}.}
\end{table}

\begin{table}[htb]
\centering
\begin{tabular}{|l|l|l|l|c|}
\hline
       & \multicolumn{3}{c|}{Ohsumed} 
        & \multicolumn{1}{c|}{20NG Avg}
       \\
Method & Micro & Macro & Rare & 5 topics \\
\hline
\tgt & 26.3 & 14.7 & 3.0 & 88.9 \\ \hline
\src & -1.0$_{\pm \text{0.9}}$ & 0.$_{\pm \text{0.5}}$ & 0.$_{\pm \text{0.1}}$ & -3.9$_{\pm \text{1.2}}$ \\
AAEME & -1.0$_{\pm \text{0.9}}$ & 0.$_{\pm \text{0.5}}$ & 0.$_{\pm \text{0.1}}$ & -3.9$_{\pm \text{1.2}}$ \\
\srcinit & 1.7$_{\pm \text{1.0}}$ & 1.8$_{\pm \text{1.7}}$ & 1.5$_{\pm \text{2.0}}$ & 0.0$_{\pm \text{1.6}}$ \\
\yangC & 0.6$_{\pm \text{0.5}}$ & 1.8$_{\pm \text{2.3}}$ & 3.7$_{\pm \text{4.7}}$ & - \\
\srcselReg & 1.4$_{\pm \text{0.5}}$ & 2.5$_{\pm \text{1.2}}$ & 4.0$_{\pm \text{1.8}}$ & 0.4$_{\pm \text{1.3}}$ \\
\srcsel & 2.0$_{\pm \text{0.9}}$ & 2.6$_{\pm \text{1.5}}$ & 1.1$_{\pm \text{1.4}}$ & 0.5$_{\pm \text{1.5}}$ \\
\rowcolor{green!10}
\srcsel & \hl{2.3$_{\pm \text{0.7}}$} & \hl{3.4$_{\pm \text{1.3}}$} & \hl{4.3$_{\pm \text{1.2}}$} &  \hl{0.5$_{\pm \text{1.5}}$} \\ 
\rowcolor{green!10} +\srcselReg & & & & \\ 
\hline
\end{tabular} 
\caption{Average accuracy gains over \tgt{} ($\pm$ std-dev) on Ohsumed and 20NG datasets.  We show macro and rare class accuracy gains for Ohsumed because of its class population skew.  Per-topic 20NG gains are in \tablename~\ref{tab:srcsel:classify:ng_detailed} in Appendix.}
\label{tab:srcsel:classify} 
\end{table}

\noindent
We next compare \srcinit, \yangC, and \srcselReg{} on the three tasks: perplexity in \tablename~\ref{tab:srcsel:perplexity_with_sd}, duplicate detection in \tablename~\ref{tab:srcsel:auc}, and classification in \tablename~\ref{tab:srcsel:classify}. All three methods are better than baselines \src\ and \concat, which are much worse than \tgt\ indicating the presence of significant concept drift.  \citet{regemnlp} provided no comparison between \yangC{} (their method) and \srcinit; we find the latter slightly better.
\noindent
On the supervised tasks, \yangC\ is often worse than \tgt\ provided \tgt\ is allowed to train for enough epochs.  If the same number of epochs are used to train the two methods, one can reach the misleading conclusion that \tgt\ is worse.  \srcselReg\ is better than \srcinit\ and \yangC\ particularly with respect to perplexity, and rare class classification (Table~\ref{tab:srcsel:classify}).  We conclude that a well-designed word stability-based regularizer can improve upon epoch-based fine-tuning.

\begin{table}[htb]
\centering
\begin{tabular}{|l|l|l|l|l|}
\hline
 & Physics & Gaming & Android & Unix  \\
\hline
\multicolumn{5}{|c|}{\yangC's reduction in Perplexity over \tgt}   \\
\hline
Original & 1.1$_{\pm \text{1.1}}$ & 1.5$_{\pm \text{1.2}}$ & 0.9$_{\pm \text{0.1}}$ & 0.7$_{\pm \text{0.8}}$  \\
+SrcInit & \hl{2.1$_{\pm \text{0.9}}$} & \hl{5.7$_{\pm \text{0.8}}$} & \hl{1.1$_{\pm \text{0.5}}$} & \hl{2.1$_{\pm \text{0.8}}$}  \\
\hline
\multicolumn{5}{|c|}{\yangC's gain in AUC over \tgt}   \\
\hline
Original & -1.2$_{\pm \text{0.4}}$ & 0.1$_{\pm \text{0.1}}$ & \hl{-0.2$_{\pm \text{0.1}}$} & \hl{-0.4$_{\pm \text{0.1}}$} 
\\ 
+SrcInit & \hl{-0.4$_{\pm \text{0.2}}$} & \hl{2.4$_{\pm \text{0.2}}$} & -0.5$_{\pm \text{0.5}}$ & -0.5$_{\pm \text{0.2}}$ 
\\
\hline
\end{tabular}
\caption{Effect of initializing with source embeddings.  We show mean gains over \tgt\ over 9 runs ($\pm$ std-dev).}
\label{tab:srcsel:init}
\end{table}

\subsubsection*{Impact of source initialization}
\tablename~\ref{tab:srcsel:init} compares \tgt\ and \yangC\ with two initializers: (1)~random  as proposed by \citet{regemnlp}, and (2)~with source embeddings.  \yangC\ after source initialization is better in almost all cases.  \srcsel\ and \srcselReg\ also improve with source initialization, but to a smaller extent.  (More detailed numbers are in \tablename~\ref{tab:srcsel:init:all} of Appendix.)  We conclude that initializing with pretrained embeddings is helpful even with regularizers.

\subsubsection*{Comparison with Meta-embeddings } 
In Tables~\ref{tab:srcsel:auc} and ~\ref{tab:srcsel:classify} we show results with the most recent meta-embedding method AAEME.
AAEME provides gains over \tgt\ in only two out of six cases\footnote{On the topic classification datasets in Table~\ref{tab:srcsel:classify}, AAEME and its variant DAEME were worse than \src.  We used the dev set to select the better of \src\ and their best method.}.

\subsubsection*{Performance of \srcsel}

We next focus on the performance of \srcsel\ on all three tasks: perplexity in \tablename~\ref{tab:srcsel:perplexity_with_sd}, duplicate detection in \tablename~\ref{tab:srcsel:auc}, and classification in \tablename~\ref{tab:srcsel:classify}.  \srcsel\ is always among the best two methods for perplexity.  In supervised tasks, \srcsel\ is the only method that provides significant gains for all topics: AUC for duplicate detection increases by 2.4\%, and classification accuracy increases by 1.4\% on average. \srcsel+\srcselReg\ performs even better than \srcsel\ on all three tasks particularly on rare words. An ablation study on other variants of \srcsel\ appear in the Appendix.
\begin{table}[htb]
\centering
\begin{tabular}{l|l|l|l|l|l}
  Pair                &     Tgt & SrcTune & RegFreq & RegSense  &   SrcSel \\ \hline
Unix topic &  & & & & \\ \hline
nice, kill & 4.6 & 4.5 & 4.4 & 4.4 & 5.2 \\
vim, emacs & 5.7 & 5.8 & 5.7 & 5.8 & 6.4 \\                      
print, cat & 5.0 & 4.9 & 4.9 & 5.0 & 5.4 \\
kill, job & 5.2 & 5.1 & 5.2 & 5.3 & 5.8 \\
make, install & 5.1 & 5.1 & 5.3 & 5.7 & 5.8 \\
character, unicode & 4.9 & 5.1 & 4.7 & 4.6 & 5.8 \\ \hline

Physics topic &  & & & & \\ \hline
lie, group         & 5.2 & 5.0 & 4.4 & 5.1 & 5.8 \\
current, electron  & 5.3 & 5.3 & 4.7 & 5.3 & 5.7 \\
potential, kinetic & 5.8 & 5.8 & 4.5 & 5.9 & 6.1 \\
rotated, spinning & 5.0 & 5.7 & 6.0 & 5.1 & 5.6 \\
x-ray, x-rays & 5.3 & 7.0 & 6.1 & 5.5 & 6.4 \\
require, cost & 4.9 & 6.2 & 5.2 & 5.1 & 5.3 \\
cool, cooling & 5.6 & 6.0 & 6.4 & 5.7 & 5.7 \\
\hline
\end{tabular}
\caption{\label{tab:srcsel:anecdotes} Example word pairs and their normalized similarity across different methods of training embeddings.}
\end{table}

\subsubsection*{Additional Experiments}
\noindent
{\bf Word-pair similarity improvements:} 
In \tablename~\ref{tab:srcsel:anecdotes}, we show  normalized\footnote{We sample a set $S$ of 20 words based on their frequency. Normalized similarity between $a$ and $b$ is $\frac{\cos(a,b)}{\sum_{w \in (S \cup b)}\cos(a,w)}$. Set $S$ is fixed across methods.} cosine similarity of word pairs pertaining to the Physics and Unix topics.  Observe how word pairs like ({\em nice, kill}), ({\em vim, emacs})
in Unix and 
({\em current, electron}), ({\em lie, group})
in Physics are brought closer together as a result of importing the larger unix/physics subset from~$\DS$.  In each of these pairs,  words (e.g. {\em nice, vim, lie, current}) have a different prominent sense in the source (Wikipedia).  Hence, methods like \srcinit, and \srcselReg\ cannot help.  In contrast, word pairs like (cost, require), (x-ray, x-rays) whose sense is the same in the two corpus benefit significantly from the source across all methods.

\noindent
{\bf Running time:} \srcsel\ is five times slower than \yangC, which is still eminently practical. $\widehat{\DS}$ was within $3\times$ the size of $\DT$ in all domains.  If $\DS$ is available, \srcsel\ is a practical and significantly more accurate option than adapting pretrained source embeddings. \srcsel+\srcselReg\ complements \srcsel\ on rare words, improves perplexity, and is never worse than \srcsel.

\begin{table}[htb]
\centering
  \begin{tabular}{|l|r|r|r|r|r|}
    \hline
    & Physics & Game & Android & Unix & Med (Rare) \\
    \hline
    \tgt & 89.7 & 88.4 & 89.4 & 89.2 & 9.4 \\
    \hline
    \srcinit & $-0.2$ & 0.6 & $-0.4$ & $-0.2$
    & $-2.1$ \\
   \srcsel  & $1.9$ & 0.5 & 0.0 & $-0.2$
    & 1.1 \\
    \hline
  \end{tabular}
  \caption{
  \label{tab:srcsel:large:c}Performance with a larger target corpus size of 10MB on the four deduplication tasks (AUC score) and one classification task (Accuracy on rare class). More details in Table~\ref{tab:srcsel:large} of Appendix.}
\end{table}

\noindent
{\bf Effect of target corpus size:}
The problem of importing source embeddings is motivated only when target data is limited. When we increase target corpus 6-fold, the gains of \srcsel\ and \srcinit\ over \tgt\ was insignificant in most cases. However, infrequent classes continued to benefit from the source as shown in~\tablename~\ref{tab:srcsel:large:c}.

\begin{table}[!ht]
\centering
  \begin{tabular}{|l|l|l|l|l|l|}
    \hline
    & Physics & Gaming & Android & Unix & Medical \\
    \hline
    \tgt & 86.7 & 82.6 & 86.8 & 85.4 & 26.3 \\
    \hline
    Elmo & -1.0$_{\pm \text{0.4}}$ & 4.5$_{\pm \text{0.3}}$ & -1.5$_{\pm \text{0.8}}$ & -2.3$_{\pm \text{0.3}}$ 
    & 3.2$_{\pm \text{1.3}}$ \\
    +\tgt & -0.8$_{\pm \text{0.4}}$ & 3.8$_{\pm \text{0.4}}$ & 0.5$_{\pm \text{0.1}}$ & -0.0$_{\pm \text{0.1}}$ 
    & 4.1$_{\pm \text{1.5}}$ \\
    +ST & -0.5$_{\pm \text{0.3}}$ & 3.0$_{\pm \text{0.2}}$ & 0.3$_{\pm \text{0.5}}$ & 0.2$_{\pm \text{0.2}}$ 
    & 3.5$_{\pm \text{0.6}}$ \\
    +\srcsel & 2.6$_{\pm \text{0.5}}$ & 4.1$_{\pm \text{0.1}}$ & 1.1$_{\pm \text{0.4}}$ & 1.5$_{\pm \text{0.2}}$ 
    & 4.6$_{\pm \text{0.9}}$ \\
    \hline
  \end{tabular}
  \caption{Gains ($\pm$ std-dev) over \tgt\ with contextual embeddings on duplicate detection (columns 2--5) and classification (column 6).}
  \label{tab:srcsel:ctxt}
\end{table}

\subsubsection*{Do Contextual Embeddings Obviate Adaptation?}

We explore if contextual word embeddings obviate the need for adapting source embeddings, in the ELMo \citep{PetersNIGCKZ2018ELMo} setting, a contextualized word representation model, pretrained on a 5.5B token corpus\footnote{\url{https://allennlp.org/elmo}}. 
We compare ELMo's contextual embeddings as-is, and also after concatenating them with each of \tgt, \srcinit, and \srcsel\ embeddings in Table~\ref{tab:srcsel:ctxt}.  First, ELMo+\tgt\ is better than \tgt\ and ELMo individually.  This shows that contextual embeddings are useful but they do not eliminate the need for topic-sensitive embeddings.  Second, ELMo+\srcsel\ is better than ELMo+\tgt.  Although \srcsel\ is trained on data that is a strict subset of ELMo, it is still instrumental in giving gains since that subset is aligned better with the target sense of words.
We conclude that topic-adapted embeddings can be useful, even with ELMo-style contextual embeddings.

Recently, BERT~\citep{devlin2018bert} has garnered a lot of interest for beating contemporary contextual embeddings on all the GLUE tasks. 
We evaluate BERT on question duplicate question detection task on the four StackExchange topics. We use pretrained BERT-base, a smaller 12-layer transformer network, for our experiments. We train a classification layer on the final pooled representation of the sentence pair given by BERT to obtain the binary label of whether they are duplicates. This is unlike the earlier setup where we used EMD on the fixed embeddings.

To evaluate the utility of a 
relevant topic focused corpus, we fine-tune the pretrained checkpoint either on $\DT$ (\srcinit) or on $\DT \cup \widehat{\DS}$ (\srcselR) using BERT's masked language model loss. The classifier is then initialized with the fine-tuned checkpoint. Since fine-tuning is sensitive to the number of update steps, we tune the number of training steps using performance on a held-out dev set.  F1 scores corresponding to different initializing checkpoints are shown in table~\ref{tab:srcsel:bert_ft}. It is clear that  pretraining the contextual embeddings on relevant target corpus helps in the downstream classification task. However, the gains of \srcselR\ over \tgt\ is not clear. This could be due to incomplete or noisy sentences in $\widehat{\DS}$. There is need for more experimentation and research to understand the limited gains of \srcselR\ over \srcinit\ in the case of BERT. We leave this for future work.

\begin{table}[htb]
\centering
\begin{tabular}{|l|r|r|r|r|}
\hline
Method & Physics & Gaming & Android & Unix \\
\hline
BERT & 87.5 & 85.3 & 87.4 & 82.7 \\
\srcinit & {\bf 88.0} & {\bf 89.2} & 88.5 & 83.5 \\
\srcselR & 87.9 & 88.4 & {\bf 88.6} & {\bf 85.1} \\
\hline
\end{tabular}
\caption{F1 scores on question de-duplication task using BERT-base and when fine-tuned on Tgt only ($\DT$) and Tgt and selected source ($\DT \cup \widehat{\DS}$)}
\label{tab:srcsel:bert_ft}
\end{table}


\newcommand\our{KYC}
\def\longname{Know Your Client}

\newcommand{\lastE}{M}

\section{On-the-fly Adaptation}
\label{sec:kyc}
State-of-the-art NLP inference uses enormous neural architectures and models trained for GPU-months, well beyond the reach of most consumers of NLP.  This has led to one-size-fits-all public API-based NLP service models by major AI companies, serving large numbers of clients.  Neither (hardware deficient) clients nor (heavily subscribed) servers can afford traditional fine tuning.  Many clients own little or no labeled data.  We now describe our study of adaptation of centralized NLP services to clients.  



Distributional mismatch between the giant general-purpose corpus used to train the central service and the corpus from which a client's instances arise leads to lower accuracy.  A common source of trouble is mismatch of word salience \citep{paik2013novel} between client and server corpora~\citep{Ruder2019Neural}.  
%
In this respect, our setting also presents a new opportunity.  Clients are numerous and form natural clusters, e.g., healthcare, sports, politics.  We want the service to exploit commonalities in existing client clusters, without explicitly supervising this space, and provide some level of generalization to new clients without re-training or fine-tuning.  

In response to the above challenges and constraints, we initiate an investigation of practical protocols for lightweight client adaptation of NLP services.  We propose a system, \our~(``\longname''), in which each client registers with the service using a simple sketch derived from its (unlabeled) corpus.
The service network takes the sketch as an additional input with each instance later submitted by the client and provides accuracy benefits to new clients immediately.

What form can a client sketch take? How should the service network incorporate it?  While this will depend on the task, we initiate a study of these twin problems focused on predictive language modeling, sentiment labeling, and named entity recognition (NER).  We show that a simple late-stage  intervention in the server network gives visible accuracy benefits, and provide diagnostic analyses and insights.  

\section{Immediate Adaptation through Corpus Sketch Registration}
\label{sec:kyc:Proposal}
We formalize the constraints on the server and client in the API setting.
(1) The server is expected to scale to a large number of clients making it impractical to adapt to individual clients.
(2) After registration, the server is expected to provide labeling immediately and response latency per instance must be kept low implying that the server's inference network cannot be too compute-intensive.
(3) Finally, the client cannot perform complex pre-processing of every instance before sending to the server, and does not have any labelled data.

\subsection{Model Design}
\paragraph{Server network and model}
These constraints lead us to design a server model that \emph{learns to   compute} client-specific model parameters from the client sketch, and requires no client-specific 
\begin{wrapfigure}{r}{0.3\hsize}
\centering
\begin{tikzpicture}[>=stealth',align=center]
\node (jl) {loss};
\node [left=3mm of jl] (y) {$\bm{y}$};
\draw [->] (y) to (jl);
\node [rectangle, draw,
minimum width=9mm, below=3mm of jl] (Stheta) {$Y_\theta$};
\draw [->] (Stheta) to (jl);
\node [circle,draw, fill=yellow!15, 
inner sep=.1mm, below=2mm of Stheta] (plus) {$+$};
\draw [->] (plus) to (Stheta);
\node [rectangle, draw, anchor=center,
minimum width=9mm, minimum height=8mm,
below left=6mm and 2mm of plus] (Etheta) {$\lastE_\theta$};
\node [anchor=center, fill=yellow!15,
below right=3mm and 2mm of plus,
outer sep=.1mm, inner sep=0mm] (g) {$\bm{g}$};
\draw [->] (Etheta) to (plus);
\draw [->] (g) to (plus);
\node [rectangle, draw, anchor=center, fill=yellow!15,
minimum width=9mm, minimum height=8mm,
below=3mm of g] (Gphi) {$G_\phi$};
\draw [->] (Gphi) to (g);
\node [anchor=center,below=6mm of Etheta] (x) {$\bm{x}$};
\node [fill=yellow!15,
anchor=center,below=3mm of Gphi] (D) {$S_c$};
\draw [->] (x) to (Etheta);
\draw [->] (D) to (Gphi);
\draw [->, dotted] (D) to (x);
\end{tikzpicture}
\caption{\raggedright \our{} overview.}  \label{fig:kyc:overview}
\end{wrapfigure}
fine-tuning or parameter learning.
The original server network is written as $\hat{\bm{y}} = Y_\theta(\lastE_\theta(\bm{x}))$ where $\bm{x}$ is the input instance, and
$Y_\theta$ is a softmax layer to get the predicted label $\hat{\bm{y}}$.
$\lastE_\theta$ is a representation learning layer that may take diverse forms depending on the task; of late, BERT \citep{devlin2018bert} is used to design $\lastE_\theta$ for many tasks.

We augment the server network to accept, with
each input $\bm{x}$, a client-specific sketch~$S_c$ as shown in \figurename~\ref{fig:kyc:overview}.     
We discuss possible forms of $S_c$ in the next subsection.  (The dotted arrow represents a generative influence of $S_c$ on $\bm{x}$.)  The server implements an auxiliary network $\bm{g} = G_\phi(S_c)$.  Here $\bm{g}$ can be regarded as a neural digest of the client sketch.
Module $\bigoplus$ combines $\lastE_\theta(\bm{x})$ and $\bm{g}$; concatenation was found adequate on the tasks we evaluated but we also discuss other options in Section~\ref{sec:kyc:Expt}.  
When the $\bigoplus$ module is concatenation we are computing a client-specific per-label bias, and even that provides significant gains, as we show in Section~\ref{sec:kyc:Expt}.

\subsection{Corpus Sketch Design}
\paragraph{Client sketch:}
The design space of client sketch $S_c$ is infinite.  We initiate a study of designing $S_c$ from the perspective of term weighting and salience in Information Retrieval \citep{paik2013novel}.  $S_c$ needs to be computed once by each client, and thereafter reused with 
every input instance~$\bm{x}$.  
Ideally, $S_c$ and $G_\phi$ should be locality preserving, in the sense that clients with similar corpora and tasks should lead to similar~$\bm{g}$s. Suppose the set of clients already registered is~$C$.

A simple client sketch is just a vector of counts of all words in the client corpus.  Suppose word $w$ occurs $n_{c,w}$ times in a client $c$, with $\sum_w n_{c,w}=N_c$.  Before input to $G_\phi$, the server normalizes these counts using counts of other clients as follows:
From all of $C$, the server will estimate a background unigram rate of word. 
Let the estimated rate for word $w$ be~$p_w$, which is calculated as:
\begin{align}
p_w &= \textstyle (\sum_{c\in C}n_{c,w})\left/\left(\sum_w \sum_{c\in C}n_{c,w}\right)\right..
\end{align}
The input into $G_\phi$ will encode, for each word $w$, how far the occurrence rate of $w$ for client $c$ deviates from the global estimate.  Assuming the multinomial word event distribution, the marginal probability of having $w$ occur $n_{c, w}$ times at client $c$ is proportional to $p_w^{n_{c,w}} (1 - p_w)^{(N_c - n_{c,w})}$.  We finally pass a vector containing the normalized negative log probabilities as input to the model:
\begin{align}
S_c \propto \Bigl( - n_{c,w}\log p_w 
- (N_c - n_{c,w})\log(1-p_w): \forall w \Bigr).
\label{eq:kyc:salience}
\end{align}
We call this the {\bf term-saliency} sketch.  We discuss other sketches like TF-IDF and corpus-level statistics like average instance length in Sec.~\ref{sec:kyc:ablation}.  

\section{Experiments}
\label{sec:kyc:Expt}
We evaluate \our\ on three NLP tasks as services: NER, sentiment classification, and auto-completion based on predictive language modeling. 
We compare \our\ against the baseline model (without the $G_\phi$ network in \figurename~\ref{fig:kyc:overview}) and the mixture of experts (MoE) model \citep{GuoSB18} (see Appendix~\ref{sec:kyc:appendix:lm}).  For all three models, the $\lastE_\theta$ network is identical in structure.  
In \our,  $G_\phi$ has two linear layers with ReLU giving a 128-dim vector~$\bm{g}$, with slight exceptions (see Appendix~\ref{sec:kyc:appendix:lm}). 
We choose datasets that are partitioned naturally across domains, used to simulate clients.  We evaluate in two settings: in-distribution (ID) on test instances from clients seen during training, and out-of-distribution (OOD) on instances from unseen clients.  For this, we perform a leave-k-client-out evaluation where given a set $D$ of clients, we remove $k$ clients as OOD test and use remaining $D - k$ as the training client set $C$. 

\noindent{\bf Named Entity Recognition (NER):} We use Ontonotes~\citep{ontonotes} which has 18 entity classes 
from 31 sources which forms our set $D$ of clients.  We perform leave-2-out test five times with 29 training clients as $C$. 
We train a cased BERT-based NER model~\citep{devlin2018bert} 
and report F-scores.   
%
\setlength\tabcolsep{2.0pt}
\begin{table}
\centering
\begin{tabular}{|l | c c c | c c c |}
\hline
& \multicolumn{3}{|c|}{\small{OOD}} & \multicolumn{3}{|c|}{\small{ID}}  \\
OOD Clients & \small{Base} & \small{MoE} & \small{{\our}} & \small{Base}& \small{MoE} & \small{{\our}} \\
\hline
\small{BC/CCTV+Phoenix} & 63.8 & 66.9 & 71.8 & 86.0 & 83.8 & 86.7 \\
\small{BN/PRI+BN/VOA} & 88.7 & 87.9 & 90.7 & 84.5 & 83.0 & 86.0\\
\small{NW/WSJ+Xinhua} &73.9 & 78.9 & 80.9 & 80.8 & 77.2 & 82.5\\
\small{BC/CNN+TC/CH} & 78.3 & 75.2 & 78.7 & 85.6 & 82.7 & 87.4\\
\small{WB/Eng+WB/a2e} & 76.2 & 69.9 & 78.4 & 86.4 & 82.6 & 87.3 \\
\hline
Average & 76.2 & 75.8 & \textbf{80.1} & 84.7 & 81.9 & \textbf{86.0}\\
\hline
\end{tabular}
\vspace{0.3em}
\caption{Test F1 on Ontonotes NER. OOD numbers are on the two listed domains whereas ID numbers are on test data of clients seen during training.}
\label{table:kyc:ner_numbers}
\end{table}
Table \ref{table:kyc:ner_numbers} shows that \our\ provides substantial gains for OOD clients. For the first two OOD clients (BC/CCTV,Phoenix), the baseline F1 score jumps from 63.8 to 71.8. MoE performs worse than baseline. We conjecture this is because separate softmax parameters over the large NER label space is not efficiently learnable. 
    
\noindent{\bf Sentiment Classification:} 
We use the popular Amazon dataset~\citep{amazon_dataset} with each product genre simulating a client.  
We retain genres with more than 1000 positive and negative reviews each and randomly sample 1000 positive and negative reviews from these 22 genres. We perform leave-2-out evaluation five times and Table~\ref{table:kyc:sentiment_numbers} shows the five OOD genre pairs. 
We use an uncased BERT model for classifcation~\citep{sentiment_repo}. 
\begin{table}
\centering
    \begin{tabular}{|l |c c c |c c c |}
    \hline
    & \multicolumn{3}{|c|}{\small{OOD}} & \multicolumn{3}{|c|}{\small{ID}}  \\
    OOD Clients & \small{Base }& \small{MoE} & \small{{\our}} & \small{Base}& \small{MoE} & \small{{\our}} \\
    \hline
    
    \small{Electronics+Games} & 86.9 & 87.4 & 87.7 & 88.6 & 88.7 &  89.0 \\
    \small{Industrial+Tools} & 87.6 & 88.3 &  87.7 & 88.4 & 88.8 &  88.9 \\
    \small{Books+Kindle Store} & 83.4 & 84.6 &  84.1 & 88.2 & 88.8 &  88.7 \\
    \small{CDs+Digital Music} & 82.4 & 83.0 &  83.2 & 89.0 & 88.9 &  88.9 \\
    \small{Arts+Automotive} & 90.2 & 90.6 &  90.4 & 88.4 & 88.6 &  88.6 \\
    \hline
    Average & 86.1 & 86.8  & 86.6 & 88.5 & 88.8 & \textbf{88.9} \\ \hline
    \end{tabular}
    \caption{Test Accuracy on Amazon Sentiment Data.}
    \label{table:kyc:sentiment_numbers}
\end{table}

Table~\ref{table:kyc:sentiment_numbers} shows that average  OOD client accuracy increases from 86.1  to 86.8 with \our. 
    
\noindent{\bf Auto-complete Task:}
We model this task as a forward language model 
and measure perplexity. We used the 20 NewsGroup dataset and treat each of the twenty topics as a client. Thus $D$ is of size 20. 
We use the state-of-art Mogrifier LSTM~\citep{MelisTP2020}. We perform leave-1-topic-out evaluation six times and OOD topics are shown in Table~\ref{tab:kyc:expts:lm}. For MoE, the client-specific parameter is only the bias and not the full softmax parameters, since that would otherwise blow up the number of trainable parameters and performed worser.
\begin{table}[htb]
\centering
    \begin{tabular}{|l|rrr|rrr|}
    \hline
    OOD & \multicolumn{3}{|c|}{OOD} & \multicolumn{3}{|c|}{ID}  \\
    Clients & Base & MoE & {\our} & Base & MoE & {\our}\\ \hline
    sci.space & 29.6 & 30.9 & 29.0 & 28.8 & 30.7 & 28.1\\
    comp.hw & 26.5 & 28.6 & 26.4 & 28.1 & 28.7 & 27.6\\
    sci.crypt & 29.7 & 29.8 & 29.6 & 27.8 & 28.1 & 27.7\\
    atheism & 28.3 & 28.1 & 28.1 & 27.9 & 28.2 & 28.0\\
    autos & 28.0 & 28.4 & 27.9 & 27.7 & 28.2 & 27.7\\
    mideast & 27.4 & 26.7 & 27.3 & 28.4 & 27.9 & 27.7\\ \hline
    Average & 28.2 & 28.7 & {\bf 27.9} & 28.0 & 28.8 & {\bf 27.7}\\ \hline
    \end{tabular}
\caption{Perplexity comparison between the baseline and {\our} on 20-NewsGroup dataset.}
\label{tab:kyc:expts:lm}
\end{table}
Table~\ref{tab:kyc:expts:lm} shows 
that \our\ performs consistently better than the baseline with average perplexity drop from 28.2 to 27.9. This drop is particularly significant because the Mogrifier LSTM is a strong baseline to start with. 
MoE is worse than baseline. 

\begin{figure}[htb] 
  \begin{subfigure}[b]{\linewidth}
    \centering
    \includegraphics[width=0.85\linewidth]{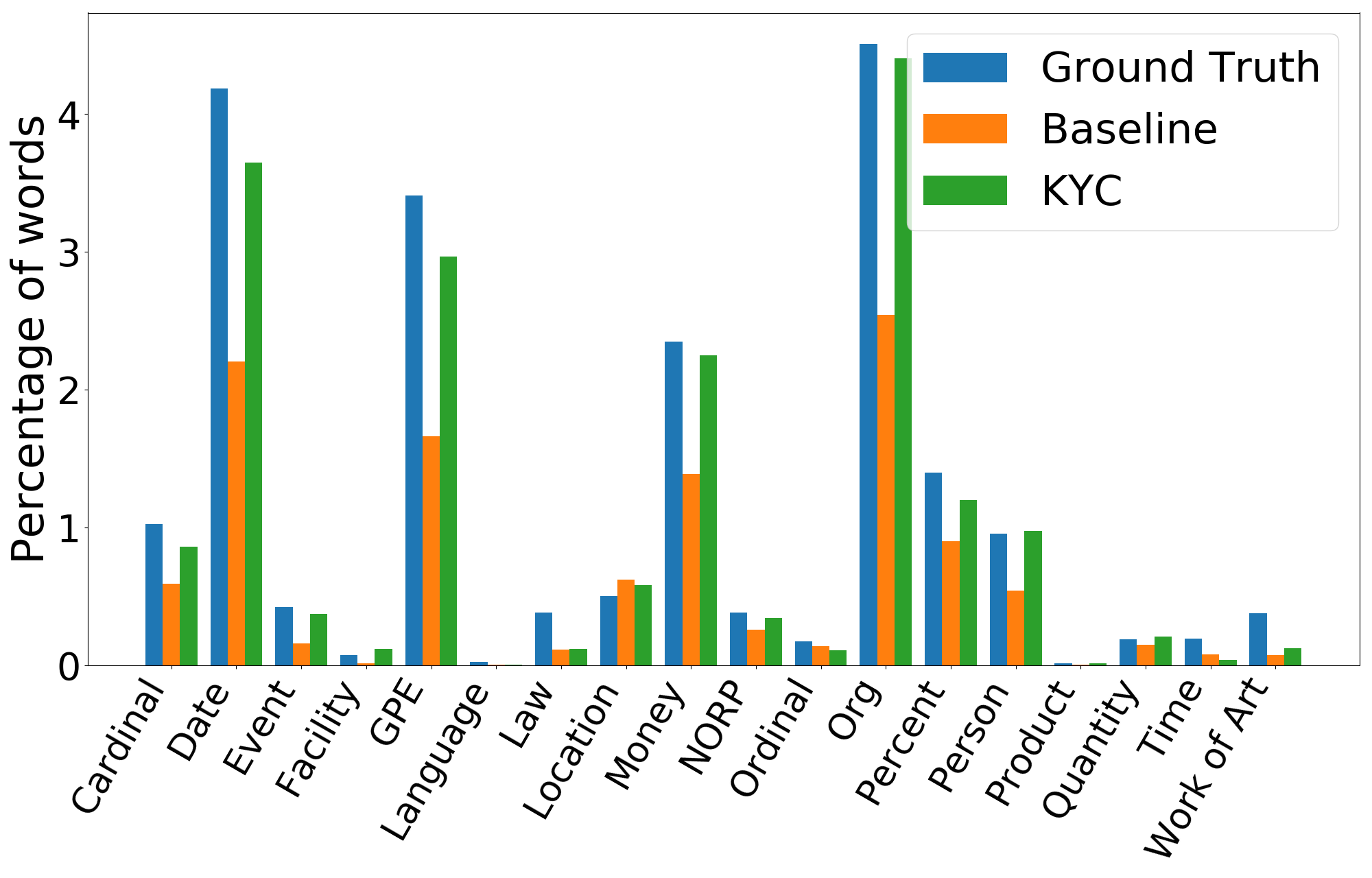} 
    \label{fig:kyc:bar_graph_xinhua} 
  \end{subfigure}

\caption{Proportion of true and predicted entity labels on OOD client NW/Xinhua.  Similar trends observed on other OOD domains~(Figure~\ref{fig:kyc:bar_graphs_more} of Appendix).}
  \label{fig:kyc:bar_graphs}
\end{figure}

\begin{figure}[htb] 
    \centering
    \includegraphics[width=.85\hsize]{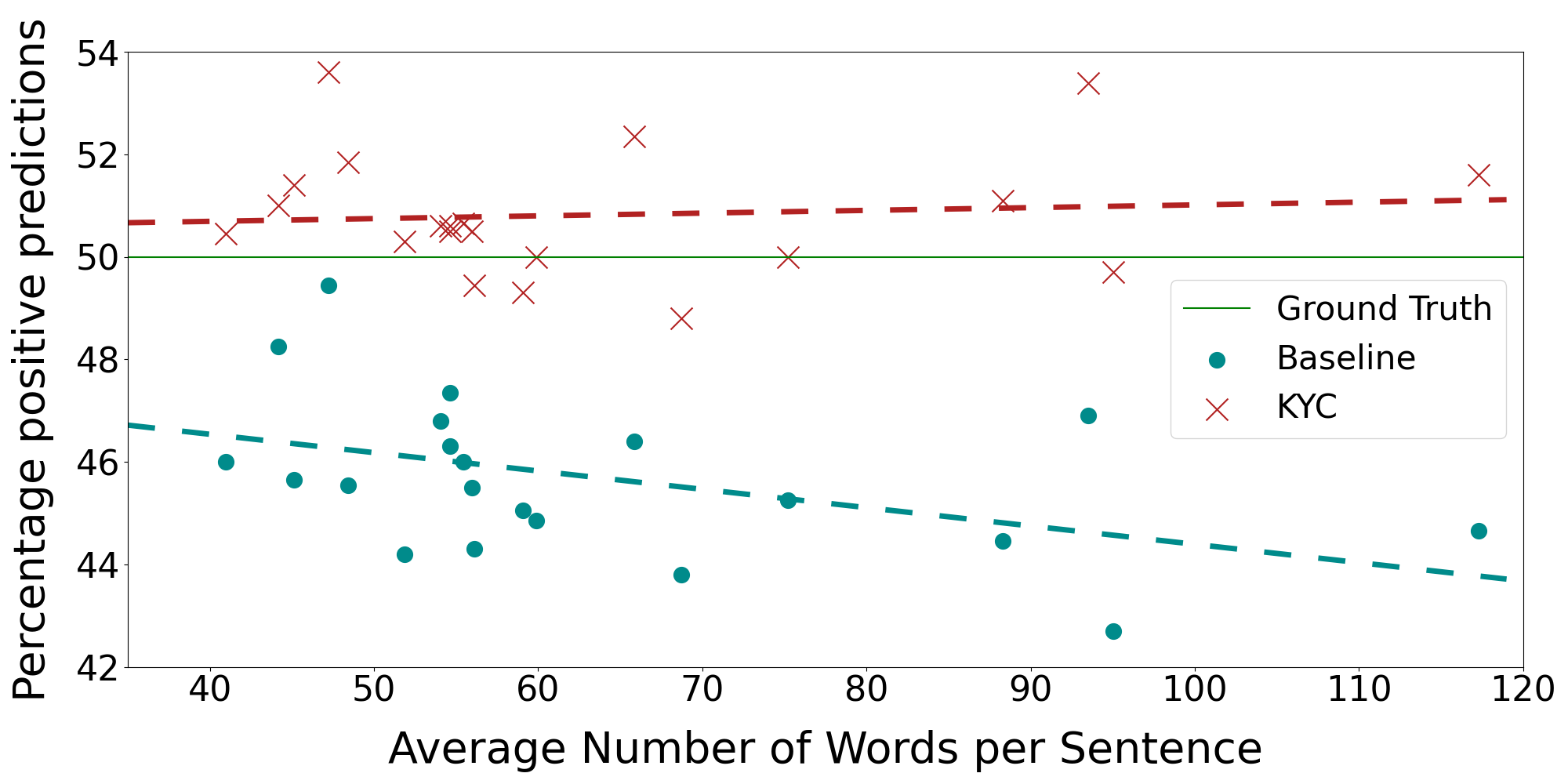} 
    \caption{Fraction Positive Predicted versus average review length by baseline and \our. Each dot/cross is a domain and the dotted lines indicate the best fit lines.} 
    \label{fig:kyc:sentiment_observation} 
\end{figure}

\paragraph{Statistical Significance: }
We verify the statistical significance of the gains obtained for the Sentiment Analysis and Auto-complete tasks; the gains in the case of NER are much larger than statistical variation. Shown in Tables \ref{table:kyc:sentiment_stat_significance} and \ref{table:kyc:auto_comp_stat_significance} are the sample estimate and standard deviation for three runs along with the p value corresponding to the null hypothesis of significance testing. In both cases, we see that the gains of {\our} over the baseline are statistically significant.

\begin{table}[htb]
\centering
\begin{tabular}{|l | c c | c|}
\hline
    OOD Clients & Base & {\our} & p-value \\
    \hline
    \small{Electronics+Games} & 86.9(0.39) & 87.7(0.33) & 0.05\\
    \small{Industrial+Tools} & 87.6(0.19) & 87.7(0.09) & 0.14\\
    \small{Books+Kindle Store} & 83.4(0.03) & 84.1(0.14) & 0.01 \\
    \small{CDs+Digital Music} & 82.4(0.24) & 83.2(0.08) & 0.02 \\
    \small{Arts+Automotive} & 90.2(0.21) & 90.4(0.31) & 0.20\\
    \hline
    Average & 86.1(0.16) & 86.6(0.13) & 0.02\\ \hline
\end{tabular}
\vspace{0.3em}
\caption{Statistical significance of results on the OOD clients by {\our} for Sentiment Classification. For every entry contains the mean with the standard deviation in parenthesis}
\label{table:kyc:sentiment_stat_significance}
\end{table}

\begin{table}[htb]
\centering
\begin{tabular}{|l | c c | c|}
\hline
    OOD Clients & Base & {\our} & p-value \\
    \hline
    sci.space & 26.5(0.4) & 26.4(0.2) & 0.39\\
    comp.hw & 29.6(0.4) & 29.0(0.3) & 0.07\\
    sci.crypt & 29.7(0.4) & 29.6(0.7) & 0.46 \\
    atheism & 28.3(0.2) & 28.1(0.2) & 0.14 \\
    autos & 28.0(0.5) & 27.9(0.4) & 0.34\\
    mideast & 27.4(0.4) & 27.3(0.4) & 0.37\\
    \hline
    Average & 28.2(0.2) & 27.9(0.0) & 0.04\\ \hline
\end{tabular}
\vspace{0.3em}
\caption{Statistical significance of results on the OOD clients by {\our} for the Auto Complete task. For every entry contains the mean with the standard deviation in parenthesis}
\label{table:kyc:auto_comp_stat_significance}
\end{table}

\subsection{Diagnostics}
We provide insights on why \our's simple method of learning per-client label biases from client sketches is so effective.  
One explanation is that the baseline had large discrepancy between the true and predicted class proportions for several OOD clients. \our\  corrects this discrepancy via {\em computed} per-client biases.  
Figure~\ref{fig:kyc:bar_graphs} shows true, baseline, and \our\ predicted class proportions for one OOD client on NER. Observe how labels like {\tt date}, {\tt GPE},  {\tt money} and {\tt org} are under-predicted by baseline and corrected by \our.
Since \our\ only corrects label biases, instances most impacted are those close to the shared decision boundary, and exhibiting properties correlated with labels but diverging across clients. We uncovered two such properties:

\noindent{\bf Ambiguous Tokens:} In NER the label of several tokens changes across clients, E.g. tokens like {\tt million,} {\tt billion} in finance clients like NW/Xinhua are {\tt money} 92\% of the times whereas in general only 50\% of the times. 
Based on client sketches, it is easy to spot finance-related topics and increase the bias of {\tt money} label. This helps \our\ correct labels of borderline tokens.  

\noindent{\bf Instance Length:}
For sentiment labeling, review length is another such property.  Figure~\ref{fig:kyc:sentiment_observation} is a scatter plot of the average review length of a client versus the fraction predicted as positive by the baseline. For most clients, review length is clustered around the mean of 61, but four clients have length $> 90$.  Length of review is correlated with label: on average, negative reviews contain 20 words more than positive ones.  This causes baseline to under-predict positives on the few clients with longer reviews.  The topics of the four outlying clients (video games, CDs, Toys\&Games) are related so that the client sketch is able to shift the decision boundary to correct for this bias. Using only normalized average sentence length as the client sketch bridges part of the improvement of {\our} over the baseline (details in Appendix C) implying that average instance length should be part of client sketch for sentiment classification tasks. 

\begin{table}[htb]
\centering
\begin{tabular}{| l | c c c c | c c c|}
\hline
 & Salience & TF-IDF & Binary-BOW & Summary & \multicolumn{3}{|c|}{Salience} \\
 Architecture$\rightarrow$& \multicolumn{4}{|c|}{Concat} & Deep & Decomp & MoE-$\bm{g}$ \\
\hline
OOD & 80.1 & 80.0 & 81.0 & 75.4 & 80.9 & 76.0 & 74.9  \\
ID & 86.0 & 85.9 & 77.8 & 81.8 & 85.9 & 85.0 & 79.8 \\
\hline
\end{tabular}
\vspace{0.3em}
\caption{Comparing variant client sketches ($S_c$) and network architectures ($\bigoplus$ and $Y_\theta$) of \our\ in Fig~\ref{fig:kyc:overview}. First row shows the client sketch type type and the second row shows the network architecture. We show F1 values for the NER task. Our saliency features perform best in summarizing target corpus. Using $\bm{g}$ in Concat architecture is the simplest and performs better or as well as other architectures.} 
\label{table:kyc:ner_method_comparison}
\end{table}


\subsection{Ablation Studies}
\label{sec:kyc:ablation}
We explored a number of alternative client sketches  and models for harnessing them.  We present a summary here; details are in the Appendix~\ref{kyc:appendix:domain_features} and~\ref{kyc:appendix:network_arch}.
Table~\ref{table:kyc:ner_method_comparison} shows average F1 on NER for three other sketches: TF-IDF, Binary bag of words, and a 768-dim pooled BERT embedding of ten summary sentences extracted from client corpus as suggested in~\citet{gensim_summarizer}.  \our's default term saliency features provides  best accuracy with TF-IDF a close second, and embedding-based sketches the worst.
Next, we compare three other architectures for harnessing $\bm{g}$ in Table~\ref{table:kyc:ner_method_comparison}:
\textbf{Deep}, where module $\bigoplus$ after concatenating $\bm{g}$ and $\lastE$ adds an additional non-linear layer so that now the whole decision boundary, and not just bias, is client-specific. \our's OOD performance increases a bit over plain concat. 
\textbf{Decompose}, which mixes two softmax matrices with a client-specific weight $\alpha$ learned from $\bm{g}$.
\textbf{MoE-$\bm{g}$}, which is like MoE but uses the client sketch for expert gating.
We observe that the last two options are worse than \our. 

\section{Further related work}
Our method addresses the mismatch between a client's data distribution and the server model. The extensive domain adaptation literature~\citep{Daume2007,BlitzerMP06,Ben-David:2006} is driven by the same goal but most of these update model parameters using labeled or unlabeled data from the target domain (client). Unsupervised Domain Adaptation summarized in \citet{unsupervised_domain_adaptation} relaxes the requirement of labelled client data, but still demands target-specific fine-tuning which inhibits scalability. Some recent approaches attempt to make the adaptation light-weight \citep{LinL2018,LiSS2020,JiaLZ2019,CaiW2019,LiuWF2020} while  others propose to use entity description~\citep{BapnaTH2017,ShahGF2019} for zero-shot adaptation.  Domain generalization is another relevant technique~\citep{ChenC2018,GuoSB18,Li2018DomainGW,WangZZ2019,VihariSSS18,CarlucciAS2019,DouCK19,VihariNS2020} where multiple domains during training are used to train a model that can generalize to new domains.  Of these, the method that seems most relevant to our setting is the mixture of experts network of \citet{GuoSB18}, with which we present empirical comparison.
Another option is to transform the client data style so as to match the data distribution used to train the server model.  Existing style transfer techniques~\citep{YangZC2018, ShenLB2017,PrabhumoyeTS2018, FuTP2018,lample2018multipleattribute,LiJH2018,GongBW2019} require access to server data distribution.

\section{Discussion}
\noindent
We presented approaches for two scenarios of unlabeled domain adaptation.

We introduced one regularization and one source-selection method for adapting word embeddings to a target topic. Both the approaches fared better than embedding transfer methods on down-stream language tasks. More interestingly, we showed (a) non-dominant sense information is either hard to recover or lost from pretrained word embeddings (b) contextual word representations still gain from adaptation methods despite their long range dependencies and extensive training on large diverse corpora. 

We then motivated and presented a lightweight client adaption for NLP service settings.  This is a promising area, ripe for further research on more complex tasks like translation.  We proposed client sketches and \our: an early prototype server network for on-the-fly adaptation.  Three NLP tasks showed considerable benefits from simple, per-label bias correction. Alternative architectures and ablations provide additional insights.


\noindent
{\bf Subsequent work: }
Language modeling has garnered immense interest and rapid developments in the recent past. Several smaller, faster, robust language models have been proposed in the BERT family of transformer based models~\citep{distilbert,roberta,albert}. It was expected and long been believed that extensive training, long-range dependencies and contextual representations of mega language models could handle any text sequence from any domain without adaptation. However, as we have shown in our work, contextual models continue to gain from further adaptation, which is confirmed further through extensive studies on various tasks by~\citet{Han19,Gururangan20} and many others. On hindsight, the utility of adaptation to a target domain is not so surprising. Language models process at the level of sentences, and could miss the broader context of the corpus, which is fixed when adapted.   

Several other exciting new ways of adaptation are emerging. \citet{Schick20} demonstrated surprisingly good results (beating LMs with 1000$\times$ parameters) with only access to very limited labeled data (32 examples) and unlabeled data from the target task. They frame any natural language understanding task as cloze-style questions---the label is verbalized as a text segment and posed as a question answering task---followed by fine-tuning with masked language model loss called Pattern Exploitative Training (PET). Succeeding work~\citep{TamRM21} further extended PET for adaptation only with labeled data (without unlabeled). 

Another recent trend with large language models~\citep{gpt3,FLAN} is parameter-update-free task adaptation method called prompting. \citet{gpt3} presented an extremely large language model with 175 billion parameters called GPT-3. They demonstrated that it can be task tuned by simply describing and presenting some examples of any task, called prompting. Prompting is appealing since we can adapt to any task/domain by simply changing the prompt content, however, these giant language models are known to be sensitive to how the task is described~\citep{Zhao21}. Since GPT-3 is unstable to prompt, few others attempted engineering prompt either in the vocabulary or soft embedding space for a given task~\citep{prompt-survey}. Prompt engineering is still appealing because we can tune without overfitting the small set of prompt parameters. 

Nevertheless, we are yet to understand how large language models can adapt quickly to new tasks/domains. Their success may have partly been because prompts help locate task data from the massive train corpus~\citep{Reynolds21}. Under such interpretation of a prompt, a task under-represented in the train data cannot be solved simply through prompting. We believe alternate adaptation methods continue to be relevant on domains under-represented or absent from the training corpus. 

\chapter{Conclusion}
\label{chap:conclusion}

Humans have evolved to extract essential features from sensory inputs, learn to ignore spurious signals, recognize patterns from the essential features, and adapt quickly and effectively to unfamiliar circumstances.  For many cognitive tasks, machine learning is not yet close to such natural capabilities.  In this dissertation, we have identified three key directions in which we have made progress toward making ML systems adapt better to domain shifts.


In Chapter~\ref{chap:dg}, we considered a natural multi-domain setting and looked at how a conventional classifier could overfit on domain-correlated signals.
We proposed \crossgrad{}, which provides a new data augmentation scheme through perturbations from a jointly trained domain predictor, and we argued the role of augmentations in improving domain generalization.
Instead of regularizing domain overfitting components implicitly through data augmentation, we developed a new algorithm called CSD that explicitly recovers the domain generalizing classifier. 
CSD decomposes classifier parameters into a common part and a low-rank domain-specific part.  
We presented a principled analysis to provide identifiability results of CSD and analytically studied the effect of rank in trading off domain-specific noise suppression.  
Both {\crossgrad} and CSD are most useful when the number of training domains is small and do not directly cover test domains well. 

In Chapter~\ref{chap:cgd}, we looked at a simpler form of domain generalization called subpopulation shift, which concerns generalization to any mixture of training domains. 
We deliberated on why our proposed domain generalization algorithms are not suitable for subpopulation shift where the datasets are characterized by high population skew.
We then presented a simple, new algorithm {\cg} that models inter-domain interactions for minority domain generalization. 
For all the proposed algorithms, we demonstrated their qualitative and empirical effectiveness through extensive evaluation using simple and real-world datasets. 

Training for domain robustness is usually triggered by a diverse user population trying to share a model, which is a common scenario with the increasing relevance of machine learning as a service.  Clients need support to estimate the accuracy they can expect from the service in their diverse settings, not aggregated performance on a few benchmarks. In Chapter~\ref{chap:aaa}, we presented AAA, a new approach to estimate the accuracy of a pretrained model, not as a single number, but as a surface over a space of attributes (arms).
AAA models accuracy surface with a Beta distribution at each arm and regresses these parameters using two Gaussian Processes (GP) to capture smoothness and generalize to unseen arms.
We introduced further innovations to GP for improved modeling of heteroscedastic uncertainty due to varying supervision across arms. 
Evaluation on real-life datasets and pretrained models
showed the efficacy of AAA, both in estimation and exploration quality.

In Chapter~\ref{chap:adaptation}, we presented approaches for two scenarios of unlabeled domain adaptation.
In the first part, we evaluated methods for adapting word embeddings to a target topic and reached the surprising conclusion that even limited corpus augmentation of target unlabeled data is more useful than adapting embeddings.
Our experiments led to the following takeaways: (a)~non-dominant sense information is either hard to recover or easily lost from pretrained word embeddings, (b)~contextual word representations from giant language models still gain from adaptation methods despite extensive training on large diverse corpora. 
We then motivated and designed a lightweight client adaption model for NLP service settings.  
Our early prototype server network called KYC for on-the-fly adaptation showed immediate gains on three NLP tasks.
This is a promising area, ripe for further research on more complex tasks like translation. 

Overall, our work explored multiple paths to improve reliability of deployed ML models in diverse real-world applications.  We began with a study of training algorithms that generalize to new domains in zero-shot mode, then proposed methods to estimate accuracy surface that identifies non-robust data regions, and finally introduced light-weight adaptation algorithms that do not require any labeled data. Our contribution lies in the insights behind our algorithms, and in the proof of their efficacy through empirical evaluation. We hope our contributions leads to further work toward building trustworthy ML systems.

\section*{Future Work}

Machine learning is rapidly changing the landscape of cognitive tasks, but there is still a long way to go before ML achieves the levels of safety and trustworthiness enjoyed by more classical engineering disciplines like mechanical or chemical engineering. 
We discuss some directions for future research below. 

\noindent
{\bf The role of data.} The root cause and bottleneck for the lack of domain generalization is because training data does not densely cover the domain space. Results in Chapter~\ref{chap:dg} corroborate this claim, where we observed consistent improvement in domain generalization with increasing number of training domains. Further, the gains from training on diverse data often exceed gains from a more sophisticated domain generalization algorithm. Promise of diverse domain data therefore calls for further study. Some exciting research questions are summarized below.
\begin{enumerate}
\item {\it Can guided acquisition of labeled data from new domains help improve gains on domain generalization?} In other words, do we gain from careful inclusion of domains in training data in order to increase training domain diversity? If so, how should the inclusion of the next domain be guided? We could use Bayesian methods to track and guide the acquisition of sparsely represented domains from a pool of unlabeled data, which would require further study to ensure a reasonable uncertainty estimate over large domain shifts.
    
\item {\it Can pretraining with self-supervision on unlabeled data help?} Existing literature, including some of our work, demonstrated that pretraining on unlabeled data can improve target-domain generalization in text~\citep{Piratla19,Gururangan20} and images~\citep{tent,arm,ttt++}. The promise of pretraining for improving target-domain generalization is exciting, but requires addressing the following questions: (i)~what is the relation between task and self-supervision loss?  (ii)~what is the relation between unlabeled data and the target domain? (iii)~how can we scale the additional pretraining step for deployed ML application that cannot afford subsequent task training on labeled data?
\end{enumerate}   

\noindent
{\bf Explicit label supervision}. When training on example-label pairs, models could learn spurious/incidental correlations across any dimension of input distribution with a low support, thereby hurting its domain robustness. There is mounting evidence~\citep{DegraveRadiology21,Geirhos20,Gururangan18,Lapuschkin19} on how models exploit shortcuts in data for classification.   These works raise alarm over the difficulty in detecting and avoiding such shortcuts.
Such shortcuts exist because the training data under-specifies the task. However, fixing the training data by adding new examples in order to specify the task well can be inefficient.  
Consider this thought experiment for the task of one-class classification of {\it butterflies} illustrated in Figure~\ref{fig:conclusion:butterflies}. The first column shows the potential spurious correlation that the model can exploit if a significant fraction of the training images look like the second column image. If the training data is improved to include images of the kind shown in the third column, we could avoid learning the spurious correlation shown in the first column. As we move down in the figure~\ref{fig:conclusion:butterflies}, we are tightening the task definition, albeit at a great expense of data-curation.

This observation calls for discovering new forms of defining tasks beyond labeled data. 
Humans, on the other hand, do not simply learn from a notebook of examples. Instead, they also learn from explanations on what constitutes a label and constant feedback from interactions that confirm their reality with the truth. Exploration of explicit forms of label supervision is a promising direction for training robust models. The ML community so far has invested in annotation pipelines for obtaining quick and dirty labeling of a large number of examples. We should consider the question: {\it can we learn models with better robustness from data with detailed annotations on relatively fewer examples?} 

\begin{figure}
    \centering
    \begin{tabular}{ccc}
    Spurious correlation & Sample image & Contradicting image \\
    Green background & \includegraphics[height=3cm]{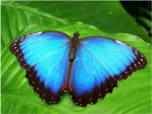} & \includegraphics[height=3cm]{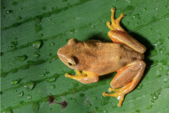} \\
    Flower & \includegraphics[height=3cm]{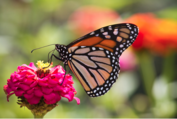} & \includegraphics[height=3cm]{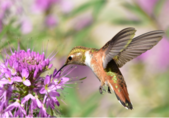} \\
    Wings & \includegraphics[height=3cm]{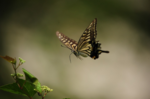} & \includegraphics[height=3cm]{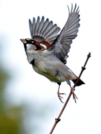} \\
    Whiskers & \includegraphics[height=3cm]{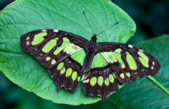} & \includegraphics[height=3cm]{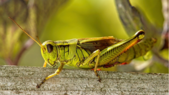} \\
    Spotted pattern & \includegraphics[height=3cm]{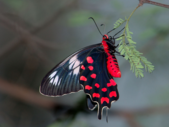} & \includegraphics[height=3cm]{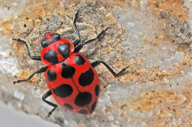} \\
    \end{tabular}
    \caption{Visual example to illustrate that it is hard to avoid spurious correlations with training data alone: a case-study of one-class butterfly classification task. The first column shows the potential spurious correlation that the model can exploit if a significant portion of training images look like the image in the second column. If the training data is improved to also include images of the kind shown in the third column, we could avoid learning the spurious correlation shown in the first column. Images are sourced from the Web and are all Creative Commons licensed.}
    \label{fig:conclusion:butterflies}
\end{figure}


\noindent
{\bf Flexible objectives.} In Chapter~\ref{chap:cgd}, we started our study of subpopulation shift robustness with the observation that ERM improves average accuracy, but hurts worst domain accuracy. We then discussed methods for improving worst domain accuracy, which were then found to hurt average accuracy to some extent. Both extremes either compromise accuracy on well-represented or under-represented training domains, and could be undesired in practice. Ideally, stakeholders should be able to specify their desired constraints on the solution flexibly, some examples are: (1) {\it the worst domain accuracy cannot be x\% worse than the best}, or (2) {\it there should be no more than p\% of the domains with accuracy worse than q\%}. Current ML systems, however, cannot accommodate such additional solution constraints. In order to make progress, we require further study on (a)~qualitative measures that predict generalization error of a hypothesis, and (b)~a suitable optimization algorithm that encodes the constraint and finds a hypothesis that satisfies it. 



\begin{appendices}
    \chapter{Details for Chapter~\ref{chap:dg}}
\label{chap:appendix:dg}

\section{Proof of Theorem~\ref{thm:csd:main}}\label{app:csd:proof}
\begin{algorithm}
\caption{$\min_{w_c, W_s, \Gamma} f(w_c, W_s, \Gamma) = \frob{W - w_c \trans{\ones} - W_s \trans{\Gamma}}^2$, s.t. $W_s \in \R^{m \times k}$ and $w_c \perp \textrm{Span}\left(W_s\right)$}\label{alg:csd:svdbased}
\begin{algorithmic}[1]
\State \textbf{Input} $W, k$
\State $w_c \leftarrow \frac{1}{D} W \cdot \ones$
\State $W_s, \Gamma \leftarrow \textrm{Top-} k \textrm{ SVD}\left(W - w_c \trans{\ones}\right)$
\State $w_c^{\textrm{new}} \leftarrow \frac{1}{\norm{\left( w_c \trans{\ones} + W_s \trans{\Gamma} \right)^+ \ones}^2} \left( w_c \trans{\ones} + W_s \trans{\Gamma} \right)^+ \ones$
\State $W_s^{\textrm{new}} \trans{\Gamma^{\textrm{new}}} \leftarrow w_c \trans{\ones} + W_s \trans{\Gamma} - w_c^{\textrm{new}} \trans{\ones}$
\State \textbf{Return} $w_c^{\textrm{new}}, W_s^{\textrm{new}}, {\Gamma^{\textrm{new}}}$
\end{algorithmic}
\end{algorithm}

\begin{proof}[Proof of Theorem 1] The high level outline of the proof is to show that the first two steps obtain the best rank-$(k+1)$ approximation of $W$ such that the row space contains $\ones$. The last two steps retain this low rank approximation while ensuring that $w_c^{\textrm{new}} \perp \Span{W_s^{\textrm{new}}}$.

The proof that the first two steps obtain the best rank-$(k+1)$ approximation follows that of the classical low rank approximation theorem of Eckart-Young-Mirsky. We first note that the minimization problem of $f$ under the given constraints, can be equivalently written as:
\begin{align}
    & \min_{\Wtilde} \frob{W - \Wtilde}^2 \nonumber \\
    \mbox{ s.t. } & \textrm{rank}(\Wtilde)\leq k+1 \mbox{ and } \ones \in \textrm{Span}\left(\trans{\Wtilde}\right). \label{eqn:csd:opt}
\end{align}
Let $\wtilde \defeq \frac{1}{D} W \ones$ and $W - \wtilde \trans{\ones} = \Utilde \Stilde \trans{\Vtilde}$ be the SVD of $W - \wtilde \trans{\ones}$. Since $\left(W - \wtilde \trans{\ones}\right)\ones = 0$, we have that $\ones \perp \textrm{Span}\left(\Vtilde\right)$. We will first show that $\Wtilde^* \defeq \wtilde \trans{\ones} + \Utilde_k \Stilde_k \trans{\Vtilde_k}$, where $\Utilde_k \Stilde_k \trans{\Vtilde_k}$ is the top-$k$ component of $\Utilde \Stilde \trans{\Vtilde}$, minimizes both $\frob{W - \Wtilde}^2$ and $\norm{W - \Wtilde}_2^2$ among all matrices $\Wtilde$ satisfying the conditions in~\eqref{eqn:csd:opt}. Let $\sigma_i$ denote the $i^{\textrm{th}}$ largest singular value of $\Utilde \Stilde \trans{\Vtilde}$.

\textbf{Optimality in operator norm}: Fix any $\Wtilde$ satisfying the conditions in~\eqref{eqn:csd:opt}. Since $\ones \in \textrm{Span}\left(\trans{\Wtilde}\right)$ and $\textrm{rank}\left(\Wtilde\right) = k+1$, there is a unit vector $\vtilde \in \Span{\Vtilde_{k+1}}$ such that $\vtilde \perp \Span{\trans{\Wtilde}}$. Let $\vtilde = \Vtilde_{k+1} x$. Since $\norm{\vtilde}=1$ and $\Vtilde$ is an orthonormal matrix, we have $\sum_i x_i^2 = 1$. We have:
\begin{align*}
    \norm{W - \Wtilde}_2^2 &\geq \norm{(W - \Wtilde) \vtilde}^2
    = \norm{W \vtilde}^2 = \norm{W \Vtilde_{k+1} x}^2 \\
    &= \norm{\Utilde_{k+1} \Stilde_{k+1} x}^2 = \sum_{i=1}^{k+1} \widetilde{\sigma}_i^2 \cdot x_i^2 \geq \tilde{\sigma}_{k+1}^2 \\ &= \norm{W - \wtilde \trans{\ones}}^2.
\end{align*}
This proves the optimality in operator norm.

\textbf{Optimality in Frobenius norm}:
Fix again any $\Wtilde$ satisfying the conditions in~\eqref{eqn:csd:opt}. Let $\sigma_i(A)$ denote the $i^\textrm{th}$ largest singular value of $A$ and $A_i$ denote the best rank-$i$ approximation to $A$. Let $W'\defeq W - \Wtilde$. We have that:
\begin{align*}
    \sigma_i(W') & = \norm{W' - W'_{i-1}}
    = \norm{W' - W'_{i-1}} + \norm{\Wtilde - \Wtilde} \\
    &\leq \norm{W' + \Wtilde - W'_{i-1} - \Wtilde} \\
    &= \norm{W - W'_{i-1} - \Wtilde} \\
    &\geq \min_{\What} \norm{W - \What},
\end{align*}
where the minimum is taken over all $\What$ such that $\Rank{\What} \leq i+k, \; \ones \in \Span{\trans{\What}}$. Picking $\What = \wtilde \trans{\ones} + \Utilde_{k+i-1} \Stilde_{k+i-1}\trans{\Vtilde}_{k+i-1}$, we see that $\sigma_i(W - \Wtilde) \geq \sigma_{k+i}$. It follows from this that $\frob{W - \Wtilde} \geq \frob{W - \wtilde \trans{\ones} - \Utilde_k \Stilde_k \trans{\Vtilde}_k}$.

For the last two steps, note that they preserve the matrix $w_c \trans{\ones} + W_s \trans{\Gamma}$. If $\overline{w}_c \trans{\ones} + \overline{W}_s \trans{\overline{\Gamma}}$ is the unique way to write $w_c \trans{\ones} + W_s \trans{\Gamma}$ such that $\overline{w}_c \perp \Span{\overline{W}_s}$, then we see that $\trans{\overline{w}_c}\left(\overline{w}_c \trans{\ones} + \overline{W}_s \trans{\overline{\Gamma}}\right) = \norm{\overline{w}_c}^2 \trans{\ones}$, meaning that $\trans{\overline{w}_c} = \norm{\overline{w}_c}^2 \left({w}_c \trans{\ones} + {W}_s \trans{\Gamma}\right)^+ \ones$. This proves the theorem.
\end{proof}

    \chapter{Details for Chapter~\ref{chap:cgd}}
\label{chap:appendix:cgd}

\section{Scaling Rule For Gradients Using Local Quadratic Approximation}
\label{appendix:cgd:grad_approx}
The gradient norms can vary widely even for a given task depending on where we are in the loss landscape.
We will now illustrate this issue. Consider any domain $i$. Since the training loss $\ell_i(\theta) \geq 0$ for all $\theta$, any $\hat{\theta}$ satisfying $\ell_i(\hat{\theta}) = 0$ is an approximate global minimizer. For any such global minimizer with $\ell_i(\hat{\theta}) = 0$ and $\norm{\nabla \ell_i(\hat{\theta})} = 0$, the local approximation of $\ell_i$ given by the second order Taylor expansion at $\hat{\theta}$ is:
\begin{align*}
{\ell_i}(\theta) \approx \frac{1}{2} \iprod{\theta - \hat{\theta}}{\nabla^2 \ell_i(\hat{\theta}) \left(\theta - \hat{\theta}\right)} \mbox{ and }
\nabla {\ell_i}(\theta) \approx {\nabla^2 \ell_i(\hat{\theta}) \left(\theta - \hat{\theta}\right)}.
\end{align*}

Denoting $\underline{\sigma}=\sigma_\textrm{min}(\nabla^2 \ell_i(\hat{\theta})) \geq 0$ and $\overline{\sigma}=\sigma_\textrm{max}(\nabla^2 \ell_i(\hat{\theta})) \geq 0$ to be the smallest and the largest eigenvalue of $\nabla^2\ell_i(\hat\theta)$, we have that $\|\nabla\ell_i(\theta)\|$ bounded between $\sqrt{\underline{\sigma}}\cdot \sqrt{\ell_i(\theta)}$ and $\sqrt{\overline{\sigma}}\cdot \sqrt{\ell_i(\theta)}$. Consequently, the gradient norm can vary widely between smallest and largest eigenvalues. Besides, the gradient scale also depends on the number of parameters, task, dataset and architecture. 

To summarize, we show that gradients are of the order of the loss, but can take a large spectrum of values, which makes the $\alpha$ updates unstable and $\eta_\alpha$ tuning difficult if used as is in~\eqref{eqn:cgd:alpha_update}. Since we are only interested in capturing if a domain transfers positively or negatively, we retain the cosine-similarity of the gradient dot-product but control their scale through $\ell(\theta)^p$, for some $p>0$. That is, we set the gradient $\nabla\ell_i(\theta)$ to $\frac{\nabla\ell_i(\theta)}{\|\nabla \ell_i(\theta)\|}\ell_i(\theta)^p$ leading to the final $\alpha$ update shown in~\eqref{eqn:cgd:final_alpha_update}.

When we set $p$ to a very large value, the gradient inner products ${\ell_i}^p {\ell_j}^p \cos(\nabla \ell_i({\theta}),\nabla \ell_j({\theta}))$, and hence the $\alpha$ value are largely influenced by the loss values. On the other hand, when we set $p$ to 0, we could get stuck in picking low loss train domains repeatedly without significant parameter update and hence not converge. In practice, neither of the extremes are ideal. We tried $p=\{0.25, 0.5, 1, 2\}$ on our synthetic setup, WaterBirds, CelebA, and found $p=0.5$ to perform well empirically.

\section{Proof of Theorem~\ref{thm:main}}\label{app:cgd:proof}
\begin{proof}[Proof of Theorem~\ref{thm:main}]
Let us consider the $t^\textrm{th}$ iteration where the iterate is $\theta^t$ and mixing weights are $\alpha^t$.
Using the notation in Section~\ref{sec:cgd:method}, we denote $\Rcal(\theta)= \frac{1}{k} \sum_i \ell_i(\theta)$, $g_i = \nabla \ell_i(\theta^t)$ and $g = \frac{1}{k} \sum_i g_i$. The update of our algorithm is given by:
\begin{align}
    \alpha_i^{t+1} &= \frac{\alpha_i^t \cdot \exp\left(\eta_\alpha \iprod{g_i}{g}\right)}{Z} \label{eqn:cgd:algo1} \\
    \theta^{t+1} &= \theta^t - \eta \sum_i \alpha_i^{t+1} g_i, \label{eqn:cgd:algo2}
\end{align}
where $Z = \sum_j \alpha_j^t \cdot \exp\left(\eta_\alpha \iprod{g_j}{g}\right)$.
Let us fix $\alpha^* = \left(1/k,\cdots,1/k\right) \in \R^k$ and use $KL(p,q) = \sum_i p_i \log \frac{p_i}{q_i}$ to denote the KL-divergence between $p$ and $q$. We first note that the update~\eqref{eqn:cgd:algo1} on $\alpha$ corresponds to mirror descent steps on the function $\iprod{g}{\sum_i \alpha_i g_i}$ and obtain:
\begin{align}
    KL(\alpha^*, \alpha^{t+1}) &= \sum_i \alpha_i^{*} \log \frac{\alpha_i^{*}}{\alpha_i^{t+1}}
    = {\sum_i \alpha_i^{*} \log \frac{\alpha_i^{*} Z}{\alpha_i^t \cdot \exp\left(\eta_\alpha \iprod{g_i}{g}\right)}} \nonumber \\
    &= \sum_i \alpha_i^{*} \log \frac{\alpha_i^{*}}{\alpha_i^t} -  \eta_\alpha \alpha_i^* \iprod{g_i}{g} + \alpha_i^* \log Z \nonumber \\
    &= KL(\alpha^*, \alpha^t) - \eta_\alpha \norm{g}^2 + \log \left(\sum_i \alpha_i^{t} \exp\left(\eta_\alpha \iprod{g_i}{g}\right)\right). \label{eqn:cgd:KL}
\end{align}
Since each $\ell_i(\cdot)$ is $G$-Lipschitz, we note that $\norm{g_i}\leq G$ for every $i$. Consequently, $\norm{g} \leq G$ and $\abs{\iprod{g_i}{g}} \leq G^2$. Using the fact that $\exp(x) \leq 1+x+x^2$ for $\abs{x} < 1$, and since $\eta_\alpha \leq \frac{1}{G^2}$, we see that:
\begin{align*}
    \log\left(\sum_i \alpha_i^{t} \exp\left(\eta_\alpha \iprod{g_i}{g}\right)\right) &\leq \log\left(\sum_i \alpha_i^{t} \left(1+ \eta_\alpha \iprod{g_i}{g} + \left(\eta_\alpha G^2\right)^2\right)\right) \\
    &=  \log\left(1+ \left(\sum_i \alpha_i^{t} \eta_\alpha \iprod{g_i}{g}\right) + \left(\eta_\alpha G^2\right)^2\right) \\
    &\leq \left(\sum_i \alpha_i^{t} \eta_\alpha \iprod{g_i}{g}\right) + \left(\eta_\alpha G^2\right)^2,
\end{align*}
where we used $\log(1+x) \leq x$ in the last step. Plugging this back into~\eqref{eqn:cgd:KL}, we obtain:
\begin{align}
    KL(\alpha^*, \alpha^{t+1}) &\leq KL(\alpha^*, \alpha^{t}) + \left(\sum_i \alpha_i^{t} \eta_\alpha \iprod{g_i}{g}\right) + \left(\eta_\alpha G^2\right)^2 - \eta_\alpha \norm{g}^2 \nonumber \\
    \Rightarrow - \sum_i \alpha_i^{t} \iprod{g_i}{g} &\leq - \norm{g}^2 + \frac{KL(\alpha^*, \alpha^{t}) - KL(\alpha^*, \alpha^{t+1})}{\eta_\alpha} + \eta_\alpha G^4. \label{eqn:cgd:KL-2}
\end{align}
We will now argue that $\sum_i \alpha_i^{t} \iprod{g_i}{g} \leq \sum_i \alpha_i^{t+1} \iprod{g_i}{g}$. To show this, it suffices to show that:
\begin{align*}
    \sum_i \alpha_i^{t} \iprod{g_i}{g} &\leq \sum_i \alpha_i^{t+1} \iprod{g_i}{g} \\
    \Leftrightarrow \left(\sum_i \alpha_i^{t} \iprod{g_i}{g}\right) \left(\sum_i \alpha_i^{t} \exp\left(\eta_\alpha \iprod{g_i}{g}\right)\right) &\leq \sum_i \alpha_i^{t} \iprod{g_i}{g} \exp\left(\eta_\alpha \iprod{g_i}{g}\right)\\
    \Leftrightarrow \sum_i \alpha_i^{t} \left(\iprod{g_i}{g}-\ipbar\right) \left(\exp\left(\eta_\alpha \iprod{g_i}{g}\right)-\expbar\right) &\geq 0,
\end{align*}
where $\ipbar = \sum_i \alpha_i^{t} \iprod{g_i}{g}$ and $\expbar = \sum_i \alpha_i^{t} \iprod{g_i}{g} \exp\left(\eta_\alpha \iprod{g_i}{g}\right)$. The last inequality follows from the fact that $\exp(\cdot)$ is an increasing function and hence for any random variable $X$, covariance of $X$ and $\exp(X)$ is greater than or equal to zero.

Now coming back to the update~\eqref{eqn:cgd:algo2}, we can argue its descent property as follows:
\begin{align*}
    \Rcal(\theta^{t+1}) &\leq \Rcal(\theta^t) - \eta \iprod{g}{\sum_i \alpha_i^{t+1}g_{i}} + \frac{\eta^2 L}{2} \norm{\sum_i \alpha_i^{t+1}g_{i}}^2 \\
    &\leq \Rcal(\theta^t) + \eta \left( - \norm{g}^2 + \frac{KL(\alpha^*, \alpha^{t}) - KL(\alpha^*, \alpha^{t+1})}{\eta_\alpha} + \eta_\alpha G^4 \right) \\
    & + \frac{\eta^2 L}{2} \norm{\sum_i \alpha_i^{t+1}g_{i}}^2 \\
    &\leq \Rcal(\theta^t) + \eta \left( - \norm{g}^2 + \frac{KL(\alpha^*, \alpha^{t}) - KL(\alpha^*, \alpha^{t+1})}{\eta_\alpha} + \eta_\alpha G^4 \right)+ \frac{\eta^2 L G^2}{2} \\
    \Rightarrow \norm{g}^2 &\leq \frac{\Rcal(\theta^t) - \Rcal(\theta^{t+1})}{\eta} + \frac{KL(\alpha^*, \alpha^{t}) - KL(\alpha^*, \alpha^{t+1})}{ \eta_\alpha} + {\eta_\alpha G^4}{} + \frac{\eta LG^2}{2}.
\end{align*}
Summing the above inequality over timesteps $t=1,\cdots,T$, we obtain:
\begin{align*}
    \frac{1}{T} \sum_{t=0}^{T-1} \norm{\nabla \Rcal(\theta^t)}^2 &\leq \frac{\Rcal(\theta^0) - \Rcal(\theta^{T})}{\eta T} + \frac{KL(\alpha^*, \alpha^{0}) - KL(\alpha^*, \alpha^{T})}{\eta_\alpha T} + {\eta_\alpha G^4}{} + \frac{\eta LG^2}{2} \\
    &\leq \frac{2B}{\eta T} + {\eta_\alpha G^4}{} + \frac{\eta LG^2}{2},
\end{align*}
where we used $KL(\alpha^*, \alpha^{t+1}) \geq 0$, $KL(\alpha^*,\alpha^0)=0$ and $\Rcal(\theta^0) - \Rcal(\theta^T) \leq 2B$ in the last step. Using the parameter choices $\eta = 2 \sqrt{\frac{B}{LG^2T}}$ and $\eta_\alpha = \sqrt{\frac{BL}{G^6T}}$, we obtain that:
\begin{align*}
    \frac{1}{T} \sum_{t=0}^{T-1} \norm{\nabla \Rcal(\theta^t)}^2 &\leq 3 \sqrt{\frac{BLG^2}{T}}.
\end{align*}
This proves the theorem.
\end{proof}

\section{Synthetic Setting Illustrated}
\label{appendix:cgd:synth}
\begin{table}[htb]
    \centering
    \begin{tabular}{c}
         \includegraphics[width=0.98\linewidth]{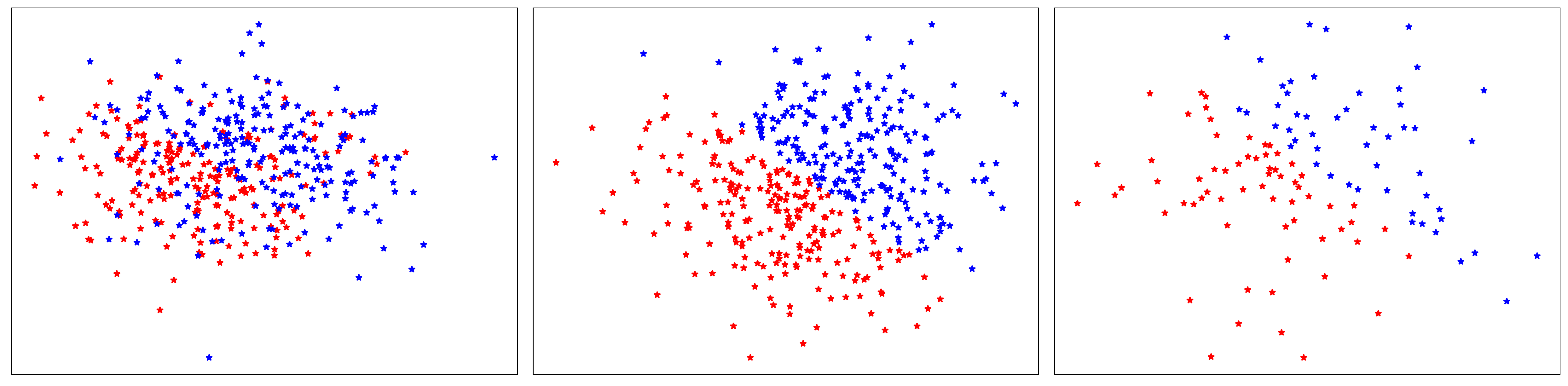}\\
         \includegraphics[width=0.98\linewidth]{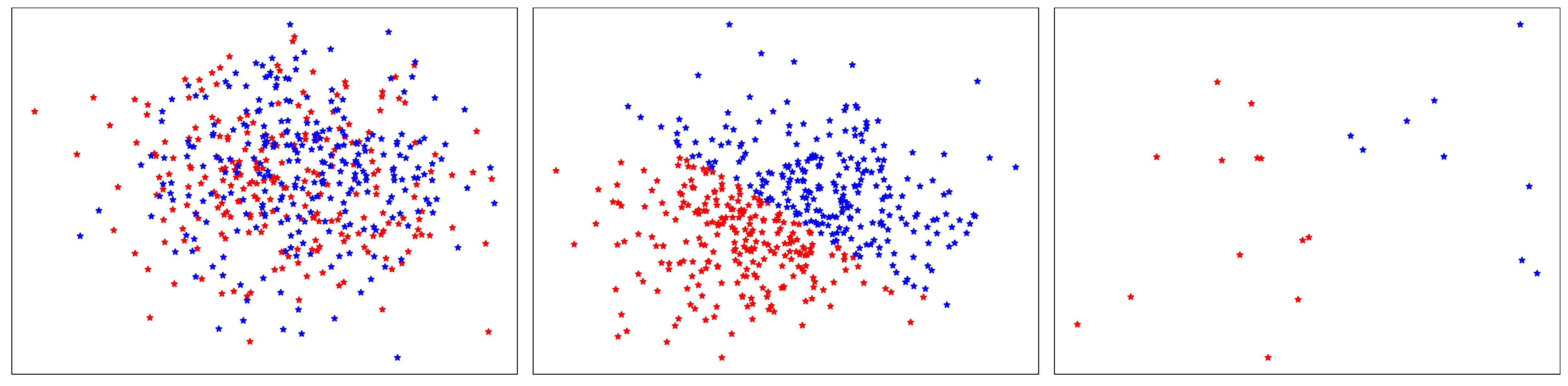}\\
         \includegraphics[width=0.98\linewidth]{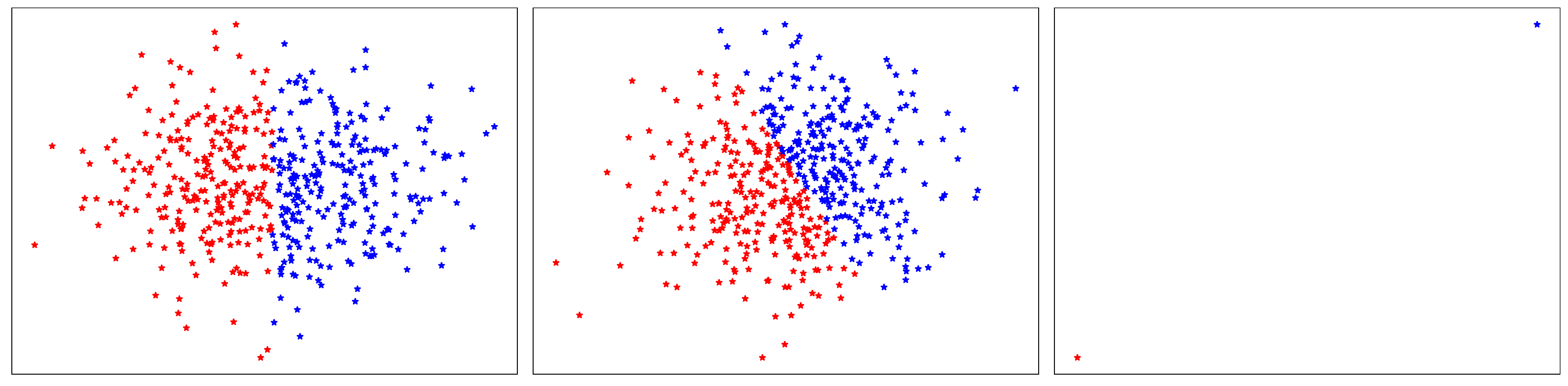}
    \end{tabular}
    \caption{Plot of data from first two features for the three synthetic settings. The three columns show the distribution of examples from the three domains in the order. The rows from top to bottom show the plots for Noisy (Section~\ref{sec:cgd:qua:noise}), Spurious (Section~\ref{sec:cgd:qua:neg}) and Rotation (Section~\ref{sec:cgd:qua:rot}) simple settings. The minority domain for Rotation setting (right bottom) has only two examples. We omitted from the plot the third feature of the spurious setting (second row). }
    \label{tab:cgd:synth_figures}
\end{table}

In Table~\ref{tab:cgd:synth_figures}, we show the data for each domain and each simple setting we considered in Section~\ref{sec:cgd:qua}. In the noise setting of the first row, the noise is only introduced in one of the domains. For the spurious case of second row, the first two features perform poorly on the first domain while the third spurious feature (not shown in the figure) performs poorly on both the second and the third domain. Shown in the last row, the classifier is gradually rotated for the rotation case and the minority domain is extremely sparse with only two examples.

\section{More Technical Details}
\label{appendix:cgd:technical}

{\bf Predictive Group Invariance}~\citep{Ahmed21} proposed an algorithm that was shown to generalize to various kinds of shifts including in-distribution shifts. The algorithm penalizes discrepancy in predictive probability distribution across domains with the same label as shown below. 

$$
\Rcal(\theta)=\sum_i \frac{n_i}{\sum_i n_i} \ell_i + \lambda\sum_c \mathbb{E}_{Q^c}[p_\theta(y\mid x)]\log \frac{\mathbb{E}_{Q^c}[p_\theta(y\mid x)]}{\mathbb{E}_{P^c}[p_\theta(y\mid x)]}
$$

Where `c' is the class label and the rest of terms in the equation follow the notation of this paper. For a given class label `c', the distributions: $Q^c, P^c$, represent the examples of the same label but belong to two different inferred domains. They used an existing work~\citep{Creager21} to stratify the examples in to domains. Here, we used the available domain annotations for creating the example splits: $Q^c, P^c$, for a fair comparison. We picked the value of $\lambda$ from \{1e-3, 1e-2, 1e-1, 1\}. 

The original paper noted that the second distribution of the KL divergence is an ``easy'' domain such that its average predictive distribution is mostly correct. In our case, since all the domains contain bias, they are all ``easy'', and hence the direction of divergence is unclear. Since {\erm} tends to learn the biases aligned with the majority domain, we set the $P^c$ distribution from the majority and $Q^c$ from the minority. 

\section{Performance at varying levels of heterogeneity}
\label{appendix:cgd:hetero}

We study performance on the simple synthetic setting of Colored-MNIST under varying levels of heterogeneity, i.e. ratio of majority to the minority domain. As discussed in the datasets section, ColoredMNIST dataset has two domains: a majority domain where the label and the digit's color match and the minority domain where they do not match. The examples are grouped in to majority and minority in the ratio of r to 1. As the value of r increases, the net (spurious) correlation of the color with the label increases, furthering the gap between the best and worst domains. The train split has 50,000 examples and validation, test split have 10,000 examples each. The validation, test split contain all the domains in equal proportions. We follow the same optimization procedure as described for the ColoredMNIST as described in the main content, with the exception that we set the weight decay to 0. 

\begin{wrapfigure}{r}{0.5\textwidth}
  \centering
  \includegraphics[width=\linewidth]{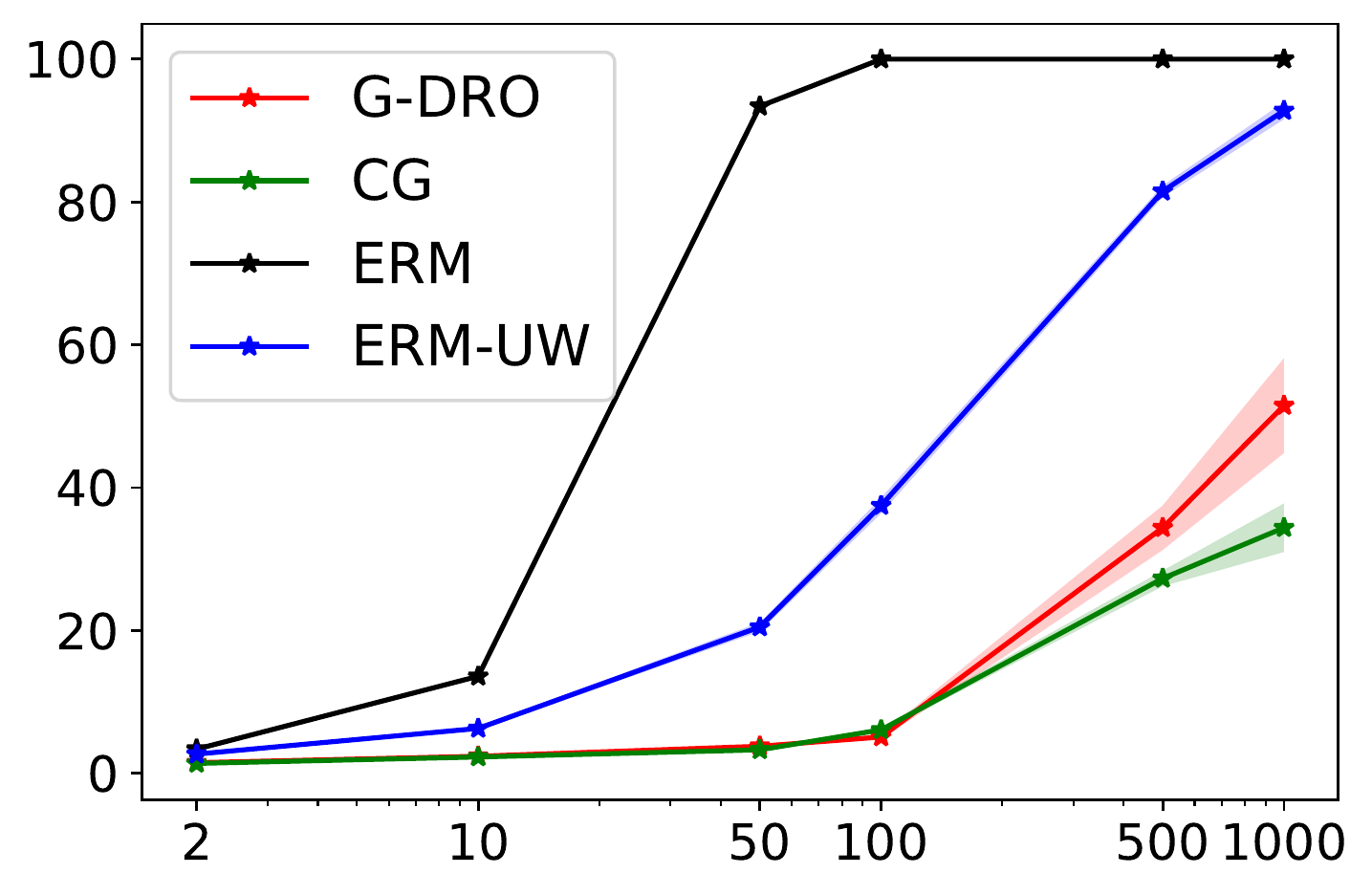}
  \caption{Worst domain error rate with increasing ratio of majority to minority. Shaded region shows one standard error. All estimates are aggregated over three runs.}
  \label{fig:cmnistd}
  \vspace{-30pt}
\end{wrapfigure}
We trained and evaluated different algorithms for the ratio (r) going from 2 to 1000. Figure~\ref{fig:cmnistd} shows the worst-domain error rate with the value of r for various algorithms. 
{\cg} performed well across all the values of r, and for the extreme value of r=1000: {\cg}'s worst error rate of 34.4\% is far better than the 51.5\% of {\dro}.

	\chapter{Details for Chapter~\ref{chap:aaa}}
\label{chap:appendix:aaa}

\section{Parametric Form of \BetaGP{}}
\label{sec:aaa:appendix:betaab}

In Section~\ref{sec:aaa:betagp}, we claimed that \BetaGP{} with (mean, scale) parameterization is better than \BetaAB{} with the standard $(\alpha, \beta)$ parameterization of the Beta distribution. In this section, we present some empirical evidence corroborating the claim.

We compare between the two parametric forms with two service models in Table~\ref{tab:aaa:betaab}. For \BetaAB, we use two GPs, one to approximate the latent value corresponding to $\alpha$, and other for $\beta$. We use soft-plus operation  to transform the latent values to their admissible positive $\alpha$, $\beta$ values.

We report macro-averaged mean square errors on two tasks in Table~\ref{tab:aaa:betaab}, when fitting on varying number of examples --- similar to the setting of Section~\ref{sec:aaa:expt:est}. We found the {\BetaAB} estimates unstable and far worse, perhaps because smoothness is not expected in either of $\alpha, \beta$ parameters across arms making the GP's bias ineffective.  

\begin{table}[H]
\centering
\begin{adjustbox}{max width=\textwidth}
\begin{tabular}{|l|c|c|c|c|} 
\hline
Service$\rightarrow$ & \multicolumn{2}{c}{\bf MF-CelebA}  &  \multicolumn{2}{c|}{\bf AC-COCOS} \\ 
Size $\downarrow$ & \BetaGP\ & \BetaAB\ & \BetaGP\ & \BetaAB\ \\ \hline
1000 & 5.4 / 0.5 & 6.6 / 0.1 & 3.7 / 0.2 & 5.6 / 0.6\\
2000 &  4.6 / 0.8 & 6.2 / 0.1 & 3.3 / 0.2 & 4.8 / 0.1\\
3500 &  4.6 / 0.3 & 6.1 / 0.1 & 3.2 / 0.2 & 5.2 / 0.4\\\hline
\end{tabular}
\end{adjustbox}
\caption{Comparison of estimation error between \BetaGP{} with (mean, scale) parameterization vs.\ \BetaAB.  \BetaAB{} is worse than {\BetaGP}. }
\label{tab:aaa:betaab}
\end{table}

\section{More Details of Gaussian Process (GP) Setup}
\label{sec:aaa:appendix:gp}
In this section, we give further details on GP training,  posterior approximation and computational cost. This section elaborates on Section~\ref{sec:aaa:betagp}. 

In all our proposed estimators, the data likelihood is modeled either by a Bernoulli or a Beta distribution. Data likelihood term of ~\BernGP{}, ~\BetaGP{}, are shown in Eqn.~\eqref{eq:bernLL}, Eqn.~\eqref{eq:pointLLSimple}, respectively. Due to the non-Gaussian nature of the data likelihood, the posterior on parameters cannot be expressed in a closed form. Several approximations exist for fitting the posterior especially for the more standard \BernGP{}, we will discuss one such method in what follows. Recall that we model two latent values $f, g$ each modeled by an independent GP. They can be seen to have been drawn from a single GP with even larger dimension and with appropriately defined kernel matrix. For the sake of explanation and with a slight abuse of notation, we denote by $\vf$, the concatenation of $f$ and~$g$.  The corresponding kernel for the concatenated vector is appropriately made by combining the kernels of either of the latent values with kernel entries corresponding to interaction between $f$ and $g$ set to~0.

Variational methods are popular for dealing with non-Gaussian likelihoods in GP. In this method, we fit a multi-variate Gaussian that closely approximates the posterior, i.e. minimizes $\mathcal{D}_\text{KL}(q(\vf)\|P(\vf|D)$. GPs are often used in their sparse avatars using \emph{inducing points}~\citep{WilsonHS16,NickischR08} that provide approximations to the full covariance matrix with large computation benefits. As a result, q(f) is parameterized by the following trainable parameters (let `d', `m' denote the input dimension and number of inducing points resp.): (a) Z, a matrix of size $d\times m$, of locations of 'm' inducing points (b) $\mu\in \mathbb{R}^m, \Sigma\in \mathbb{R}^{m\times m}$, denoting mean and covariance of the inducing points. 
%
In order to minimize $\mathcal{D}_\text{KL}(q(\vf)\|P(\vf|D)$, a pseudo objective called Evidence Lower Bound (ELBO), shown below, is employed:
\begin{equation}
  q^*(\vf) = \argmax_{\vf \sim q(\vf)} \mathbb{E}_q[\log P(D|\vf)]-\mathcal{D}_\text{KL}(q(\vf)|| P(\vf))
  \label{eqn:elbo}
\end{equation}

The first term above in Eqn.~\eqref{eqn:elbo} maximizes data likelihood, which is Equation~\ref{eq:pointLLSimple} in our case.
The second term is a regularizer that regresses the posterior fit $q(\vf)$ close to the prior distribution $P(\vf)$ which is set to standard Normal. 
We optimize using this objective over all the parameters involved through gradient descent. The required integrals in \eqref{eqn:elbo} can be computed using Monte Carlo methods~\citep{HensmanMG15}. We describe further implementation details in the next section.

\noindent
{\bf Implementation Details} \\
We use routines from GPytorch\footnote{\url{https://gpytorch.ai/}} library to implement the variational objective. Specifically, we extend \href{https://docs.gpytorch.ai/en/stable/models.html#gpytorch.models.ApproximateGP}{ApproximateGP} with \href{https://docs.gpytorch.ai/en/stable/variational.html?highlight=variationalstrategy#id1}{VariationalStrategy}, both of which are GPytorch classes, and set them to learn inducing point locations.

Number of \emph{inducing points} when set to a very low value could overly smooth the surface and can have high computation overhead when set to a large value. We set the number of inducing points to 50 for all the tasks. The choice of 50 over a larger number is only to ensure reasonable computation speed. 

Since we keep getting more observations as we explore, we use the following strategy for scheduling the parameter updates. We start with the examples in the seed set $\dseed$ and update for 1,000 steps. We explore using the variance of the estimated posterior. We pick 12 arms with highest variance and label one example for each arm. After every new batch of observations, we make 50 update steps on all the data. As a result, we keep on updating the parameters as we explore more. We use Adam Optimizer with learning rate $10^{-3}$. At each step, we update over observations from all the arms. 

We use the feature representations of the network used to model joint potentials described in Section~\ref{sec:aaa:AttribNoise} to also initialize the deep kernel induced by~$\Emb$. A final new linear layer of default output size 20 is added to project the feature representations. 

In our proposed method~\DirGPR{}, described in Section~\ref{sec:aaa:pool}, we take the kernel average of three neighbours for any arm with fewer than five observations. 

\section{More Task and Dataset Details}
\label{sec:aaa:appendix:task}

\subsection{MF-CelebA, MF-IMDB}
We hand-picked 12 binary attributes relevant for gender classification of the 40 total available attributes in the CelebA dataset. 
The twelve binary attributes are listed in Table~\ref{tab:aaa:celeba:attrs}, these constitute the $\attlist$. $\attspace$ is the combination of twelve binary attributes and is $2^{12}=4,096$ large. 
The attributes related to hair color are retained in this list due to the recent finding that hair-color is spuriously correlated with the gender in CelebA~\citep{Sagawa19}. We ignored several other gender-neutral or rare attributes. 

\begin{table}[htb]
    \centering
    \begin{tabular}{|l|r|r|}\hline
    \thead{Index} & \thead{Name} & \thead{Num. labels} \\\hline
    1 & Black Hair? & 2\\\hline
    2 & Blond Hair? & 2\\\hline
    3 & Brown Hair? & 2\\\hline
    4 & Smiling? & 2\\\hline
    5 & Male? & 2\\\hline
    6 & Chubby? & 2\\\hline
    7 & Mustache? & 2\\\hline
    8 & No Beard? & 2\\\hline
    9 & Wearing Hat? & 2\\\hline
    10 & Blurry? & 2\\\hline
    11 & Young? & 2\\\hline
    12 & Eyeglasses? & 2\\\hline
    \end{tabular}
    \caption{Attribute list of MF-CelebA, MF-IMDB.}
    \label{tab:aaa:celeba:attrs}
\end{table}

\subsection{AC-COCOS, AC-COCOS10K}

COCOS is a scene classification dataset, where pixel level supervision is provided. Methods are usually evaluated on pixel level classification accuracy. For simplicity, we cast it in to an object recognition task. The subset of ten animal labels we consider is shown in Table~\ref{tab:aaa:cocos:primary}. We consider in our task five coarse background (stuff) labels by collapsing the fine labels to coarse using the label hierarchy shown in Table~\ref{tab:aaa:cocos:stuff} and is the same as the official hierarchy~\citep{COCOS}.

We now describe how we cast the scene label classification task to an animal classification task. We first identify the subset of images in the train and validation set of the COCOS dataset which contain only one animal label, it could contain multiple background labels. If the image contains multiple animals, we exclude it, leaving around 23,000 images in the dataset. This is implemented in the routine: \texttt{filter\_ids\_with\_single\_object} of \texttt{cocos3.py} in the attached code.  We also retrofit the scene classification models for animal classification. When the model (service model) labels pixels with more than one animal label, we retain the label associated with the largest number of pixels. This is implemented in \texttt{fetch\_preds} routine of \texttt{cocos3.py}. Recall our calibration method makes use of the logit scores given by the attribute predictors, since we are aggregating prediction from multiple pixels, we do not have access to the logit scores. We simply set the logit score to +1 if a label is found in the image and -1 otherwise. 

We follow the same protocol for both the tasks: AC-COCOS, AC-COCOS10K. The only difference between the two is the service model, AC-COCOS is a stronger service model trained on 164K size training data compared to AC-COCOS10K which is a model trained on a previous version of the dataset that is only 10K large. In these tasks, we use the same model for predicting the attributes and task labels since the pre-trained model we use is a scene label classifier. Both the pretrained\footnote{\url{https://github.com/kazuto1011/deeplab-pytorch}} models were trained using ResNet101 architecture.  

\begin{minipage}{0.2\textwidth}
    \centering
    \begin{tabular}{|l|}\hline
    \thead{Name}\\\hline
    bird\\\hline
    cat\\\hline
    dog\\\hline
    horse \\\hline
    sheep \\\hline 
    cow \\\hline
    elephant \\\hline 
    bear \\\hline 
    zebra \\\hline
    giraffe \\\hline
    \end{tabular}
    \captionof{table}{List of ten animals in AC tasks}
    \label{tab:aaa:cocos:primary}
\end{minipage}
\hfill
\begin{minipage}{0.8\textwidth}
    \centering
    \begin{tabular}{|l|l|}\hline
         \thead{Coarse label} & Example of constituent stuff classes \\\hline
         water-other & sea, river \\\hline
        ground-other & ground-other, playingfield, platform, railroad, pavement \\\hline
        sky-other & sky-other, clouds \\\hline
        structural-other & structural-other, cage, fence, railing, net \\\hline
        furniture-other & furniture-other, stairs, light, counter, mirror-stuff \\\hline
    \end{tabular}
    \captionof{table}{80 stuff labels in the COCOS dataset are collapsed in to five coarse labels. Few examples are shown for each coarse label in the right column.}
    \label{tab:aaa:cocos:stuff}
\end{minipage}


\section{Statistics of Accuracy Surface}
\label{sec:aaa:appendix:surface_stats}
We show in Table~\ref{tab:aaa:dataset:stats}, details about our two data sets such as the number of attributes, number of arms and number of active arms. Active arms are arms with a support of at least five and are the ones used for evaluation. Large number of arms as shown in the table exclude the possibility of manual supervision, since it is hard to obtain and label data that covers all the arms.  

In Figures~\ref{fig:aaa:task:stats1} and~\ref{fig:aaa:task:stats2}, we show some statistics that illustrate the shape of the accuracy surface. We note that, although the service model's mean accuracy is high, the accuracy of the arms in the 10\% quantile is abysmally low while arms in the top-quantiles have near perfect accuracy. This further motivates for why we need an accuracy surface instead of single accuracy estimate.  

\begin{table}[H]
\begin{center}
\begin{tabular}{|l|r|r|r|}
\hline
Dataset & \# attributes & \# arms & \# active arms \\\Xhline{4\arrayrulewidth}
CelebA & 12 & 4096 & 398\\\hline
COCOS & 6 & 320 & 176 \\\hline
\end{tabular}
\end{center}
\caption{Attribute statistics per dataset. First and second column show number of attributes and total possible combinations of the attributes. Third column shows number of attribute combinations (arms) with at least a support of five in the unlabeled data. These are the arms on which accuracy surface is evaluated.}
\label{tab:aaa:dataset:stats}
\end{table}

\begin{figure}
    \centering
    \begin{subfigure}[t]{0.48\textwidth}
    \vspace{0pt}
        \includegraphics[width=\textwidth]{images/aaa/quantiles_new.pdf}
        \caption{We show the mean and ten quantiles of per-arm accuracy: 0, 0.1, 0.3, 0.5, 0.7, 0.9, 1. for each task when evaluated on their corresponding dataset (quantile 0 corresponds to the worst value). Observe the disparity between the best and the worst arms in terms of accuracy. In all the cases, also note how the large mean accuracy (macro-averaged over arms) does not do justice to explaining the service model's vulnerabilities.}
    \label{fig:aaa:task:stats1}
    \end{subfigure}
    \hfill
    \begin{subfigure}[t]{0.48\textwidth}\vspace{0pt}
        \includegraphics[width=\textwidth]{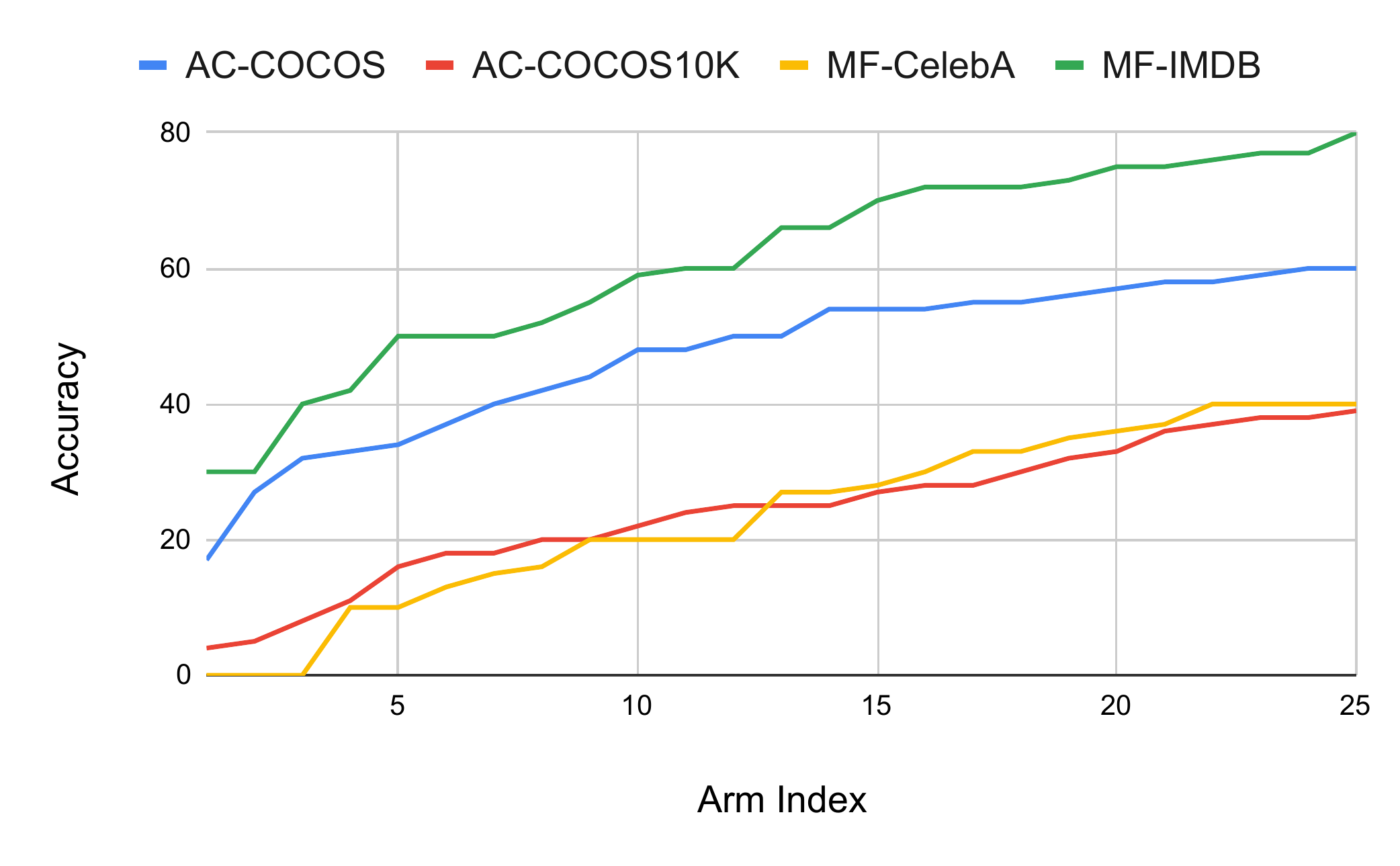}
        \caption{We show here the 25 worst arm accuracies for each of the service models. The large number of arms with accuracies much worse than mean accuracy further illustrates our argument for why we need accuracy surfaces.}
        \label{fig:aaa:task:stats2}
    \end{subfigure}
    \caption{Arm accuracies.}
    \label{fig:aaa:task}
\end{figure}

\section{Standard Deviation and More Evaluation Metrics}
\label{sec:aaa:appendix:metrics}
We include standard deviation accompanying numbers in Table~\ref{tab:aaa:est} for macro MSE in Table~\ref{tab:aaa:macromse} and for worst MSE in Table~\ref{tab:aaa:worstmse}.

In the main content of the paper, we gave results using macro and worst MSE. In this following section, we show results using two other metrics.  We follow the same setup as in Section~\ref{sec:aaa:expt:est}.

{\bf Micro-averaged MSE:} We assign importance to each arm based on its support. The error per arm is multiplied by its support (in $U$). Results with this error are shown in Table~\ref{tab:aaa:micromse}. The best predictor with this metric is the point estimate given by the \cpred{} estimator which is not surprising since very few arms with high frequency dominate this metric.\\
{\bf Infrequent MSE:} In Table~\ref{tab:aaa:infreqmse}, we show MSE evaluated only on the 50 arms that are least frequent in $U$. \\
For each of the above evaluation metrics, the trend between \BetaGP\, \DirGP{}, \DirGPR\ is statistically significant. 

\begin{table}
\centering
\begin{tabular}{|l|llllll|}\hline
Service $\downarrow$ &\thead{\cpred} & \thead{\PerArmBeta} & \thead{\BernGP{}} & \thead{\BetaGP{}} & \thead{\DirGP{}} & \thead{\DirGPR{}} \\\hline
AC-COCOS10K & 5.4 / 0.2 & 7.0 / 0.6 & 7.0 / 0.7 & 7.1 / 0.3 & 5.3 / 0.2 & 4.7 / 0.1\\
AC-COCOS & 3.2 / 0.1 & 4.3 / 0.3 & 3.5 / 0.3 & 3.3 / 0.2 & 2.8 / 0.0 & 2.8 / 0.1\\
MF-IMDB & 1.2 / 0.0 & 1.6 / 0.1 & 1.7 / 0.2 & 2.2 / 0.2 & 1.4 / 0.1 & 1.4 / 0.1\\
MS-CelebA & 5.2 / 0.1 & 4.7 / 0.2 & 4.9 / 0.4 & 4.6 / 0.8 & 4.1 / 0.1 & 4.3 / 0.1\\\hline
\end{tabular}
\caption{\emph{Macro-averaged} MSE along with standard deviation on all tasks. Shown after trailing '/' is the standard deviation.}
\label{tab:aaa:macromse}
\end{table}

\begin{table}
\centering
\begin{tabular}{|l|llllll|}\hline
Service $\downarrow$ &\thead{\cpred} & \thead{\PerArmBeta} & \thead{\BernGP{}} & \thead{\BetaGP{}} & \thead{\DirGP{}} & \thead{\DirGPR{}} \\\hline
AC-COCOS10K & 3.0 / 0.0 & 3.4 / 0.2 & 3.6 / 0.2 & 3.5 / 0.1 & 3.2 / 0.1 & 3.3 / 0.4\\
AC-COCOS & 1.4 / 0.0 & 1.8 / 0.2 & 1.6 / 0.1 & 1.6 / 0.1 & 1.7 / 0.1 & 2.1 / 0.3\\
MF-IMDB & 0.2 / 0.0 & 0.3 / 0.1 & 0.2 / 0.0 & 0.3 / 0.0 & 0.3 / 0.1 & 0.8 / 0.1\\
MF-CelebA & 0.8 / 0.0 & 0.7 / 0.1 & 0.7 / 0.1 & 0.7 / 0.2 & 0.9 / 0.1 & 1.2 / 0.1\\\hline
\end{tabular}
\caption{\emph{Micro-averaged} MSE along with standard deviation on all tasks. Shown after trailing '/' is the standard deviation.}
\label{tab:aaa:micromse}
\end{table}

\begin{table}
\centering
\begin{tabular}{|l|llllll|}\hline
Service $\downarrow$ &\thead{\cpred} & \thead{\PerArmBeta} & \thead{\BernGP{}} & \thead{\BetaGP{}} & \thead{\DirGP{}} & \thead{\DirGPR{}} \\\hline
AC-COCOS10K & 15.0 / 0.8 & 15.6 / 0.3 & 13.2 / 2.2 & 14.3 / 3.0 & 11.7 / 1.7 & 10.4 / 1.5\\
AC-COCOS & 9.4 / 0.4 & 10.0 / 0.4 & 8.6 / 0.7 & 7.9 / 0.9 & 6.8 / 0.5 & 5.7 / 0.5\\
MF-CelebA & 8.2 / 0.2 & 8.4 / 0.7 & 7.6 / 1.3 & 6.6 / 0.7 & 4.4 / 0.6 & 3.9 / 0.7\\
MF-IMDB & 35.9 / 0.6 & 30.3 / 1.2 & 28.1 / 2.7 & 25.9 / 2.7 & 22.6 / 1.4 & 23.3 / 2.3\\\hline
\end{tabular}
\caption{\emph{Worst} MSE along with standard deviation on all tasks. Shown after trailing '/' is the standard deviation.}
\label{tab:aaa:worstmse}
\end{table}

\begin{table}
\centering
\begin{tabular}{|l|llllll|}\hline
Service $\downarrow$ &\thead{\cpred} & \thead{\PerArmBeta} & \thead{\BernGP{}} & \thead{\BetaGP{}} & \thead{\DirGP{}} & \thead{\DirGPR{}} \\\hline
AC-COCOS10K & 7.3 / 0.3 & 11.8 / 1.5 & 10.8 / 1.9 & 12.4 / 0.1 & 7.3 / 0.2 & 6.4 / 0.4\\
AC-COCOS & 4.0 / 0.1 & 6.9 / 1.2 & 4.8 / 0.3 & 4.9 / 0.3 & 3.8 / 0.1 & 3.8 / 0.3\\
MF-CelebA & 2.9 / 0.0 & 2.8 / 0.0 & 3.3 / 1.2 & 3.9 / 0.4 & 3.2 / 0.2 & 2.9 / 0.3\\
MF-IMDB & 11.7 / 0.2 & 11.3 / 0.3 & 11.4 / 0.6 & 11.0 / 0.7 & 9.3 / 1.1 & 9.4 / 0.6\\\hline
\end{tabular}
\caption{\emph{Infrequent} MSE along with standard deviation on all tasks. Shown after trailing '/' is the standard deviation.}
\label{tab:aaa:infreqmse}
\end{table}

	\chapter{Details for Chapter~\ref{chap:adaptation}}
\label{chap:appendix:adaptation}

\section{Additional Results for Word Embedding Adaptation}


\paragraph{Ablation studies on \srcsel}
We compare variants in the design of \srcsel\ in Tables~\ref{tab:srcsel:ablation:class}, \ref{tab:srcsel:wect:ablation_lm} and ~\ref{tab:srcsel:wect:ablation_auc}. 
In \srcselR\ we run the SrcSel without weighting the source snippets by the $Q(w,C)$ score in \eqref{eq:srcsel:SnippetSelect}.  
We observe that the performance is worse than with the $Q(w,C)$ score. Next, we check if the score would suffice in down-weighting irrelevant snippets without help from our IR based selection.  
In \srcselC\ we include 5\% random snippets from $\DS$ in addition to those in \srcsel\ and weigh them all by their $Q(w,C)$ score.  
We find in Table~\ref{tab:srcsel:wect:ablation_auc} that the performance drops compared to \srcsel.  Thus, both the $Q(w,C)$ weighting and the IR selection are important components of our source selection method.

\begin{table}[htb]
\centering
\begin{tabular}{|l|l|l|} \hline
 & Micro Accuracy & Macro Accuracy \\
\hline
\tgt & 26.3$_{\pm 0.5}$ & 14.7$_{\pm 1.2}$ \\
\hline
\srcselR & 27.3$_{\pm 0.3}$ & 16.1$_{\pm 1.6}$ \\
\srcsel & 28.3$_{\pm 0.4}$ & 17.3$_{\pm 0.7}$ \\ \hline
\end{tabular}
\caption{Ablation studies on \srcsel\ for classification on the Medical dataset. }
\label{tab:srcsel:ablation:class}
\end{table}

\begin{table}[htb]
\centering
\begin{tabular}{|l|l|l|l|l|} \hline
 & \multicolumn{4}{|c|}{LM Perplexity}  \\ \hline
 & Physics & Gaming & Android & Unix \\
\hline
\tgt & 121.9$_{\pm 0.6}$ & 185.0$_{\pm 0.3}$ & 142.7$_{\pm 2.7}$ & 159.5$_{\pm 1.2}$ \\ \hline
SSR & 114.8$_{\pm 0.2}$ & 172.7$_{\pm 1.5}$ & 131.6$_{\pm 0.7}$ & 151.8$_{\pm 1.1}$ \\
\srcsel & 116.1$_{\pm 0.9}$ & 173.3$_{\pm 0.6}$ & 136.7$_{\pm 1.1}$ & 153.1$_{\pm 0.1}$ \\
\hline
\end{tabular}
\caption{Ablation studies on \srcsel\ over LM perplexity. SSR refers to \srcselR\ .}
\label{tab:srcsel:wect:ablation_lm} 
\end{table}

\begin{table}[htb]
\centering
\begin{tabular}{|l|l|l|l|l|} \hline
 & \multicolumn{4}{|c|}{Question Dedup: AUC} \\ \hline
 & Physics & Gaming & Android & Unix \\
\hline
\tgt & 86.7$_{\pm 0.4}$ & 82.6$_{\pm 0.4}$ & 86.8$_{\pm 0.5}$ & 85.3$_{\pm 0.3}$ \\ \hline
\srcselR & 89.2$_{\pm 0.2}$ & 85.6$_{\pm 0.4}$ & 87.5$_{\pm 0.3}$ & 86.8$_{\pm 0.2}$ \\
\srcselC & 88.7$_{\pm 0.3}$ & 84.8$_{\pm 0.3}$ & 87.0$_{\pm 0.5}$ & 85.8$_{\pm 0.3}$ \\
\srcsel & 90.4$_{\pm 0.2}$ & 85.4$_{\pm 0.5}$ & 87.4$_{\pm 0.4}$ & 87.5$_{\pm 0.1}$ \\
\hline
\end{tabular}
\caption{Ablation studies on \srcsel\ over AUC.}
\label{tab:srcsel:wect:ablation_auc}
\end{table}

\begin{table}[htb]
\centering
\begin{tabular}{|l|l|l|l|l|l|}
\hline
 &
Sci & Com & Pol & Rel & Rec 
\\ \hline
\tgt & 92.2 & 79.9 & 94.8 & 87.3 & 90.3  \\\hline
\src
 & -0.1$_{\pm \text{0.8}}$ & -9.1$_{\pm \text{0.7}}$ & -3.3$_{\pm \text{0.9}}$ & -1.0$_{\pm \text{2.3}}$ & -6.0$_{\pm \text{0.5}}$ 
\\
ST & 0.0$_{\pm \text{1.1}}$ & 0.0$_{\pm \text{1.3}}$ & -0.1$_{\pm \text{1.5}}$ & 0.1$_{\pm \text{2.5}}$ & 0.2$_{\pm \text{1.1}}$  \\
RS & 0.9$_{\pm \text{0.9}}$ & -0.2$_{\pm \text{1.3}}$ & 0.2$_{\pm \text{1.3}}$ & 1.2$_{\pm \text{2.0}}$ & 0.1$_{\pm \text{0.8}}$  \\
\srcsel & 1.2$_{\pm \text{0.8}}$ & 0.1$_{\pm \text{1.4}}$ & 0.5$_{\pm \text{1.2}}$ & 0.5$_{\pm \text{2.6}}$ & 0.3$_{\pm \text{1.0}}$  \\\hline
\end{tabular}
\caption{Average accuracy gains over \tgt\ and $\pm$std-dev on five classification domains: Science, Computer, Politics, Religion, Recreation, from 20 NG dataset. ST and RS abbreviate to \srcinit\ and \srcselReg\ respectively.}
\label{tab:srcsel:classify:ng_detailed}
\end{table}

\paragraph{Critical hyper-parameters}
The number of neighbours $K$ used for computing embedding based stability score as shown in~(\ref{eqn:srcsel:sim_score}) is set to 10 for all the tasks. We train each of the different embedding methods for a range of different epochs: \{5, 20, 80, 160, 200, 250\}. The $\lambda$ parameter of \srcselReg\ and \yangC\ is tuned over \{0.1, 1, 10, 50\}. Pre-trained embeddings $\ES$ are obtained by training a CBOW model for 5 epochs on a cleaned version of 20160901 dump of Wikipedia. All the embedding sizes are set to $300$. 


\begin{table}[htb]
\centering
\setlength\tabcolsep{3.0pt}
\begin{tabular}{|l|r|r|r|r|r|r|r|r|} \hline
  & \multicolumn{4}{|c|}{Perplexity} & \multicolumn{4}{|c|}{AUC}  \\
\hline
 & Physics & Gaming & Android & Unix & Physics & Gaming & Android & Unix \\
\hline
\tgt & 121.9  & 185.0 & 142.7 & 159.5 & 86.7 & 82.6 & 86.8 & 85.3 \\
\hline
\yangC & 2.1$_{\pm \text{0.8}}$ & 7.0$_{\pm \text{0.7}}$ & 1.8$_{\pm \text{0.4}}$ & 3.4$_{\pm \text{1.0}}$ & -0.4$_{\pm \text{0.4}}$ & 2.3$_{\pm \text{0.3}}$ & -0.6$_{\pm \text{0.4}}$ & -0.3$_{\pm \text{0.3}}$ \\
\yangC-rinit & -1.6$_{\pm \text{0.9}}$ & 1.2$_{\pm \text{0.6}}$ & 1.6$_{\pm \text{0.2}}$ & 2.6$_{\pm \text{0.8}}$ & -1.2$_{\pm \text{0.2}}$ & 0.$_{\pm \text{0.3}}$ & -0.2$_{\pm \text{0.5}}$ & -0.3$_{\pm \text{0.3}}$ \\
\hline
\srcselReg & 5.0$_{\pm \text{0.1}}$ & 13.8$_{\pm \text{0.3}}$ & 6.7$_{\pm \text{0.7}}$ & 9.7$_{\pm \text{0.8}}$ & -0.3$_{\pm \text{0.3}}$ & 2.1$_{\pm \text{0.4}}$ & -0.6$_{\pm \text{0.3}}$ & -0.3$_{\pm \text{0.5}}$ \\
\srcselReg-rinit & 3.6$_{\pm \text{1.0}}$ & 11.1$_{\pm \text{0.5}}$ & 7.0$_{\pm \text{1.2}}$ & 8.9$_{\pm \text{1.4}}$ & 0.7$_{\pm \text{0.2}}$ & 1.2$_{\pm \text{0.2}}$ & -0.3$_{\pm \text{0.6}}$ & -0.2$_{\pm \text{0.5}}$ \\
\hline
\srcsel & 5.8$_{\pm \text{0.9}}$ & 11.7$_{\pm \text{1.8}}$ & 6.0$_{\pm \text{1.2}}$ & 6.3$_{\pm \text{1.0}}$ & 3.7$_{\pm \text{0.2}}$ & 2.8$_{\pm \text{0.5}}$ & 0.6$_{\pm \text{0.4}}$ & 2.2$_{\pm \text{0.1}}$ \\
\srcselR-rinit & 5.8$_{\pm \text{1.0}}$ & 12.5$_{\pm \text{0.4}}$ & 10.4$_{\pm \text{1.4}}$ & 7.9$_{\pm \text{1.0}}$ & 2.5$_{\pm \text{0.1}}$ & 2.$_{\pm \text{0.5}}$ & 0.4$_{\pm \text{0.4}}$ & 1.5$_{\pm \text{0.2}}$ \\
\hline
\end{tabular}
\caption{Source initialization vs random initialization (with suffix {\it -rinit}) on \yangC\ , \srcselReg\ , \srcsel\ . Shown in the table is the average gain over \tgt\ for each method.}
\label{tab:srcsel:init:all}
\end{table}

\begin{table}[htb]
\centering
\begin{tabular}{|l|l|l|l|l|l|l|}
\hline
Method & Physics & Gaming & Android & Unix & Med (Micro) & Med (Rare) \\
\hline
\tgt & 89.7$_{\pm \text{0.3}}$ & 88.4$_{\pm \text{0.2}}$ & 89.4$_{\pm \text{0.3}}$ & 89.2$_{\pm \text{0.3}}$ & 31.4$_{\pm \text{0.9}}$ & 9.4$_{\pm \text{3.0}}$ \\
\srcinit & 89.5$_{\pm \text{0.2}}$ & 89.0$_{\pm \text{0.2}}$ & 89.0$_{\pm \text{0.2}}$ & 89.0$_{\pm \text{0.2}}$ & 31.3$_{\pm \text{0.4}}$ & 7.3$_{\pm \text{2.6}}$ \\
\srcsel & 91.6$_{\pm \text{0.3}}$ & 88.9$_{\pm \text{0.8}}$ & 89.4$_{\pm \text{0.2}}$ & 89.0$_{\pm \text{0.2}}$ & 31.3$_{\pm \text{0.9}}$ & 10.5$_{\pm \text{1.8}}$ \\
\hline
\end{tabular}
\caption{Performance with a larger target corpus size of 10MB on the four deduplication tasks (AUC score) and one classification task (Accuracy) shown in the last two columns.}
\label{tab:srcsel:large}
\end{table}
\section{Additional results for Immediate Adaptation}
\subsection{Reproducibility/Implementation Details}
\label{sec:kyc:appendix:lm}
In this section we provide details about the dataset, architecture and training procedures used for each of the three tasks.

\paragraph{NER:}
We use the standard splits provided in the Ontonotes dataset \cite{ontonotes}. Our codebase builds on the official PyTorch implementation released by \cite{devlin2018bert}. We finetune a cased BERT base model with a maximum sequence length of 128 tokens for 3 epochs which takes 3 hours on a Titan X GPU. 

\paragraph{Sentiment Classification:}
As described previously, we use the Amazon dataset \cite{amazon_dataset}. For each review, we use the standard protocol to convert the rating to a binary class label by marking reviews with 4 or 5 stars as positive, reviews with 1 or 2 stars as negative and leaving out reviews with 3 stars. We randomly sample data points from each domain to select 1000, 200 and 500 positive and negative reviews each for the train, validation and test splits, respectively. We leave out the domains that have insufficient examples, leaving us with 22 domains.
\indent We use the finetuning protocol provided by the authors of \cite{sentiment_repo} and use the uncased BERT base model with a maximum sequence length of 256 for this task. We train for 5 epochs (which takes 4 hours on a Titan X GPU) and use the validation set accuracy after every epoch to select the best model. 

\paragraph{Auto Complete Task:}
We use 20NewsGraoup dataset while regarding each content class label as a client. We remove header, footer from the content of the documents and truncate the size of each client to around 1MB. We use word based tokenizer with a vocabulary restricted to top 10,000 tokens and demarcate sentence after 50 tokens. The reported numbers in Table~\ref{tab:kyc:expts:lm} are when using TF-IDF vector for domain sketch. We diIn this section, we reportd not evaluate other kinds of domain sketch on this task. We train all the methods for 40 epochs with per epoch train time of 4 minutes on a Titan X GPU. 

We adopt the tuned hyperparameters corresponding to PTB dataset to configure the baseline~\citet{MelisTP2020}. Since the salience information from the client sketch can be trivially exploited in  perplexity reduction and thereby impede learning desired hypothesis beyond trivially copying the salience information, we project the sketch vector to a very small dimension of 32 before fanning it out to the size of vocabulary. We did not use any non-linearity in $G_\phi$ and also employ dropout on the sketches.

\paragraph{Details of MoE method:}
MoE~\citep{GuoSB18} employs a shared encoder and a client specific classifier.  We implemented their proposal to work with our latest encoder networks. Our implementation of their method is to the best of our efforts faithful to their scheme. The only digression we made is in the design of discriminator: we use a learnable discriminator module that the encoder fools while they adopt MMD based metric to quantify and minimize divergence between clients. This should, in our opinion, only work towards their advantage since MMD is not sample efficient especially given the small size of our clients.

\begin{figure*}[tb]
  \begin{subfigure}{0.5\linewidth}
   \centering
    \includegraphics[width=0.9\linewidth]{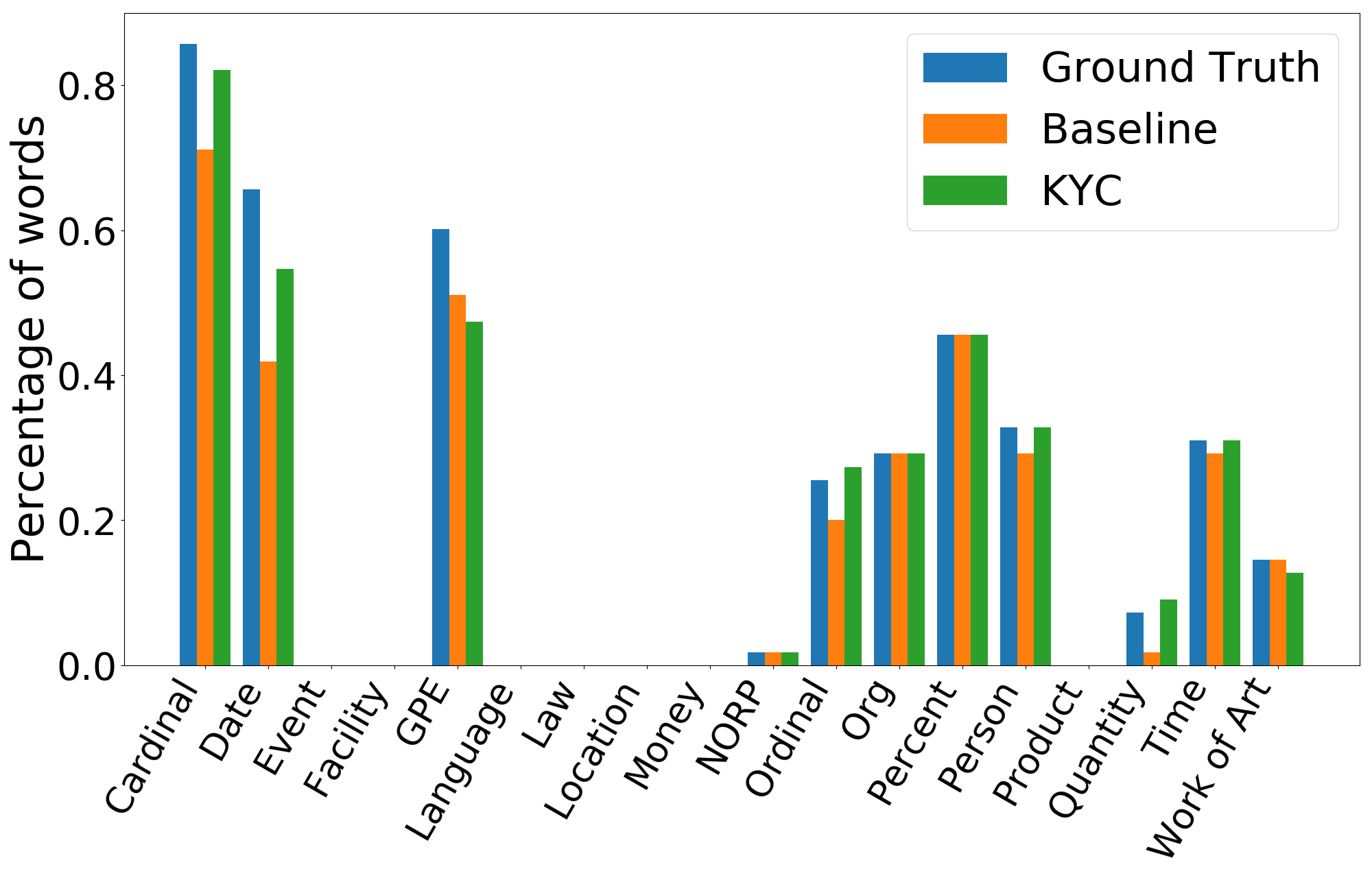} 
    \end{subfigure}
  \begin{subfigure}{0.5\linewidth}
  \centering
    \includegraphics[width=0.9\linewidth]{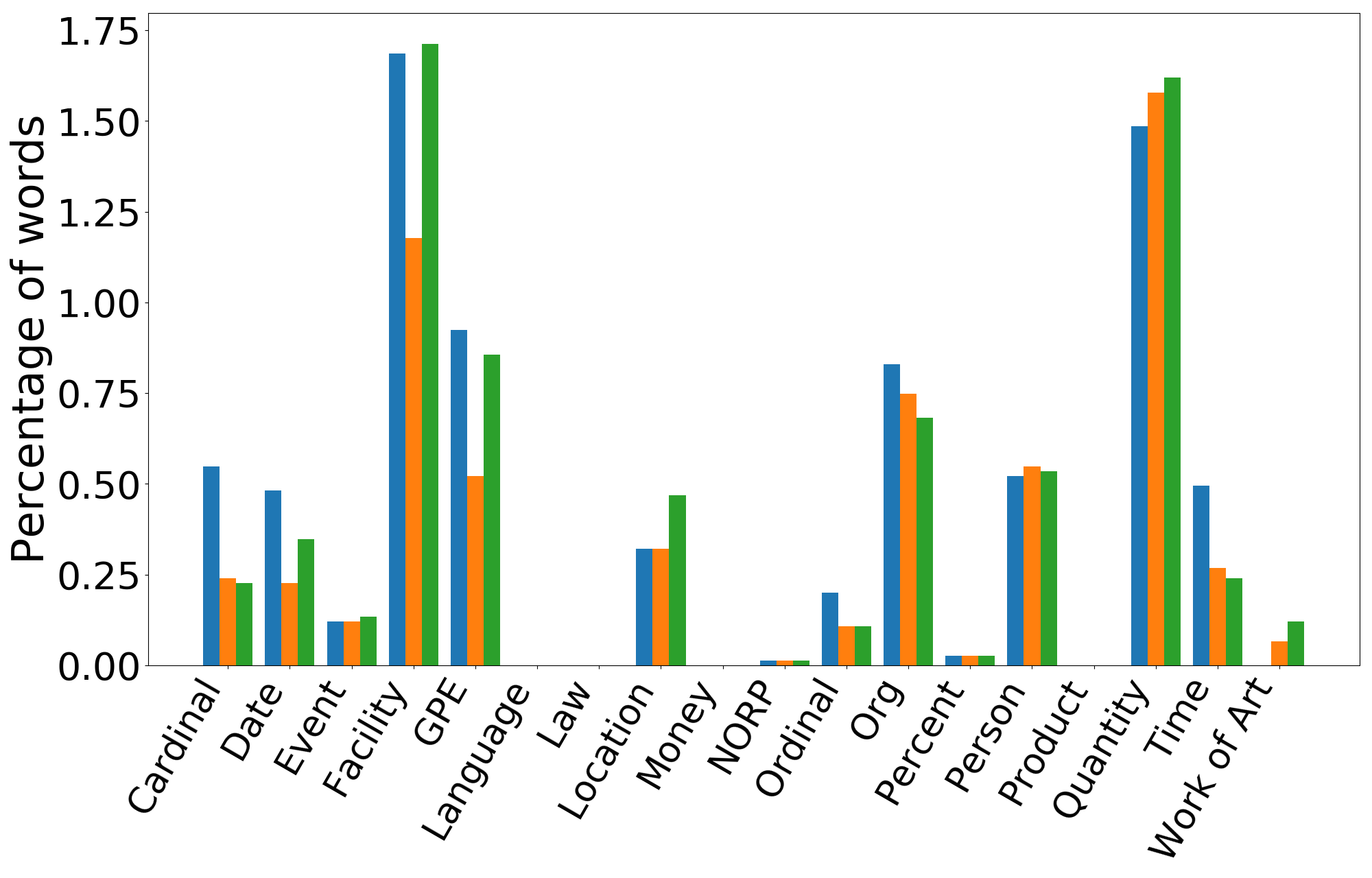} 
\end{subfigure}
\caption{Proportion of true and predicted entity labels for different OOD clients (left) BC/Phoenix (right) BC/CCTV.}
  \label{fig:kyc:bar_graphs_more}
\end{figure*}

\begin{table}[!htbp]
\centering
\begin{tabular}{|l | c c | c c|}
\hline
& \multicolumn{2}{|c|}{\small{OOD}} & \multicolumn{2}{|c|}{\small{ID}}  \\
OOD Clients & \small{Base}  & \small{{\our}} & \small{Base} & \small{{\our}} \\
\hline
\small{BC/CCTV + BC/Phoenix} & 63.8 & 70.1 & 86.00 & 86.7 \\
\small{BN/PRI + BN/VOA} & 88.7 & 91.6 & 84.5 & 86.2\\
\small{NW/WSJ + NW/Xinhua} & 73.9 & 79.2 & 80.8 & 82.2\\
\small{BC/CNN + TC/CH} & 78.3 & 80.4 & 85.6 & 87.1\\
\small{WB/Eng + WB/a2e} & 76.2 & 78.9 & 86.4 & 87.5 \\
\hline
Average & 76.2 & 80.0 & 84.7 & 85.9\\ \hline
\end{tabular}
\vspace{0.3em}
\caption{Performance on the NER task on the Ontonotes dataset when using TF-IDF as the client sketch.}
\label{table:kyc:ner_tfidf}
\end{table}

\begin{table}[!htbp]
\centering
\begin{tabular}{|l | c c | c c|}
\hline
& \multicolumn{2}{|c|}{\small{OOD}} & \multicolumn{2}{|c|}{\small{ID}}  \\
OOD Clients & \small{Base}  & \small{{\our}} & \small{Base} & \small{{\our}} \\
\hline
\small{BC/CCTV + BC/Phoenix} & 63.8 & 75.3 & 86.0 & 79.3 \\
\small{BN/PRI + BN/VOA} & 88.7 & 90.5 & 84.5 & 78.7\\
\small{NW/WSJ + NW/Xinhua}& 73.9 & 82.7 & 80.8 & 71.4\\
\small{BC/CNN + TC/CH} & 78.3 & 80.3 & 85.6 & 79.9\\
\small{WB/Eng + WB/a2e} & 76.2 & 76.4 & 86.4 & 79.6 \\
\hline
Average & 76.2 & 81.0 & 84.7 & 77.8 \\\hline
\end{tabular}
\vspace{0.3em}
\caption{Performance on the NER task on the Ontonotes dataset when using Binary Bag of Words as the client sketch.}
\label{table:kyc:ner_bloom}
\end{table}

\begin{table}
\centering
\begin{tabular}{|l | c c | c c|}
\hline
& \multicolumn{2}{|c|}{\small{OOD}} & \multicolumn{2}{|c|}{\small{ID}}  \\
OOD Clients & \small{Base}  & \small{{\our}} & \small{Base} & \small{{\our}} \\
\hline
\small{BC/CCTV + BC/Phoenix} & 63.8 & 61.5 & 86.0 & 83.0\\
\small{BN/PRI + BN/VOA} & 88.7 & 82.3 & 84.5 & 85.2 \\
\small{NW/WSJ + NW/Xinhua} & 73.9 & 82.3 & 80.8 & 75.0\\
\small{BC/CNN + TC/CH} & 78.3 & 72.5 & 85.6 & 83.2\\
\small{WB/Eng + WB/a2e} & 76.2 & 78.3 & 86.4 & 82.5\\
\hline
Average & 76.2 & 75.4 & 84.7 & 81.8\\\hline
\end{tabular}
\vspace{0.3em}
\caption{Performance on the NER task on the Ontonotes dataset when using sentence embddings averaged over an extracted summary.}
\label{table:kyc:ner_summary}
\end{table}

\begin{table}[!htbp]
\centering
    \begin{tabular}{|l |c c c |c c c |}
    \hline
    & \multicolumn{3}{|c|}{\small{OOD}} & \multicolumn{3}{|c|}{\small{ID}}  \\
    OOD Clients & \small{Base } & \small{Sali-} & \small{Avg} & \small{Base} & \small{Sali-} & \small{Avg} \\
     &  & \small{ence} & \small{Len} & & \small{ence} & \small{Len} \\
    \hline
    \small{Electronics+Games} & 86.4 & 88.1 & 86.9 & 88.5 & 89.0 & 88.6 \\
    \small{Industrial+Tools} & 87.4 & 87.6 & 88.3 & 88.2 & 88.9 & 88.8 \\
    \small{Books+Kindle Store} & 83.5 & 84.6 & 84.5 & 88.0 & 88.9 & 89.0 \\
    \small{CDs+Digital Music} & 82.5 & 83.0 & 83.1 & 89.0 & 89.0 & 89.0 \\
    \small{Arts+Automotive} & 89.9 & 90.6 & 90.2 & 88.2 & 88.6 & 88.5 \\
    \hline
    Average & 86.0 & 86.8  & 86.6 & 88.4 & 88.8 & 88.8 \\ \hline
    \end{tabular}
    \caption{Accuracy on the Sentiment Analysis task when using average review length as the client sketch. Columns ``Saliency" and ``Avg Len" refer to using {\our} with the default saliency features and normalized review lengths as client sketches, respectively.}
    \label{table:kyc:sentiment_length_numbers}
\end{table}

\subsection{Results with Different Client Sketches}
\label{kyc:appendix:domain_features}
In this section we provide results on every OOD split for the different client sketches described in Section \ref{sec:kyc:ablation} along with more details.
\begin{itemize}
    \item \textbf{TF-IDF}: This is a standard vectorizer used in Information Retrieval community for document similarity. We regard all the data of the client as a document when computing this vector.  
    The corresponding numbers using this sketch are shown in Table~\ref{table:kyc:ner_tfidf} and are only slightly worse than the salience features.
    \item \textbf{Binary Bag of Words (BBoW)}: A binary vector of the same size as vocabulary is assigned to each client while setting the bit corresponding to a word on if the word has occurred in the client's data.  
    We notice an improvement on the OOD set but a significant drop in ID numbers as seen in Table~\ref{table:kyc:ner_bloom},~\ref{table:kyc:ner_method_comparison}. We attribute this to the strictly low representative power of BBoW sketches compared to the other sketches. The available train data for NER is laced with rogue clients which are not labeled and are instead assigned the default tag: ``O''. Proportion of {\our}'s improvement on this task comes from the ability to distinguish bad clients and keeping their parameters from not affecting other clients. This, however, is not possible when the representative capacity of the sketch is compromised. Thereby we do worse on ID using this sketch but not on OOD meaning the model does worse on the bad clients (which are only part of ID, and not OOD). 
    \item \textbf{Contextualized Embedding of Summary}: We also experiment with using deep-learning based techniques to extract the topic and style of a client by using the ``pooled" BERT embeddings averaged over sentences from the client. Since the large number of sentences from every client would lead to most useful signals being killed upon averaging, we first use a Summary Extractor \cite{gensim_summarizer} to extract roughly 10 sentences per client and average the sentence embeddings over these sentences only. This method turns out to be ineffective in comparison to the other client sketches, indicating that sentence embeddings do not capture all the word-distribution information needed to extract useful correction.
    \item \textbf{Average Instance Length:} For the task of Sentiment Analysis, we also experiment with passing a single scalar indicating average review length as the client sketch in order to better understand and quantify the importance of average review length on the performance of {\our}. We linearly scale the average lengths so that all train clients have values in the range $[-1, 1]$. As can be seen in Table \ref{table:kyc:sentiment_length_numbers}, this leads to a significant improvement over the baseline.  In particular, the OOD splits  CDs + Digital Music and Books + Kindle Store have reviews that are longer than the average and consequently result in improvements when augmented with average length information. The gains from review length alone are not higher than our default term-saliency sketch indicating that term frequency captures other meaningful properties as well. 
\end{itemize}

\subsection{Results with Different Model Architectures}
\label{kyc:appendix:network_arch}
In this section we provide results for the different network architecture choices described in Section~\ref{sec:kyc:ablation}
\begin{itemize}
    \item \textbf{Deep}: 
    The architecture used is identical to that shown in Figure~\ref{fig:kyc:overview} barring $\bigoplus$, which now consists of an additional 128-dimensional non-linear layer before the final softmax transform $Y_\theta$.
    \item \textbf{Decompose}: The final softmax layers is decomposed in to two. A scalar $\alpha$ is predicted from the client sketch using $G_\phi$ similar to {\our}. The final softmax layer then is obtained through convex combination of the two softmax layers using $\alpha$. Figure~\ref{fig:kyc:overview_decompose} shows the overview of the architecture.
    \item \textbf{MoE-$\bm{g}$}: We use the client sketch as the drop-in replacement for encoded instance representation employed in \citet{GuoSB18}. 
    The architecture is sketched in  Figure~\ref{fig:kyc:overview_moe}. 
    As shown in Table~\ref{table:kyc:ner_moe}, this method works better than the standard MoE model, but worse than {\our}.
\end{itemize}

\begin{table}[!htbp]
\centering
\begin{tabular}{|l | c c | c c|}
\hline
& \multicolumn{2}{|c|}{\small{OOD}} & \multicolumn{2}{|c|}{\small{ID}}  \\
OOD Clients & \small{Base}  & \small{{\our}} & \small{Base} & \small{{\our}} \\
\hline
\small{BC/CCTV + BC/Phoenix} & 64.8 & 74.5 & 85.6 & 86.8 \\
\small{BN/PRI + BN/VOA} & 89.5 & 90.0 & 84.1 & 85.6\\
\small{NW/WSJ + NW/Xinhua} & 74.4 & 80.6 & 80.2 & 92.8\\
\small{BC/CNN + TC/CH} & 78.0 & 79.6 & 86.1 & 87.5 \\
\small{WB/Eng + WB/a2e}&   75.6 & 79.9 & 85.8 & 87.1\\
\hline
Average & 76.5 & 80.9 & 84.4 & 86.0 \\\hline
\end{tabular}
\vspace{0.3em}
\caption{Performance on the NER task on the Ontonotes dataset using KYC-Deep.}
\label{table:kyc:ner_deep}
\end{table}

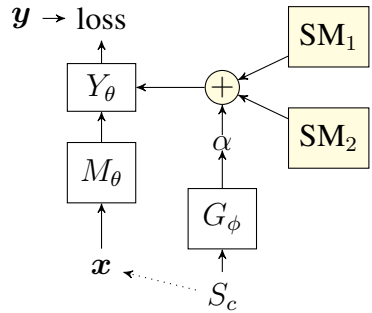
\begin{figure}
\begin{center}
\begin{tikzpicture}[>=stealth',align=center]
\node (loss) {loss};
\node [left=3mm of loss] (y) {$\bm{y}$};

\node [rectangle, draw,
minimum width=9mm, below=3mm of loss] (Ytheta) {$Y_\theta$};
\node [circle,draw, 
inner sep=.1mm, right=9mm of Ytheta, fill=yellow!15, ] (combination) {$+$};
\node [rectangle, draw, anchor=center,
minimum width=9mm, minimum height=8mm,
below=4mm of Ytheta] (Etheta) {$\lastE_\theta$};
\node [anchor=center,below=6mm of Etheta] (x) {$\bm{x}$};
\node [anchor=center,
below = 4mm of combination,
outer sep=.1mm, inner sep=0mm] (alpha) {$\alpha$};
\node [rectangle, draw, anchor=center,
minimum width=9mm, minimum height=8mm,
below=5mm of alpha] (Gphi) {$G_\phi$};
\node [rectangle, draw, anchor=center,
minimum width=9mm, minimum height=8mm,
above right=1mm and 7mm  of combination, fill=yellow!15, ] (sm1) {SM$_1$};
\node [rectangle, draw, anchor=center,
minimum width=9mm, minimum height=8mm,
below right=1mm and 7mm  of combination, fill=yellow!15, ] (sm2) {SM$_2$};
\node [anchor=center,below=3mm of Gphi] (sc) {$S_c$};
\draw [->] (y) to (loss);
\draw [->] (combination) to (Ytheta);
\draw [->] (sm2) to (combination);
\draw [->] (sm1) to (combination);
\draw [->] (alpha) to (combination);
\draw [->] (Ytheta) to (loss);
\draw [->] (Etheta) to (Ytheta);
\draw [->] (Gphi) to (alpha);
\draw [->] (x) to (Etheta);
\draw [->] (sc) to (Gphi);
\draw [->, dotted] (sc) to (x);
\end{tikzpicture}
\end{center}
\caption{Decompose overview: $\bigoplus$ indicates a weighted linear combination. SM$_i$, $i \in \{1,2\}$ represent the softmax matrices which are combined using weights $\alpha$.}  \label{fig:kyc:overview_decompose}
\end{figure}

\begin{table}[!htbp]
\centering
\begin{tabular}{|l | c c | c c|}
\hline
& \multicolumn{2}{|c|}{\small{OOD}} & \multicolumn{2}{|c|}{\small{ID}}  \\
OOD Clients & \small{Base}  & \small{{\our}} & \small{Base} & \small{{\our}} \\
\hline
\small{BC/CCTV + BC/Phoenix} & 64.1 & 56.0 & 85.6 & 86.3  \\
\small{BN/PRI + BN/VOA} & 89.6 & 89.9 & 84.6 & 85.5\\
\small{NW/WSJ + NW/Xinhua} & 72.3 & 68.2 & 81.2 & 80.0\\
\small{BC/CNN + TC/CH} & 78.5 & 77.5 & 85.9 & 86.6\\
\small{WB/Eng + WB/a2e} & 75.5 & 71.0 & 86.1 & 86.7 \\
\hline
Average & 76.0 & 72.5 & 84.7 & 85.2\\ \hline
\end{tabular}
\vspace{0.3em}
\caption{Performance on the NER task on the Ontonotes dataset using Decompose.}
\label{table:kyc:ner_softmax}
\end{table}

\begin{figure}[ht]
\begin{center}
\begin{tikzpicture}[>=stealth',align=center]

\node (loss) {loss};
\node [left=3mm of loss] (y) {$\bm{y}$};
\node [circle,draw, 
inner sep=.1mm, below = 4mm of loss] (combination) {$+$};

\node [below left=5mm and 15mm of combination] (pred1) {p$_1(y|x)$};
\node [right=3mm of pred1] (pred2) {p$_2(y|x)$};
\node [right = 2mm of pred2] (dots) {$\ldots$};
\node [right=2mm of dots] (predn) {p$_n(y|x)$};
\node [rectangle, draw,
minimum width=9mm, below=5mm of pred1] (Ytheta1){$Y_{\theta_1}$};
\node [rectangle, draw,
minimum width=9mm, below=5mm of pred2] (Ytheta2){$Y_{\theta_2}$};
\node [rectangle, draw,
minimum width=9mm, below=5mm of predn] (Ythetan){$Y_{\theta_n}$};
\node [rectangle, draw, anchor=center,
minimum width=9mm, minimum height=8mm,
below=35mm of loss] (Etheta) {$\lastE_\theta$};
\node [anchor=center,
below right=7mm and 30mm of combination,
outer sep=.1mm, inner sep=0mm, fill=yellow!15, ] (alpha) {$\alpha$};
\node [rectangle, draw, anchor=center,
minimum width=9mm, minimum height=8mm,
below=10mm of alpha, fill=yellow!15, ] (Gphi) {$G_\phi$};
\node [anchor=center,below=10mm of Gphi, fill=yellow!15, fill=yellow!15, ] (sc) {$S_c$};
\node [anchor=center,below=6mm of Etheta] (x) {$\bm{x}$};
\draw [->] (x) to (Etheta);
\draw [->] (Etheta) to (Ytheta1);
\draw [->] (Etheta) to (Ytheta2);
\draw [->] (Etheta) to (Ythetan);
\draw [->] (Ytheta1) to (pred1);
\draw [->] (Ytheta2) to (pred2);
\draw [->] (Ythetan) to (predn);
\draw [->] (pred1) to (combination);
\draw [->] (pred2) to (combination);
\draw [->] (predn) to (combination);
\draw [->] (alpha) to (combination);
\draw [->] (Gphi) to (alpha);
\draw [->] (sc) to (Gphi);
\draw [->] (combination) to (loss);
\draw [->] (y) to (loss);
\draw [->, dotted] (sc) to (x);
\end{tikzpicture}
\end{center}
\caption{ MoE-$\bm{g}$ overview: $\bigoplus$ indicates a weighted linear combination. p$_i(y|x)$ represents the $i^{th}$ expert's predictions and $\alpha$ represents weights for expert gating.}  \label{fig:kyc:overview_moe}
\end{figure}
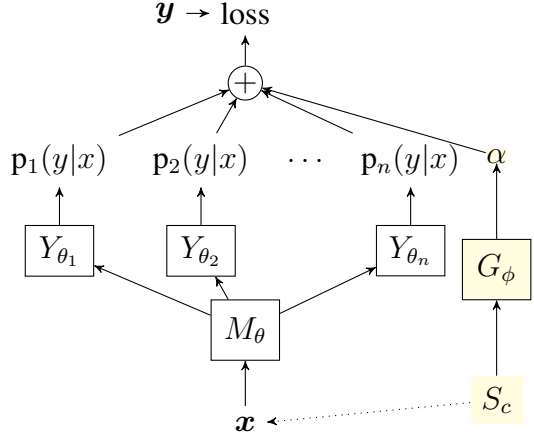

\begin{table}[!htbp]
\centering
\begin{tabular}{|l | c c | c c|}
\hline
& \multicolumn{2}{|c|}{\small{OOD}} & \multicolumn{2}{|c|}{\small{ID}}  \\
OOD Clients & \small{Base}  & \small{{\our}} & \small{Base} & \small{{\our}} \\
\hline
\small{BC/CCTV + BC/Phoenix} & 64.8 & 74.7 & 85.6 & 84.0 \\
\small{BN/PRI + BN/VOA} & 89.5 & 88.3 & 84.1 & 83.6\\
\small{NW/WSJ + NW/Xinhua} & 74.4 & 61.6& 80.2 & 64.8\\
\small{BC/CNN + TC/CH} & 78.0 & 73.7 & 86.1 & 82.1 \\
\small{WB/Eng + WB/a2e}&   75.6 & 76.3 & 85.8 & 84.4\\
\hline
Average & 76.5 & 74.9 & 84.4 & 79.8 \\\hline
\end{tabular}
\vspace{0.3em}
\caption{Performance on the NER task on the Ontonotes dataset using MoE-$\bm{g}$.}
\label{table:kyc:ner_moe}
\end{table}

\end{appendices}

\newpage
\setlength{\parskip}{5mm}
\titlespacing{\chapter}{0cm}{0mm}{0mm}
\titleformat{\chapter}[display]
  {\normalfont\huge\bfseries}
  {\chaptertitlename\ \thechapter}{20pt}{\Huge}

\bibliographystyle{unsrtnat}
\bibliography{main}

\begin{center}
{\huge Publications}
\end{center}
\begin{adjustwidth}{-5pt}{0pt}
{\Large \bf Conference}
\end{adjustwidth}
\begin{enumerate}
\item \href{https://arxiv.org/pdf/2110.02619.pdf}{Focus on the Common Good: Group Distributional Robustness Follows}\\
{\bf V Piratla}, P Netrapalli, S Sarawagi
\\\textit{International Conference on Learning Representations (ICLR) 2022}.


\item \href{https://arxiv.org/pdf/2108.06514.pdf}{Active Assessment of Prediction Services as Accuracy Surface Over Attribute Combinations}\\
{\bf V Piratla}, S Chakrabarty, S Sarawagi \\\textit{Neural Information Processing Systems (NeurIPS) 2021}.
\item \href{https://arxiv.org/pdf/2108.06721.pdf}{Training for the Future: A Simple Gradient Interpolation Loss to Generalize Along Time}\\
A Nasery, S Thakur, {\bf V Piratla}, A De, S Sarawagi\\ \textit{Neural Information Processing Systems (NeurIPS) 2021}.
\item \href{https://openreview.net/forum?id=lE9zA0yGygK}{NLP Service APIs and Models for Efficient Registration of New Clients}\\
S Shah, \textbf{V Piratla}, S Sarawagi, S Chakrabarti\\ \textit{Findings at Empirical Methods in Natural Language Processing (EMNLP), 2020.}
\item \href{https://arxiv.org/pdf/2003.12815.pdf}{Efficient Domain Generalization via Common-Specific Low-Rank Decomposition}\\
\textbf{V Piratla}, P Netrapalli, S Sarawagi\\ \textit{International Conference on Machine Learning (ICML) 2020}.
\item \href{https://arxiv.org/abs/1906.02688}{Topic Sensitive Attention on Generic Corpora Corrects Sense Bias in Pretrained Embeddings}\\
\textbf{V Piratla}, S Sarawagi, S Chakrabarti\\ 
\textit{Annual Meeting of the Association for Computational Linguistics (ACL) 2019 (Oral)}.
\item  \href{https://www.aclweb.org/anthology/D19-1435.pdf}{Parallel iterative edit models for local sequence transduction}\\
A Awasthi, S Sarawagi, R Goyal, S Ghosh, {\bf V Piratla}\\ \textit{Empirical Methods in Natural Language Processing (EMNLP), 2019}.
\item \href{https://openreview.net/pdf?id=r1Dx7fbCW}{Generalizing Across Domains via Cross-Gradient Training}\\
S Shankar*, \textbf{V Piratla}*, S Chakrabarti, S Chaudhuri, P Jyothi, S Sarawagi [Shared first author]\\
\textit{International Conference on Learning Representations (ICLR) 2018}.

\begin{adjustwidth}{-20pt}{0pt}
{\Large \bf Workshop}
\end{adjustwidth}
    \item \href{https://vihari.github.io/assets/dg_for_dr.pdf}{Untapped Potential of Data Augmentation: A Domain Generalization Viewpoint}\\
    \textbf{V Piratla}, S Shankar\\ \textit{ICML 2020 Workshop on 
Uncertainty \& Robustness in Deep Learning.} 
\end{enumerate}


\newpage
\setlength{\parskip}{5mm}
\titlespacing{\chapter}{0cm}{0mm}{0mm}
\titleformat{\chapter}[display]
  {\normalfont\huge\bfseries \centering}
  {\chaptertitlename\ \thechapter}{20pt}{\Huge}

\chapter*{Acknowledgments}
\label{ch:Acknowledgments}
\addcontentsline{toc}{chapter}{\nameref{ch:Acknowledgments}}
\vspace{10mm}
I am extremely fortunate to have been advised by Prof. Sunita Sarawagi and Prof. Soumen Chakrabarti. They were immensely patient and supportive throughout. Apart from their high-quality research, I admire their way of life and outlook towards research. I enrolled for a PhD partly because I was impressed by their clear and sharp thinking, which I hoped to emulate. 
I learned a lot from writing papers with them, which felt more like an artistic expression of scientific thoughts with chiselled ablations and experiments. Under their supervision, I developed a good taste in the choice of problems and a strong scientific temper, such as not ignoring unexpected results or not accepting an unsatisfying explanation.
I used to come out of our meetings mostly feeling optimistic and empowered, which I later realised is a sign of positive supervision. Sunita and Soumen helped me focus completely on my research by handling many non-technical issues regarding my leave, allocation of funds or conference travel. I am very proud of our well-crafted and slow-baked work that reflected our collective contemplative thoughts.

In 2019, I started working with Dr Praneeth Netrapalli as an intern at Microsoft Research. Our collaboration has drastically improved my research approach. I started to devise and explain algorithms through simple synthetic settings, which often also helped improve confidence and clarity in understanding the merits of our proposed algorithm. He once reminded me that we are only reporters of nature, which helped me understand his refreshing perspective towards research. I am very grateful for his time, which made me a better researcher and led to two publications.  

I have been lucky to find great mentors through my RPC committee and Google PhD Fellowship: Prof. Preethi Jyothi, Prof. Shivaram Kalyanakrishnan and Dr Pradeep Shenoy. They provided much-needed and timely feedback on past research and advised on future work. I am also grateful for the insightful comments from my external reviewers: Prof. Vineeth Balasubramanian and Prof. Soma Biswas, which further improved the quality of this thesis. I would also like to thank IITB and the Department of Computer Science for hosting me as a PhD student and Google for providing financial stability. Mr Vijay Ambre of the Computer Science Department needs a special mention for being easily accessible and often going above and beyond to resolve our administrative issues. 

Several others have contributed significantly to placing me on a PhD journey and several others in helping me sustain it. I will always be grateful to Dr Sudheendra Hangal for trusting a second-year undergraduate student in the Himalayas to do research, which gave me my first glimpse into the world of research. I am indebted to my parents for imbibing in me a sense of curiosity and the value of education. I am very glad for the emphasis they placed very early in my childhood towards high-quality education over any other material comforts. 

Five years of my PhD have passed like a breeze in the warm company of my friends and lab-mates: Anjali, Rasna, Shiv, Karthik, Tikaram, Abhijeet, Prathamesh and Lokesh. They always listened patiently to my venting stories related to research or otherwise. I will continue to cherish the memorable time I spent with my friends in the hills, restaurants, malls and yoga centres.



\end{document}